\documentclass{article}

    \PassOptionsToPackage{numbers, compress}{natbib}

 \usepackage[main, final]{neurips_2026}

\usepackage[utf8]{inputenc} 
\usepackage[T1]{fontenc}    
\usepackage{hyperref}       
\usepackage{url}            
\usepackage{booktabs}       
\usepackage{amsfonts}       
\usepackage{nicefrac}       
\usepackage{microtype}      
\usepackage{xcolor}         

\usepackage{graphicx}
\usepackage{booktabs}

\usepackage{amsthm}

\usepackage{url}            
\usepackage{amsfonts}       
\usepackage{nicefrac}       
\usepackage{xcolor}         

\usepackage{amssymb}
\usepackage{mathtools}
\usepackage{enumitem}
\usepackage{multirow}
\usepackage{subcaption}  

\usepackage[export]{adjustbox}
\usepackage{caption}

\usepackage{bbm}

\usepackage{comment}

\usepackage{rotating}  
\usepackage{pdflscape}  

\usepackage{xfrac}

\usepackage{fontawesome5}

\usepackage[absolute]{textpos}

\definecolor{codegreen}{rgb}{0,0.6,0}
\definecolor{codegray}{rgb}{0.5,0.5,0.5}
\definecolor{codepurple}{rgb}{0.58,0,0.82}
\definecolor{backcolour}{rgb}{0.95,0.95,0.92}
\definecolor{codepurple}{rgb}{0.58,0,0.82}
\definecolor{papercolor}{HTML}{0668E1}
\definecolor{darkred}{rgb}{0.68,0.05,0.0}
\definecolor{tab_blue}{RGB}{230,245,255} 
\definecolor{tab_purple}{RGB}{245,230,255} 
\definecolor{mygreen}{RGB}{66,150,83}

\definecolor{codeblue}{rgb}{0.25,0.5,0.5}
\definecolor{codekw}{rgb}{0.85, 0.18, 0.50}

\definecolor{darkblue}{rgb}{0.0,0.0,0.65}
\definecolor{darkred}{rgb}{0.68,0.05,0.0}
\definecolor{darkgreen}{rgb}{0.0,0.29,0.29}
\definecolor{darkpurple}{rgb}{0.47,0.09,0.29}

\hypersetup{
    colorlinks,
    linkcolor={mygreen},
    citecolor={darkblue},
    urlcolor={codekw}
}

\newtheorem{theorem}{Theorem}
\newtheorem{proposition}{Proposition}
\newtheorem{corollary}{Corollary}

\newsavebox{\largestimage}

\title{
Harnessing Data Asymmetry in Manifold Learning}

\author{
    Thomas Dag\`es$^{1,2}$\thanks{Corresponding author: \tt{tom.dages@tum.de}} \quad
    Simon Weber$^3$ \quad
    Daniel Cremers$^{1,2}$ \quad
    Ron Kimmel$^4$\\[0.5em]
    $^{1}$Technical University of Munich \quad
    $^{2}$Munich Center for Machine Learning\\
    $^{3}$University of Oxford \quad
    $^{4}$Technion - Israel Institute of Technology
}

\begin{document}

{
\definecolor{somegray}{gray}{0.5}
\newcommand{\darkgrayed}[1]{\textcolor{somegray}{#1}}
\begin{textblock}{17}(-0.5, 1)  %
\centering
\darkgrayed{Accepted for publication at the 40th Annual Conference on Neural Information Processing Systems (NeurIPS) 2026.}
\end{textblock}
}

\maketitle
\begingroup
    \renewcommand{\thefootnote}{}
    \footnotetext{\faGithub\ \url{https://github.com/Tommoo/Finsler-Manifold-Learning}}
\endgroup
\setcounter{footnote}{0}

\begin{abstract}
  Manifold learning is a fundamental task at the core of data analysis and visualisation. It aims to capture the simple underlying structure of complex high-dimensional data by preserving pairwise dissimilarities in low-dimensional embeddings. Traditional methods rely on symmetric Riemannian geometry, thus forcing symmetric dissimilarities and embedding spaces, e.g.\ Euclidean. However, this discards in practice valuable asymmetric information inherent to the non-uniformity of data samples. We suggest to harness this asymmetry by switching to Finsler geometry, an asymmetric generalisation of Riemannian geometry, and propose a Finsler manifold learning pipeline that constructs asymmetric dissimilarities and embeds in a Finsler space. This greatly broadens the applicability of existing asymmetric embedders beyond traditionally directed data to any data. We also modernise asymmetric embedders by generalising current reference methods to asymmetry, like Finsler t-SNE, Finsler UMAP, and Finsler PHATE. On controlled synthetic and large real datasets, we show that our asymmetric pipeline reveals valuable information lost in the traditional pipeline, e.g.\ density hierarchies, and consistently provides superior quality embeddings than their Euclidean counterparts.
\end{abstract}

\section{Introduction}

Manifold learning is a core task in data analysis. 
It seeks low-dimensional representations of high-dimensional data by preserving pairwise dissimilarities, capturing the underlying manifold.
Classical methods share a three-stage pipeline: 
(P1) \textit{Data dissimilarity}
-- compute pairwise dissimilarities of data points, 
(P2) \textit{Embedding dissimilarity}
-- choose how to measure embedding dissimilarities, and 
(P3)
\textit{Optimisation} -- optimise the embedding to fit the dissimilarities. 
We will refer to (P1), (P2), and (P3) as these pipeline stages.
As the hidden data manifold is usually Riemannian, 
these methods target symmetric dissimilarities and embed in a canonical Riemannian space, typically Euclidean. 

However, we advocate that, due to the sampling, equipping the manifold with an asymmetric Finsler metric is both more natural, generalising the Riemannian perspective \cite{chern1996finsler}, and captures additional valuable information that is lost in the symmetric setting (Fig.~\ref{fig: US cities}).
This novel asymmetric perspective is
encouraged after noticing that existing methods construct asymmetric data dissimilarities 
during (P1),
undesired
when using Riemannian theory, requiring poorly justified post hoc averaging to fix the issue, discarding information in the process.

We thus construct asymmetric data dissimilarities
(P1)
and must then define
(P2)
and optimise
(P3)
their embedding. 
As Euclidean space 
is symmetric, it cannot represent 
asymmetry during (P2),
so classical methods would fail. 
Recently, 
it was suggested to embed asymmetric data to a canonical Finsler space 
(P2)
via Finsler MDS \cite{dages2025finsler} 
(P3),
but it is limited to oracle asymmetric data 
(P1).
Thanks to our asymmetric data construction
(P1),
we greatly broaden its, and generally any asymmetric embedder's 
(P2-P3),
applicability to any data. 
Additionally, Finsler MDS' 
optimizer
(P3)
as provided in \cite{dages2025finsler}
is slow, unstable, unscalable, and
without
extra
perks like clustering
(P2).
We propose to overcome these issues by adapting modern methods
(P2-P3), like references t-SNE \cite{van2008visualizing} and UMAP \cite{mcinnes2018umap}, to 
asymmetry.

Our contributions are as follows:
    (a)
    Reveal
    Riemannian
    theory inconsistency
    in classical 
    \textit{data dissimilarities} (P1).
    (b)
    Embrace sampling-induced asymmetry in 
    \textit{data dissimilarities} (P1)
    via a Finsler
    lens,
    enriching our view on sampled data
    with information that symmetric methods cannot capture,
    and opening up applications of asymmetric methods 
    (P2-P3),
    e.g.\ Finsler MDS, even to traditionally symmetric data.
    (c)
    Embed arbitrary data in the canonical Finsler space (Finsler 
    \textit{embedding dissimilarities} (P2)
    and 
    \textit{optimisation} stages
    (P3)). 
    We generalise modern optimisation-based methods to asymmetric data and Finsler embeddings, and introduce in particular Finsler t-SNE and Finsler UMAP
    (P2-P3).
    (d)
    We experimentally demonstrate the advantage of our asymmetric Finsler pipeline
    (P1-P2-P3).
    Through visualisations,
    we show that Finsler embeddings not only recover the hidden manifold structure but also 
    reveal additional information hidden in the sampling-induced asymmetry.
    In extensive classification benchmarks, we 
    outperform Euclidean baselines on 
    a wide range of
    clustering metrics 
    suggesting better quality embeddings with our Finsler framework.

\begin{figure}[t]
    \centering
    \includegraphics[width=\textwidth]{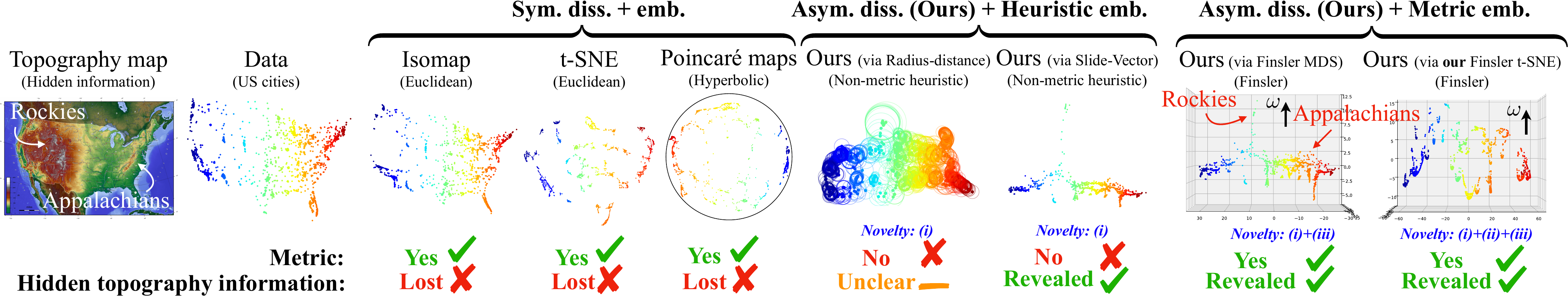}%
    \caption{
    Motivation on real data: how asymmetry can arise and why preserving it matters. 
    We aim to recover the underlying smooth US manifold from US cities (latitude-longitude).
    Hidden external factors
    bias the sampling density: fewer cities lie in 
    mountain ranges.
    Reweighing distances by local density yields asymmetric dissimilarities that encode differences in geographical setting. Symmetrising and embedding in symmetric spaces (e.g.\ Isomap, t-SNE, UMAP, Poincaré maps) discards this information. Our asymmetric dissimilarity construction can be fed to existing traditional asymmetric embedders (e.g.\ slide-vector, radius-distance), but these are heuristic and non-metric. Finsler geometry provides a principled metric framework, enabling the existing Finsler MDS \cite{dages2025finsler} and our novel and more scalable Finsler t-SNE and UMAP to better
    embed and reveal the hidden terrain.
    }
    \label{fig: US cities}
    \vspace{-1em}
\end{figure}%

\section{Related works}

\noindent \textbf{Manifold learning.}
This 
classical
field \cite{shepard1962analysis,kruskal1964multidimensional,fischl1999cortical,wandell2000visualization} seeks low-dimensional embeddings 
preserving
pairwise dissimilarities \cite{saeed2018survey}. 
The many methods
are 
often
grouped by whether they preserve structure locally \cite{mcinnes2018umap,roweis2000nonlinear,donoho2003hessian,zhang2006mlle,belkin2001laplacian,belkin2003laplacian,zhang2004principal,coifman2005geometric,coifman2006diffusion,lin2023hyperbolic,lin2024tree,suzuki2019hyperbolic, gou2023discriminative,martin2005visualizing,tang2016visualizing,venna2010information,liu2024curvature}, globally \cite{weinberger2005nonlinear,weinberger2006unsupervised,brand2005nonrigid,funke2020low}, or both \cite{van2008visualizing,pearson1901pca,hotelling1933analysis,scholkopf1998nonlinear,schwartz1989numerical,tenenbaum2000global,aflalo2013spectral,pai2022deep,schwartz2019intrinsic,bronstein2006generalized,rosman2010nonlinear,bracha2024wormhole,joharinad2025isumap,kim2024inductive,klimovskaia2020poincare}. 
Most assume symmetric dissimilarities, excluding asymmetric data 
 such as directed graphs or physical systems. 
Fewer works tackle
asymmetry, 
either by artificially reweighing symmetric dissimilarities or leaving the metric setting. 
Traditionally, they either
\textit{a)} decompose separately symmetric and skew-symmetric parts (ASYMSCAL \cite{chino1978graphical,chino2011asymmetric}, Gower \cite{constantine1978graphical,gower1977analysis}),
or \textit{b)} add
per-point objects \cite{martin2005visualizing,gower1977analysis,young1975asymmetric,okada1987nonmetric,tanioka2018asymmetric}, e.g.\ circles in radius-distance models 
that add to the Euclidean embedding dissimilarity the difference in radius of each point's circle
\cite{okada1987nonmetric,tanioka2018asymmetric}, 
or \textit{c)} use
global-vectors \cite{leeuw1982theory,borg2007modern,okada2012bayesian}, e.g.\ in the slide-vector model \cite{leeuw1982theory} 
embedding dissimilarities are the Euclidean distance between a non-perturbed embedding point and a globally shifted one.
These asymmetric embedders are not only heuristic but also non-metric, as for instance 
embeddings are not guaranteed to vary faithfully with data dissimilarity changes.

Recently, Finsler MDS \cite{dages2025finsler} extended classical Multi-Dimensional Scaling (MDS) \cite{shepard1962analysis,kruskal1964multidimensional,cox2000multidimensional} to metric handling of asymmetric dissimilarities via Finsler geometry, but it targets inherently asymmetric data (e.g.\ current flows on a river) and 
was not made applicable
to general datasets (e.g.\ images) usually treated as symmetric. 
It also relies on 
classical MDS rather than modern scalable methods like t-SNE \cite{van2008visualizing} and UMAP \cite{mcinnes2018umap}.
Beyond \cite{dages2025finsler}, we
construct asymmetric dissimilarities on any data 
(P1),
unlocking all 
asymmetric embedders 
(P2-P3)
to any data, 
accelerate Finsler MDS optimisation 
(P3),
and extend modern 
reference methods, e.g.\ t-SNE and UMAP, to asymmetry 
(P2-P3),
enabling large-scale embedding.
To the best of our knowledge, no prior work, including \cite{dages2025finsler},
extends modern methods like t-SNE and UMAP to asymmetric data, 
scales to large asymmetric datasets,
or builds asymmetric dissimilarities on any data (e.g.\ image datasets).

\noindent \textbf{Finsler computer vision.}
Finsler geometry \cite{ohta2009heat,bao2012introduction,mirebeau2014efficient,bonnans2022linear} provides a 
theoretical
asymmetry framework. Remarkably, it has been largely unexplored in computer vision, with rare exceptions in robotics \cite{ratliff2021generalized}, image processing \cite{chen2016geoevo,chen2015global,chen2017elastica,chen2016minpath,yang2019geodesic,dages2025metric}, shape analysis \cite{weber2024finsler},
wildfire simulation \cite{gahtan2026wildfire},
and manifold learning \cite{dages2025finsler}.

\section{The traditional manifold learning pipeline}

The goal of manifold learning is to find an embedding $y_1, \ldots, y_N\in\mathbb{R}^m$
of $N$ points 
$x_1,\ldots,x_N\in\mathbb{R}^n$
into a low-dimensional space, i.e.\ $m<n$, while preserving a collection of data dissimilarities $\boldsymbol{D}\in\mathbb{R}^{N\times N}_+$.
The original points 
$x_i$
are often assumed to lie on a manifold $\mathcal{X}\subset\mathbb{R}^n$, equipped with a metric $L$ given at point $x$ on its tangent plane $\mathcal{T}_x\mathcal{X}$ by $L_x:\mathcal{T}_x\mathcal{X}\to\mathbb{R}_+$ and defining the concept of distance and in turn geodesic distance $d_L:\mathcal{X}\times\mathcal{X}\to\mathbb{R}_+$.

A major assumption in traditional techniques is that the dissimilarities are symmetric $\boldsymbol{D} = \boldsymbol{D}^\top$, as the metric $L$ of the underlying manifold $\mathcal{X}$ is naturally symmetric for most datasets consisting in a collection of independently sampled points, e.g.\ image classification datasets. As such, the data manifold is often assumed to be Riemannian and is to be embedded into a Riemannian space $\mathbb{R}^m$.

\noindent \textbf{Riemannian manifolds.} The manifold $\mathcal{X}$ is Riemannian if it is equipped with a Riemannian metric $R$, 
which has quadratic form
$R_x(u) = \sqrt{u^\top M(x) u} = \lVert u\rVert_{M(x)}$.
The positive definite matrix $M(x)$ is the metric tensor and fully describes the metric. 
These metrics are symmetric as $R_x(-u) = R_x(u)$ for any tangent vector $u$. By integrating the length of infinitesimal tangent vectors $R_x(\gamma'(t)dt)$ of curves $\gamma(t)$ on $\mathcal{X}$, the Riemannian length of a curve 
is independent of its traversal direction, making Riemannian geodesic distances also symmetric.

In addition to dimensionality reduction, we usually aim to flatten the embedded data manifold in  $\mathbb{R}^m$ to show the simple intrinsic nature of the original data manifold $\mathcal{X}$. We thus opt for canonical embedding spaces with simple geometry, having geodesic curves given by the usual straight segments and simple closed formulas for distance calculations. The canonical Riemannian space is the Euclidean space, with distances computed in $\mathbb{R}^m$ as $d_E(y_i, y_j) = \lVert y_j - y_i\rVert_2$.

\noindent \textbf{Data dissimilarities (P1).}
Traditional manifold learning methods follow a shared pipeline.
First, the smooth manifold is discretely approximated on the 
samples by a proximity graph, e.g.\ the universal kNN, radius, or rarer adaptive graphs \cite{gabriel1969new,toussaint1980relative,jaromczyk1992relative,costa2005estimating,zelnik2004self,correa2012locally,arias2017spectral,dyballa2023ian,alvarez2011global,zhang2011adaptive,mekuz2006parameterless,amenta1998new,bernardini2002ball,belkin2009constructing,boissonnat2007manifold,inkaya2015adaptive,zhu2016natural},
based on the ambient Euclidean distance. 
Taking the manifold metric
as the Euclidean distance along close points, i.e.\ graph edges, $L_{x_i}(x_j - x_i) \approx \lVert x_j - x_i\rVert_2$ approximates in the first order the graph edges $x_j-x_i$ to lie on the tangent plane $\mathcal{T}_{x_i}\mathcal{X}$.
Modern methods often tweak the metric to a new Riemannian metric $R$ accounting for local sampling disparities, typically $R_{x_i}(x_j - x_i) \approx \tfrac{1}{\sigma_i}\lVert x_j - x_i\rVert_2$ with $\sigma_i$ estimated from local densities.
Data dissimilarities $p_{ij}$ are then built by optionally 
remapping
distances
via a transformation $h^p$:
$p_{ij} \!=\! h^p(R_{x_i}(x_j - x_i))$, 
in other words 
$ p_{ij} \!=\! h_i^p(\lVert x_j - x_i\rVert_2)$\footnote{And optional extra $i$-based transforms, e.g.\ UMAP's distance threshold.}. 
Commonly if $j$ neighbours $i$,
\begin{equation}
    \label{eq: traditional data dissimilarities}
        p_{ij} = \!\!\underbrace{\lVert x_j - x_i\rVert_2}_{\text{
        MDS \cite{schwartz1989numerical}, Isomap \cite{tenenbaum2000global}}} ; \,\,
        \underbrace{\mathrm{SM}\Big(-\tfrac{\lVert x_j - x_i \rVert_2^2}{2\sigma_i^2}\Big)
        }_{\text{t-SNE \cite{van2008visualizing}}}  \,\,;\,\, 
        \underbrace{e^{-\frac{\lVert x_j - x_i\rVert_2 - \rho_i}{\sigma_i}}}_{\text{UMAP \cite{mcinnes2018umap}}}
\end{equation}
where $\rho_i$ is another local estimate
and $\mathrm{SM}$ is the softmax.
If $i$ does not conversely neighbour $j$, then $p_{ji}$ is undefined.
Populating the graph weights $p_{ij}$ into a matrix provides the sparse dissimilarity matrix $\boldsymbol{D}$. Modern methods (t-SNE, UMAP) 
use
the sparsity for computational efficiency. 
Earlier works
(MDS, Isomap), geodesically extend $\boldsymbol{D}$ into a dense matrix with Dijkstra's algorithm \cite{dijkstra1959note} on the neighbourhood graph to account for long range behaviours. But this operation is costly as it requires fixing the connectivity, and lengths of shortest paths are overestimated due to undersampling, requiring expensive additional effort to filter only consistent shortest paths \cite{schwartz2019intrinsic,bronstein2006generalized,rosman2010nonlinear,bracha2024wormhole,bronstein2006efficient,bronstein2006robust,rosman2008topologically}.
Directed neighbourhood graphs like kNN lead to binary asymmetry
due to missing reverse edges.
Local metric tweaking of $R$, with transforms $h_i^p$ depending on $i$, yield ``incompatible'' \cite{mcinnes2018umap} weights $p_{ij}$ with non-binary asymmetry\footnote{
Binary asymmetry: either $p_{ij}=p_{ji}$ or one of them is undefined.
Non-binary asymmetry: $p_{ij}\neq p_{ji}$ yet both are defined.},
see Fig.~\ref{fig: asymmetry types}. 
As asymmetry violates the Riemannian 
assumption,
the $p_{ij}$ are arbitrarily symmetrised,
commonly\footnote{
If $p_{ij}$ is defined but not $p_{ji}$, then $p_{ji}$ is treated as $0$ in symmetrisation.
}
\begin{equation}
    \label{eq: traditional symmetrisation}
    p_{ij} \leftarrow \underbrace{\frac{p_{ij} + p_{ji}}{2},\max(p_{ij}, p_{ji})}_{\text{Classical MDS, Isomap}} \,\,;\,\,
    \underbrace{\frac{p_{ij} + p_{ji}}{2N}}_{\text{t-SNE}} \,\,;\,\,
    \underbrace{p_{ij} + p_{ji} - p_{ij}p_{ji}}_{\text{UMAP}}.
\end{equation}

\begin{figure}[t]
    \centering
    \adjincludegraphics[
            width=\textwidth,
            trim={{0} {0} {0.022\width} {0}},clip,
            ]{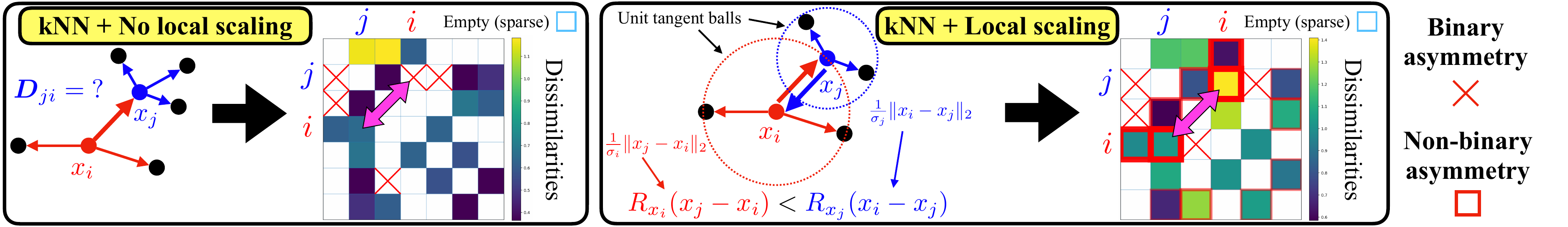}
    \vspace{-1em}
    \caption{
    Left: binary asymmetry from absent reverse edges in directed proximity graphs.
    Right: non-binary asymmetry with reciprocal edges having differing distances from locally tweaking the metric and approximating geodesic with tangent space distances.
    }
    \label{fig: asymmetry types}
\end{figure}

\noindent \textbf{Embedding dissimilarities (P2).}
The canonical Riemannian
space is Euclidean, 
with
embedding distances 
$d_{ij} \triangleq d_E(y_i,y_j) = \lVert y_j-y_i\rVert_2$ 
transformed 
to
dissimilarities $q_{ij}\leftarrow h^q(\lVert y_j - y_i\rVert_2)$, 
e.g.\
\begin{equation}
    \label{eq: traditional embedding dissimilarities}
        q_{ij} = \underbrace{d_{ij}\!\triangleq\!\lVert y_j-y_i\rVert_2}_{\text{
        MDS, Isomap
        }} \,;\,
        \underbrace{\frac{1}{\mathcal{Z}_{y}}\bigg(\!1 + \frac{d_{ij}^2}{\nu}\!\bigg)^{\!-\frac{\nu+1}{2}}
        }_{\text{t-SNE}} \,\,;\,\,
        \underbrace{(1+ad_{ij}^{2b})^{-1}}_{\text{UMAP}},
\end{equation}
with 
parameters
$\nu$ \cite{van2009learning}, $a$, and $b$ and normalisation 
$\mathcal{Z}_y$.
The $q_{ij}$
are computed 
for 
all pairs
and are symmetric (Euclidean space):
$h^q$ is independent of $i$.

\noindent \textbf{Optimisation (P3).}
An objective $\mathcal{L}(p,q)$ is minimised, e.g.
\begin{equation}
    \label{eq: traditional objectives}
    \mathcal{L} = \underbrace{\mathrm{MSE}_w(p, q)}_{\text{
    MDS, Isomap}}
    \; ;\;
    \underbrace{\mathrm{KL}(p \,\Vert\, q)
    }_{\text{t-SNE}
    } \;;\;
    \underbrace{\mathrm{CE}(p,q)
    }_{\text{UMAP}},
\end{equation}
where $\mathrm{KL}$, $\mathrm{CE}$, and $\mathrm{MSE}_w$ are the KL-divergence \cite{kullback1951information}, 
cross-entropy, and mean squared error with optional weights $w$ to focus 
on consistent pairs \cite{rosman2010nonlinear,schwartz2019intrinsic,bracha2024wormhole}.
MDS is solved 
iteratively
via the SMACOF algorithm \cite{leeuw1977application,borg2007modern}, but 
cheaper kernel PCA (Isomap) is usually done in practice\footnote{In Isomap, $w_{ij}\!=\!1$ and distances $p_{ij}$, $q_{ij}$ are squared in the objective.}.
Gradient descent is done in t-SNE and UMAP (with negative sampling for UMAP). Efficient implementations, which are crucial for large datasets, 
require explicit closed-form gradient and update rules\footnote{
We found an error in the update rule of t-SNE \cite{van2009learning} and its implementations, e.g.\ in Scikit-learn \cite{pedregosa2011scikit}, see Sec.~\ref{sec: tSNE detailed update rule}.
} (Sec.~\ref{sec: update rules of t-SNE and UMAP}).

\section{Limitations of traditional manifold learning}
\label{sec: limitations of traditional manifold learning}

We argue in this work that there are two fundamental issues with traditional methods relating to Riemannian metrics:
\textit{(i)} the 
\textit{data dissimilarity} (P1)
methodology cannot be Riemannian as it leads to asymmetry, 
and \textit{(ii)} asymmetry is
incompatible with the Euclidean 
\textit{embedding definition} (P2)
and \textit{optimisation}
(P3).
Also, symmetry discards information on the sampling itself.

\noindent \textbf{Non-Riemannian methodology
(P1)}.
A major issue with existing methods, claiming to be Riemannian, is that the 
\textit{data dissimilarities} stage (P1)
leads to asymmetric 
dissimilarities.
This is prohibited in Riemannian geometry as Riemannian metrics are symmetric.
Rather than switching to asymmetric tools, like Finsler geometry, that would account for the natural asymmetry of the data,
unjustified heuristic fixes
symmetrise them 
(Eq.~\ref{eq: traditional symmetrisation}), discarding information in the process.
Yet even from a Riemannian perspective, this symmetrisation is problematic.
While it can be acceptable for binary asymmetry, 
as it converts 
the directed neighbourhood graph to another valid undirected one, it is highly problematic with non-binary asymmetry, which we now focus on.
The issue arises from combining local metric tweaking $R_{x_i}(x_j - x_i) = c(x_i) \lVert x_j - x_i\rVert_2$,
e.g.\ $c(x_i) \!=\! \tfrac{1}{\sigma_i}$  
(manifold with non-uniform isotropic metric $M(x)\!=\!c^2(x)I$),
with approximating geodesic distances between close points with tangent space distances, $d_R(x_i, x_j) \approx R_{x_i}(x_j - x_i)$, leading to 
$p_{ij} \!=\! h^p(R_{x_i} (x_j - x_i))$.
Since $R$ is Riemannian, it must have symmetric manifold distances $d_R(x_i, x_j) = d_R(x_j, x_i)$. Yet tangent space distances $R_{x_i}$ and $R_{x_j}$ can greatly vary in non-uniform metrics, even between close points, if the local scaling $c(x_i)$ and $c(x_j)$ does.
This problematic yet universal methodology puts into question the claimed rigorous Riemannian theoretical foundations existing methods might rely on,
even when trying to justify post hoc symmetrisation 
fixes
(Eq.~\ref{eq: traditional symmetrisation}) with layers of 
complex theory\footnote{Such as the probabilistic union averaging with fuzzy logic in UMAP.}.
Nevertheless, the traditional pipeline has shown practical success in many cases, thus preserving the asymmetry might not always be necessary, especially in tasks where sampling variation is a nuisance.
But in other cases,
symmetry fixing 
loses important information compared to an asymmetric perspective: the asymmetric disparities from the sampling itself, e.g.\ from disparate sampling densities. Such 
information is lost 
in 
symmetric settings 
(see Fig.~\ref{fig: US cities}
and Sec.~\ref{sec: what is lost in symmetrisation}).
Rather than working towards theoretical symmetry fixes (Sec.~\ref{sec: fixing symmetric pipeline}), we thus propose 
from now on
to preserve this information and 
work with the naturally arising asymmetric dissimilarities $\boldsymbol{D}^\top\neq \boldsymbol{D}$, requiring to drop Riemannian tools for asymmetry-compatible ones.

\noindent \textbf{Asymmetry incompatibility (P2-P3)}.
Traditional pipelines assume symmetric dissimilarities 
(P1)
to embed data in the symmetric 
Euclidean space $\mathbb{R}^m$
(\textit{embedding dissimilarities} (P2)
and \textit{optimisation}
(P3)).
As such, traditional 
methods would 
crash or at best discard
the naturally arising asymmetry
(P1),
losing the extra information on the discrete sampling distribution.
Recently, \cite{dages2025finsler} proposed 
Finsler geometry 
for asymmetry by replacing
the Euclidean 
with the asymmetric canonical Randers space
for \textit{embedding dissimilarities} (P2).
However, this method was only applicable to 
data with given asymmetry
(P1),
like digraphs or physical systems with currents,
but not arbitrary datasets traditionally considered symmetric
yet having natural asymmetry as we have discussed
(absent construction of asymmetric \textit{data dissimilarities} (P1)).
Also,
this method is slow and numerically unstable
(\textit{optimisation} 
(P3)).
Moreover,
it does not possess the benefits of recent reference symmetric methods
(P2-P3),
e.g.\ t-SNE or UMAP, providing high speeds, numerical stability, and clustering properties. It is thus desirable to adapt the modern reference techniques to handle asymmetric data.

\section{Asymmetric manifold learning}

In this paper, we suggest to embrace the natural asymmetry arising from traditional pipelines. To do so, we need to
switch from Riemannian geometry to 
its asymmetric generalisation:
Finsler geometry.

\noindent \textbf{Finsler geometry.}
Finsler manifolds \cite{bao2012introduction} generalise Riemannian ones 
\cite{chern1996finsler} and allow asymmetry. Such manifolds $\mathcal{X}$ are equipped with a Finsler metric $F$, 
with
$F_x: \mathcal{T}_x\mathcal{X}\to\mathbb{R}_+$ 
such that $F_x(u)=0$ if and only if $u=0$, it follows the triangular inequality, and is positive-homogeneous, i.e.\ $F_x(\lambda u) = \lambda F_x(u)$ for any $\lambda>0$.
Many families of parametric Finsler metrics exist \cite{javaloyes2011definition,matsumoto1972c,randers1941asymmetrical,matsumoto1989slope,kropina1959projective}. We work with the common \cite{weber2024finsler,dages2025finsler,dages2025metric} Randers metric \cite{randers1941asymmetrical} for simplicity, $F_x(u) = \lVert u\rVert_{M(x)} + \omega(x)^\top u$ with $\lVert \omega(x)\rVert_{M^{-1}(x)}<1$, where the added linearity
breaks the symmetry. Note that Riemannian metrics have $\omega \equiv 0$.

To embed 
given
asymmetric data in $\mathbb{R}^m$ 
(P2),
\cite{dages2025finsler} recently proposed to replace the canonical Riemannian metric (Euclidean) with a canonical Randers metric $F^C$ having constant $\omega$, typically chosen along the last dimension, perturbing the Euclidean metric: $F_x^C(u) = \lVert u\rVert_2 + \omega^\top u$.
Note that it becomes Euclidean when $\omega = 0$ and is flat just like it as geodesics are the usual Euclidean segments. However, their lengths depend on the traversal direction, with a simple closed-form expression $d_{F^C}(x,y) = \lVert y-x\rVert_2 + \omega^\top (y-x)$.
Schematically, the canonical Randers space combines a Riemannian hyperplane $\mathbb{R}^{m-1}$ with an orthogonal direction of asymmetry $\omega$. 
For symmetric data, the embedding collapses to the $m-1$ hyperplane. It is thus 
established
to compare asymmetric embeddings to the canonical Randers space $\mathbb{R}^{m+1}$ to symmetric embeddings in $\mathbb{R}^m$ \cite{dages2025finsler}.
Note that choosing $\omega$ along the last dimension axis is without loss of generality. Indeed, should we denote $d_{\omega}$ to be $d_{F^C}$ with Randers component $\omega$, then applying a rotation $U$ to the embedding space leaves Finsler distances unchanged as $d_{U\omega}(Ux,Uy)=\lVert U(y-x)\Vert_2+(U\omega)^\top U(y-x)=d_\omega(x,y)$.

Finsler
embeddings were found in \cite{dages2025finsler} by replacing 
Euclidean distances $q_{ij} = d_{ij}$ with
Finsler ones $q_{ij}^F = d_{F^C}(y_i, y_j) \triangleq d_{ij}^F$
(P2),
and minimising the MSE loss by the Finsler SMACOF algorithm
(P3).
However, the 
original algorithm from \cite{dages2025finsler}
is slow and memory expensive due to reshaping linear systems of equations to $N^2\times N^2$ matrices, with unstable pseudo-inverse calculations in practice (unlike the original SMACOF)\footnote{ We also provide a corrected and more stable and scalable update rule in Sec.~\ref{sec: corrected update rule of finsler mds}, which avoids matrix vectorisation and unstable solvers for pseudo-inversion, yet it is still less scalable than sparse methods like Finsler t-SNE and Finsler UMAP.}.
Although it 
handles asymmetry, it does not use any of the modern manifold learning knowledge for  
using
sparse dissimilarities and fast calculations as in t-SNE or UMAP, which also provide scalability and clustering properties. Additionally, \cite{dages2025finsler} did not provide any way to apply their method to traditionally considered symmetric manifolds,
e.g.\ image datasets, drastically limiting their application scope to physically asymmetric data like current maps or small digraphs
(P1).

\noindent \textbf{From symmetric to asymmetric manifolds.}
We now return to the traditional manifold learning pipeline and its 
\textit{data dissimilarities} construction (P1).
Asymmetry naturally arose when combining local Riemannian metric tweaking with approximating geodesic with tangent lengths. Although it is an issue for Riemannian geometry, it naturally fits in a Finsler perspective, which we propose to adopt.
This process in fact equips the data manifold $\mathcal{X}$ with an implicit Finsler metric $F$, where \textit{geodesic} distances between close points are given by the scaled Riemannian \textit{tangent} lengths $d_F(x_i, x_j) = R_{x_i}(x_j - x_i)$,
see Sec.~\ref{sec: finsler metric on the data} for further explanations.
We thus propose to keep the computed asymmetric dissimilarities $p_{ij} = h^p(R_{x_i}(x_j - x_i)) = h_i^p(\lVert x_j - x_i\rVert_2)$ without artificial symmetrisation\footnote{Normalisations though are preserved, for instance in Finsler t-SNE raw dissimilarities are normalised by $p_{ij} \leftarrow \tfrac{p_{ij}}{N}$.}.
We can optionally extend $p_{ij}$ geodesically on the directed neighbourhood graph if a dense dissimilarity matrix $\boldsymbol{D}$ is needed, otherwise $p_{ij}$ form a sparse $\boldsymbol{D}$.
Should they be symmetric, i.e.\ $R$ is uniform and $h_i^p$ is independent of $i$, then the data manifold $\mathcal{X}$ is equipped with a symmetric metric and 
symmetric Euclidean embeddings should be done.
If not, then we should embed the Finsler manifold $\mathcal{X}$ to an asymmetric Finsler space, e.g.\ the canonical Randers space.
Switching to Finsler metrics not only resolves the theoretical issue with the traditional methodology, its bypassing of symmetrisation preserves valuable natural asymmetric information that encodes the sampling, e.g.\ density disparities.
See Sec.~\ref{sec: asymptotic analysis link sampling density and directionality} for a quantitative analysis expressing the link between the computed directionality and the sampling density and refer to Sec.~\ref{sec: on vanishing asymmetry} for special cases when our asymmetric construction can collapse with the symmetric one.

\noindent \textbf{Finsler manifold learning.}
We propose to generalise existing symmetric manifold learning methods to asymmetric data and Finsler embeddings while preserving their unique advantages and characteristics 
(\textit{embedding dissimilarities} (P2) and \textit{optimisation} (P3) stages).
Our extension is general and can be applied to most methods, yet we focus as examples on the references t-SNE and UMAP.
Given traditional methods with symmetric embedding dissimilarities $q_{ij} = h^q(\lVert x_j - x_i\rVert_2)$ and Riemannian objectives $\mathcal{L}(p,q)$, we 
propose to generalise them to Finsler methods by embedding into the canonical Randers 
space with asymmetric embedding dissimilarities $q_{ij}^F = h^q(d_{F^C}(x_j - x_i))$
(P2),
e.g.\
\begin{equation}
    \label{eq: finsler embedding dissimilarities}
    q_{ij}^F = \hspace{-1.5em} \underbrace{d_{ij}^F}_{\text{Finsler MDS \cite{dages2025finsler}}} \!\!\! ; \,
    \underbrace{\frac{1}{\mathcal{Z}_y^F}\bigg(\!1 + \frac{(d_{ij}^F)^2}{\nu}\!\bigg)^{\!-\frac{\nu+1}{2}}
    }_{\text{Finsler t-SNE}} \,;\,
    \underbrace{(1+a (d_{ij}^F)^{2b})^{-1}}_{\text{Finsler UMAP}},
\end{equation}
and then minimise the Finsler objective $\mathcal{L}(p, q^F)$
(P3),
e.g.\
\begin{equation}
    \mathcal{L} = 
    \underbrace{\mathrm{MSE}_w(p, q^F)}_{\text{Finsler MDS \cite{dages2025finsler}}} \; ;\;
    \underbrace{\mathrm{KL}(p\,\Vert\,q^F)
    }_{\text{Finsler t-SNE}} \;;\;
    \underbrace{\mathrm{CE}(p, q^F)
    }_{\text{Finsler UMAP}}.
\end{equation}
Analysis of Euclidean objectives yields their minimisation\footnote{Like closed-form spectral solutions (Isomap), 
majorisation
(SMACOF), or gradient descent and variants (t-SNE, UMAP).}.
Similar analysis of Finsler objectives generalises the update rules (P3).
Note that closed-form spectral solutions often do not generalise in the Finsler setting. More expensive iterative alternatives, e.g.\ the Finsler SMACOF or gradient descent procedures, are then necessary. The reason is that asymmetry destroys the quadratic kernel form, so eigendecomposition is then no longer the solution, and instead constrained iterative optimisation is required. 
We thus focus in this work on non-spectral methods, and push to Sec.~\ref{sec: Asymmetric Finsler generalisation of spectral methods in manifold learning} how to generalise spectral methods, like Diffusion maps \cite{coifman2006diffusion}, to Finsler embeddings.
We also show in the supplementary how to generalise to our Finsler asymmetric pipeline other existing non-spectral methods. Notably, we derive Finsler PHATE from PHATE \cite{moon2019visualizing} in Sec.~\ref{sec: Finsler PHATE} as an example on how to naturally generalise symmetric approaches that do not seem to have an explicit data symmetrisation step at first glance.

In t-SNE and UMAP, explicit gradients are needed for efficient implementations in optimised code languages.
Remarkably, gradients of canonical Finsler distances also share the same antisymmetry as Euclidean ones (Theorem \ref{th: gradient canonical Finsler distances}, Sec.~\ref{sec: proofs finsler tsne umap}),
suggesting beyond the arguments in \cite{dages2025finsler} that this is the proper generalisation of the Euclidean space. 
Calculations provide explicit gradients (Theorems \ref{th: finsler tsne update rule} and \ref{th: finsler umap update rule}, 
proven in Sec.~\ref{sec: proofs finsler tsne umap}).
Of note, updates are no longer just a combination of (asymmetric) forces along the rays connecting points, but also of external current forces\footnote{This echoes the time dynamics in currents, e.g.\ boats on a river, given by the Zermelo metric, which is
Randers
\cite{zermelo1931navigationsproblem,shen2003finsler,dages2025finsler}.}. 
The update rules then follow gradient descent with any optional tricks from the Euclidean methods, e.g.\ negative sampling \cite{mikolov2013distributed,tang2016visualizing} for (Finsler) UMAP.

\begin{theorem}[Finsler t-SNE]
    \label{th: finsler tsne update rule}
    Let $\mathcal{L} = -\sum_{ij} p_{ij}\ln q_{ij}^F$ be the Finsler t-SNE 
    loss
    with $q_{ij}^F$ from Eq.~\ref{eq: finsler embedding dissimilarities}. 
    Denoting 
    $t_{ij}^F \!=\! \big(1+\tfrac{(d_{ij}^F)^2}{\nu}\big)^{-1}$
    and $\delta_{pq^F}^{ij} \!=\! p_{ij} - q_{ij}^F$,
    then
    {
    \begin{equation*}
        {\frac{\partial \mathcal{L}}{\partial y_i}}\! =  
        \frac{\nu+1}{\nu}\! \sum\limits_j \!\Bigg(\Bigg[\! 
        \delta_{pq^F}^{ij}
        t_{ij}^F \frac{d_{ij}^F}{d_{ij}} + 
        \delta_{pq^F}^{ji}
        t_{ji}^F\frac{d_{ji}^F}{d_{ij}} \!\Bigg]\!(y_i - y_j)
        + \Bigg[ -
        \delta_{pq^F}^{ij}
        t_{ij}^F d_{ij}^F + 
        \delta_{pq^F}^{ji}
        t_{ji}^F d_{ji}^F \!\Bigg]\omega\Bigg).
    \end{equation*}
    }%
\end{theorem}

\begin{theorem}[Finsler UMAP]
    \label{th: finsler umap update rule}
    Let $\mathcal{L} = -\sum p_{ij}\ln q_{ij}^F + (1-p_{ij})\ln(1-q_{ij}^F)$ be the Finsler UMAP 
    loss
    with
    $q_{ij}^F$
    from 
    Eq.~\ref{eq: finsler embedding dissimilarities}. 
    Denoting $c_{ij}^a = -\ln q_{ij}^F$ and $c_{ij}^r = -\ln (1-q_{ij}^F)$ to be the asymmetric attractive and repulsive forces 
    and $b' = 2b-1$,
    then
    $\tfrac{\partial c_{ij}^a}{\partial y_i}= -\tfrac{\partial c_{ij}^a}{\partial y_j}$ and $\tfrac{\partial c_{ij}^r}{\partial y_i} = - \tfrac{\partial c_{ij}^r}{\partial y_j}$ with
    {
    \begin{align*}
        \frac{\partial c_{ij}^a}{\partial y_i} \!=\! 2ab \frac{(d_{ij}^F)^{b'}}{d_{ij}} q_{ij}^F (y_i - y_j) - 2ab (d_{ij}^F)^{b'} q_{ij}^F\omega
        \text{ and }
        \frac{\partial c_{ij}^r}{\partial y_i} = \frac{-2b q_{ij}^F}{d_{ij} d_{ij}^F} (y_i - y_j) + \frac{2b q_{ij}^F}{d_{ij}^F} \omega.
    \end{align*}
    }%
\end{theorem}

\noindent \textbf{Combining Finsler data and embedding perspectives.}
For any dataset, we can combine both our asymmetric understanding of the 
data manifold, with extracted asymmetric dissimilarities $\boldsymbol{D}\neq \boldsymbol{D}^\top$
(P1),
and our generalisation of existing manifold learning methods to Finsler embeddings
(P2-P3).
To the best of our knowledge, we are the first to provide this perspective. In particular, \cite{dages2025finsler} proposed a Finsler embedding method
(P2-P3)
for given asymmetric data, like physical systems or digraphs, but did not propose ways to find asymmetry on arbitrary datasets, like collections of images
(P1).
We can also just apply the asymmetric Fisnler embedders (P2-P3) directly to directed graphs (see Sec.~\ref{sec: directed graphs}).

\noindent \textbf{Summary of main technical novelties.}
This work combines several simple yet important technical novelties.
\textit{a)}~We provide a way to compute asymmetric dissimilarities for ordinary data (P1), e.g.\ image datasets, and are not constrained as in prior work \cite{dages2025finsler} to given asymmetric data like directed graphs. 
\textit{b)}~Taking inspiration from \cite{dages2025finsler} that focused solely on adapting MDS to Finsler embeddings, we provide a framework to generalise any manifold learning approach to asymmetric Finsler embeddings (P2), in part by switching from Euclidean to Finsler embedding distances. The canonical Randers model \cite{dages2025finsler} yields interpretable asymmetric embeddings separating the symmetric and asymmetric components, with 
$\tfrac{1}{2}(d_{ij}^F+d_{ji}^F)=\| y_j-y_i\|_2$ 
and 
$\tfrac{1}{2}(d_{ij}^F-d_{ji}^F)=\omega^\top(y_j-y_i)$.
\textit{c)}~We derive explicit gradients for efficient sparse implementations. This enables asymmetric embeddings at large scale, e.g.\ Finsler UMAP can be evaluated on over one million ImageNet samples, whereas prior work \cite{dages2025finsler} could not scale beyond a couple thousand points.
We also provide other minor novelties, such as discovering incorrect update rules in prior work like t-SNE \cite{van2009learning} and Finsler MDS \cite{dages2025finsler}.

\section{Experiments}

We provide qualitative and quantitative experiments to 
validate
our approach. Full details -- data, implementation, results -- are 
in
the appendix (Secs.~\ref{sec: implementation details} and \ref{sec: further experiments}).

\subsection{Synthetic datasets}

As high-dimensional data is often sampled from unknown incomprehensible manifolds, evaluating embeddings is non-trivial. It is standard
to test new methods on synthetic data whose underlying manifold is known and understood. Such mandatory experiments offer fundamental insight into the quality of embeddings and how well they preserve manifolds. We test several synthetic setups.

\textbf{Planar manifold.} \textit{Data.}
We use $N=300$ points on the unit Euclidean disk in $\mathbb{R}^2$, sampled on a uniform polar grid. In Cartesian coordinates this yields non-uniform density (denser at the centre).

\noindent\textit{Methods.}
We embed with symmetric baselines (Isomap, t-SNE, UMAP) in $\mathbb{R}^m$ with $m=2$, matching the manifold dimension of $\mathcal{X}$. Following \cite{dages2025finsler}, Finsler embeddings use the canonical Randers space $\mathbb{R}^{m+1}=\mathbb{R}^3$ with an extra coordinate to encode asymmetry. To reveal 
sampling
disparities, we use methods in our asymmetric pipeline (Finsler MDS, Finsler t-SNE, and Finsler UMAP). 
Note that without our pipeline
(P1-P2-P3),
the original Finsler MDS \cite{dages2025finsler} would not run here as they did not provide ways to compute asymmetric dissimilarities on such data 
(P1).
We also run the original Finsler SMACOF and a new solution with gradient-descent (GD) on the Finsler stress
(P3),
which scales to larger $N$ and avoids pseudo-inverse instabilities we found in Finsler SMACOF
(P3).

\noindent\textit{Results.}
Although the underlying manifold is the Euclidean disk, asymmetric density-dependent dissimilarities $p_{ij}$ endow it with a Finsler metric: higher local densities (lower $\sigma_i$) increase $p_{ij}$, so going from dense to sparse regions costs more than the reverse. Embedding with our asymmetric $p_{ij}$ in the canonical Randers space exposes this (see Fig.~\ref{fig: toy data single peak n 300}): high-density areas are mapped to lower $z$ than sparse ones, while the disk geometry is preserved in top-down views along $xy$ hyperplanes orthogonal to $\omega$ (upwards $z$ axis). This density information is lost in symmetric embeddings. Isomap preserves the disk but does not explicitly reveal density variation, while the references t-SNE and UMAP poorly represent the data. 
We rule out an artifact from the extra Randers coordinate by giving Euclidean (resp.\ Finsler) methods an extra (resp.\ minus) embedding dimension (Sec.~\ref{sec: toy planar data embedding dim extra minus}). 
At larger $N$ (Sec.~\ref{sec: toy planar data number of datapoints}), we see limits of Finsler MDS, the local asymmetric nature of purely local methods (Finsler UMAP), and confirm that t-SNE outperforms UMAP but is slower (Euclidean and Finsler).

\begin{figure}[t]
    \captionsetup[subfigure]{labelformat=empty,justification=centering,position=top}
    \centering
    \includegraphics[width=0.9\columnwidth]{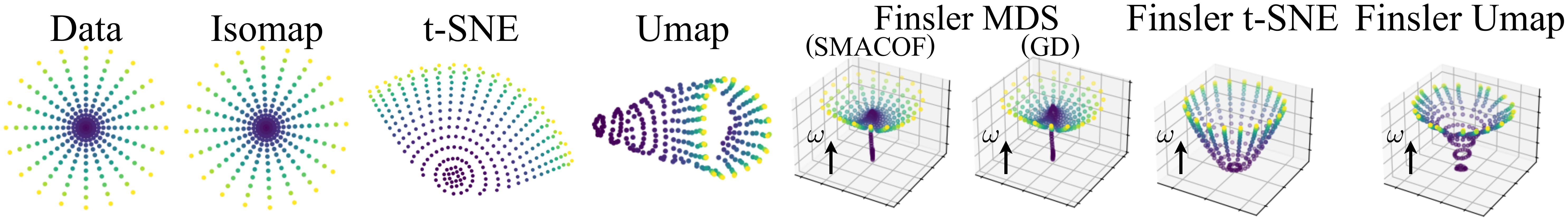}%
    \caption{
    Toy planar data with non-uniform 
    sampling,
    embedded with symmetric baselines and our Finsler methods using asymmetric dissimilarities. Variation in the $z$ coordinate reveals asymmetric distances 
    encoding density disparity,
    while a top view ($xy$ only) preserves the manifold as in Isomap.
    }
    \label{fig: toy data single peak n 300}
    \vspace{-1em}
\end{figure}

\textbf{Curved manifold.} \textit{Data.}
The data is $N\approx2000$ points sampled on the Swiss roll: an intrinsically flat 2D manifold embedded in $\mathbb{R}^3$ with extrinsic curvature. Points are obtained by regular sampling of the unit square, then mapped to the rolled surface. This yields higher sample 
counts
where mean curvature is larger (toward the roll’s inside). See Sec.~\ref{sec: data swiss roll} for details.

\noindent\textit{Methods.}
We embed the 
Swiss roll 
with
Isomap, t-SNE, and UMAP,
and 
Finsler MDS-GD, t-SNE, and UMAP.
To 
boost
manifold preservation in (Finsler) UMAP, we geodesically extend distances $d_{ij}$ in the kNN graph 
($k=15$)
and 
keep
the 50 smallest edges per node to 
get
a sparse graph.

\noindent\textit{Results.}
Our Finsler embeddings visually expose the 
hidden sampling
asymmetry: denser regions map lower than sparse ones (see Fig.~\ref{fig: toy data swiss roll}). In contrast, Euclidean methods do not reveal it, whether preserving the manifold (Isomap) or not (t-SNE, UMAP). Finsler UMAP highlights the asymmetry more than Finsler MDS and t-SNE, whose embeddings are near-flat and almost orthogonal to the $z$ axis of asymmetry. Thus, our approach not only recovers the band-like structure in top-down views similar to Euclidean mappings but also adds density-based information lost in Euclidean methods.

\begin{figure}[t]
  \centering
  \setlength{\tabcolsep}{2pt}
  \resizebox{0.9\textwidth}{!}{%
      \begin{tabular}{@{\hspace{-5pt}}c c c c c c c c}
    
        \begin{tabular}[t]{@{}c@{}}
          \scriptsize Data\\[0.25em]
          \adjincludegraphics[width=0.12\textwidth,
            trim={{0.1\width} {0.1\width} {0.1\width} {0.1\width}},clip,
            ]{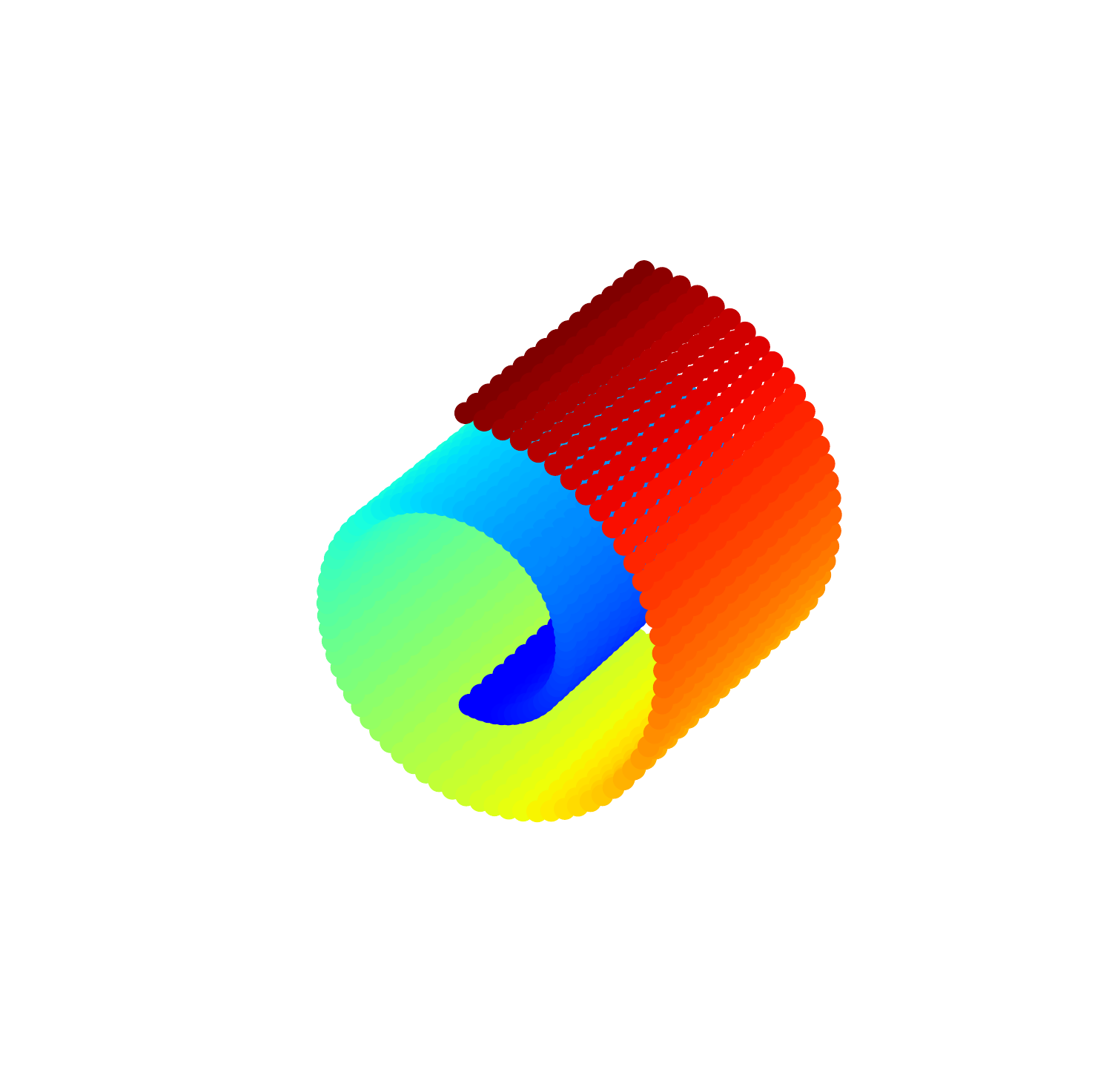}
        \end{tabular} 
        &
    
        \begin{tabular}[t]{@{}c@{}}
          \scriptsize Isomap\\[0.25em]
          \adjincludegraphics[width=0.12\textwidth,
            trim={{0.1\width} {0.1\width} {0.1\width} {0.1\width}},clip,
            ]{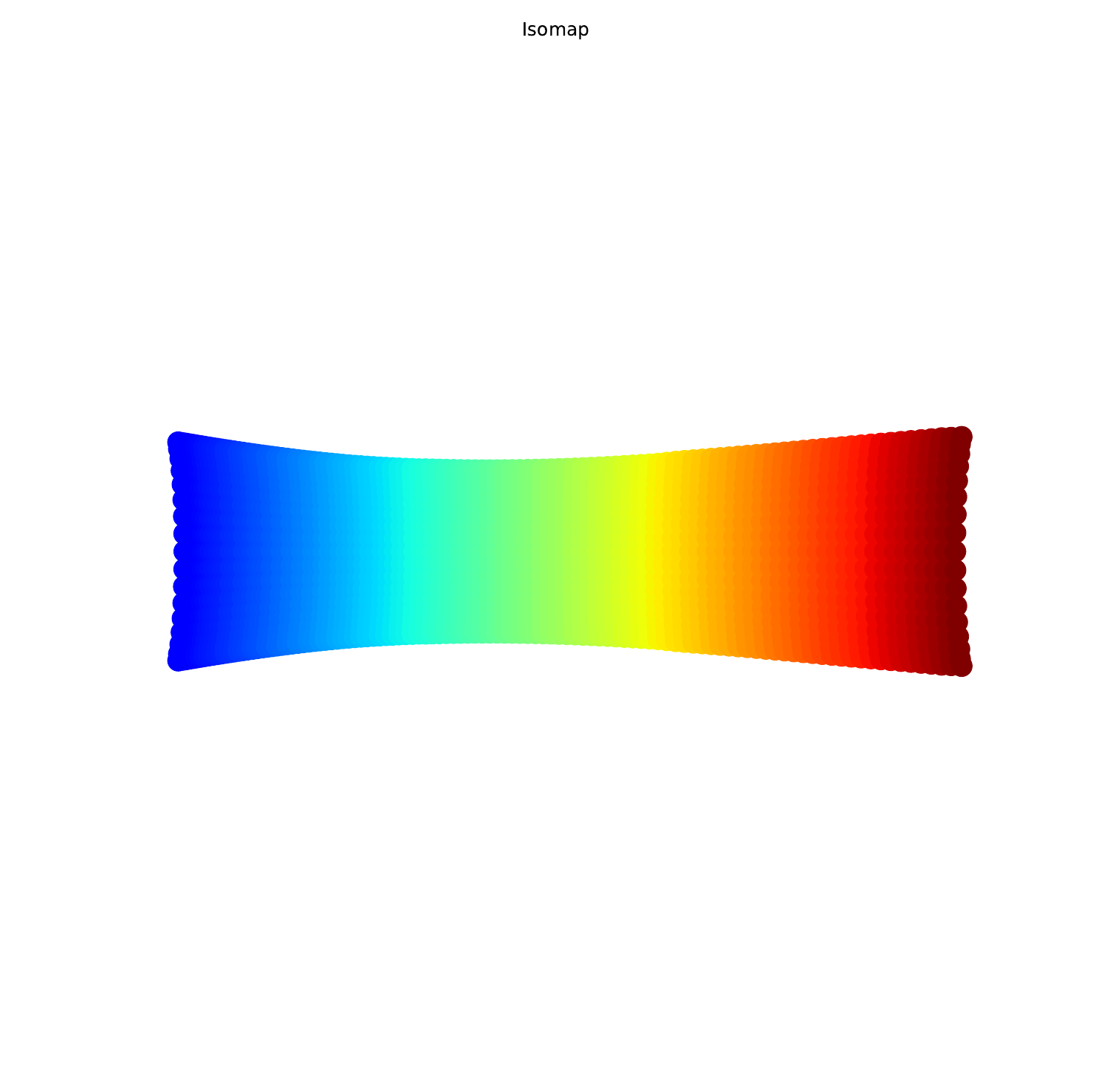}
        \end{tabular} &
    
        \begin{tabular}[t]{@{}c@{}}
          \scriptsize t-SNE\\[0.25em]
          \adjincludegraphics[width=0.09\textwidth,
            trim={{0.24\width} {0.13\width} {0.21\width} {0.13\width}},clip,
            ]{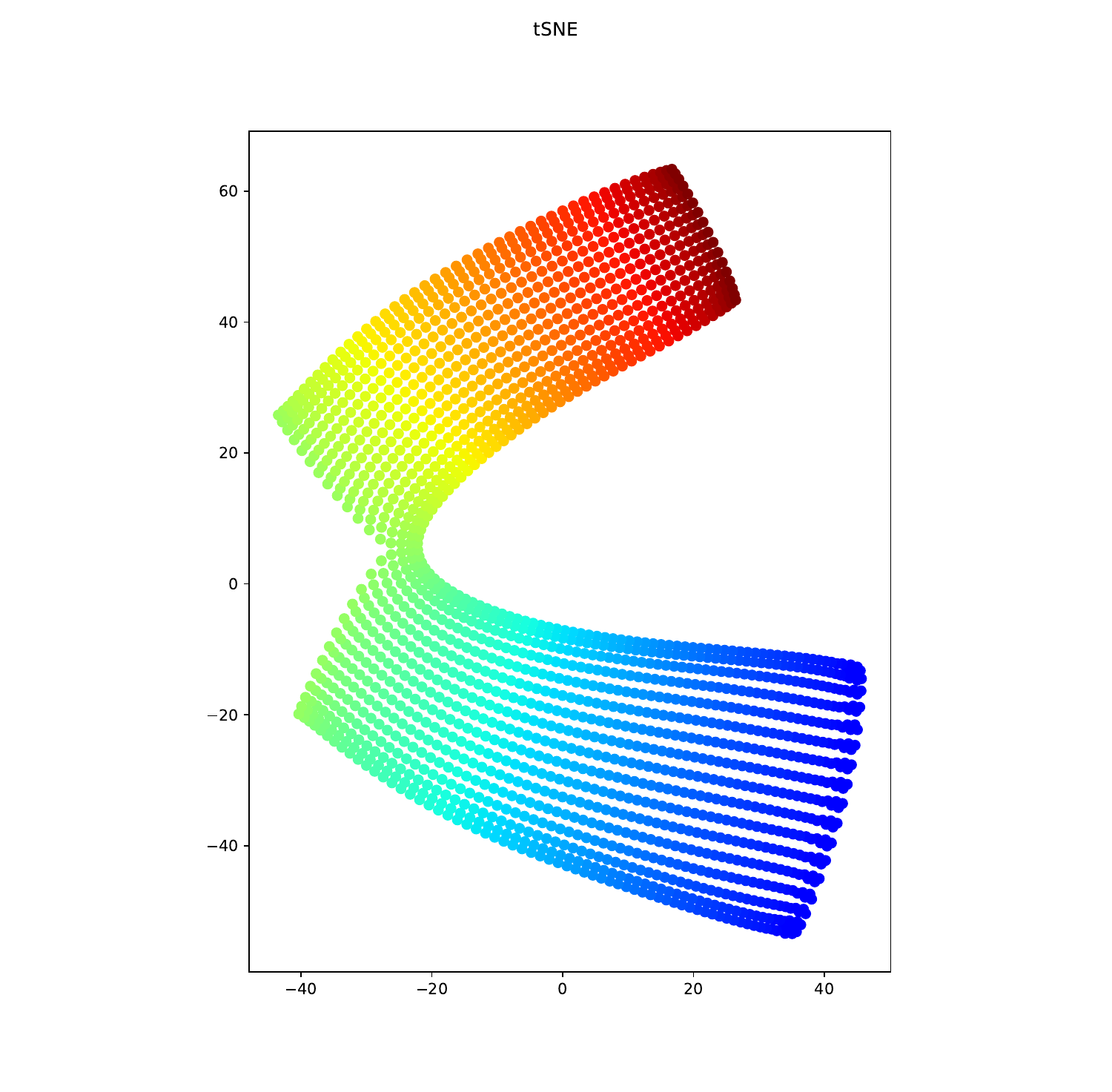}
        \end{tabular} &
    
        \begin{tabular}[t]{@{}c@{}}
          \scriptsize UMAP\\[0.25em]
          \adjincludegraphics[width=0.105\textwidth,
            trim={{0.19\width} {0.13\width} {0.17\width} {0.14\width}},clip,
            ]{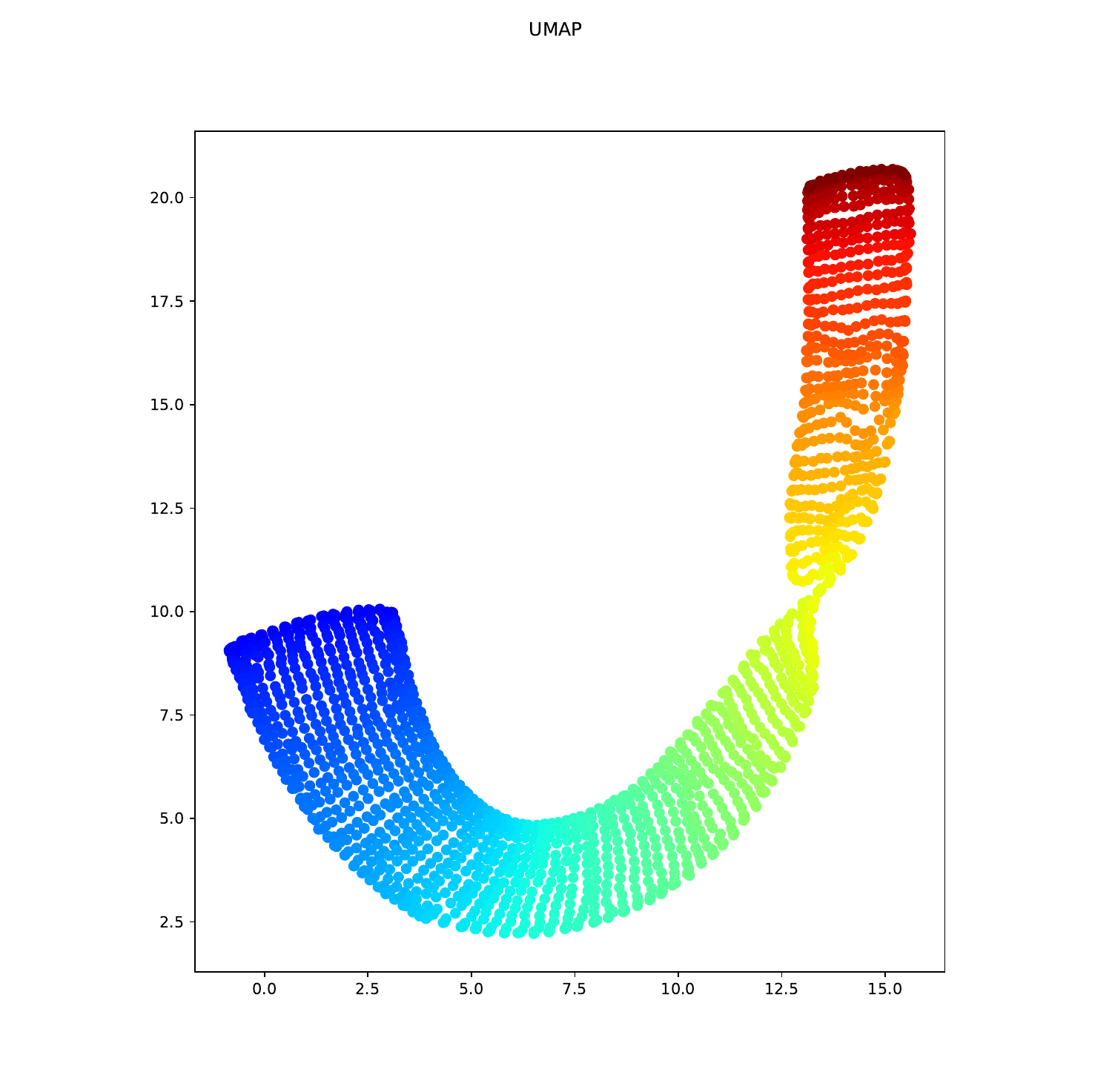}
        \end{tabular}
    
        &
    
        &
    
        \begin{tabular}[t]{@{}c@{}}
          \scriptsize Finsler MDS\\[0.25em]
          \adjincludegraphics[width=0.09\textwidth,
            trim={{0.15\width} {0.33\width} {0.04\width} {0.35\width}},clip,
            ]{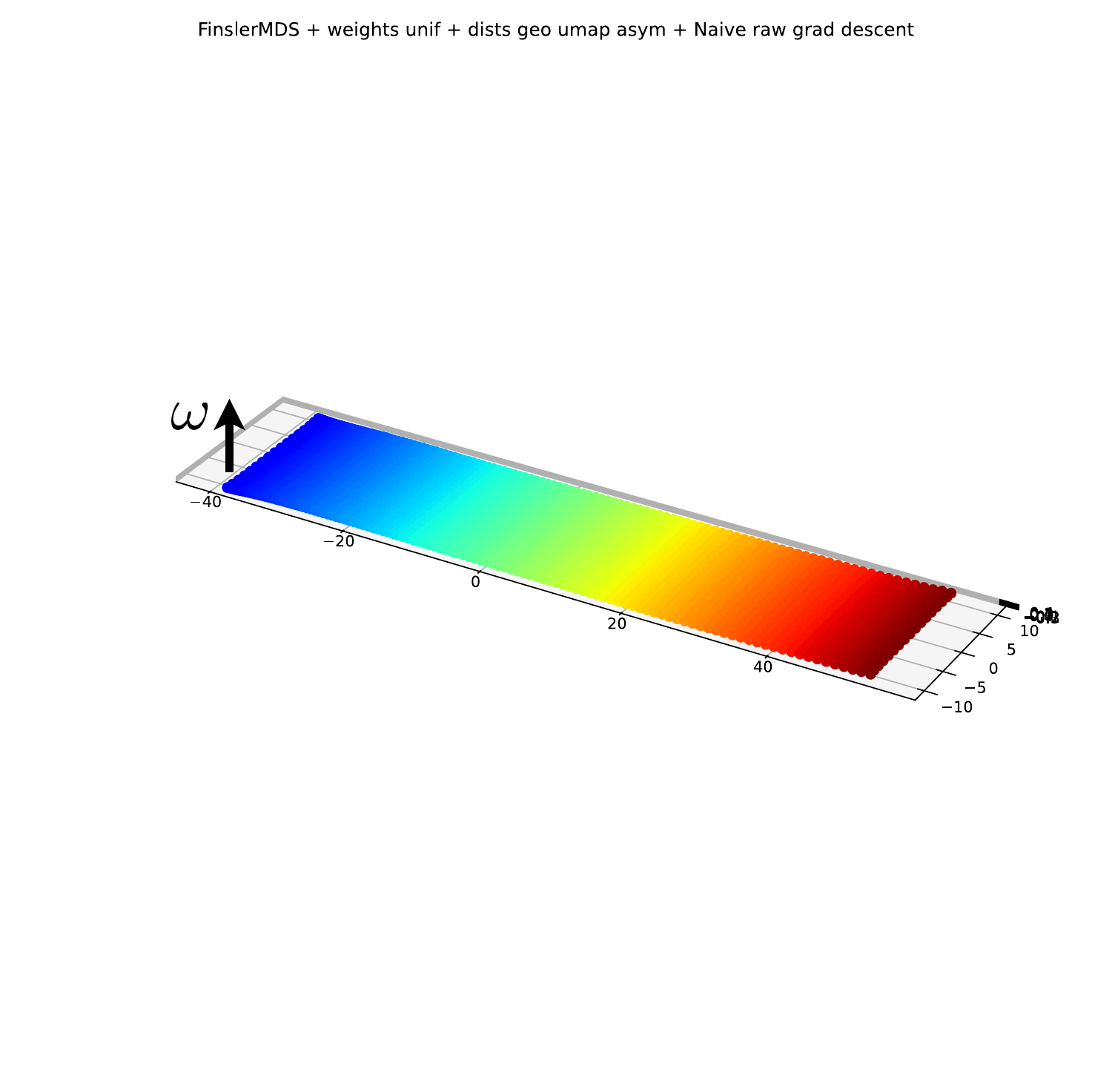} \\
          \adjincludegraphics[width=0.09\textwidth,
            trim={{0.2\width} {0.23\width} {0.15\width} {0.27\width}},clip,
            ]{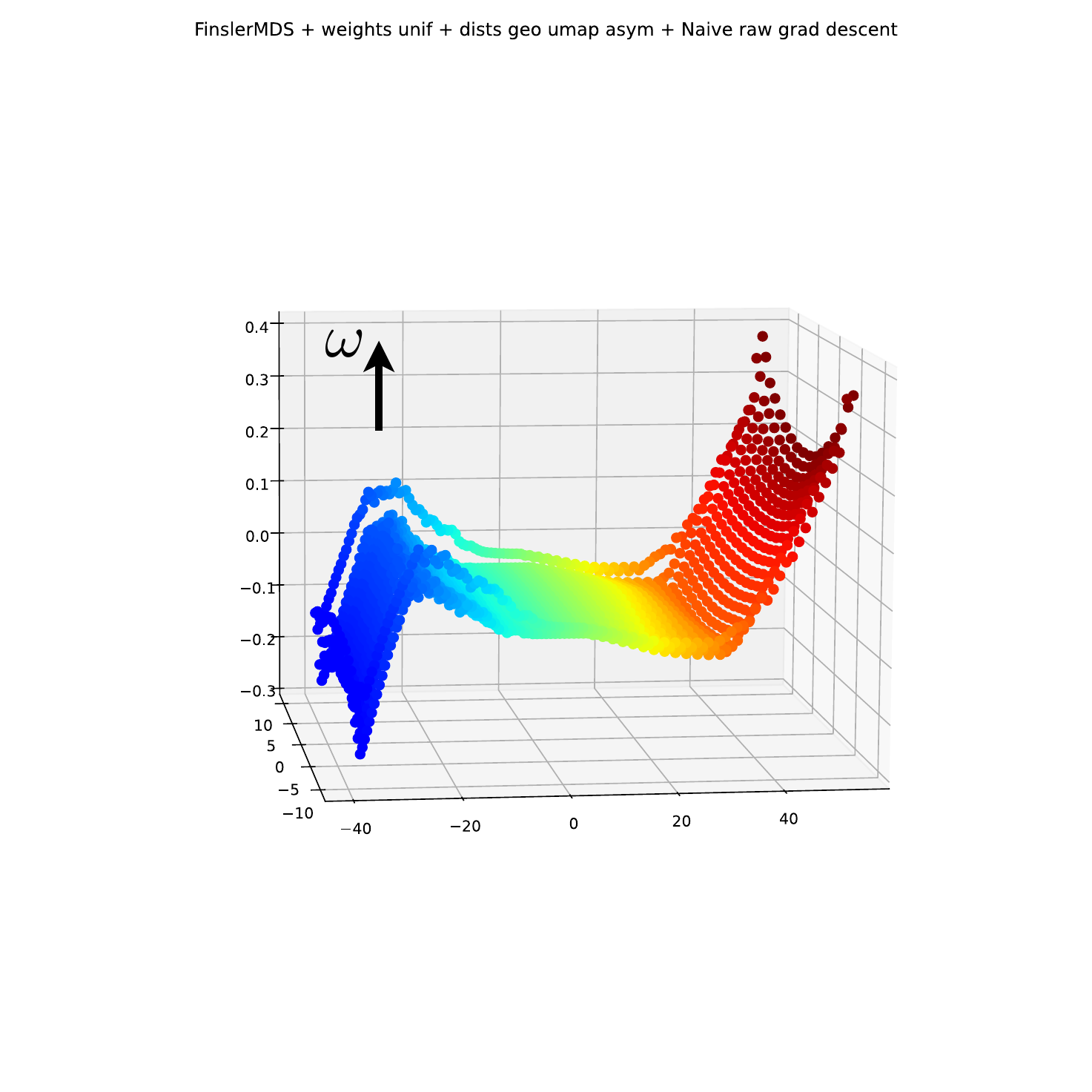}
        \end{tabular} &
    
        \begin{tabular}[t]{@{}c@{}}
          \scriptsize Finsler t-SNE\\[0.25em]
          \adjincludegraphics[width=0.07\textwidth,
            trim={{0.14\width} {0.25\width} {0.08\width} {0.28\width}},clip,
            ]{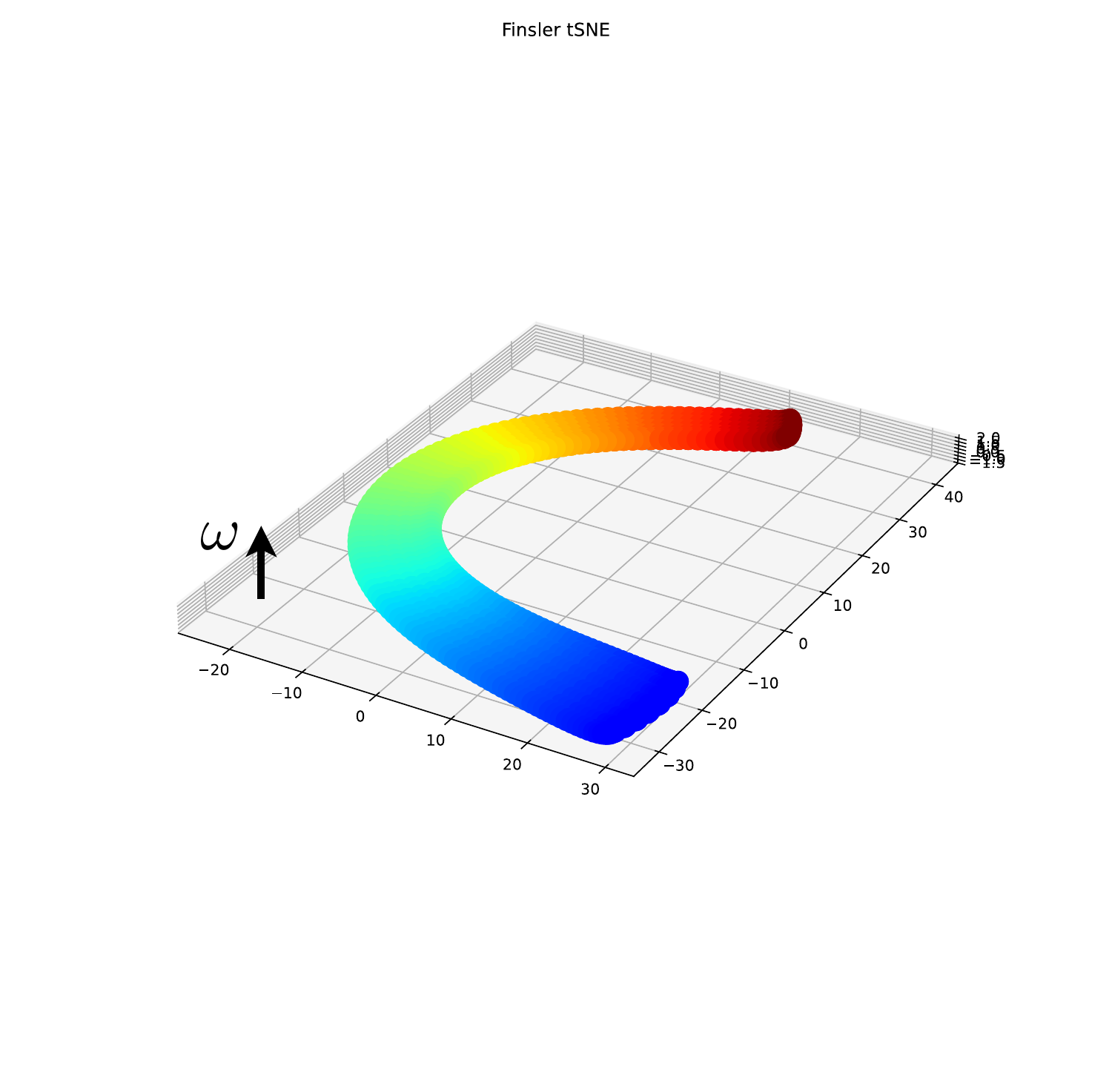} \\
          \adjincludegraphics[width=0.09\textwidth,
            trim={{0.17\width} {0.23\width} {0.08\width} {0.27\width}},clip,
            ]{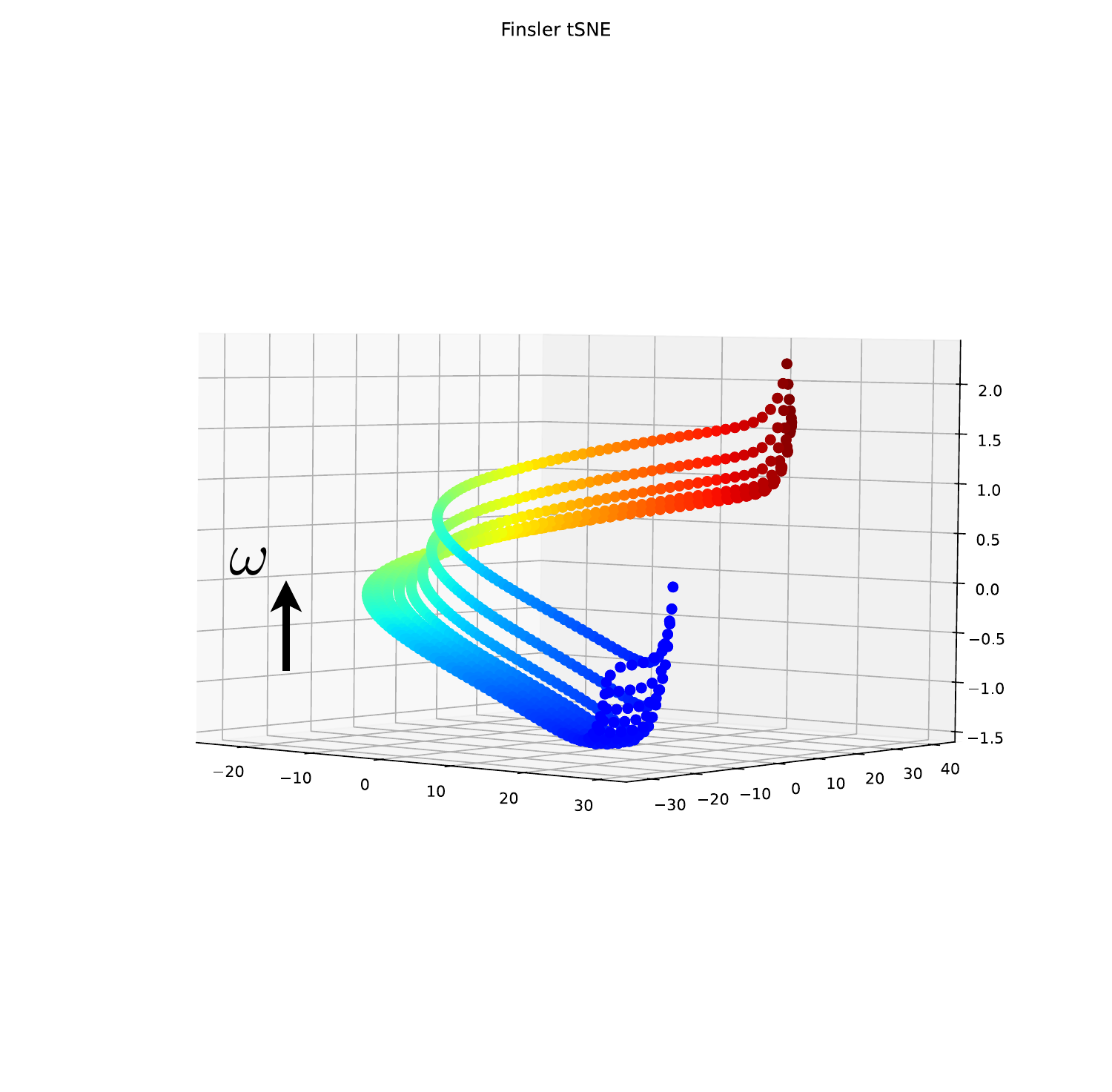}
        \end{tabular} &
    
        \begin{tabular}[t]{@{}c@{}}
          \scriptsize Finsler UMAP\\[0.25em]
          \adjincludegraphics[width=0.048\textwidth,
            trim={{0.35\width} {0.08\width} {0.3\width} {0.1\width}},clip,
            ]{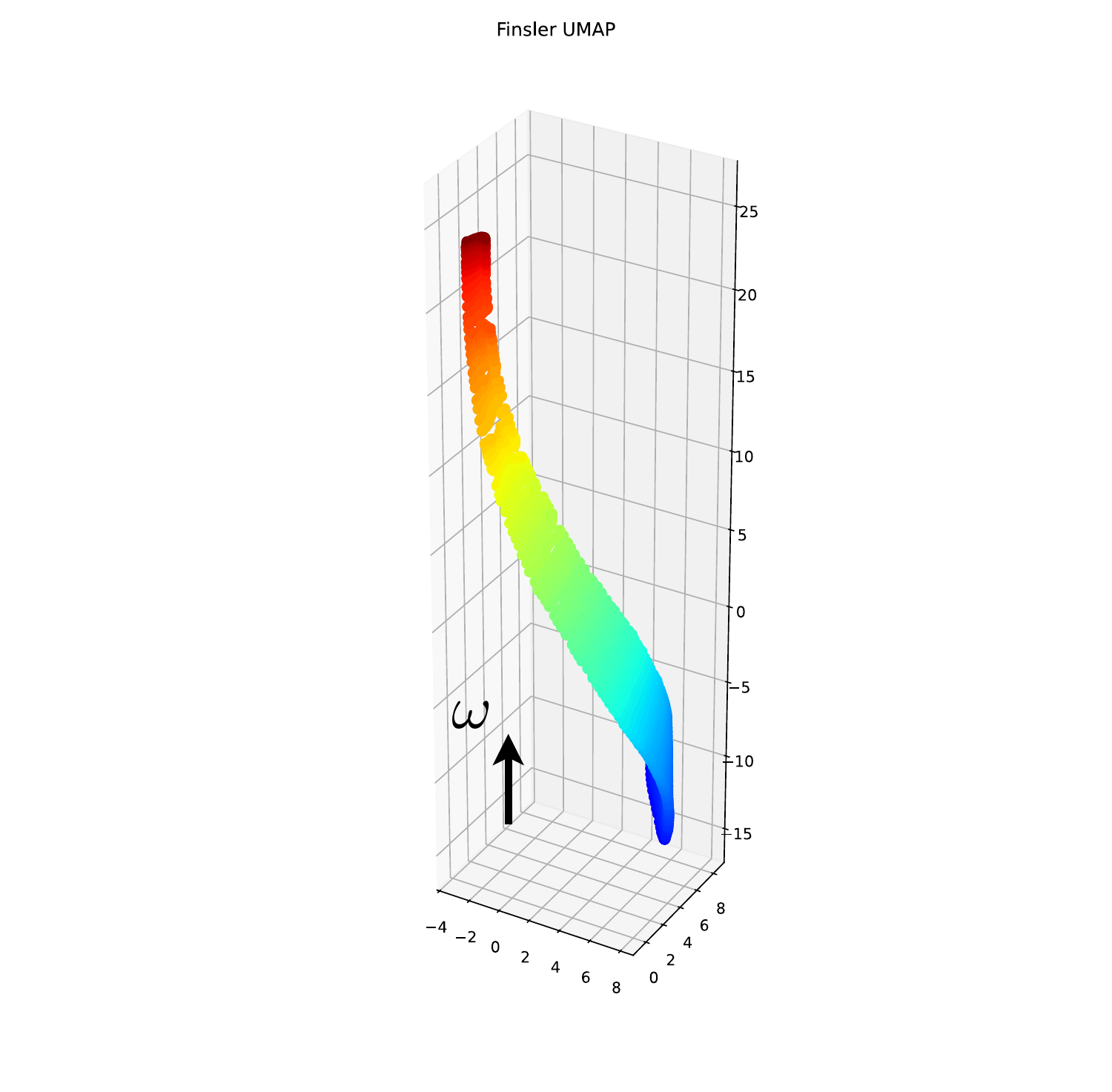}
        \end{tabular}
    
      \end{tabular}
  }%
  \caption{
  Swiss roll embeddings. Euclidean baselines aim only to preserve the manifold, with debatable success. Our asymmetric Finsler embeddings additionally reveal density asymmetry: dense points (blue) are embedded lower than sparse ones (red). For Finsler MDS and t-SNE we show
  the embedding (top) and a rescaled zoomed-in side view (bottom).
  }
  \label{fig: toy data swiss roll}
  \vspace{-1em}
\end{figure}

\begin{figure}[t]
    \centering
    \includegraphics[width=\textwidth]{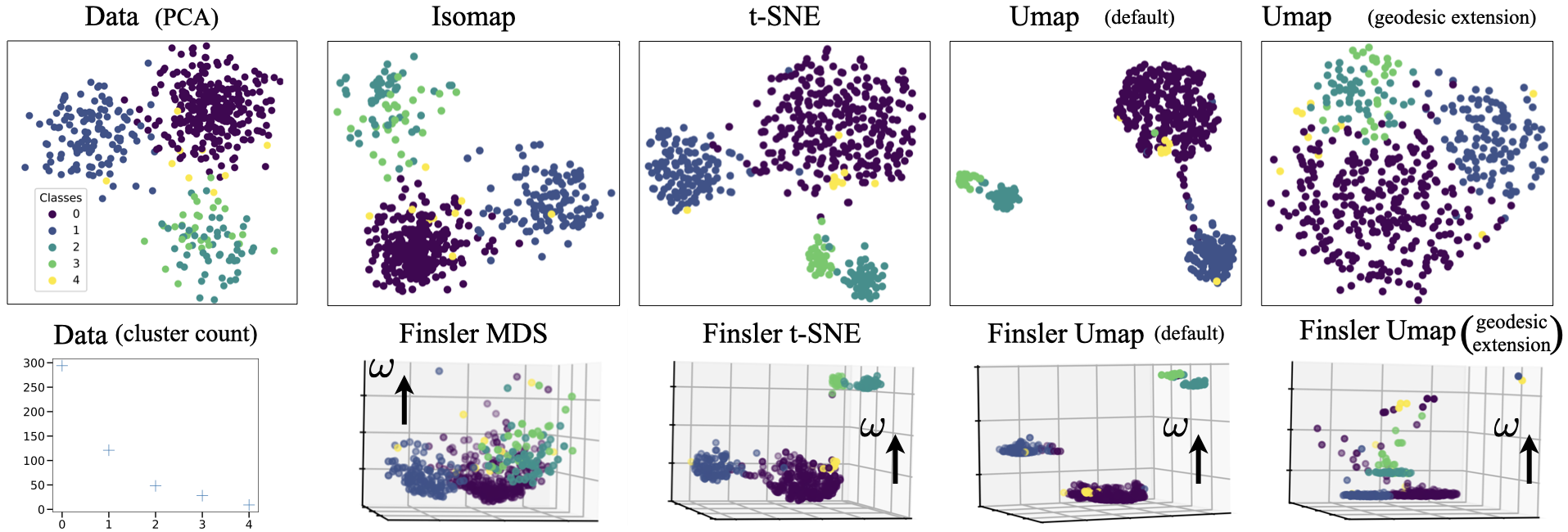}%
    \caption{
    Our Finsler approach clusters and reveals cluster hierarchy, as sparser clusters are embedded higher. 
    This is shown either locally (Finsler t-SNE, default Finsler UMAP) or globally  (Finsler MDS, extended Finsler UMAP), 
    and
    is absent in Euclidean methods.
    }
    \label{fig: persistence summary}
    \vspace{-1em}
\end{figure}

\textbf{Clustered manifold.}
In Fig.~\ref{fig: persistence summary}, we 
show
that our methods not only
accurately cluster high-dimensional data, they also provide a semantic hierarchy
between clusters
(see Sec.~\ref{sec: clustered manifold toy exp} for full details).

\subsection{Real datasets}

\textbf{US Cities.}
As a warm-up, we embed $N=2000$ mainland US cities (lat--lon) from US Zip codes \cite{simplemaps2026uszips} (Fig.~\ref{fig: US cities}).  
Although altitude is absent from input coordinates, it biases 
the distribution of cities
and is thus encoded in asymmetric relationships between cities.
Symmetrising discards this information. As expected, symmetric Isomap, t-SNE (Euclidean), Poincar\'e maps (Hyperbolic) \cite{klimovskaia2020poincare} miss it, whereas asymmetric models -- radius-distance, slide-vector, Finsler MDS, our Finsler t-SNE --, enabled here thanks to our 
\textit{data dissimilarities} construction (P1),
preserve and reveal it
in side views (Fig.~\ref{fig: US cities}), while providing similar top-down views to their symmetric references (see Fig.~\ref{fig: US cities top down}).

\textbf{Classification datasets.}
The standard
non-toy
manifold learning test evaluates label alignment in unsupervised embeddings of classification datasets. Methods favouring clustering (t-SNE, UMAP) over manifold preservation 
(MDS)
tend to produce label-aligned clusters. Quality scores are estimated by clustering the embedding, e.g.\ kMeans \cite{steinhaus1956division,lloyd1982least,mcqueen1967some,forgy1965cluster}, and measuring overlap with groundtruth labels. Label-agnostic scores also exist to see if clusters are well-behaved, but nice-looking cluster shapes are unrelated to the data manifold, making them unreliable as embedding quality measurements.

\noindent \textit{Data.}
We
perform extensive non-toy evaluation on
16 
reference
benchmarks: 1 tabular dataset (Iris) and 15 image datasets spanning
various
resolution, colour, size, and complexity (MNIST, Fashion-MNIST, Kuzushiji-MNIST, EMNIST, EMNIST-Balanced, CIFAR10, CIFAR100, DTD, Caltech101, Caltech256, OxfordFlowers102, Oxford-IIIT Pet, GTSRB, Imagenette, ImageNet). We do not subsample, e.g.\ we embed over one million images for ImageNet. As vanilla t-SNE \cite{van2008visualizing} is slow, we ran it and its Finsler variant only on the smaller datasets. See Sec.~\ref{sec: data reference classification datasets} for details.

\noindent \textit{Scores.}
We use the Adjusted Mutual Information (AMI) \cite{vinh2009information}, Adjusted Rank Index (ARI) \cite{hubert1985comparing}, Normalized Mutual Information (NMI) \cite{vinh2009information}, Homogeneity (HOM) \cite{rosenberg2007v}, Completeness (COM) \cite{rosenberg2007v}, V-Measure (V-M) \cite{rosenberg2007v}, and Fowlkes-Mallows Index (FMI) \cite{fowlkes1983method} to measure
the alignment of kMeans embedding clusters with groundtruth labels.
We also study label-unrelated 
clustering
scores in Sec.~\ref{sec: appendix scores},
namely Silhouette \cite{rousseeuw1987silhouettes}, Davies-Bouldin Index \cite{davies2009cluster}, and Calinski-Harabasz Index \cite{calinski1974dendrite},
and non-clustering scores like 
Trustworthiness 
\cite{venna2001neighborhood}, 
Continuity
\cite{venna2001neighborhood},
Local Continuity Meta-Criteria 
(LCMC)
\cite{chen2009local}, 
and Denoised Embedding Manifold Preservation 
(DeMAP)
\cite{moon2019visualizing}
in  Sec.~\ref{sec: raw values and label-unrelated scores}, 
along with
supervised classifier accuracy using nearest neighbour and linear classifiers.

\noindent \textit{Methods.}
To measure the gain from accounting for sampling asymmetry and Finsler space embeddings, we compare the reference methods favouring clustering, t-SNE and UMAP, with their Finsler counterparts. Default deterministic initialisation makes (Finsler) t-SNE deterministic
-- one run needed,
unlike (Finsler) UMAP as it relies on random negative sampling
-- we report the mean over 10 runs. 
Euclidean (resp.\ Finsler) methods embed in $\mathbb{R}^m$ (resp.\ $\mathbb{R}^{m+1}$), with $m=2$.
We also report Euclidean (resp.\ Finsler) results in $\mathbb{R}^3$ (resp.\ $\mathbb{R}^2$) in Sec.~\ref{sec: full visualisations classification datasets}, and include ablations on the embedding dimensionality $m$ (Sec.~\ref{sec: embedding dim classification datasets}) and on the 
asymmetry emphasis level
$\lVert\omega\rVert_2$ (Sec.~\ref{sec: levels of emphasis on asymmetry}).
In Sec.~\ref{sec: raw values and label-unrelated scores}, we also evaluate Finsler PHATE and compare it with the reference PHATE approach.

\noindent\textit{Results.}
Across datasets, our asymmetric pipeline with Finsler embeddings consistently surpasses Euclidean baselines on label-related 
unsupervised 
metrics (Fig.~\ref{fig: classif kmeans on umap finsler umap + tsne finsler tsne label-related only})
implying higher quality embeddings. Finsler UMAP (resp.\ t-SNE) outperforms UMAP (resp.\ t-SNE) on 
(almost)
every dataset and label-related score.
Finsler embeddings thus preserve the data more faithfully, though clusters may look less ``nice'' (see label-unrelated scores in Sec.~\ref{sec: raw values and label-unrelated scores}). 
Likewise, for supervised scores (Tab.~\ref{tab: supervised evaluation}) our Finsler methods either outperform by large magnitudes or are on par.
We provide qualitative results on large datasets in Sec.~\ref{sec: full visualisations classification datasets}
showcasing better cluster-label alignment with our Finsler pipeline. 
For example in Caltech256, UMAP sometimes formed clusters decorrelated with labels, but Finsler UMAP did not. Full 
similar
results
for all 
16
datasets appear in
Sec.~\ref{sec: reference classification datasets}
due to length constraints.
Also our Finsler methods are robust to the choice of asymmetry emphasis $\lVert \omega\rVert_2$ (Sec.~\ref{sec: levels of emphasis on asymmetry}), with strong performance 
on
a range of values.

\begin{table}[t]
    \centering
    \caption{
    Supervised accuracy on image benchmarks
    for two classifiers. 
    We report mean test accuracy over 2-stratified 5-fold CV (10 runs).
    Best/ties are bold/underlined. Finsler variants consistently outperform or are on par of Euclidean ones.
    }
    \label{tab: supervised evaluation}
    \vspace{0.5em}
    \setlength{\tabcolsep}{3.5pt}
    \Huge
    \resizebox{0.8\columnwidth}{!}{%
    \begin{tabular}{llccccc}
        \toprule
        \textbf{Supervised Evaluator} & \textbf{Embedding} & \textbf{DTD} & \textbf{Caltech101} & \textbf{OxfordFlowers102} & \textbf{OxfordIIITPet} & \textbf{Imagenette} \\
        \midrule
        \multirow{4}{*}{\textbf{5-NN classifier}} 
        & UMAP           & 0.51 & 0.87 & 0.55 & 0.92 & \underline{0.98} \\
        & Finsler UMAP   & \textbf{0.56} & \textbf{0.88} & \textbf{0.62} & \textbf{0.93} & \underline{0.98} \\
        \cmidrule(lr){2-7}
        & t-SNE          & 0.57 & \underline{0.88} & \textbf{0.63} & \underline{0.92} & \underline{0.98} \\
        & Finsler t-SNE  & \textbf{0.58} & \underline{0.88} & 0.62 & \underline{0.92} & \underline{0.98} \\
        \midrule
        \multirow{4}{*}{\textbf{Linear classifier}} 
        & UMAP           & 0.34 & 0.53 & 0.37 & 0.80 & 0.97 \\
        & Finsler UMAP   & \textbf{0.43} & \textbf{0.66} & \textbf{0.48} & \textbf{0.87} & \textbf{0.98} \\
        \cmidrule(lr){2-7}
        & t-SNE          & 0.33 & 0.55 & 0.29 & 0.90 & \underline{0.98} \\
        & Finsler t-SNE  & \textbf{0.42} & \textbf{0.76} & \textbf{0.38} & \textbf{0.91} & \underline{0.98} \\
        \bottomrule
    \end{tabular}
    }%
\end{table}

\begin{figure}[t]
    \centering
    \begin{subfigure}{0.35\textwidth}
        \adjincludegraphics[
            width=\textwidth, 
            trim={{0.\width} {0.\height} {0.5\width} {0.\height}},
            clip,
        ]{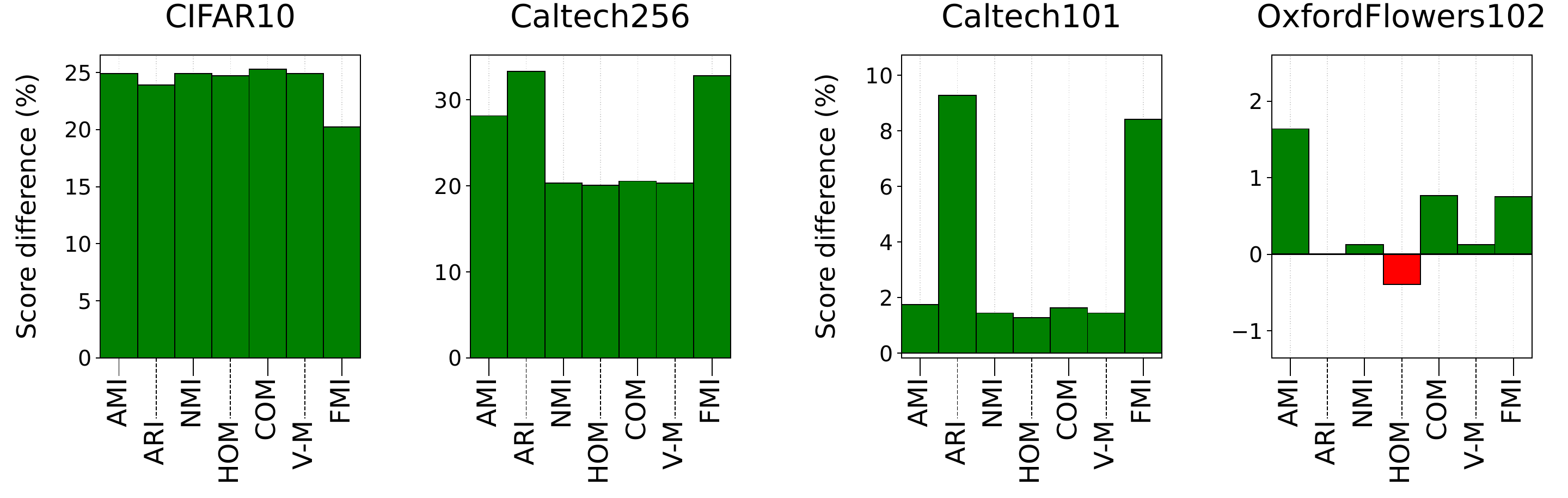}%
        \caption*{Finsler UMAP vs.\ UMAP}
    \end{subfigure}%
    \hspace{2em}
    \begin{subfigure}{0.35\textwidth}
        \adjincludegraphics[
            width=\textwidth, 
            trim={{0.5\width} {0.\height} {0.\width} {0.\height}},
            clip,
        ]   {Pictures/classif_datasets/quantitative_main_pct_diff_fixed_omega_0.001__subset.pdf}%
        \caption*{Finsler t-SNE vs.\ t-SNE}
    \end{subfigure}
    \caption{Percentage difference in mean performance between our Finsler methods and 
    their 
    Euclidean baselines. Positive (resp.\ negative) differences, in green (resp.\ red), means we get better (resp.\ worse) scores 
    for
    kMeans 
    vs.\
    class labels. 
    See Sec.~\ref{sec: raw values and label-unrelated scores} for 
    extensive
    results on all 16 tested datasets, 
    with
    visual summaries (Fig.~\ref{fig: classif kmeans on umap finsler umap + tsne finsler tsne label-related only zoom in}), raw performance values, and standard deviations (Tabs.~\ref{tab: classif kmeans on umap finsler umap} and \ref{tab: classif kmeans on tsne finsler tsne}).
    }
    \label{fig: classif kmeans on umap finsler umap + tsne finsler tsne label-related only fixed omega}
    \vspace{-1em}
\end{figure}

\section{Conclusion}

We propose a novel general manifold learning pipeline, 
which captures, harnesses, and reveals the asymmetry arising from discrete samples, e.g.\ 
from
non-uniform densities.
It is further motivated when observing that the traditional symmetric pipeline constructs asymmetric dissimilarities, requiring symmetrisation,
discarding valuable information in the process.
We deliberately construct asymmetric dissimilarities and embed in the canonical Finsler, rather than Euclidean, space, which supports asymmetry.
Our approach greatly broadens the applicability of the rare existing asymmetric embedders, that were all limited to traditionally asymmetric data, to any data. 
We also generalise the modern reference symmetric methods to handle 
asymmetric data, especially our 
Finsler t-SNE, Finsler UMAP, and Finsler PHATE generalisations of t-SNE, UMAP, and PHATE,
providing asymmetric embedders 
for 
large scale data. 
Through extensive evaluations, from controlled well-understood small-scale settings to uncontrolled large-scale ones, we demonstrate that our asymmetric Finsler pipeline provides novel insights through visualisations and superior quality embeddings to their Euclidean counterparts.

\noindent \textbf{Limitations and future work.
}
We fix a specific Finsler metric for the embedding space, yielding increasing distortions the more the data's structure deviates from 
it,
as in the 
Euclidean case.
We also provide
a single asymmetry direction, yet 
applications might prefer disentangling more asymmetry directions, thus requiring other Finsler metrics.
By default,
our asymmetry is inferred from 
sampling densities. Yet, if points had features, we could derive from them other non-uniform local concepts, implying differently semantic asymmetries -- an exciting avenue we plan to explore.
While density disparities can reveal meaningful biases in data sampling, it can be sensitive to noise or resampling. Downstream applications may also not need the directional information from non-uniform sampling densities.
Finally, Finsler geometry is not limited to manifold learning, yet remains largely uncharted in computer vision and manifold learning.

\clearpage

\section*{Acknowledgments}
TD was supported by the Munich Center for Machine Learning (MCML) via the Thomas Bayes Postdoctoral Fellowship.
RK was partially supported by the Technion Strategic Program Fund on Learning Geometry.

\bibliographystyle{unsrt}
\bibliography{main}

\clearpage
\appendix

\begin{sidewaystable}
    \centering
    \captionof{table}{
    Summary and comparison of manifold learning pipelines, using either the traditional Euclidean perspective or our novel Finsler one. As standard approaches naturally lead to data asymmetry, thus requiring artificially symmetrising the data to remain within symmetric Riemannian geometry, our pipeline embraces this asymmetry and preserves it by embedding into a Finsler space rather than a Riemannian one. With this new perspective, we generalise the reference modern manifold learning methods to handle asymmetric data.
    }
    \label{tab: summary comparison euclidean and finsler pipelines}
    \resizebox{\textwidth}{!}{%
    \begin{tabular}{c l lll}
        \toprule[2pt]
        \textbf{Geometry} & \textbf{Pipeline} & \multicolumn{3}{c}{\textbf{Method}}
        \\
        \cmidrule(lr){3-5}
        \textbf{\shortstack{Metric\\ (Canonical space)}}
        & 
        & \textbf{Multi-dimensional scaling}
        & \textbf{Student t-distributed stochastic neighbour embedding}
        & \textbf{Uniform manifold approximation}
        \\
        \midrule[2pt]
        \multirow{19}{*}{\textbf{\shortstack{Riemann\\[1em] \big(Euclidean\big)}}}
        &
        & \textbf{Classical MDS \cite{schwartz1989numerical}  
        }
        & \textbf{t-SNE \cite{van2008visualizing,van2009learning}}
        & \textbf{UMAP \cite{mcinnes2018umap}}
        \\
        \cmidrule(lr){2-5}
        & \textbf{Data dissimilarities}
        & $p_{ij} = \lVert x_j - x_i\rVert_2$ \textit{(to be extended geodesically)}
        & $p_{ij} = \frac{e^{- \frac{\lVert x_j - x_i\rVert_2^2}{2\sigma_i^2}}}{\sum\limits_k e^{- \frac{\lVert x_k - x_i\rVert_2^2}{2\sigma_i^2}}}$
        & $p_{ij} = e^{-\frac{\lVert x_j - x_i\rVert_2 - \rho_i}{\sigma_i}}$ 
        \\
        & \textbf{Symmetrisation}
        & $p_{ij} \leftarrow \frac{p_{ij} + p_{ji}}{2} \text{ OR } \max(p_{ij}, p_{ji})$ 
        & $p_{ij} \leftarrow \frac{p_{ij} + p_{ji}}{2N}$
        & $p_{ij} \leftarrow p_{ij} + p_{ji} - p_{ij}p_{ji}$ 
        \\
        \cmidrule(lr){2-5}
        & \textbf{Embedding dissimilarities}
        & $q_{ij} = \lVert y_j-y_i\rVert_2$
        & $q_{ij} = \frac{(1 + \frac{\lVert y_j - y_i\rVert_2^2)}{\nu})^{-\frac{\nu+1}{2}}}{\sum\limits_{k\neq l} (1 + \frac{\lVert y_l - y_k\rVert_2^2)}{\nu})^{-\frac{\nu+1}{2}}}$
        & $q_{ij} = \frac{1}{1+a\lVert y_j - y_i\rVert_2^{2b}}$ 
        \\
        \cmidrule(lr){2-5}
        & \textbf{Optimisation objective}
        & $\mathcal{L} = \sum w_{ij} (q_{ij} - p_{ij})^2$
        & $\mathcal{L} = -\sum p_{ij}\ln q_{ij}$
        & $\mathcal{L} = -\sum [p_{ij}\ln q_{ij} + (1-p_{ij})\ln(1-q_{ij})]$ \\
        \cmidrule(lr){2-5}
        & \multirow{9}{*}{\textbf{Update rule}}
        & SMACOF
        & Gradient descent
        & Gradient descent with negative sampling
        \\
        & 
        & $\boldsymbol{Y}^{(k+1)} = V^\dagger B\big(\boldsymbol{Y}^{(k)}\big)\boldsymbol{Y}^{(k)}$
        & $\frac{\partial \mathcal{L}}{\partial y_i} = 
        2\frac{\nu+1}{\nu} \sum\limits_j (p_{ij} - q_{ij}) \left(1+\tfrac{\lVert y_j - y_i\rVert_2^2}{\nu}\right)^{-1}(y_i - y_j)$
        & $\frac{\partial \mathcal{L}}{\partial y_i} = 2 \sum\limits_{j} 
        \left[ p_{ij} \cdot \frac{\partial c_{ij}^a}{\partial y_i}+ (1 - p_{ij}) \cdot \frac{\partial c_{ij}^r}{\partial y_i}\right]$ 
        \\
        &
        & $
            \begin{cases}
            V_{ij} =
                \begin{cases}
                    -w_{ij} & \text{ if } i \neq j, \\
                    \sum_{k\neq i} w_{ik} & \text{ if } i = j.
                \end{cases}
                \\
            B_{ij}(\boldsymbol{Y}) =     
            \begin{cases}
                 - w_{ij}\frac{\boldsymbol{D}_{i,j}}{\lVert y_j - y_i\rVert_2}& \text{if } i \neq j
                 \\
                 -\sum_{k\neq i}B_{ik}(\boldsymbol{Y})& \text{if } i = j.
            \end{cases}
            \end{cases}
          $
        &
        & $
            \begin{dcases}
                \frac{\partial c_{ij}^a}{\partial y_i} = \frac{2ab \lVert y_j-y_i\rVert_2^{2(b-1)}}{1+a\lVert y_j - y_i\rVert_2^{2b}} (y_i - y_j)  
                \\
                \frac{\partial c_{ij}^r}{\partial y_i} = \frac{-2b }{\lVert y_j-y_i\rVert_2^2 \left(1+a\lVert y_j - y_i\rVert_2^{2b}\right)} (y_i - y_j) \\
                \frac{\partial c_{ij}^a}{\partial y_j} = - \frac{\partial c_{ij}^a}{\partial y_i} 
                \quad\text{and}\quad
                \frac{\partial c_{ij}^r}{\partial y_j} = - \frac{\partial c_{ij}^r}{\partial y_i}
            \end{dcases}
        $
        \\
        \midrule[2pt]
        \multirow{24}{*}{\textbf{\shortstack{Finsler\\[1em] (Canonical\\ Randers)}}}
        & 
        & \textbf{Finsler MDS (Ours, corrected, extends \cite{dages2025finsler})}
        & \textbf{Finsler t-SNE (Ours)}
        & \textbf{Finsler UMAP (Ours)}
        \\
        \cmidrule(lr){2-5}
        & \textbf{Data dissimilarities}
        & $p_{ij} = \lVert x_j - x_i\rVert_2$ \textit{(to be extended geodesically)}
        & $p_{ij} = \frac{e^{- \frac{\lVert x_j - x_i\rVert_2^2}{2\sigma_i^2}}}{\sum\limits_k e^{- \frac{\lVert x_k - x_i\rVert_2^2}{2\sigma_i^2}}}$
        & $p_{ij} = e^{-\frac{\lVert x_j - x_i\rVert_2 - \rho_i}{\sigma_i}}$ 
        \\
        & 
        & 
        & $p_{ij} \leftarrow \frac{p_{ij}}{N}$
        & 
        \\
        \cmidrule(lr){2-5}
        & \textbf{Embedding dissimilarities}
        & $q_{ij}^F = d_{F^C}(y_i, y_j) = \lVert y_j - y_i\rVert_2 + \omega^\top (y_j - y_i)$
        & $q_{ij}^F = \frac{(1 + \frac{d_{F^C}(y_i, y_j)^2)}{\nu})^{-\frac{\nu+1}{2}}}{\sum\limits_{k\neq l} (1 + \frac{d_{F^C}(y_l, y_k)^2)}{\nu})^{-\frac{\nu+1}{2}}}$
        & $q_{ij}^F = \frac{1}{1+a d_{F^C}(y_i, y_j)^{2b}}$
        \\ 
        \cmidrule(lr){2-5}
        & \textbf{Optimisation objective}
        & $\mathcal{L} = \sum w_{ij} (q_{ij}^F - p_{ij})^2$
        & $\mathcal{L} = -\sum p_{ij}\ln q_{ij}^F$
        & $\mathcal{L} = -\sum [p_{ij}\ln q_{ij}^F + (1-p_{ij})\ln(1-q_{ij}^F)]$
        \\
        \cmidrule(lr){2-5}
        & \multirow{13}{*}{\textbf{Update rule}}
        & Finsler SMACOF
        & Gradient descent
        & Gradient descent with negative sampling
        \\
        &
        & $ \boldsymbol{Y}^{(k+1)} = 
        {V}^\dagger
        \left[
            {B}(\boldsymbol{Y}^{(k)}){\boldsymbol{Y}^{(k)}}
            -
            \frac{1}{2}{C}
        \right]
        \left(
            {I}_m
            -
            \frac{\omega\omega^\top}{1+\lVert \omega\rVert_2^2}
        \right)$
        & 
        $\frac{\partial \mathcal{L}}{\partial y_i} = \frac{\nu+1}{\nu} \sum\limits_j \Bigg[ (p_{ij} - q_{ij}^F) \bigg(1+\frac{d_{F^C}(y_i, y_j)^2}{\nu}\bigg)^{-1} \frac{d_{F^C}(y_i, y_j)}{\lVert y_j - y_i\rVert_2}$ 
        & 
        $
        \frac{\partial \mathcal{L}}{\partial y_i} = \sum\limits_{j} \left[ p_{ij} \cdot \frac{\partial c_{ij}^a}{\partial y_i} 
        + p_{ji}\frac{\partial c_{ji}^a}{\partial y_i} 
        + (1 - p_{ij}) \cdot \frac{\partial c_{ij}^r}{\partial y_i}
        + (1 - p_{ji}) \cdot \frac{\partial c_{ji}^r}{\partial y_i}\right]
        $
        \\
        &
        & 
        \multirow{3}{*}{
        $
            \begin{cases}
            V_{ij} =
                \begin{cases}
                    -w_{ij} & \text{ if } i \neq j, \\
                    \sum_{k\neq i} w_{ik} & \text{ if } i = j.
                \end{cases}
            \\
            \boldsymbol{B}_{ij}(\boldsymbol{Y})
            =
            \begin{cases}
                -\frac{
                    w_{ij}
                   (\boldsymbol{D}_{ij} + \boldsymbol{D}_{ji})
                }{
                    2d_{ij}(\boldsymbol{Y})
                },
                & \text{if }\, i\ne j\\
                -\sum_{k\ne i}{B}_{ik}(\boldsymbol{Y}),
                & \text{if }\, i=j.
            \end{cases}
            \\
            {C} = \bigl[{W}\odot(\boldsymbol{D}-\boldsymbol{D}^\top)\bigr] \boldsymbol{1}_N \omega^\top
            \end{cases}
        $
        }
        & 
        $\hspace{6em}+ (p_{ji} - q_{ji}^F)\bigg(1+\frac{d_{F^C}(y_j, y_i)^2}{\nu}\bigg)^{-1} \frac{d_{F^C}(y_j, y_i)}{\lVert y_i - y_j\rVert_2} \Bigg](y_i - y_j)$ 
        &
        \multirow{4}{*}{
        $
            \begin{dcases}
                \frac{\partial c_{ij}^a}{\partial y_i} = \frac{2ab \frac{(d_{F^C}(y_i, y_j))^{2b-1}}{\lVert y_j-y_i\rVert_2}}{1+a (d_{F^C}(y_i, y_j))^{2b}} (y_i - y_j) 
                - \frac{2ab (d_{F^C}(y_i, y_j))^{2b-1}}{1+a (d_{F^C}(y_i, y_j))^{2b}}\omega \\
                \frac{\partial c_{ij}^r}{\partial y_i} = \frac{-2b}{(1+a(d_{F^C}(y_i, y_j))^{2b})\lVert y_j-y_i\rVert_2 d_{F^C}(y_i, y_j)}(y_i - y_j) \\
                \hspace{4em}+ \frac{2b}{(1+a(d_{F^C}(y_i, y_j))^{2b})d_{F^C}(y_i. y_j)}\omega \\
                \frac{\partial c_{ij}^a}{\partial y_j} = - \frac{\partial c_{ij}^a}{\partial y_i} 
                \quad\text{and}\quad
                \frac{\partial c_{ij}^r}{\partial y_j} = - \frac{\partial c_{ij}^r}{\partial y_i}
            \end{dcases}
        $
        }
        \\
        &
        &
        & 
        $\hspace{2em}+\frac{\nu+1}{\nu} \sum\limits_j \Bigg[ -(p_{ij} - q_{ij}^F) \bigg(1+\frac{d_{F^C}(y_i, y_j)^2}{\nu}\bigg)^{-1} d_{F^C}(y_i, y_j)$ 
        \\
        &
        &
        & 
        $\hspace{6em}+ (p_{ji} - q_{ji}^F)\bigg(1+\frac{d_{F^C}(y_j, y_i)^2}{\nu}\bigg)^{-1} d_{F^C}(y_j, y_i) \Bigg]\omega$ 
        \\
        &
        &
        &
        & 
        \\
        &
        &
        &
        &
        \\
        &
        &
        &
        &
        \\
        \bottomrule[2pt]
    \end{tabular}%
    }
\end{sidewaystable}

\begin{center}
    \LARGE  \textbf{Supplementary Material}    
\end{center}

\section*{Appendix overview}

\noindent
This 
appendix
is organised as follows.
\begin{itemize}
    \setlength\itemsep{1em}
    
    \item 
    \textbf{Sec.~\ref{sec: sup mat More details on existing manifold learning methods}} further presents existing methods. In particular, it provides full details on the t-SNE and UMAP methods and corrects a bug in the update rule of t-SNE that is present not only in the original paper \cite{van2009learning} but also in common reference libraries like Scikit-learn \cite{pedregosa2011scikit}.

    \item \textbf{Sec.~\ref{sec: discussions metric on the data}} further discusses the metric used on the input sampled data, such as clarifications regarding the information lost by symmetrisation, the link between the used Riemannian metric and the resulting Finsler metric, or asymptotic limits. Of note, it also revisits symmetrisation towards a possible justified direction for fully symmetric manifold learning.

    \item \textbf{Sec.~\ref{sec: corrected update rule of finsler mds}} provides corrected update rules for Finsler MDS, fixing the various issues with derivations in the original work \cite{dages2025finsler}. Notably, it reveals that the original expensive unstable implementation of \cite{dages2025finsler} can be made much simpler, faster, more stable, and handle more points.

    \item 
    \textbf{Sec.~\ref{sec: proofs finsler tsne umap}} comprises in the proof of Theorems \ref{th: finsler tsne update rule} and \ref{th: finsler umap update rule}, which are the full derivations of the update rules for our Finsler t-SNE and Finsler UMAP.

    \item 
    \textbf{Sec.~\ref{sec: Asymmetric Finsler generalisation of spectral methods in manifold learning}} presents how to generalise traditional spectral methods to asymmetry using Finsler geometry.

    \item 
    \textbf{Sec.~\ref{sec: Finsler PHATE}} explains Finsler PHATE, which provides an example on how to generalise to asymmetric Finsler embeddings existing symmetric manifold learning techniques that do not seem to use at first glance symmetrisation fixes in the computation of their data dissimilarities.

    \item
    \textbf{Sec.~\ref{sec: implementation details}} includes extensive implementation details for all experiments, including those in the supplementary material, allowing full independent reproducibility based on the paper alone.

    \item
    \textbf{Sec.~\ref{sec: further experiments}} contains further experiments, such as ablation studies, full visualisations, or raw performance tables.

\end{itemize}

For conciseness, we denote in row-stacked format $\boldsymbol{X}\in\mathbb{R}^{N\times n}$ as the data points $x_1,\ldots,x_N$ and $\boldsymbol{Y}\in\mathbb{R}^{N\times m}$ as the embedded points $y_1,\ldots,y_N$.

As a reminder, we provide a visual overview of our pipeline in Fig.~\ref{fig: pipelines overview} and a detailed summary comparing our Finsler methods to their Euclidean counterparts in Tab.~\ref{tab: summary comparison euclidean and finsler pipelines}.

\begin{figure}[h]%
    \centering
    \includegraphics[width=\textwidth]{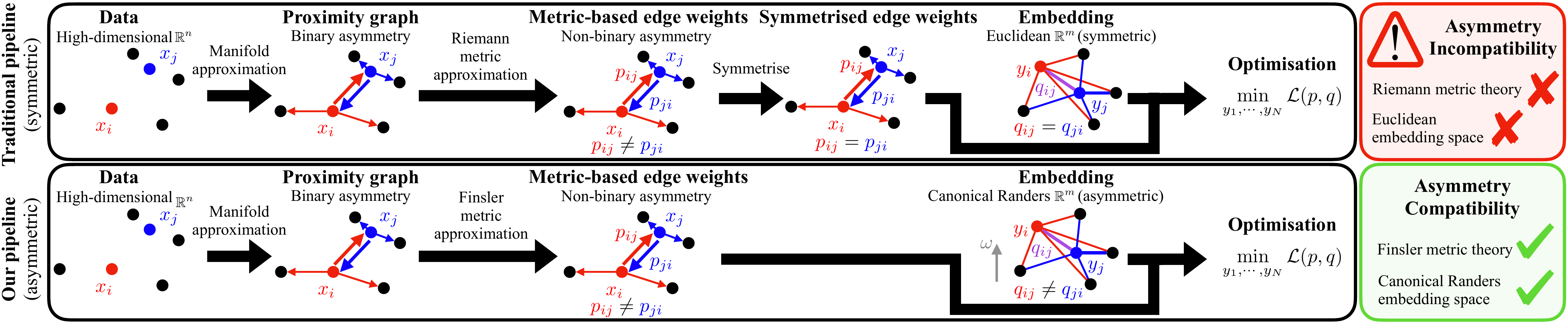}
    \captionof{figure}{
    The traditional manifold learning pipeline 
    leads to asymmetric data dissimilarities on sampled data due to directed proximity graphs and local distance transforms. As this violates the Riemannian manifold assumption and is incompatible with symmetric Euclidean space embeddings, heuristic symmetrisation of data dissimilarities is required yet theoretically unjustified.
    We propose to equip the data manifold with a Finsler metric allowing asymmetric dissimilarities. Embeddings are then performed in a canonical Finsler space, enabling us not only to accurately capture the structure of the data but also harness and reveal the natural asymmetry of the sampling.
    }
    \label{fig: pipelines overview}
    \vspace{-1em}
\end{figure}

\clearpage

\section{More details on existing manifold learning methods}
\label{sec: sup mat More details on existing manifold learning methods}

\subsection{Update rules of classical MDS}

The minimisation objective in classical MDS is the weighted stress function
\begin{equation}
    \mathcal{L} = \sum w_{ij}(q_{ij} - p_{ij})^2 = \sum w_{ij}(\lVert y_j -y_i\rVert_{ij} - \boldsymbol{D}_{ij})^2.
\end{equation}

\noindent\textbf{SMACOF.} There is no closed-form solution to the stress function of Vanilla MDS \cite{schwartz1989numerical}, even for uniform weights $w_{ij}=1$.
For minimisation, the SMACOF algorithm \cite{leeuw1977application,borg2007modern} is widely used. This essentially iteratively minimises a simple majorising function of the stress, where finding the minimisation of the majorising function is easy. In practice, starting from some arbitrary initialisation, the SMACOF algorithm iteratively decreases the stress until reaching a local minimum using the following update rule
\begin{equation}
    \boldsymbol{Y}^{(k+1)} = V^\dagger B\big(\boldsymbol{Y}^{(k)}\big)\boldsymbol{Y}^{(k)},
\end{equation}
where
\begin{align}
    V_{ij} &=
        \begin{cases}
            -w_{ij} & \text{ if } i \neq j, \\
            \sum_{k\neq i} w_{ik} & \text{ if } i = j.
        \end{cases}
        \\
    \label{eq:B}
    B_{ij}(\boldsymbol{Y}) &=     
    \begin{cases}
         - w_{ij}\frac{\boldsymbol{D}_{ij}}{\lVert y_j - y_i\rVert}& \text{if } i \neq j
         \\
         -\sum_{k\neq i}B_{ik}(\boldsymbol{Y})& \text{if } i = j.
    \end{cases}
\end{align}
and $V^\dagger$ is the pseudo-inverse of $V$. It can be proven that this strategy is equivalent to scaled gradient descent with constant step size \cite{bronstein2006multigrid}.

It is well-known that the SMACOF algorithm is highly sensitive to initialisation. Thus, in the field, it is considered good practice to initialise with the Isomap embedding (corresponding to uniform weights $w_{ij} = 1$ and a relaxed optimisation objective).

\noindent\textbf{Isomap.} To achieve a simple closed-form solution, the stress function in classical MDS is often relaxed to the strain function, with uniform weights $w_{ij}=1$. Applying a squaring of the dissimilarities $p_{ij}\leftarrow p_{ij}^2$ and $q_{ij} \leftarrow q_{ij}^2$, the strain objective is given by
\begin{equation}
    \mathcal{L} = \sum(q_{ij} - p_{ij})^2 = \sum(\lVert y_j - y_i\rVert_2^2 - \boldsymbol{D}_{ij}^2)^2
\end{equation}
Switching from the stress to strain makes the objective quadratic and thus simpler to optimise. Indeed, the strain can be rewritten using the Gram matrix of the centred data, and is given by double centring of the squared-distance matrix $\boldsymbol{G} = -\tfrac{1}{2} \boldsymbol{J} \boldsymbol{D^{(2)}}\boldsymbol{J}^\top$, where $\boldsymbol{J} = \boldsymbol{I} - \tfrac{1}{N} \boldsymbol{1}_N\boldsymbol{1}_N^\top$ is the centring projector and $(\boldsymbol{D^{(2)}})_{ij} = \boldsymbol{D}_{ij}^2$ is the matrix of squared distances (see Sec.~\ref{sec: spectral methods appendix} for more details). This makes this relaxed version of classical MDS a kernel problem, with kernel $\boldsymbol{G}$. The solution is then provided by eigendecomposition of $\boldsymbol{G}$. When using geodesic dissimilarities $\boldsymbol{D}$, this approach is called Isomap \cite{tenenbaum2000global}. Although slightly different from solving the initial raw stress function, Isomap is nevertheless accepted as the reference method in the field for solving the classical MDS problem when weights are uniform $w_{ij} = 1$: it provides almost identical (good) solutions in practice, while being much faster as it does not require iterative minimisation and avoiding getting stuck in frequent local minima of the stress function. In prior works, the MDS objective often interchangeably means the stress or the strain objective, leading to much confusion, as in \cite{dages2025finsler}.

\subsection{Update rules of PHATE}
\label{sec: update rules of phate}

PHATE \cite{moon2019visualizing} is a popular visualisation method for biological data. However, at its core, PHATE is simply MDS applied to symmetric dissimilarities built from potentials computed from a diffusion operator. As such, it is a symmetric method, with symmetric dissimilarities and embedding. Similar to the traditional pipeline, PHATE uses some asymmetric intermediate concepts to build data dissimilarities, but these are always symmetric by design.

We here break down the core pipeline of PHATE. Given or computed distances $d_{ij}$ are locally scaled to $\tfrac{d_{ij}}{\varepsilon_{i}^{(k)}}$ via normalisation using the distance to the $k$-th neighbour $\varepsilon_{i}^{(k)}$. The normalised distances are then exponentially $\alpha$-decayed into affinities $\widetilde{\boldsymbol{K}}$
\begin{equation}
    \widetilde{\boldsymbol{K}}_{ij} = e^{-\left(\frac{d_{ij}}{\varepsilon_i^{(k)}}\right)^\alpha},
\end{equation}
which can be asymmetric due to different local normalisation scale $\varepsilon_i^{(k)}$. This adaptive $\alpha$-decaying radial basis kernel, which was proposed in the original paper and is the default, can also be replaced with any other kernel or precomputed affinities, eventually asymmetric.
By default and in the original paper, the affinities are symmetrised
\begin{equation}
    \boldsymbol{K}=\frac{\widetilde{\boldsymbol{K}}+\widetilde{\boldsymbol{K}}^\top}{2}.
\end{equation}

Row-normalisation gives a generally non-symmetric diffusion operator $\boldsymbol{P}$
\begin{equation}
    \boldsymbol{P} = \boldsymbol{D}_{\boldsymbol{K}}^{-1}\boldsymbol{K},
\end{equation}
where $\boldsymbol{D}_{\boldsymbol{K}}$ is a diagonal matrix with diagonal entries given by $(\boldsymbol{D}_{\boldsymbol{K}})_{ii} = \sum \boldsymbol{K}_{ij}$. Running conceptually diffusion for $t$-steps, its $t$-steps rows are used to compute potentials $\boldsymbol{U}_i$
\begin{equation}
    \boldsymbol{U}_i = -\log \boldsymbol{P}_{i:}^t,
\end{equation}
which are simply high-dimension features for each vertex $i$ (based on diffusion).
From these features, PHATE simply builds data dissimilarities $\boldsymbol{\Delta}_{ij}$ using the Euclidean distance between these features
\begin{equation}
    \boldsymbol{\Delta}_{ij}=\lVert \boldsymbol{U}_j-\boldsymbol{U}_i\rVert_2,
\end{equation}
which are by design symmetric $\boldsymbol{\Delta}^\top = \boldsymbol{\Delta}$. This concludes the dissimilarity construction stage of PHATE (P1). Note that any intermediate asymmetry disappears in the Euclidean distance on potentials.

PHATE embeddings are then found by applying the traditional Euclidean MDS to the symmetric Euclidean dissimilarities $\boldsymbol{\Delta}$ (P2-P3). Thus, PHATE is at its core a particularly simple symmetric method: convert the input raw data to diffusion features, compute Euclidean distances on them, and apply MDS on those distances.

\subsection{Update rules of t-SNE and UMAP}
\label{sec: update rules of t-SNE and UMAP}

From the optimisation perspective, both t-SNE and UMAP simply perform gradient descent on their objective. For fast implementations even on CPU hardware and to fully use the sparsity of the data, analytically computed gradients are directly implemented, rather than relying on more costly algorithms such as backpropagation via the chain-rule. We here present the analytical gradients of both objectives.

\subsubsection{t-SNE.}
\label{sec: tSNE detailed update rule}

The optimisation objective in t-SNE is the KL-divergence between sparse data dissimilarities $p_{ij}$ and the embedded dissimilarities $q_{ij}$ given by a globally normalised Student distribution with $\nu$ degrees of freedom: $q_{ij} = \tfrac{t_{ij}}{\sum_{k\neq l} t_{kl}}$ where $t_{ij} = \big(1+\tfrac{\lVert x_j - x_i\rVert_2^2}{\nu}\big)^{-\tfrac{\nu+1}{2}}$. Recall that the KL-divergence between $p_{ij}$ and $q_{ij}$ is given by
\begin{equation}
    \mathcal{L} =  \sum p_{ij} \ln \left( \frac{p_{ij}}{q_{ij}} \right).
\end{equation}
In the original t-SNE work \cite{van2008visualizing}, degrees of freedom were not considered, i.e.\ $\nu=1$, these only came with the follow-up work \cite{van2009learning} and is the standard in modern libraries.

However, there are several mistakes in the proofs of the works of Van der Maaten and Hinton \cite{van2008visualizing,van2009learning}. While in the original work, double mistakes end up cancelling each other out, leading to a correct result, they perform another mistake in the follow-up work which is not cancelled out, leading to an incorrect update rule in the paper. Unfortunately, this implies that the analytical gradient used in most libraries, even reference ones like Scikit-learn \cite{pedregosa2011scikit}, are incorrect.

As such, we propose to rederive in detail the calculation of the gradient and fix the t-SNE update rule.

\begin{theorem}[Fixed t-SNE gradient]
    \label{th: correct tSNE update rule}
    The correct update rule for t-SNE with $\nu$ degrees of freedom is
    $$
    \frac{\partial \mathcal{L}}{\partial y_i} = 
        2\frac{\nu+1}{\nu} \sum\limits_j (p_{ij} - q_{ij}) \left(1+\tfrac{\lVert y_j - y_i\rVert_2^2}{\nu}\right)^{-1}
        (y_i - y_j)
    $$
\end{theorem}

\begin{proof}
    Recall that $p_{ij}$ are calculated from the data and are thus constants from the perspective of the optimisation. The quantities varying with respect to the embedding $\boldsymbol{Y}$ are $q_{ij}$. Each $q_{ij}$ depends on all $t_{kl}$, and thus on the embedding of all other points, which explains why the gradient in $y_i$ depends on all other $y_j$ and not just local ones in the original space (i.e.\ local $x_j$ around $x_i$).

    We can rewrite the KL-divergence as
    \begin{equation}
        \mathcal{L} = \sum p_{ij} \ln p_{ij} - \sum p_{ij} \ln q_{ij}.
    \end{equation}
    The first term depends only on the data $p_{ij}$ and is independent from the embedding. As such, minimising the KL-divergence is the same as minimising the simplifed loss
    \begin{equation}
        \mathcal{L} = -\sum p_{ij} \ln q_{ij}.
    \end{equation}

    Denote $d_{ij} = \lVert y_j - y_i\rVert_2$ for conciseness. Recalling\footnote{An incorrect formula for the differentiation of the distance was the reason for the first set of mistakes in the original work \cite{van2008visualizing}.} that $\tfrac{\partial d_{ij}}{\partial y_i} = \tfrac{y_i - y_j}{\lVert y_i - y_j\rVert_2}$, we have
    \begin{align}
        \frac{\partial \mathcal{L}}{\partial y_i} 
        &= \sum_j \frac{\partial \mathcal{L}}{\partial d_{ij}} \frac{\partial d_{ij}}{\partial y_i} + \frac{\partial \mathcal{L}}{\partial d_{ji}}\frac{\partial d_{ji}}{\partial y_i} \\
        &= \sum_j \left(\frac{\partial \mathcal{L}}{\partial d_{ij}} + \frac{\partial \mathcal{L}}{\partial d_{ji}}\right) \frac{y_i - y_j}{\lVert y_i - y_j\rVert_2} \\
        &= 2 \sum_j \frac{\partial \mathcal{L}}{\partial d_{ij}} \frac{y_i - y_j}{\lVert y_i - y_j\rVert_2}.
    \end{align}

    We now focus on the term $\tfrac{\partial \mathcal{L}}{\partial d_{ij}}$. Note that $q_{ij}$ explicitly depends on every $d_{kl}$ for all pairs $(k,l)$, due to the denominator $Z = \sum_{k\neq l} t_{kl}$ in $q_{ij}$. We then have
    \begin{align}
        \frac{\partial \mathcal{L}}{\partial d_{ij}} 
        &= - \sum\limits_{k\neq l} p_{kl} \Bigg[ \frac{\partial}{\partial d_{ij}}\left(-\frac{\nu +1}{2}\ln\left(1+\frac{d_{kl}^2}{\nu}\right)\right) 
        - \frac{\partial}{\partial d_{ij}}(\ln(Z))\Bigg] \\
        &= -\sum\limits_{k\neq l} p_{kl} \Bigg[-\frac{\nu+1}{2}\frac{1}{1+\frac{d_{kl}^2}{\nu}}\frac{2d_{kl}}{\nu}\frac{\partial d_{kl}}{\partial d_{ij}} - \frac{1}{Z}\frac{\partial Z}{\partial d_{ij}}\Bigg]. 
    \end{align}
    Since we have $\tfrac{\partial d_{kl}}{\partial d_{ij}} = 1$ if $(i,j)=(k,l)$ and $0$ otherwise, and since we also have $\tfrac{\partial Z}{\partial d_{ij}} = -\tfrac{\nu+1}{2}\big(1+\tfrac{d_{ij}^2}{\nu}\big)^{-\frac{\nu+3}{2}}\frac{2d_{ij}}{\nu}$, and since $\sum_{k\neq l}p_{kl} =1$, we have
     \begin{align}
        \frac{\partial \mathcal{L}}{\partial d_{ij}} 
        &= \frac{\nu+1}{\nu}\frac{p_{ij}d_{ij}}{1+\frac{d_{ij}^2}{\nu}} - \sum\limits_{k\neq l} p_{kl} \frac{\nu+1}{\nu}\frac{d_{ij}\left(1+\frac{d_{ij}^2}{\nu}\right)^{-\frac{\nu+3}{2}}}{Z} \\
        &= \frac{\nu+1}{\nu}(p_{ij} - q_{ij}) d_{ij} \left(1+\frac{d_{ij}^2}{\nu}\right)^{-1}.
    \end{align}
    Returning to $\tfrac{\partial \mathcal{L}}{\partial y_i}$, and since $d_{ij} = \lVert y_i - y_j\rVert_2$ we get the desired result
    \begin{equation}
        \frac{\partial \mathcal{L}}{\partial y_i} = 2 \frac{\nu+1}{\nu}\sum\limits_j (p_{ij} - q_{ij})\left(1+\frac{d_{ij}^2}{\nu}\right)^{-1} (y_i - y_j).
    \end{equation}
    
\end{proof}

The correct derivation of the t-SNE update rule in Theorem \ref{th: correct tSNE update rule} contrasts with the formula derived by Van der Maaten \cite{van2009learning}, which is 
\begin{gather}
    \begin{bmatrix}
        \textbf{Incorrect} \\
        \textbf{\fontencoding{U}\fontfamily{futs}\selectfont\char 49\relax} 
    \end{bmatrix} \nonumber
    \\
    \frac{\partial \mathcal{L}}{\partial y_i} = 2\frac{\nu+1}{\nu}\sum\limits_j (p_{ij}-q_{ij}) \left(1+\frac{d_{ij}^2}{\nu}\right)^{-\frac{\nu+1}{2}}(y_i - y_j).
    \\
    \begin{bmatrix}
        \textbf{Incorrect} \\
        \textbf{\fontencoding{U}\fontfamily{futs}\selectfont\char 49\relax} 
    \end{bmatrix} \nonumber
\end{gather}

The issue is the incorrect exponent of the term $\big(1+\tfrac{d_{ij}^2}{\nu}\big)$, which should be $-1$ but was written $-\tfrac{\nu+1}{2}$ in \cite{van2009learning}, which is incorrect as soon as $\nu \neq 1$. This happens when the embedding dimension is $m\ge3$ as the usual recipe is $\nu = \max(m-1, 1)$\footnote{In \cite{van2009learning}, the recipeis $\nu = m-1$ so that planar embeddings have one degree of freedom, matching the original work \cite{van2008visualizing}. In Scikit-learn, the max with 1 is taken avoid crashes when embedding to a one-dimensional space.}. Common implementations of t-SNE, such as in the reference Scikit-learn library unfortunately carry on this mistake.

\subsubsection{UMAP.}

The optimisation objective in UMAP is the cross-entropy between the data dissimilarities $p_{ij}$ and the embedded dissimilarities $q_{ij}$ given by a (optionally tweaked) Student distribution with $1$ degree of freedom but without global normalisation: $q_{ij} = (1+a\lVert y_j - y_i\rVert_2^{2b})^{-1}$, where $a$ and $b$ can tweak the distribution away from a rigorous Student distribution. Recall that the cross-entropy  loss between $p_{ij}$ and $q_{ij}$ is given by
\begin{equation}
    \mathcal{L} = \sum p_{ij} \ln \left(\frac{p_{ij}}{q_{ij}}\right) + (1 - p_{ij}) \ln \left(\frac{1 - p_{ij}}{1 - q_{ij}}\right).
\end{equation}

Denote $c_{ij}^a = -\ln q_{ij}$ and $c_{ij}^r = -\ln(1-q_{ij})$ the symmetric attractive and repulsive forces respectively. Due to the absence of global normalisation, unlike in t-SNE, $q_{ij}$ only explicitly depends on $y_i$ and $y_j$. The gradient of the objective is then
\begin{equation}
    \frac{\partial \mathcal{L}}{\partial y_i} = 2 \sum\limits_{j} p_{ij} \cdot \frac{\partial c_{ij}^a}{\partial y_i}+ (1 - p_{ij}) \cdot \frac{\partial c_{ij}^r}{\partial y_i},
\end{equation}
with
\begin{equation}
    \begin{dcases}
        \frac{\partial c_{ij}^a}{\partial y_i} = \frac{2ab \lVert y_j-y_i\rVert_2^{2(b-1)}}{1+a\lVert y_j - y_i\rVert_2^{2b}} (y_i - y_j)  
        \\
        \frac{\partial c_{ij}^a}{\partial y_j} = - \frac{\partial c_{ij}^a}{\partial y_i} \\
        \frac{\partial c_{ij}^r}{\partial y_i} = \frac{-2b }{\lVert y_j-y_i\rVert_2^2 \left(1+a\lVert y_j - y_i\rVert_2^{2b}\right)} (y_i - y_j)
        \\
        \frac{\partial c_{ij}^r}{\partial y_j} = - \frac{\partial c_{ij}^r}{\partial y_i}
    \end{dcases}
\end{equation}

The attractive forces contribute sparsely to the loss due to weighting with the sparse $p_{ij}$. On the other hand, the repulsive forces are dense. To speed up computation, UMAP uses negative sampling to only compute repulsive forces on a small number of random samples. This makes UMAP a random algorithm beyond initialisation. Additionally, UMAP is efficiently implemented with multi-threading, which adds another level of randomness.

\subsection{Spectral methods}
\label{sec: spectral methods appendix}

Spectral methods in manifold learning are a special type of Euclidean methods where the objective function was carefully designed such that it can be rewritten as a kernel problem and thus be solved directly via (generalised) eigendecomposition of a matrix constructed from the data dissimilarities $\boldsymbol{D}$. Spectral methods can be designed using various paradigms, such as distance preservation in Isomap \cite{tenenbaum2000global}, local linear preservation, e.g.\ LLE-based methods \cite{roweis2000nonlinear,donoho2003hessian,zhang2006mlle,zhang2004principal}, and the Laplacian in Laplacian eigenmaps \cite{belkin2003laplacian} and Diffusion maps \cite{coifman2006diffusion}.
From the perspective of this paper, it is important to understand that spectral methods can systematically be given by solving optimisation problems involving (squared) pairwise Euclidean distances between embedded points $y_i$ (analogues to our embedding dissimilarities) and fitting to some pairwise data terms (analogues to our data dissimilarities).

\paragraph{Isomap.} 
The secret in relaxing the stress to the strain is the use of quadratic distances. Stacking up the squared distances into a matrix 
$\boldsymbol{D^{(2)}_Y}$, i.e.\ $(\boldsymbol{D^{(2)}_Y})_{ij} = \lVert y_j - y_i\rVert_2^2$, and assuming that the embedding $\boldsymbol{Y}$ is centred, expanding the squared distances leads to $\boldsymbol{J}\boldsymbol{D^{(2)}_Y}\boldsymbol{J} = -2\boldsymbol{Y}\boldsymbol{Y}^\top$. In other words the Gram matrix $\boldsymbol{G_Y} = \boldsymbol{Y}\boldsymbol{Y}^\top = -\tfrac{1}{2}\boldsymbol{J}\boldsymbol{D^{(2)}_Y}\boldsymbol{J}$. As the strain objective is given by $\mathcal{L} = \lVert \boldsymbol{J}(\boldsymbol{D^{(2)}_Y} - \boldsymbol{D^{(2)}})\boldsymbol{J}\rVert_F^2$, where here $\lVert \cdot\rVert_F$ is the Frobenius norm, Isomap minimises the objective $\mathcal{L} = 4\lVert \boldsymbol{G_Y} - \boldsymbol{G}\rVert_F^2$, where $\boldsymbol{G} = -\tfrac{1}{2}\boldsymbol{J}\boldsymbol{D^{(2)}}\boldsymbol{J}$. This is a typical kernel formulation. As $\boldsymbol{G_Y}$ has rank at most $m$, since $\boldsymbol{Y}\in\mathbb{R}^m$, a closed-form solution is provided by taking the top $m$ eigenvectors of $\boldsymbol{G}$. 
The main advantage of spectral methods is that they are solved by eigendecomposition of a matrix, rather than by more costly iterative solvers that also risk getting stuck in local minima. This can often come at the cost of less robustness, e.g.\ to noise, as eigenvectors tend to be unstable (unlike eigenvalues).

\paragraph{Laplacian eigenmaps.}
Denote $\boldsymbol{\Delta}$ an estimated graph Laplacian matrix on the data, with mass matrix $\boldsymbol{A}$ and stiffness $\boldsymbol{W}$ on the neighbourhood graph equipped with edge weights $W_{ij} = e^{-\frac{\lVert x_j - x_i\rVert_2^2}{t}}$ for some small time scale $t$. The connectivity weights $\boldsymbol{W}$ are analogues of the data dissimilarities $p$. Laplacian eigenmaps \cite{belkin2003laplacian} proposes to embed the data by taking the eigenvectors associated to the smallest eigenvalues of the generalised eigendecomposition $\boldsymbol{\Delta}f = \lambda \boldsymbol{A}f$. This spectral approach is the solution to an alternative optimisation problem. 
Laplacian eigenmaps solves the constrained optimisation $\mathrm{tr}(\boldsymbol{Y}^\top \boldsymbol{\Delta} \boldsymbol{Y})$ subject to the constraint $\boldsymbol{Y}^\top \boldsymbol{A} \boldsymbol{Y} = \boldsymbol{I}$. 
For simplicity, assume that  $\boldsymbol{\Delta} = \boldsymbol{A} - \boldsymbol{W}$ is the unnormalised graph Laplacian, with $A_{i} = \sum_j W_{ij}$, where $W_{ij}$ (resp.\ $A_{i}$) populate the matrix $\boldsymbol{W}$ (resp.\ the diagonal matrix $\boldsymbol{A}$). 
The constrained objective can then be rewritten as $\mathcal{L} = \tfrac{1}{2}\sum W_{ij} \lVert y_j - y_i\rVert_2^2 = \mathrm{tr}(\boldsymbol{Y}^\top \boldsymbol{\Delta}\boldsymbol{Y})$. Using different Laplacians will modify the $y_i$ terms. For instance, for the symmetric graph Laplacian $\boldsymbol{\Delta} = \boldsymbol{I} - \boldsymbol{A}^{-\frac{1}{2}}\boldsymbol{W}\boldsymbol{A}^{-\frac{1}{2}}$, this choice rescales the data terms in the objective as the constrained objective becomes $\mathcal{L} = \tfrac{1}{2}\sum W_{ij} \Big\lVert \tfrac{y_j}{\sqrt{A_{j}}} - \tfrac{y_i}{\sqrt{A_{i}}}\Big\rVert_2^2 = \mathrm{tr}(\boldsymbol{Y}^\top \boldsymbol{\Delta} \boldsymbol{Y})$.

\paragraph{Diffusion maps.}
Diffusion maps \cite{coifman2006diffusion} are tightly connected to Laplacian eigenmaps \cite{belkin2003laplacian}. Denote again $\boldsymbol{\Delta}$, $\boldsymbol{A}$, and $\boldsymbol{W}$ to be some Laplacian, its mass matrix, and its stiffness matrix. Again, $W_{ij} = e^{-\frac{\lVert x_j - x_i\rVert_2^2}{t}}$. The idea behind diffusion maps is to normalise the Laplacian and then view it as a Markov process transition matrix $\boldsymbol{P}$ simulating heat diffusion. With this perspective, the goal is to preserve diffusion distances, rather than raw distances. Denoting $\lambda_k$ (in decreasing order) and $\psi_k$ the eigenvalues and eigenvectors of $\boldsymbol{P}$, the squared diffusion distance at time step $t$ is given by $D_t^2(x_i, x_j) = \sum_k \lambda_k^{2t} \big((\psi_k)_j - (\psi_k)_i\big)^2$. It can be rewritten, using the Euclidean norm as $D_t^2(x_i, x_j) = \lVert \phi_t(x_i) - \phi_t(x_j)\rVert_2^2$, where $\phi_t(x_i) = (\lambda_1^t (\psi_1)_i, \lambda_2^t (\psi_2)_i, \ldots)$. To preserve the diffusion distance in $\mathbb{R}^m$ embeddings, diffusion maps simply suggests to take as embedding $\phi_t(x_i)$ truncated to its first $m$ columns.
Like other spectral approaches, this embedding can be seen as the solution to an alternative optimisation objective involving Euclidean distances between embedded points $y_i$.
For simplicity, assume that the transition matrix $\boldsymbol{P}$ is symmetric, otherwise we could involve the $\boldsymbol{\Pi}$-weights of the stationary distribution for reversibility of the Markov process, i.e.\ $\boldsymbol{\Pi} \boldsymbol{P} = \boldsymbol{P}^\top \boldsymbol{\Pi}$ for the diagonal matrix $\boldsymbol{\Pi}$ with diagonal entries $\pi_i$, and symmetrise by $\boldsymbol{P}^t \leftarrow \boldsymbol{\Pi}\boldsymbol{P}^t$.
Similar to Laplacian eigenmaps, with $\boldsymbol{I} - \boldsymbol{P}^t$ instead of $\boldsymbol{\Delta}$, minimising the quadratic form $\mathcal{L} = \tfrac{1}{2}\sum \boldsymbol{P}_{ij}^t\lVert y_j - y_i\rVert_2^2 = \mathrm{tr}(\boldsymbol{Y}^\top (\boldsymbol{I} - \boldsymbol{P}^t)\boldsymbol{Y})$ subject to $\boldsymbol{\boldsymbol{Y}^\top \boldsymbol{Y}} = \boldsymbol{\Lambda}^{2t}$ yields the diffusion maps solution $\phi_t$ (note that the constraint was scaled by the eigenvalues in order to get them to multiply the eigenfunctions). 
As such, diffusion maps is simply Laplacian eigenmaps using the Markov transition matrix as weights in the optimisation and a scaled constraint.
An alternate view of diffusion maps is to see it as a special case of spectral MDS, which applies MDS to diffusion distances, as discussed in the original paper of spectral MDS \cite{aflalo2013spectral}.

\paragraph{Locally linear embeddings.} 
In locally linear embeddings (LLE) and its variations, the goal is to preserve the locally linear structure of the data. Thus, rather than raw distance dissimilarities, LLE-based methods focus on local reconstruction weights, which are inherently connected to distances, expressing each point as a weighted average of its neighbours. In vanilla LLE \cite{roweis2000nonlinear}, the objective to minimise is
\begin{equation}
    \mathcal{L} = \sum\Big\lVert y_i - \sum\limits_j W_{ij}y_j\Big\rVert_2^2,
\end{equation}
subject to centring and normalisation constraints on $\boldsymbol{Y}$, e.g.\ $\boldsymbol{Y}^\top \boldsymbol{1} = 0$ and $\boldsymbol{Y}^\top\boldsymbol{Y} = N\boldsymbol{I}$.
The local reconstruction weights $W_{ij}$, acting as analogues to $p_{ij}$, are given by solving the least squares problem for vertices $j$ within the neighbourhood $\mathcal{N}(i)$ of $i$
\begin{equation}
    \min\limits_{\substack{w\\\sum\limits_j W_{ij} = 1}} \sum\Big\lVert x_i - \sum\limits_{j\in\mathcal{N}(i)} W_{ij} x_j\Big\rVert_2^2,
\end{equation}
with closed-form solution provided by Lagrangian optimisation. Expanding the quadratic objective of LLE, it can be rewritten in matrix form as $\mathcal{L} = \lVert (\boldsymbol{I}-\boldsymbol{W})\boldsymbol{Y}\rVert_F^2 = \mathrm{tr}(\boldsymbol{Y}^\top \boldsymbol{G_W} \boldsymbol{Y})$, where $\boldsymbol{W}$ has entries $W_{ij}$, $\boldsymbol{G_W} = (\boldsymbol{I}-\boldsymbol{W})^\top( \boldsymbol{I}-\boldsymbol{W})$, and $\mathrm{tr}$ is the trace operator. This is a kernel problem, with kernel $\boldsymbol{G_W}$. The solution to LLE is thus to take the eigenvectors of $\boldsymbol{G_W}$ with lowest eigenvalues.

\subsection{Autoencoders}

Neural autoencoders are not commonly considered manifold learning techniques. Nevertheless, they can provide dimensionality reduction from high dimensions to the predefined latent dimension of the trained autoencoder, which are usually quite high for manifold learning (often a few hundred dimensions). Embedding to other dimensions is not possible without retraining the network, or involving classical manifold learning techniques, typically PCA, on top of the latents.
Standard autoencoders learn by reconstruction and practitioners usually take symmetric latent geometry on top of their latents: Euclidean $d_{ij}=\lVert z_j-z_i\rVert_2$ or spherical-cosine similarity. Asymmetric distances are non-standard, and we are unaware of such autoencoder constructions. Latent geometry commonly enters contrastive methods, such as CEBRA \cite{schneider2023learnable}, but this requires labels or external information, unlike traditional manifold learning techniques that only consider pairwise dissimilarities. 
In this manifold learning work, we focused on true manifold learning methods, which study unordered data without class knowledge. Hence, we did not explore autoencoders.

\subsection{On the dissimilarities terminology}

In the manifold learning community, many other terminologies can appear instead of ``dissimilarities'': ``distances'', ``affinities'', ``similarities'', ``probabilities''... Each terminology brings its own conceptual flavour, especially: should a small value mean that points are close or distant? In this work, we chose a single terminology for simplicity and ease of reading. Following the most related work \cite{dages2025finsler}, we adopted ``dissimilarity'', which is similar conceptually to distances, a most fitting choice when based in metric theory, while allowing transforms of them.
Should a method use similarities instead, like t-SNE and UMAP (larger values mean closer distance), then our ``dissimilarities'' terminology does not bring confusion as we also call the embedding similarities as ``dissimilarities''. No matter whether the method is actually using similarities or dissimilarities, our embeddings cannot be misinterpreted: points close by in the input space will be embedded close by in the Finsler embedding space, resolving any possible confusion.

\section{Further discussions on the metric on sampled data}
\label{sec: discussions metric on the data}

\subsection{What is lost in symmetrisation}
\label{sec: what is lost in symmetrisation}

We here quantify the lost information incurred by symmetrisation. 
Write
\begin{equation}
    \boldsymbol{D}=\boldsymbol{D}_S+\boldsymbol{D}_A,
\end{equation}
where
\begin{equation}
    \boldsymbol{D}_S=\frac{\boldsymbol{D}+\boldsymbol{D}^\top}{2} 
    \quad \text{and} \quad
    \boldsymbol{D}_A= \frac{\boldsymbol{D}-\boldsymbol{D}^\top}{2}.
\end{equation}
By construction, $\boldsymbol{D}_S$ is symmetric, i.e.\ $\boldsymbol{D}_S^\top  = \boldsymbol{D}_S$, while $\boldsymbol{D}_A$ is skew-symmetric, i.e.\ $\boldsymbol{D}_A^\top = -\boldsymbol{D}_A$.
For any symmetric reconstruction with embedded dissimilarities $\boldsymbol{Q}=\boldsymbol{Q}^\top$, we have the following embedding error
\begin{equation}
    \lVert \boldsymbol{D}-\boldsymbol{Q}\rVert_F^2=\lVert \boldsymbol{D}_S-\boldsymbol{Q}\rVert_F^2+\lVert \boldsymbol{D}_A\rVert_F^2,
\end{equation}
as symmetric and skew-symmetric matrices are orthogonal in the Frobenius inner product. Therefore, $\lVert \boldsymbol{D}_A\rVert_F^2$ is an irreducible squared reconstruction loss for any symmetric embedding model. No method using symmetrisation can recover it.
Note that this result does not claim that $\boldsymbol{D}_A$ is always useful, but it identifies exactly what is discarded under squared error.
Should sampling variation, which induces asymmetric data dissimilarities, be a nuisance, then symmetrisation can provide desirable results, as symmetrisation discards the directional component by construction.
In some cases, like the US cities in Fig.~\ref{fig: US cities}, this discarded information has semantic meaning that could be useful downstream. Our method accounts for, preserves, and reveals this directional information should practitioners need it.

\subsection{Implicit Finsler metric arising from sampling variation}
\label{sec: finsler metric on the data}

\begin{figure}[t]
    \centering
    \includegraphics[width=\columnwidth]{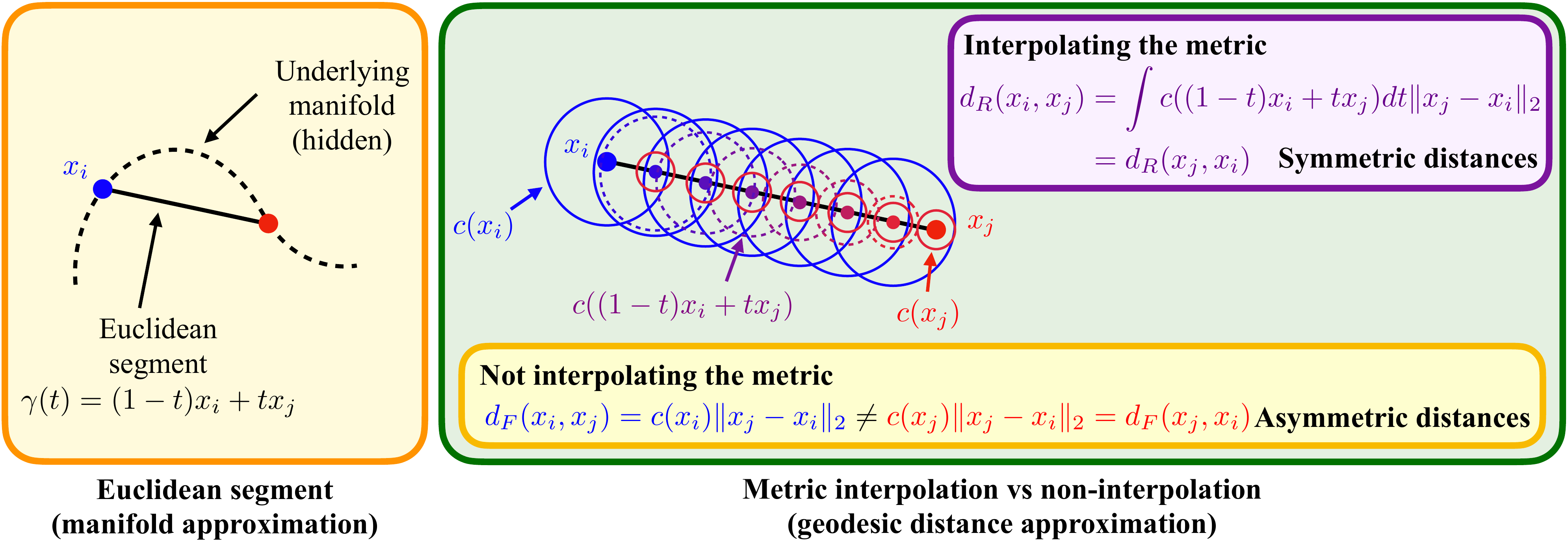}
    \caption{
    While the underlying manifold may be equipped with a Riemannian metric, typically isotropic $R_x(u) = c(x) \lVert u \rVert_2$ with local scale $c(x)$, computing distances on discretised samples leads to asymmetry due to approximations that might not hold in practice. This is particularly true for high-dimensional data where neighbouring samples might not be sufficiently close to have almost identical local metric due to undersampling. Thus not interpolating the metric along path between neighbouring points creates asymmetric geodesic distances. The constructed asymmetric distances could actually be those of an implicit Finsler metric, rather than the original Riemannian metric, that accounts for extra information on how the discrete samples are distributed.
    }
    \label{fig: manifold metric interpolation approximations}
\end{figure}

We here further discuss about the meaning of the Finsler metric arising by using asymmetric dissimilarities in the traditional pipeline based-on Riemannian distances. We use Fig.~\ref{fig: manifold metric interpolation approximations} as illustration.

The important subtlety here is on the meaning of the \emph{distance} between two data points. Consider the traditional Riemannian setting. For nearby points $x_i,x_j$, two approximations are made: 
\textit{i)} replace the unknown manifold geodesic path by the ambient segment $\gamma(t) = (1-t) x_i + t x_j$
\begin{equation}
    d_R(x_i,x_j)\approx\int_0^1 R_{\gamma(t)}(x_j - x_i)
\end{equation}
and 
\textit{ii)}
freeze the varying metric along that segment at its source, thus
\begin{equation}
    R_{\gamma(t)}(x_j-x_i)\approx c(x_i)\|x_j-x_i\|, \qquad \forall t\in[0,1).
\end{equation}
Combining both gives $d_R(x_i,x_j)\approx c(x_i)\|x_j-x_i\|$. Under smooth, sufficiently dense sampling, the second approximation is first-order accurate and $c(x_i)\approx c(x_j)$, so forward and reverse estimates agree. But in finite, non-uniformly sampled high-dimensional data, neighbouring scale estimates can differ substantially: ii) does not hold $c(x_i)\neq c(x_j)$. Then $c(x_i)\|x_j-x_i\|$ and $c(x_j)\|x_i-x_j\|$ are directed and cannot both equal the geodesic distance $d_R$ of the symmetric Riemannian metric $R$.

If these one-sided quantities are nevertheless retained as the operative local distances, as in our pipeline, the resulting distance being modeled is asymmetric and cannot be Riemannian: it is most likely instead some implicit Finsler metric $F$.

Thus $d_F(x_i,x_j)=R_{x_i}(x_j-x_i)$ should be read as defining the retained Finsler manifold distance between nearby points, not as a claim that $R$ has become non-reversible or that this equals $d_R$. The construction specifies the local directed geodesic distance used by the algorithms rather than explicitly reconstructing a continuous tangent-plane formula for $F$. Directed shortest-path extension then supplies global geodesics if a dense matrix is required, although sparse methods like (Finsler) t-SNE or UMAP use the local directed distances directly.

\subsection{Formal quantitative link between sampling density and data directionality}
\label{sec: asymptotic analysis link sampling density and directionality}

To provide further theoretical insight, we give the following simple local asymptotic view. Let $N$ points be sampled from a smooth positive density $\rho$ on a $d$-dimensional manifold, and let $r_k(x)$ denote a kNN radius. Under the standard kNN regime,
\begin{equation}
    r_k(x)\asymp\left(\frac{k}{Nv_d\rho(x)}\right)^{\frac{1}{d}},
\end{equation}
where $v_d$ is the volume of the unit $d$-ball.
After removing the dataset-wide factor $(\tfrac{Nv_d}{k})^{\frac{1}{d}}$, which contains no spatial information, let $c(x)\propto \tfrac{1}{r_k(x)}\propto\rho(x)^{\frac{1}{d}}$. For the one-sided length $\ell_{i\to j}=c(x_i)\lVert x_j-x_i\rVert$, take a reciprocal neighbour $x_j=x_i+u$ in local normal coordinates. The skew is given by
\begin{equation}
    \frac{\ell_{i\to j}-\ell_{j\to i}}{\ell_{i\to j}+\ell_{j\to i}}
    =\frac{c(x)-c(x+u)}{c(x)+c(x+u)}.
\end{equation}
By Taylor expansion,
\begin{equation}
    c(x+u)=c(x)+\nabla c(x)^\top u+O(\lVert u\Vert_2^2).
\end{equation}
Thus
\begin{equation}
    \frac{\ell_{i\to j}-\ell_{j\to i}}{\ell_{i\to j}+\ell_{j\to i}}
    =\frac{-\nabla c(x)^\top u+O(\lVert u\Vert_2^2)}
    {2c(x)+\nabla c(x)^\top u+O(\lVert u\Vert_2^2)}
    =-\frac{\nabla c(x)^\top u}{2c(x)}+O(\lVert u\Vert_2^2).
\end{equation}
As $\nabla\log c(x) = \frac{\nabla c(x)}{c(x)}$, we have
\begin{equation}
    \frac{\ell_{i\to j}-\ell_{j\to i}}{\ell_{i\to j}+\ell_{j\to i}}
    =-\frac{1}{2}\nabla\log c(x)^\top u+O(\lVert u\Vert_2^2).
\end{equation}
Since $c(x)\propto\rho(x)^{\frac{1}{d}}$, $\nabla\log c(x)=d^{-1}\nabla\log\rho(x)$, hence
\begin{equation}
    \label{eq: asymptotic analysis dissimilarity skew}
    \frac{\ell_{i\to j}-\ell_{j\to i}}{\ell_{i\to j}+\ell_{j\to i}}=-\frac{1}{2d}\nabla\log\rho(x_i)^\top u+O(\lVert u\Vert_2^2).
\end{equation}

Thus the leading continuous directional signal is the local directional derivative of log sampling density. Method-specific transforms (e.g.\ squaring distances) yield a corresponding first-order term with transform-dependent sign and scale. However, the common point is that the leading skew is controlled by local-scale differences.

Note that the above discussion concerns solely the input data dissimilarities. 
Regarding the embedding side, we do not need an asymptotic analysis. For canonical Randers geometry, the embedding antisymmetry is given by
\begin{equation}
    \frac{d_F(y_i,y_j)-d_F(y_j,y_i)}{2}
    =\omega^\top(y_j-y_i).
\end{equation}
With $\omega=\lVert\omega\rVert_2 e_{m+1}$ without loss of generality, the last coordinate is therefore the scalar potential used by the embedding to fit the antisymmetric distance component. We thus have the two-step (data to embedding) mechanism:
first, local scaling turns density gradients into asymmetric edge skew, and
second, the Randers coordinate represents the fitted skew through pairwise coordinate differences.

\subsection{On vanishing asymmetry}
\label{sec: on vanishing asymmetry}

If the asymmetric dissimilarities are built solely from density disparities, then our construction can become similar to that of the symmetric Riemannian pipeline.

\noindent \textbf{Dense-sampling limit.}
Locally, we achieve continuum limits, even when the underlying smooth density $\rho$ is not uniform. Schematically, for smooth metrics, sufficiently close points will have (almost) identical local metrics $R_{x_j}\approx R_{x_i}$, e.g.\ due to (almost) identical local density scale $c(x_i)\approx c(x_j)$. Thus, the asymmetry arising for metric differences vanishes $R_{x_j}(x_i - x_j) \asymp_{x_j\to x_i} R_{x_i}(x_j-x_i)$.
More rigorously, neighbour displacements satisfy $\lVert u\rVert_2=O(r_k)$ (see Eq.~\ref{eq: asymptotic analysis dissimilarity skew}), so normalised local skew is $O(r_k)$ and vanishes as neighbourhoods shrink. After removing the common dataset-wide kNN scale, forward and reverse one-sided integrals of a smooth reversible length converge to the same path integral. Thus this local-length model has a symmetric continuum limit when $N\to\infty$, $k\to\infty$, and $k/N\to0$.

\noindent \textbf{Constant density.}
For constant density, i.e.\ $\nabla\rho=0$, the leading deterministic term for the non-binary asymmetry vanishes. However, finite-sample scale-estimation variability of non-asymptotic regimes and non-reciprocal graph edges (binary asymmetry) can remain. A regular but sparse kNN graph may therefore still have binary or noisy asymmetry. In this case, the computed asymmetry might not be meaningful and preserving it might not compute an advantage over traditional symmetrisation.

Our asymmetric construction should have the strongest benefit for finite, irregularly sampled data when sampling variation is scientifically meaningful. When it is not, $\lVert \omega\Vert_2=0$ or symmetrisation could be appropriate.
Note that the cases discussed above only tackle the situation when the data dissimilarity construction depends solely on local density differences. This might not be the case in other constructions, e.g.\ oracle directed graphs or using external non-smooth local features.

\subsection{Towards a fully symmetric pipeline}
\label{sec: fixing symmetric pipeline}

We here provide a possible theoretically-justified direction for
the traditional manifold learning pipeline
based on
metric geometry 
that aims to avoid asymmetric
dissimilarities.

The goal is to find a better symmetric estimate of the manifold distance $d_R(x_i, x_j)$ between close points. 
We still approximate the shortest geodesic path $\gamma(t): x_i \rightsquigarrow x_j$ as the Euclidean line $(1-t)x_i + tx_j$ between 
close points.
For approximating its length $d_R(x_i, x_j) = \lvert\gamma\rvert = \int_t\lVert\gamma'(t)\rVert_{M(\gamma(t))}dt$, instead of a single metric estimate $M(\gamma(0)) = c^2(x_i)I$, we propose to interpolate 
the metric scale $c(\gamma(t)) = (1-t)c(x_i) + tc(x_j)$, which is equivalent to linearly interpolating infinitesimal steps $F_{\gamma(t)}(x_j - x_i)dt$. Integrating along the segment provides the following symmetric approximation.

\begin{theorem}
    \label{th: symmetric geodesic dist interp scale}
    Given an isotropic non-uniform Riemannian metric $R$, with $M(x) = c^2(x)I$, then the geodesic distance between close points $x_i$ and $x_j$ is approximately given by
    $$d_R(x_i,x_j) = \frac{c(x_i) + c(x_j)}{2}\lVert x_j - x_i\rVert_2 = d_R(x_j, x_i).$$
\end{theorem}

\begin{proof}
    Denote $\gamma(t) = (1-t) x_i + t x_j$ with $t\in[0,1]$ a parametrisation of the Euclidean segment $x_j -x_i$, and $\gamma(1-t) = t x_i + (1-t)x_j$ the segment in the opposite direction. The estimated metric given by interpolating the isotropic scale $M\big(\gamma(t)\big)$ is given by 
    \begin{equation}
        M\big(\gamma(t)\big) = \Big((1-t) c(x_i) + t c(x_j)\Big)^2 I.
    \end{equation}
    
    The length of the curve $\gamma(t)$, denoted $|\gamma(t)|$, is given by
    \begin{align}
        |\gamma(t)| 
        &= \int_0^1 \lVert \gamma'(t) \rVert_{M\big(\gamma(t)\big)} dt \\
        &= \int_0^1 \lVert x_j - x_i \rVert_{\big((1-t) c(x_i) + t c(x_j)\big)^2 I} dt \\
        &= \int_0^1 \big((1-t)c(x_i) + t c(x_j)\big) \lVert x_j - x_i \rVert_2 dt \\
        &= \frac{c(x_i) + c(x_j)}{2} \lVert x_j - x_i \rVert_2.
    \end{align}
    
    As the interpolation leads to an isotropic metric, it is symmetric, leading to symmetric distances, meaning that the estimated geodesic distance is the same in both directions of traversal $|\gamma(t)| = |\gamma(1-t)|$.
\end{proof}

As such, linearly interpolating the metric leads to averaging out start and end-point metric scale. Since commonly $c(x_i) = \tfrac{1}{\sigma_i}$, our Riemannian fix suggests to replace $\sigma_i$ in the calculation of $p_{ij}$ by the harmonic mean $\sigma_{ij}$ of $\sigma_i$ and $\sigma_j$.

\begin{corollary}
    \label{th: theoretical justified symmetric geodesic neighbour distances}
    Given an isotropic non-uniform Riemannian metric $R$, with $M(x) = c^2(x)I$ and $c(x_i) = \tfrac{1}{\sigma_i}$, then the geodesic distance between close points $x_i$ and $x_j$ is approximately given by
    $$d_{R}(x_i, x_j) = \frac{1}{\sigma_{ij}}\lVert x_j-x_i\rVert_2 \quad\text{with}\quad \frac{1}{\sigma_{ij}} = \frac{1}{2}\bigg(\frac{1}{\sigma_i} + \frac{1}{\sigma_j}\bigg).$$
\end{corollary}

Note that other close approximations of the geodesic distance between close points are possible. For instance, instead of interpolating the metric scale $c(\gamma(t))$, i.e.\ the length of infinitesimal steps $F_{\gamma(t)}(x_j - x_i)dt$, we could instead interpolate the metric itself $M(\gamma(t)) = (1-t)M(x_i) + tM(x_j)$. Integrating this alternative interpolation along the segment provides the following alternative symmetric approximation.

\begin{theorem}
    \label{th: symmetric geodesic dist interp metric}
    Given an isotropic non-uniform Riemannian metric $R$, with $M(x) = c^2(x)I$, then the geodesic distance between close points $x_i$ and $x_j$ is approximately given by
    $$d_R(x_i,x_j) = \frac{2\big(c^2(x_i) + c(x_i)c(x_j)+  c^2(x_j)\big)}{3\big(c(x_i) + c(x_j)\big)}\lVert x_j - x_i\rVert_2 = d_R(x_j, x_i).$$
\end{theorem}

\begin{proof}
    Denote $\gamma(t) = (1-t) x_i + t x_j$ with $t\in[0,1]$ a parametrisation of the Euclidean segment $x_j -x_i$, and $\gamma(1-t) = t x_i + (1-t)x_j$ the segment in the opposite direction. The estimated linearly interpolated metric $M\big(\gamma(t)\big)$ is given by 
    \begin{align}
        M\big(\gamma(t)\big) &= (1-t) M(x_i) + t M(x_j)\\
        &= \Big((1-t)c^2(x_i) + tc^{2}(x_j)\Big)I.
    \end{align}
    
    The length of the curve $\gamma(t)$, denoted $|\gamma(t)|$, is given by
    \begin{align}
        |\gamma(t)| 
        &= \int_0^1 \lVert \gamma'(t) \rVert_{M\big(\gamma(t)\big)} dt \\
        &= \int_0^1 \lVert x_j - x_i \rVert_{\big((1-t)c^2(x_i) + tc^{2}(x_j)\big)I} dt \\
        &= \int_0^1 \sqrt{(1-t)c^2(x_i) + t c^2(x_j)} dt
        \; \lVert x_j - x_i \rVert_2 \\
        &= \int_0^1 \sqrt{c^2(x_i) + t(c^2(x_j)-c^2(x_i))} dt
        \; \lVert x_j - x_i \rVert_2.
        \label{eq: symmetric geodesic dist interp metric sqrt int}
    \end{align}
    Changing integration variables to $u = c^2(x_i) + t\big(c^2(x_j)-c^2(x_i)\big)$, we have
    \begin{align}
        \int_0^1 \sqrt{c^2(x_i) + t(c^2(x_j)-c^2(x_i))}dt &= \frac{1}{c^2(x_j) - c^2(x_i)} \int_{c^2(x_i)}^{c^2(x_j)} u^{\frac{1}{2}} du\\
        &= \frac{2}{3(c^2(x_j) - c^2(x_i))} \left[(c^2(x_j))^{\frac{3}{2}}-(c^2(x_i))^{\frac{3}{2}}\right]\\
        &= \frac{2}{3}
       \frac{\big(c(x_j)-c(x_i)\big)\big(c^2(x_i)+c(x_i)c(x_j)+c^2(x_j)\big)}{\big(c(x_j)-c(x_i)\big)\big(c(x_i)+c(x_j)\big)}\\
       &= \frac{2\big(c^2(x_i) + c(x_i)c(x_j)+  c^2(x_j)\big)}{3\big(c(x_i) + c(x_j)\big)}.
    \end{align}
    Plugging this into Eq.~\ref{eq: symmetric geodesic dist interp metric sqrt int}, we get the desired approximation result:
    \begin{equation}
        |\gamma(t)| = \frac{2\big(c^2(x_i) + c(x_i)c(x_j)+  c^2(x_j)\big)}{3\big(c(x_i) + c(x_j)\big)} \lVert x_j - x_i \rVert_2.
    \end{equation}
    
    As in the other alternative approximation, as the interpolated metric is isotropic, it is symmetric, leading to symmetric distances, meaning that the estimated geodesic distance is the same in both directions of traversal $|\gamma(t)| = |\gamma(1-t)|$.
\end{proof}

Note that although both provided approximations are not exactly the same, they are approximately identical when the points are very close on a smooth manifold with smooth metric, as then $c(x_i)\approx c(x_j)$ and $\big(c(x_j) - c(x_i)\big)^2 = O(\lVert x_j - x_i\rVert_2^2)$.
Both approximations differ up to second order terms as
\begin{align}
    \frac{2\big(c^2(x_i) + c(x_i)c(x_j)+  c^2(x_j)\big)}{3\big(c(x_i) + c(x_j)\big)} 
    &= c(x_i)+\frac{c(x_j) - c(x_i)}{2} +\frac{(c(x_j) - c(x_i))^2}{6(2c(x_i)+c(x_j) - c(x_i))} \\
    &= c(x_i)+\frac{c(x_j) - c(x_i)}{2} + O(\lVert x_j - x_i\rVert_2^2),
\end{align}
and 
\begin{equation}
    \frac{c(x_i) + c(x_j)}{2} = c(x_i)+\frac{c(x_j) - c(x_i)}{2}.
\end{equation}

As such, we suggest to take the simpler construction given by interpolating the scale rather than the metric. In any case,
dissimilarities $p_{ij}$ constructed from the symmetric estimate $d_R(x_i, x_j)$ using $p_{ij} = h^p(d_R(x_i, x_j))$ are guaranteed to be symmetric, unlike the ones from the traditional pipeline $p_{ij} = h^p(R_{x_i}(x_j - x_i))$ that incorrectly rely on tangent rather than geodesic estimates of the distance between close points.

Our rigorous estimation significantly differs from the omnipresent approach in the literature of approximating $|\gamma(t)|$ by $M(x_i) \lVert x_j -x_i\rVert_2$ and $|\gamma(1-t)|$ by $M(x_j) \lVert x_j -x_i\rVert_2$\footnote{And often, as in t-SNE, $M(x_i) = \tfrac{1}{\sigma_i}$ so the asymmetric distance estimates $\tfrac{\lVert x_j - x_i\rVert_2}{\sigma_i}$ and $\tfrac{\lVert x_j - x_i\rVert_2}{\sigma_j}$ lead to asymmetric dissimilarities.}, as this leads for non-uniform metrics to asymmetric distance estimates, which violates the Riemannian assumption. 
We have thus provided a 
possible direction towards a theoretically grounded approach to fully symmetric manifold learning. 

While seemingly simple, this alternative approach is not readily compatible with modern efficient techniques. The reason is that this symmetrisation is not done at the end of the \textit{data construction} stage but inside it. This conflicts with the efficient sparse implementations that assume a fixed number $k$ of non-empty (non-sparse) weights connecting each node, provided by the non-symmetrised kNN graph. However, our symmetrisation idea would break this assumption inside the data construction stage, rather than at the end of it, requiring heavy code rewrite of each method and critically losing their efficiency. 

Additionally, while theoretically grounded, such a better theoretically-justified symmetric approach would, by construction, not possess asymmetry. Just like symmetrisation heuristics, it would thus not be able to reveal hidden information stemming from asymmetric relationships between points due to sampling (see Fig.~\ref{fig: US cities}).
Thus, we did not further investigate developing and efficiently implementing an end-to-end theoretically-justified fully symmetric pipeline to manifold learning. We may consider this in future work.

\subsection{Assumptions on the data manifold}

We do not use any major assumptions on the data manifold. As such, our pipeline and Finsler embedding methods can be applied to any sort of data where quantitative pairwise relationships can be computed.

In particular,
our pipeline and Finsler methods do not estimate or require the intrinsic data manifold dimensionality. The target dimension is chosen by the user according to the task, typically 2 or 3 for visualisation, which is the most common application of manifold learning. In our synthetic experiments we know the manifold dimension, but the method is agnostic to it. These controlled experiments allow us to examine the geometric preservation of the method, as we do not understand high-dimensional manifolds (what should a curved sampled 50 dimensional manifold intrinsically look like?). For visualisation of real data, we choose a two-dimensional symmetric subspace plus one designated asymmetry coordinate. The extra dimension convention was established by the existing Finsler manifold learning literature \cite{dages2025finsler}. Our embedding dimension is therefore a representation choice, not an assumption that the intrinsic dimension is known.

\section{Corrected update rule of Finsler MDS}
\label{sec: corrected update rule of finsler mds}

For completeness, we here present the Finsler MDS method. Due to small errors in the original derivation, we fully derive the accurate update rules, and prove convergence limits.

The Finsler MDS problem \cite{dages2025finsler} consists in minimising the Finsler stress objective for data dissimilarities $\boldsymbol{D}_{ij}\ge 0$
\begin{equation}
    \sigma_F^2(\boldsymbol{Y})
    =
    \sum_{i,j=1}^{N}
    w_{ij}\bigl(d_F(y_i,y_j)-\boldsymbol{D}_{ij}\bigr)^2,
\end{equation}
where $\boldsymbol{D}_{ii}=0$ and $\boldsymbol{W}$, with positive entries $\boldsymbol{W}_{ij} = w_{ij}\ge 0$ and $w_{ii}=0$, are optional weights that can give different importances to different pairs $ij$. In the original work, the weighting scheme is assumed to be symmetric $\boldsymbol{W}^\top = \boldsymbol{W}$, but the data dissimilarities $\boldsymbol{D}$ might not be.

In the original paper \cite{dages2025finsler}, there are some small errors in the derivation of the Finsler stress, in particular in its equations (7) and (9), and in turn the update rule (Proposition 2) is not fully accurate. We here provide the corrected update rule, with full derivation for it including a correct derivation of the Finsler stress.

Throughout, \(\langle \boldsymbol{A},\boldsymbol{B}\rangle_F=\mathrm{tr}(\boldsymbol{A}^\top \boldsymbol{B})\) denotes
the Frobenius inner product, $d_{ij}(\boldsymbol{Y}) = \lVert y_j - y_i\rVert_2$ is the embedding Euclidean distance, and write $r_{ij}(\boldsymbol{Y}) = \omega^\top(y_j-y_i)$. The data dissimilarities can be decomposed into symmetric $\boldsymbol{D}_S = \tfrac{\boldsymbol{D} + \boldsymbol{D}^\top}{2}$ and antisymmetric $\boldsymbol{D}_A = \frac{\boldsymbol{D}-\boldsymbol{D}^\top}{2}$ parts.
Denote $\boldsymbol{V}\in\mathbb{R}^{N\times N}$ as
\begin{equation}
    \boldsymbol{V}_{ij}
    =
    \begin{cases}
        -w_{ij}, & i\ne j,\\[2pt]
        \displaystyle\sum_{k\ne i}w_{ik}, & i=j,
    \end{cases}
\end{equation}
Let $\boldsymbol{B}(\boldsymbol{Y})\in\mathbb{R}^{N\times N}$ be
\begin{equation}
    \boldsymbol{B}_{ij}(\boldsymbol{Y})
    =
    \begin{cases}
        \displaystyle
        -\frac{
            w_{ij}
           (\boldsymbol{D}_{\mathrm{S}})_{ij}
        }{
            d_{ij}(\boldsymbol{Y})
        },
        & \text{if }\, i\ne j \text{ and } d_{ij}(\boldsymbol{Y})>0,\\[10pt]
        0,
        & 
           \text{if }\, i\ne j \text{ and } d_{ij}(\boldsymbol{Y})=0,\\[4pt]
        \displaystyle
        -\sum_{k\ne i}\boldsymbol{B}_{ik}(\boldsymbol{Y}),
        & \text{if }\, i=j.
    \end{cases}
\end{equation}
This matrix, depending on the embedding $\boldsymbol{Y}$, is symmetric $\boldsymbol{B}(\boldsymbol{Y})^\top = \boldsymbol{B}(\boldsymbol{Y})$ and satisfies
$\boldsymbol{B}(\boldsymbol{Z})\boldsymbol{1}_N=\boldsymbol{0}$ and 
$\boldsymbol{1}_N^\top\boldsymbol{B}(\boldsymbol{Z})=\boldsymbol{0}$, where $\boldsymbol{1}_N$ is the $N$-dimensional vector with all entries equal to $1$.
For any matrix $\boldsymbol{A}$, we denote its Moore-Penrose pseudo-invert as $\boldsymbol{A}^\dagger$.
Last,
let $\boldsymbol{C} = \bigl[\boldsymbol{W}\odot(\boldsymbol{D}-\boldsymbol{D}^\top)\bigr] \boldsymbol{1}_N \omega^\top$ where $\odot$ is the Hadamard product.

Denoting $\boldsymbol{Y}^{(k)}$ the $k$-th iterate of the embedding, the correct Finsler MDS update rule is given by the following.

\begin{proposition}(Finsler SMACOF -- Corrected)
    \label{prop: Finsler SMACOF}
    Under the assumptions above, a majorisation–minimisation update for the Finsler stress is
    $$ \boldsymbol{Y}^{(k+1)} = 
    \boldsymbol{V}^\dagger
    \left[
        \boldsymbol{B}(\boldsymbol{Y}^{(k)}){\boldsymbol{Y}^{(k)}}
        -
        \frac{1}{2}\boldsymbol{C}
    \right]
    \left(
        \boldsymbol{I}_m
        -
        \frac{\omega\omega^\top}{1+\lVert \omega\rVert_2^2}
    \right).
    $$
\end{proposition}

\begin{proof}
    The proof is based on first finding a simple quadratic majoriser which can then be minimised uniquely.
    
    We first rewrite the Finsler stress in (corrected) matrix form.
    Since $d_{ji}(\boldsymbol{Y})=d_{ij}(\boldsymbol{Y})$, $r_{ji}(\boldsymbol{Y})=-r_{ij}(\boldsymbol{Y})$,
    $(\boldsymbol{D}_{\mathrm{S}})_{ji}
    =(\boldsymbol{D}_{\mathrm{S}})_{ij}$, and
    $(\boldsymbol{D}_{\mathrm{A}})_{ji}
    =-(\boldsymbol{D}_{\mathrm{A}})_{ij}$, pairing both the terms $ij$ and $ji$ in the Finsler stress yields
    \begin{align}
        &(d_{ij}(\boldsymbol{Y})+r_{ij}(\boldsymbol{Y})-D_{ij})^2
        +(d_{ji}(\boldsymbol{Y})+r_{ji}(\boldsymbol{Y})-D_{ji})^2\\
        &\qquad=(d_{ij}(\boldsymbol{Y})+r_{ij}(\boldsymbol{Y})-D_{ij})^2
        +(d_{ij}(\boldsymbol{Y})-r_{ij}(\boldsymbol{Y})-D_{ji})^2\\
        &\qquad=
        \bigl[
            (d_{ij}(\boldsymbol{Y})-(\boldsymbol{D}_{\mathrm{S}})_{ij})
            +
            (r_{ij}(\boldsymbol{Y})-(\boldsymbol{D}_{\mathrm{A}})_{ij})
        \bigr]^2\nonumber\\
        &\qquad\quad+
        \bigl[
            (d_{ij}(\boldsymbol{Y})-(\boldsymbol{D}_{\mathrm{S}})_{ij})
            -
            (r_{ij}(\boldsymbol{Y})-(\boldsymbol{D}_{\mathrm{A}})_{ij})
        \bigr]^2\\
        &\qquad=
        2
        \bigl(d_{ij}(\boldsymbol{Y})-(\boldsymbol{D}_{\mathrm{S}})_{ij}\bigr)^2
        +
        2
        \bigl(r_{ij}(\boldsymbol{Y})-(\boldsymbol{D}_{\mathrm{A}})_{ij}\bigr)^2.
    \end{align}
    Consequently,
    \begin{equation}
        \sigma_F^2(\boldsymbol{Y})
        =
        2
        \sum_{i<j}w_{ij}
        \left[
            \bigl(
                d_{ij}(\boldsymbol{Y})
                -
                (\boldsymbol{D}_{\mathrm{S}})_{ij}
            \bigr)^2
            +
            \bigl(
                r_{ij}(\boldsymbol{Y})
                -
                (\boldsymbol{D}_{\mathrm{A}})_{ij}
            \bigr)^2
        \right].
        \label{eq:paired-stress}
    \end{equation}

    Let
    \begin{equation}
        \Psi(\boldsymbol{Y})
        =
        \sum_{i<j}
        w_{ij}
        (\boldsymbol{D}_{\mathrm{S}})_{ij}
        d_{ij}(\boldsymbol{Y})
        = \mathrm{tr}\big(\boldsymbol{Y}^\top\boldsymbol{B}(\boldsymbol{Y})\boldsymbol{Y}\big).
    \end{equation}
    Using matrix notations for terms in Eq.~\ref{eq:paired-stress}, note the following for each terms
    \begin{align}
        \sum_{i,j}w_{ij}D_{ij}^{\,2}
        &=
        \langle
            \boldsymbol{W},
            \boldsymbol{D}\odot\boldsymbol{D}
        \rangle_F
        = \mathrm{tr}\big((\boldsymbol{W}\odot \boldsymbol{D})\boldsymbol{D}^\top\big),\\
        \sum_{i<j}w_{ij}d_{ij}(\boldsymbol{Y})^2
        &=
        \mathrm{tr}(\boldsymbol{Y}^\top\boldsymbol{V}\boldsymbol{Y}),\\
        \sum_{i<j}w_{ij}r_{ij}(\boldsymbol{Y})^2
        &=
        \mathrm{tr}\bigl(
            \boldsymbol{Y}^\top\boldsymbol{V}\boldsymbol{Y}
            \omega\omega^\top
        \bigr),
        \\
        \sum_{i<j}
        w_{ij}(\boldsymbol{D}_{\mathrm{A}})_{ij}
        r_{ij}(\boldsymbol{Y})
        &=
        -\frac{1}{2}
        \langle\boldsymbol{C},\boldsymbol{Y}\rangle_F = 
        -\frac{1}{2}
        \mathrm{tr}(\boldsymbol{C}\boldsymbol{Y}^\top).
    \end{align}
    Therefore, the correct matrix form expression for the Finsler stress is given by
    \begin{align}
        \sigma_F^2(\boldsymbol{Y})
        &=
        2
        \mathrm{tr}(\boldsymbol{Y}^\top\boldsymbol{V}\boldsymbol{Y})
        +
        2
        \mathrm{tr}\bigl(
            \boldsymbol{Y}^\top\boldsymbol{V}\boldsymbol{Y}
            \omega\omega^\top
        \bigr)
        +
        2\mathrm{tr}(\boldsymbol{C}\boldsymbol{Y}^\top) \nonumber\\
        &\qquad-
        4\mathrm{tr}\big(\boldsymbol{Y}^\top\boldsymbol{B}(\boldsymbol{Y})\boldsymbol{Y}\big)
        +
        \mathrm{tr}\big((\boldsymbol{W}\odot \boldsymbol{D})\boldsymbol{D}^\top\big)
        \\
        &=
        2
        \mathrm{tr}(\boldsymbol{Y}^\top\boldsymbol{V}\boldsymbol{Y}(\boldsymbol{I}_m + \omega\omega^\top))
        +
        2\mathrm{tr}(\boldsymbol{C}\boldsymbol{Y}^\top)
        -
        4\Psi(\boldsymbol{Y})
        +
        \mathrm{tr}\big((\boldsymbol{W}\odot \boldsymbol{D})\boldsymbol{D}^\top\big).
        \label{eq:matrix-stress}
    \end{align}

    Quadratic majorisation of the Finsler stress is then achieved by finding an linear minimiser of $\Psi(\boldsymbol{Y})$.

    For any alternative embedding $\boldsymbol{Z}$, using the Cauchy-Schwarz inequality whenever \(d_{ij}(\boldsymbol{Z})>0\), we have
    \begin{equation}
        d_{ij}(\boldsymbol{Y})
        \ge
        \frac{
            (\boldsymbol{y}_i-\boldsymbol{y}_j)^\top
            (\boldsymbol{z}_i-\boldsymbol{z}_j)
        }{
            d_{ij}(\boldsymbol{Z})
        },
    \end{equation}
    with equality when \(\boldsymbol{Y}=\boldsymbol{Z}\).
    For \(d_{ij}(\boldsymbol{Z})=0\), we use the supporting inequality
    \(d_{ij}(\boldsymbol{Y})\ge0\), which also has equality at
    \(\boldsymbol{Y}=\boldsymbol{Z}\).
    Since \(w_{ij}(\boldsymbol{D}_{\mathrm{S}})_{ij}\ge0\), summing
    these inequalities yields
    \begin{equation}
        \Psi(\boldsymbol{Y}) = \mathrm{tr}(\boldsymbol{Y}^\top \boldsymbol{B}(\boldsymbol{Y})\boldsymbol{Y})
        \ge
        \mathrm{tr}\bigl(
            \boldsymbol{Y}^\top\boldsymbol{B}(\boldsymbol{Z})\boldsymbol{Z}
        \bigr) 
        \qquad \forall\, \boldsymbol{Y}, \boldsymbol{Z}.
    \end{equation}
    Equivalently, as $\mathrm{tr}\bigl(\boldsymbol{Y}^\top\boldsymbol{B}(\boldsymbol{Z})\boldsymbol{Z}\bigr) = \Psi(\boldsymbol{Z}) + \langle \boldsymbol{B}(\boldsymbol{Z})\boldsymbol{Z}, \boldsymbol{Y} - \boldsymbol{Z}\rangle_F$, we have
    \begin{equation}
        \boldsymbol{B}(\boldsymbol{Z})\boldsymbol{Z}
        \in
        \partial\Psi(\boldsymbol{Z}),
        \label{eq:selected-subgradient}
    \end{equation}
    where \(\partial\Psi\) is the convex subdifferential with
    respect to the Frobenius inner product.

    We can now derive the quadratic majoriser $g(\boldsymbol{Y}, \boldsymbol{Z})$ of the Finsler stress. Indeed, it follows from Eq.~\ref{eq:matrix-stress} that
    \begin{equation}
        g(\boldsymbol{Y},\boldsymbol{Z})
        =         
        2
        \mathrm{tr}(\boldsymbol{Y}^\top\boldsymbol{V}\boldsymbol{Y}(\boldsymbol{I}_m + \omega\omega^\top))
        +
        2\mathrm{tr}(\boldsymbol{C}\boldsymbol{Y}^\top)
        -
        4\mathrm{tr}(\boldsymbol{Y}^\top \boldsymbol{B}(\boldsymbol{Z})\boldsymbol{Z})
        +
        \mathrm{tr}\big((\boldsymbol{W}\odot \boldsymbol{D})\boldsymbol{D}^\top\big).
        \label{eq:majoriser}
    \end{equation}
    is a quadratic majoriser of the stress:
    \begin{equation}
        \sigma_F^2(\boldsymbol{Y})
        \le
        g(\boldsymbol{Y},\boldsymbol{Z}) \quad \forall\, \boldsymbol{Y},\boldsymbol{Z},
        \qquad
        \text{and}
        \qquad
        g(\boldsymbol{Z},\boldsymbol{Z})
        =
        \sigma_F^2(\boldsymbol{Z})
        \quad \forall\, \boldsymbol{Z}.
        \label{eq:majorization-properties}
    \end{equation}

    We can now find the Finsler SMACOF update rule by minimising the majoriser.

    Let \(\mathcal{I}_1,\ldots,\mathcal{I}_q\) be the connected
    components of the graph with an edge between \(i\) and \(j\)
    whenever \(w_{ij}>0\).
    We fix the translation freedom by working on the subspace
    \begin{equation}
        \mathcal{H}
        =
        \left\{
            \boldsymbol{Y}\in\mathbb{R}^{N\times m}:
            \sum_{i\in\mathcal{I}_\ell}\boldsymbol{y}_i
            =
            \boldsymbol{0}
            \text{ for every }\ell
        \right\}.
    \end{equation}
    Then, the unique minimiser of \(g(\cdot,\boldsymbol{Z})\) on
    \(\mathcal{H}\) is given by
    \begin{equation}
        \boldsymbol{Y}^*
        =
        \boldsymbol{V}^\dagger
        \left[
            \boldsymbol{B}(\boldsymbol{Z})\boldsymbol{Z}
            -
            \frac{1}{2}\boldsymbol{C}
        \right]
        \bigl(
            \boldsymbol{I}_m
            +
            \omega\omega^\top
        \bigr)^{-1},
        \label{eq:smacof-update}
    \end{equation}
    providing the Finsler SMACOF update rule.
    The unconstrained minimisers are obtained by independently
    translating the connected components.

    To prove this, 
    differentiating Eq.~\ref{eq:majoriser} gives
    \begin{equation}
        \nabla_{\boldsymbol{Y}}g(\boldsymbol{Y},\boldsymbol{Z})
        =
        4
        \boldsymbol{V}\boldsymbol{Y}(\boldsymbol{I}_m + \omega \omega^\top)
        +
        2\boldsymbol{C}
        -
        4
        \boldsymbol{B}(\boldsymbol{Z})\boldsymbol{Z}.
    \end{equation}
    Hence the first-order optimality equation is
    \begin{equation}
        \boldsymbol{V}\boldsymbol{Y}^*(\boldsymbol{I}_m + \omega\omega^\top)
        =
        \boldsymbol{B}(\boldsymbol{Z})\boldsymbol{Z}
        -
        \frac{1}{2}\boldsymbol{C}.
        \label{eq:normal-equation}
    \end{equation}
    
    For each connected component \(\mathcal{I}_\ell\), let
    \(\boldsymbol{u}_\ell\in\mathbb{R}^N\) be its indicator
    vector.
    The symmetry and componentwise zero row sums of
    \(\boldsymbol{B}(\boldsymbol{Z})\) give
    \(\boldsymbol{u}_\ell^\top\boldsymbol{B}(\boldsymbol{Z})
    =\boldsymbol{0}\).
    Moreover, the skew-symmetry of
    \(\boldsymbol{W}\odot(\boldsymbol{D}-\boldsymbol{D}^\top)\) gives
    \(\boldsymbol{u}_\ell^\top\boldsymbol{C}=\boldsymbol{0}\).
    Thus each column of the right-hand side of
    Eq.~\ref{eq:normal-equation} lies in the range of
    \(\boldsymbol{V}\).
    
    Since \(\boldsymbol{I}_m + \omega\omega^\top\) is positive definite, multiplication
    on the left by \(\boldsymbol{V}^\dagger\) and on the right by
    \((\boldsymbol{I}_m + \omega\omega^\top)^{-1}\) gives Eq.~\ref{eq:smacof-update}.
    The restriction of \(\boldsymbol{V}\) to the componentwise
    centred subspace is positive definite, so this minimiser
    is unique on \(\mathcal{H}\).
    Adding any matrix \(\boldsymbol{T}\) with
    \(\boldsymbol{V}\boldsymbol{T}=\boldsymbol{0}\) leaves the objective
    unchanged and generates the remaining minimisers.

    To conclude, we use the Sherman-Morrison formula giving
    \begin{equation}
        \left(\boldsymbol{I}_m + \omega\omega^\top\right)^{-1} = 
        \boldsymbol{I}_m
        -
        \frac{\omega\omega^\top}{1+\lVert \omega\rVert_2^2}.
    \end{equation}

\end{proof}

Note that the update rule in fact does not require costly vectorisation of matrices or costly pseudo-inversion solvers of $N^2\times N^2$ matrices, as was done in the original work \cite{dages2025finsler} and its official code repository\footnote{In this work, we used the official vectorised implementation, rather than the fixed cheaper one, since we detected the issues with the original Finsler MDS \cite{dages2025finsler} towards the late stages of the camera-ready deadline. Results are similar, albeit slightly improved and more stable with the fixed version.}. 
In fact, it does not even require pseudo-inverts in the common uniform weights $w_{ij}$ case, see below Eq.~\ref{eq:uniform-update}.
Nevertheless, we can still provide the corrected update rule in vectorised format similar to \cite{dages2025finsler}.
Using column-stacking vectorisation, define
\begin{equation}
    \boldsymbol{K}=(\boldsymbol{I}_m+\omega\omega^\top)\otimes\boldsymbol{V}.
\end{equation}
Then,
using the factorisation
\begin{equation}
    \big((\boldsymbol{I}_m+\omega\omega^\top)\otimes\boldsymbol{V}\big)^\dagger
    =
    (\boldsymbol{I}_m+\omega\omega^\top)^{-1}\otimes\boldsymbol{V}^\dagger,
\end{equation}
Eq.~\ref{eq:smacof-update} can be equivalently written in matrix form as
\begin{equation}
    \mathrm{vec}(\boldsymbol{Y}^*)
    =
    \boldsymbol{K}^\dagger
    \mathrm{vec}\left(
        \boldsymbol{B}(\boldsymbol{Z})\boldsymbol{Z}
        -
        \frac{1}{2}\boldsymbol{C}
    \right),
\end{equation}

Note that for uniform weights \(w_{ij}=1\) when \(i\ne j\),
\begin{equation}
    \boldsymbol{V}
    =
    N\boldsymbol{I}_N-\boldsymbol{1}_N\boldsymbol{1}_N^\top.
\end{equation}
Since the right-hand side of Eq.~\ref{eq:normal-equation}
is centred, the update becomes for uniform weights
\begin{equation}
    \boldsymbol{Y}^*
    =
    \frac1N
    \left[
        \boldsymbol{B}(\boldsymbol{Z})\boldsymbol{Z}
        -
        \frac{1}{2}\boldsymbol{C}
    \right]
    \left(
        \boldsymbol{I}_m
        -
        \frac{\omega\omega^\top}{1+\lVert \omega\rVert_2^2}
    \right).
    \label{eq:uniform-update}
\end{equation}

Thanks to our corrected update rules, we can prove the following convergence results, which the original work \cite{dages2025finsler} did not provide.

\begin{theorem}
    Let \(\boldsymbol{Y}^{(0)}\in\mathcal{H}\), and generate the sequence
    \(\{\boldsymbol{Y}^{(k)}\}_{k\ge0}\) by
    \begin{equation}
        \boldsymbol{Y}^{(k+1)}
        =
        \boldsymbol{V}^\dagger
        \left[
            \boldsymbol{B}(\boldsymbol{Y}^{(k)})\boldsymbol{Y}^{(k)}
            -
            \frac{1}{2}\boldsymbol{C}
        \right]\left(\boldsymbol{I}_m - \frac{\omega\omega^\top}{1+\lVert \omega\rVert_2^2}\right).
        \label{eq:iteration}
    \end{equation}
    Then the following statements hold.
    \begin{enumerate}
        \item The stress values decrease and converge:
        \[
            \sigma_F^2(\boldsymbol{Y}^{(k+1)})
            \le
            \sigma_F^2(\boldsymbol{Y}^{(k)}).
        \]
    
        \item The sequence \(\{\boldsymbol{Y}^{(k)}\}\) is bounded, and
        \[
            \sum_{k=0}^{\infty}
            \|\boldsymbol{Y}^{(k+1)}-\boldsymbol{Y}^{(k)}\|_F^2<\infty.
        \]
        In particular,
        \[
            \boldsymbol{Y}^{(k+1)}-\boldsymbol{Y}^{(k)}
            \longrightarrow
            \boldsymbol{0}.
        \]
    
        \item Every accumulation point \(\boldsymbol{Y}^\star\)
        satisfies
        \[
            \boldsymbol{V}\boldsymbol{Y}^\star(\boldsymbol{I}_m + \omega\omega^\top)
            +
            \frac{1}{2}\boldsymbol{C}
            \in
            \partial\Psi(\boldsymbol{Y}^\star).
        \]
    
        \item If an accumulation point \(\boldsymbol{Y}^\star\)
        satisfies
        $
            \|\boldsymbol{y}_i^\star
            -
            \boldsymbol{y}_j^\star\|_2>0
        $
        whenever
        $
            w_{ij}(\boldsymbol{D}_{\mathrm{S}})_{ij}>0,
        $
        then
        \[
            \boldsymbol{V}\boldsymbol{Y}^\star(\boldsymbol{I}_m + \omega\omega^\top)
            +
            \frac{1}{2}\boldsymbol{C}
            -
            \boldsymbol{B}(\boldsymbol{Y}^\star)\boldsymbol{Y}^\star
            =
            \boldsymbol{0}.
        \]
        Equivalently,
        \(\nabla\sigma_F^2(\boldsymbol{Y}^\star)=\boldsymbol{0}\).
    \end{enumerate}
\end{theorem}

\begin{proof}
    The majorisation properties and the minimising property
    of the update imply
    \begin{equation}
        \sigma_F^2(\boldsymbol{Y}^{(k+1)})
        \le
        g(\boldsymbol{Y}^{(k+1)},\boldsymbol{Y}^{(k)})
        \le
        g(\boldsymbol{Y}^{(k)},\boldsymbol{Y}^{(k)})
        =
        \sigma_F^2(\boldsymbol{Y}^{(k)}).
        \label{eq:descent-chain}
    \end{equation}
    Since the stress is positive, its values converge.
    
    To establish boundedness, consider an edge with \(w_{ij}>0\).
    Eq.~\ref{eq:paired-stress} and monotonicity give
    \begin{equation}
        \|\boldsymbol{y}_i^k-\boldsymbol{y}_j^k\|_2
        \le
        (\boldsymbol{D}_{\mathrm{S}})_{ij}
        +
        \sqrt{
            \frac{\sigma_F^2(\boldsymbol{Y}^{(0)})}{2w_{ij}}
        }.
        \label{eq:edge-bound}
    \end{equation}
    Paths in each finite connected component therefore bound
    all within-component distances.
    The componentwise centering constraints then bound the
    coordinates.
    An isolated vertex is fixed at the origin by these
    constraints.
    Hence \(\{\boldsymbol{Y}^{(k)}\}\) is bounded.
    
    Let
    \(\boldsymbol{\Delta}_k=\boldsymbol{Y}^{(k+1)}-\boldsymbol{Y}^{(k)}\).
    Because \(g(\cdot,\boldsymbol{Y}^{(k)})\) is quadratic and its
    gradient vanishes at \(\boldsymbol{Y}^{(k+1)}\),
    \begin{equation}
        g(\boldsymbol{Y}^{(k)},\boldsymbol{Y}^{(k)})
        -
        g(\boldsymbol{Y}^{(k+1)},\boldsymbol{Y}^{(k)})
        =
        2
        \mathrm{tr}\bigl(
            \boldsymbol{\Delta}_k^\top
            \boldsymbol{V}\boldsymbol{\Delta}_k(\boldsymbol{I}_m + \omega\omega^\top)
        \bigr).
    \end{equation}
    Combining this identity with Eq.~\ref{eq:descent-chain} yields
    \begin{equation}
        \sigma_F^2(\boldsymbol{Y}^{(k)})
        -
        \sigma_F^2(\boldsymbol{Y}^{(k+1)})
        \ge
        2\mathrm{tr}\bigl(
            \boldsymbol{\Delta}_k^\top
            \boldsymbol{V}\boldsymbol{\Delta}_k(\boldsymbol{I}_m + \omega\omega^\top)
        \bigr).
        \label{eq:sufficient-decrease}
    \end{equation}
    
    If \(\boldsymbol{V}\ne\boldsymbol{0}\), let
    \(\lambda_+(\boldsymbol{V})\) be its smallest positive
    eigenvalue.
    Since \(\boldsymbol{\Delta}_k\in\mathcal{H}\) and
    \(\boldsymbol{I}_m + \omega\omega^\top\) is positive definite,
    \begin{equation}
        \mathrm{tr}\bigl(
            \boldsymbol{\Delta}_k^\top
            \boldsymbol{V}\boldsymbol{\Delta}_k(\boldsymbol{I}_m + \omega\omega^\top)
        \bigr)
        \ge
        \lambda_+(\boldsymbol{V})
        \lambda_{\min}(\boldsymbol{I}_m + \omega\omega^\top)
        \|\boldsymbol{\Delta}_k\|_F^2.
    \end{equation}
    Summing Eq.~\ref{eq:sufficient-decrease} over \(k\) gives
    \begin{equation}
        \sum_{k=0}^{\infty}\|\boldsymbol{\Delta}_k\|_F^2
        \le
        \frac{
            \sigma_F^2(\boldsymbol{Y}^{(0)})
        }{
            2\lambda_+(\boldsymbol{V})
            \lambda_{\min}(\boldsymbol{I}_m + \omega\omega^\top)
        }
        <\infty.
    \end{equation}
    Thus \(\boldsymbol{\Delta}_k\to\boldsymbol{0}\).
    If \(\boldsymbol{V}=\boldsymbol{0}\), all weights vanish and the
    centred sequence is identically zero, so the conclusions
    hold directly.
    
    Let \(\boldsymbol{Y}^{k_\ell}\to\boldsymbol{Y}^\star\) be a
    convergent subsequence.
    Since \(\boldsymbol{\Delta}_k\to\boldsymbol{0}\), also
    \(\boldsymbol{Y}^{k_\ell+1}\to\boldsymbol{Y}^\star\).
    Define
    \begin{equation}
        \boldsymbol{G}_k
        =
        \boldsymbol{B}(\boldsymbol{Y}^{(k)})\boldsymbol{Y}^{(k)}
        \in
        \partial\Psi(\boldsymbol{Y}^{(k)}).
    \end{equation}
    The normal equation gives
    \begin{equation}
        \boldsymbol{G}_{k_\ell}
        =
        \boldsymbol{V}\boldsymbol{Y}^{k_\ell+1}(\boldsymbol{I}_m + \omega\omega^\top)
        +
        \frac{1}{2}\boldsymbol{C}
        \longrightarrow
        \boldsymbol{V}\boldsymbol{Y}^\star(\boldsymbol{I}_m + \omega\omega^\top)
        +
        \frac{1}{2}\boldsymbol{C}.
    \end{equation}
    For every \(\boldsymbol{T}\in\mathbb{R}^{N\times m}\), the
    subgradient inequality reads
    \begin{equation}
        \Psi(\boldsymbol{T})
        \ge
        \Psi(\boldsymbol{Y}^{k_\ell})
        +
        \langle
            \boldsymbol{G}_{k_\ell},
            \boldsymbol{T}-\boldsymbol{Y}^{k_\ell}
        \rangle_F.
    \end{equation}
    Passing to the limit proves
    \begin{equation}
        \boldsymbol{V}\boldsymbol{Y}^\star(\boldsymbol{I}_m + \omega\omega^\top)
        +
        \frac{1}{2}\boldsymbol{C}
        \in
        \partial\Psi(\boldsymbol{Y}^\star).
        \label{eq:limit-inclusion}
    \end{equation}
    
    If \(d_{ij}(\boldsymbol{Y}^\star)>0\) for every pair with
    \(w_{ij}(\boldsymbol{D}_{\mathrm{S}})_{ij}>0\), then
    \(\Psi\) is differentiable at \(\boldsymbol{Y}^\star\), with
    \begin{equation}
        \nabla\Psi(\boldsymbol{Y}^\star)
        =
        \boldsymbol{B}(\boldsymbol{Y}^\star)\boldsymbol{Y}^\star.
    \end{equation}
    Equation Eq.~\ref{eq:limit-inclusion} consequently becomes
    the stationarity equation
    \begin{equation}
        \boldsymbol{V}\boldsymbol{Y}^\star(\boldsymbol{I}_m + \omega\omega^\top)
        +
        \frac{1}{2}\boldsymbol{C}
        -
        \boldsymbol{B}(\boldsymbol{Y}^\star)\boldsymbol{Y}^\star
        =
        \boldsymbol{0},
    \end{equation}
    or, equivalently,
    \begin{equation}
        \nabla\sigma_F^2(\boldsymbol{Y}^\star)
        =
        4
        \boldsymbol{V}\boldsymbol{Y}^\star(\boldsymbol{I}_m + \omega\omega^\top)
        +
        2\boldsymbol{C}
        -
        4
        \boldsymbol{B}(\boldsymbol{Y}^\star)\boldsymbol{Y}^\star
        =
        \boldsymbol{0}.
    \end{equation}
\end{proof}

\section{Update rules of Finsler t-SNE and Finsler UMAP}
\label{sec: proofs finsler tsne umap}

We here prove Theorems \ref{th: finsler tsne update rule} and \ref{th: finsler umap update rule}.
An important preliminary result to carry on the calculations is the following.

\begin{theorem}[Gradient canonical Finsler distances]
    \label{th: gradient canonical Finsler distances}
    The gradients of the canonical Finsler distance are given by $\frac{\partial d_{F^C}}{\partial x} (x,y) = \frac{x-y}{\lVert y-x\rVert_2} - \omega$ and $\frac{\partial d_{F^C}}{\partial y} (x,y) = \frac{y-x}{\lVert y-x\rVert_2} + \omega$. In particular, $\frac{\partial d_{F^C}}{\partial x} = - \frac{\partial d_{F^C}}{\partial y}$.
\end{theorem}

Although the result might look trivial as the proof is immediate, the fact that the gradient of the asymmetric canonical Randers distance shares the same antisymmetry as the gradient of the Euclidean distance is remarkable. It is another particularity of this Finsler metric making it an asymmetric analogue to the Euclidean space that was not observed in \cite{dages2025finsler}. We can then explicitly compute the gradients for the Finsler t-SNE and Finsler UMAP update rules.

We first derive the update rule for Finsler t-SNE.

\begin{proof}[Proof (Theorem \ref{th: finsler tsne update rule} -- Finsler t-SNE)]
    For compactness, denote $d_{ij}^F = d_{F^C}(y_i, y_j)$ to be the asymmetric Finsler distance in the canonical Randers space between the embedded points $y_i$ and $y_j$. Following the proof of standard Euclidean t-SNE, we write
    \begin{equation}
        \frac{\partial \mathcal{L}}{\partial y_i} 
            = \sum_j \frac{\partial \mathcal{L}}{\partial d_{ij}^F} \frac{\partial d_{ij}^F}{\partial y_i} + \frac{\partial \mathcal{L}}{\partial d_{ji}^F}\frac{\partial d_{ji}^F}{\partial y_i}.
    \end{equation}
    Using Theorem \ref{th: gradient canonical Finsler distances}, we have $\tfrac{\partial d_{ij}^F}{\partial y_i} = \tfrac{y_i - y_j}{\lVert y_i - y_j\rVert_2} - \omega$ and $\tfrac{\partial d_{ji}^F}{\partial y_i} = \tfrac{y_i - y_j}{\lVert y_i - y_j\rVert_2} + \omega$. We then have
    \begin{align}
        \frac{\partial \mathcal{L}}{\partial y_i} 
        = &\sum\limits_j \left(\frac{\partial \mathcal{L}}{\partial d_{ij}^F} + \frac{\partial \mathcal{L}}{\partial d_{ji}^F}\right) \frac{y_i - y_j}{\lVert y_i-y_j \rVert_2} 
        + \sum\limits_j \left(-\frac{\partial \mathcal{L}}{\partial d_{ij}^F} + \frac{\partial \mathcal{L}}{\partial d_{ji}^F}\right) \omega.
    \end{align}
    We now focus on the term $\tfrac{\partial \mathcal{L}}{\partial d_{ij}^F}$, which behaves exactly like its counterpart in standard Euclidean t-SNE. Nevertheless, we redo the derivations for completeness. Denoting the denominator $Z = \sum_{k\neq l} (t_{kl}^F)^{\frac{\nu+1}{2}}$ in $q_{ij}^F$, we have
    \begin{align}
        \frac{\partial \mathcal{L}}{\partial d_{ij}^F} 
        &= - \sum\limits_{k\neq l} p_{kl} \Bigg[ \frac{\partial}{\partial d_{ij}^F}\left(-\frac{\nu +1}{2}\ln\left(1+\frac{(d_{kl}^F)^2}{\nu}\right)\right) 
        - \frac{\partial}{\partial d_{ij}^F}(\ln(Z))\Bigg] \\
        &= -\sum\limits_{k\neq l} p_{kl} \Bigg[-\frac{\nu+1}{2}\frac{1}{1+\frac{(d_{kl}^F)^2}{\nu}}\frac{2d_{kl}^F}{\nu}\frac{\partial d_{kl}^F}{\partial d_{ij}^F} - \frac{1}{Z}\frac{\partial Z}{\partial d_{ij}^F}\Bigg].
    \end{align}
    Since $\tfrac{\partial d_{kl}^F}{\partial d_{ij}^F} = 1$ if $(i,j)=(k,l)$ and $0$ otherwise, $\tfrac{\partial Z}{\partial d_{ij}^F} = -\tfrac{\nu+1}{2}\big(1+\tfrac{(d_{ij}^F)^2}{\nu}\big)^{-\frac{\nu+3}{2}}\frac{2d_{ij}^F}{\nu}$, and $\sum_{k\neq l}p_{kl} =1$, we have
     \begin{align}
        \frac{\partial \mathcal{L}}{\partial d_{ij}^F} 
        &= \frac{\nu+1}{\nu}\frac{p_{ij}d_{ij}^F}{1+\frac{(d_{ij}^F)^2}{\nu}} 
        - \sum\limits_{k\neq l} p_{kl} \frac{\nu+1}{\nu}\frac{d_{ij}^F\left(1+\frac{(d_{ij}^F)^2}{\nu}\right)^{-\frac{\nu+3}{2}}}{Z} \\
        &= \frac{\nu+1}{\nu}(p_{ij} - q_{ij}^F) d_{ij}^F \left(1+\frac{(d_{ij}^F)^2}{\nu}\right)^{-1}.
    \end{align}
    Returning to $\tfrac{\partial \mathcal{L}}{\partial y_i}$, we get the desired result
    \begin{align*}
        \frac{\partial \mathcal{L}}{\partial y_i} &= 
        \frac{\nu+1}{\nu} \sum\limits_j \Bigg[ (p_{ij} - q_{ij}^F) t_{ij}^F \frac{d_{F^C}(y_i, y_j)}{\lVert y_j - y_i\rVert_2} 
        + (p_{ji} - q_{ji}^F)t_{ji}^F\frac{d_{F^C}(y_j, y_i)}{\lVert y_i - y_j\rVert_2} \Bigg](y_i - y_j)\\
        &+
        \frac{\nu+1}{\nu} \sum\limits_j \Bigg[ -(p_{ij} - q_{ij}^F) t_{ij}^F d_{F^C}(y_i, y_j)
        + (p_{ji} - q_{ji}^F)t_{ji}^F d_{F^C}(y_j, y_i) \Bigg]\omega.
    \end{align*}
    \qed
    
\end{proof}

We now derive the update rule for Finsler UMAP.

\begin{proof}[Proof (Theorem \ref{th: finsler umap update rule} -- Finsler UMAP)]
    The gradient of the objective is given by
    \begin{equation}
        \frac{\partial \mathcal{L}}{\partial y_i} = \sum\limits_{j} p_{ij} \cdot \frac{\partial c_{ij}^a}{\partial y_i} 
        + p_{ji}\frac{\partial c_{ji}^a}{\partial y_i} 
        + (1 - p_{ij}) \cdot \frac{\partial c_{ij}^r}{\partial y_i}
        + (1 - p_{ji}) \cdot \frac{\partial c_{ji}^r}{\partial y_i}.
    \end{equation}

    Denote $d_{ij}^F = d_{F^C}(y_i, y_j)$ the canonical Randers distance between.
    We first focus on the attractive forces $c_{ij}^a = -\ln(q_{ij}^F)$. We have
    \begin{align}
        \frac{\partial c_{ij}^a}{\partial y_i} 
        &= -\frac{1}{q_{ij}^F}\frac{\partial q_{ij}^F}{\partial d_{ij}^F}\frac{\partial d_{ij}^F}{\partial y_i} \\
        &= \frac{1+a(d_{ij}^F)^{2b}}{\big(1+a(d_{ij}^F)^{2b}\big)^2}2ab (d_{ij}^F)^{2b-1} \frac{\partial d_{ij}^F}{\partial y_i}.
    \end{align}
    Using Theorem \ref{th: gradient canonical Finsler distances}, we have
    \begin{equation}
        \frac{\partial c_{ij}^a}{\partial y_i} 
        = \frac{2ab\frac{(d_{ij}^F)^{2b-1}}{\lVert y_i - y_j\rVert_2}}{1+a(d_{ij}^F)^{2b}}(y_i - y_j) - \frac{2ab(d_{ij}^F)^{2b-1}}{1+a(d_{ij}^F)^{2b}}\omega.
    \end{equation}
    Likewise, a similar calculation provides
    \begin{align}
        \frac{\partial c_{ij}^a}{\partial y_j} 
        &= \frac{2ab\frac{(d_{ij}^F)^{2b-1}}{\lVert y_i - y_j\rVert_2}}{1+a(d_{ij}^F)^{2b}}(y_j - y_i) + \frac{2ab(d_{ij}^F)^{2b-1}}{1+a(d_{ij}^F)^{2b}}\omega\\
        &= - \frac{\partial c_{ij}^a}{\partial y_i}.
    \end{align}

    We now focus on the repulsive forces $c_{ij}^r = -\ln (1-q_{ij}^F)$. We have
    \begin{align}
        \frac{\partial c_{ij}^r}{\partial y_i} 
        &= \frac{1}{1-q_{ij}^F} \frac{\partial q_{ij}^F}{\partial d_{ij}^F}\frac{\partial d_{ij}^F}{\partial y_i}\\
        &= \frac{1}{1-q_{ij}^F} \frac{-1}{\big(1+a(d_{ij}^F)^{2b}\big)^2}2ab(d_{ij}^F)^{2b-1}\Bigg(\frac{y_i - y_j}{\lVert y_i-y_j\rVert_2} - \omega\Bigg) \\
        &= \frac{-2b (q_{ij}^F)^2}{(1-q_{ij}^F)}\Big((q_{ij}^F)^{-1}-1\Big)(d_{ij}^F)^{-1}\Bigg(\frac{y_i - y_j}{\lVert y_i - y_j\rVert_2} - \omega\Bigg) \\
        &= -2b q_{ij}^F (d_{ij}^F)^{-1}\Bigg(\frac{y_i - y_j}{\lVert y_i - y_j\rVert_2} - \omega\Bigg) \\
        &= \frac{-2b}{\Big(1+a(d_{ij}^F)^{2b}\Big) d_{ij}^F \lVert y_i - y_j \rVert_2} (y_i - y_j) + \frac{2b}{\Big(1+a(d_{ij}^F)^{2b}\Big) d_{ij}^F}\omega.
    \end{align}
    Likewise, a similar calculation provides
    \begin{align}
        \frac{\partial c_{ij}^r}{\partial y_j} 
        &= \frac{-2b}{\Big(1+a(d_{ij}^F)^{2b}\Big) d_{ij}^F \lVert y_i - y_j \rVert_2} (y_j - y_i) - \frac{2b}{\Big(1+a(d_{ij}^F)^{2b}\Big) d_{ij}^F}\omega \\
        &= -\frac{\partial c_{ij}^r}{\partial y_i}.
    \end{align}
    \qed

\end{proof}

\section{Asymmetric Finsler generalisation of spectral methods in manifold learning}
\label{sec: Asymmetric Finsler generalisation of spectral methods in manifold learning}

As mentioned in the main manuscript, we can generalise spectral methods to our Finsler setting. 
As discussed in Sec.~\ref{sec: spectral methods appendix}, spectral methods, e.g.\ Isomap \cite{tenenbaum2000global}, Laplacian eigenmaps \cite{belkin2003laplacian}, diffusion maps \cite{coifman2006diffusion}, and LLE \cite{roweis2000nonlinear}, can be reformulated as quadratic optimisation problems involving a minimisation on pairwise Euclidean distances between embedded points $y_i$. We propose to naively generalise such methods by replacing the Euclidean distance $\lVert \cdot\rVert_2$ involving embedded points $y_i$ by the Finsler distance $d_F$ of the embedding space, e.g.\ the canonical Randers space distance. While this simple switch theoretically generalises these methods to embed asymmetric data, it comes at a major downside: the solution of the new Finsler objectives are no longer given by a simple eigendecomposition. Indeed, the asymmetry breaks the kernel formulation: the Finsler generalisation of spectral methods is no longer given by spectral optimisation. Rather than an elegant closed-form solution, we must instead perform iterative minimisation, e.g.\ via gradient descent, which is slower and more susceptible to get stuck in local minima.
For this reason, we did not implement these generalisations and focused instead on those that do not strongly distort the solvers, like Finsler t-SNE and Finsler UMAP.

As examples, here are some generalised Finsler objectives $\mathcal{L}$ that can be minimised by gradient descent. By doing so we generalise spectral methods to asymmetry.

\paragraph{Finsler Isomap.}
Our proposed Finsler objective generalising Isomap to asymmetry is $\mathcal{L} = \sum\limits w_{ij}(d_F^2(y_i, y_j) - \boldsymbol{D}_{ij}^2)^2$. Note that unlike in the Euclidean case, we can include non-uniform weights $w_{ij}$, as this does not impact the solver\footnote{In the Euclidean case, non-uniform weights make the problem no longer spectral and iterative optimisation must be performed.}. This method is similar to Finsler MDS, by minimising the Finsler strain rather than the Finsler stress. As in the Euclidean case, both methods should lead to similar results, but this time Finsler Isomap is not solved with a closed-form solution. Instead, gradient descent should be performed. In contrast, Finsler MDS speeds up the descent via the Finsler SMACOF algorithm, which produces faster descent than raw gradient descent. As such, we recommend Finsler MDS rather than Finsler Isomap.

\paragraph{Finsler Laplacian eigenmaps.}
Generalising Laplacian eigenmaps to asymmetry, we aim to minimise the Finsler objective $\mathcal{L} = \tfrac{1}{2}\sum W_{ij} d_F^2(s_iy_i, s_jy_j)$, where $s_i$ is an optional scaling factor to account for different choices of Laplacians: $s_i = 1$ for the unnormalised graph Laplacian and $s_i = \tfrac{1}{\sqrt{A_i}}$ for the symmetric graph Laplacian. The optimisation should be constrained to $\boldsymbol{Y}^\top \boldsymbol{A}\boldsymbol{Y} = \boldsymbol{I}$, which can be enforced by Lagrangian optimisation. The final proposed objective is thus $\mathcal{L} = \tfrac{1}{2}\sum W_{ij} d_F^2(s_iy_i, s_jy_j) + \lambda \lVert \boldsymbol{Y}^\top \boldsymbol{A}\boldsymbol{Y} - \boldsymbol{I}\rVert_F^2$ for some scalar $\lambda$\footnote{\label{footnote: lagrangian multiplier as hyperparameter} Which can be treated as a hyperparameter.}. Gradient descent would be performed on the Lagrangian formulation. Note that in our Finsler case and unlike in the Euclidean case, the weights $W_{ij}$ may be asymmetric.

\paragraph{Finsler diffusion maps.}
Generalising diffusion maps to asymmetry, we aim to minimise the Finsler objective  $\mathcal{L} = \tfrac{1}{2}\sum \boldsymbol{P}_{ij}^t d_F^2(y_i, y_j)$, subject to the constraint $\boldsymbol{\boldsymbol{Y}^\top \boldsymbol{Y}} = \boldsymbol{\Lambda}^{2t}$. To enforce this, we propose a Lagrangian formulation 
objective $\mathcal{L} = \tfrac{1}{2}\sum \boldsymbol{P}_{ij}^t d_F^2(y_i, y_j) + \lambda \lVert \boldsymbol{\boldsymbol{Y}^\top \boldsymbol{Y}} - \boldsymbol{\Lambda}^{2t}\rVert_F^2$ for some scalar $\lambda$, as in Finsler Laplacian eigenmaps, to be minimised with gradient descent.
Note that in our Finsler and unlike in Euclidean case, the weights $W_{ij}$ may be asymmetric.

\paragraph{Finsler locally linear embeddings.}
Our proposed Finsler objective generalising the LLE method to asymmetry is
$\mathcal{L} = \sum d_F^2(y_i, \sum_j W_{ij}y_j)$ subject to centring and normalisation constraints on $\boldsymbol{Y}$, where local weights $W_{ij}$ are solution to the same constrained least-squares as in the Euclidean case. However, should the ambient space of the data be itself asymmetric, which would greatly justify the use of Finsler LLE over Euclidean LLE, for instance with a Finsler distance $d_F^{(X)}$ instead of the Euclidean distance, then the weights $W_{ij}$ should be the solution to the
following
constrained optimisation
problem:
$\min \big(d_F^{(X)}(x_i, \sum_{j\in\mathcal{N}(i)} W_{ij} x_j)\big)^2$ with $\sum_j W_{ij} = 1$ for all $i$. Both the Finsler LLE  main objective and the optional Finsler variant defining the local weights $W_{ij}$ can be solved via gradient descent.

\section{Finsler PHATE: an example of Finsler generalisation of methods without post hoc symmetrisation of data dissimilarities}
\label{sec: Finsler PHATE}

Some existing symmetric manifold learning methods do not seemingly have at first glance final data dissimilarities given by a symmetrisation fix, they are already symmetric by construction. As such, it might not seem reasonable a priori to generalise these approaches to Finsler asymmetric methods. However, in their data dissimilarity construction stage (P1), intermediate asymmetric concepts often naturally appear nevertheless. They are not explicitly symmetrised to compute the final data dissimilarities, other operations are used that lead to dropping their inherent asymmetry.
But, if the data construction stage (P1) invokes asymmetric concepts, this begs the question: why not account for, preserve, and reveal this natural asymmetry in the data? As such, it is actually meaningful to generalise such existing symmetric methods to Finsler methods.

In this section, we illustrate how to do so on the widely popular PHATE method.
As explained in Sec.~\ref{sec: update rules of phate}, PHATE is just MDS on high-dimensional features $\boldsymbol{U}_i$, called potentials. It does not involve any symmetrisation fix on its final dissimilarities, as they are given by Euclidean distances between potentials $\boldsymbol{\Delta}_{ij} = \lVert \boldsymbol{U}_j - \boldsymbol{U}_i\rVert_2$. Yet, the data dissimilarities construction phase (P1) involves several asymmetric concepts. The first is the kernel affinities $\widetilde{\boldsymbol{K}}$ resulting from adaptively locally rescaling the original distances. By default, this asymmetry is discarded by a symmetrisation fix, yielding the symmetrised kernel $\boldsymbol{K}$. The other emergence of asymmetry is the diffusion operator $\boldsymbol{P}$, given by row-normalisation of the affinities. Regardless of whether the affinities are symmetric or not, row-normalisation results in an asymmetric diffusion matrix $\boldsymbol{P}$, and in turn an asymmetric $t$-step diffusion $\boldsymbol{P}^t$ used to compute the potentials $\boldsymbol{U}_i = -\log \boldsymbol{P}_{i:}^t$.

A naive way to generalise PHATE to a Finsler asymmetric method would simply be to consider the potentials $\boldsymbol{U}$ as the input data point cloud, and apply our Finsler MDS pipeline on it. A new proximity graph computed in feature space could yield binary asymmetry and local metric scaling based on the density of the features would provide non-binary asymmetry. While this Finsler generalisation of PHATE, that we call \emph{naive Finsler PHATE}, is of interest on its own, a major downside of this generalisation is that it does not account for the different asymmetries present in the data dissimilarity construction phase (P1) of the original method, and instead builds a completely new data dissimilarity construction phase (P1) on top of it. 
We propose instead to generalise directly inside the original construction phase (P1) and account for the symmetries present there and discussed above. We call this alternative method \emph{Finsler PHATE} and now explain it.

Asymmetry is present in the entire data dissimilarity construction stage (P1), yet it is removed just when computing the final data dissimilarities as the Euclidean distance between diffusion potentials $\boldsymbol{\Delta}_{ij} = \lVert \boldsymbol{U}_j - \boldsymbol{U}_i\rVert_2$. Right before that, the data construction had resulted in an asymmetric $t$-step diffusion matrix $\boldsymbol{P}^t$. While the Euclidean distance is not a symmetrisation fix, it acts like it. We propose to modify it so that it accounts for the diffusion asymmetry. The Euclidean symmetric PHATE dissimilarities can be rewritten as
\begin{equation}
    \boldsymbol{\Delta}_{ij} = \left[\sum\limits_k (\boldsymbol{U}_{jk} - \boldsymbol{U}_{ik})^2\right]^{\frac{1}{2}}.
\end{equation}
PHATE thus sums equally the diffusion discrepancies from points $i$ and $j$ to all other points $k$. In particular, it does not account for how easy it is to reach $k$ from either points, which is given by $\boldsymbol{P}_{ik}^t$ (and $\boldsymbol{P}_{jk}^t$). With this new interpretation of the Euclidean distance using the tools built during the original data dissimilarity construction phase, our Finsler generalisation of PHATE to Finsler PHATE becomes natural. In Finsler PHATE, we propose to weigh the contribution of each term according to the diffusion probability $\boldsymbol{P}_{ik}^t$. This provides the asymmetric data dissimilarities
\begin{equation}
    \vec{\boldsymbol{\Delta}}_{ij}=\left[N\sum\limits_k(\boldsymbol{P}^t)_{ik}(\boldsymbol{U}_{jk}-\boldsymbol{U}_{ik})^2\right]^{\frac{1}{2}}.
\end{equation}
Multiplication by $N$ is used in order to obtain the Euclidean distance when the diffusion is uniform\footnote{Should the original data dissimilarities have involved a dimensionality normalisation, i.e.\ $\tfrac{1}{N}$ weighting, then our Finsler generalisation would not have needed to scale the weights by $N$ to recover the same dissimilarities in the uniform weight case.}, i.e.\ when $\boldsymbol{P}_{ik}^t = \tfrac{1}{N}$ then $\vec{\boldsymbol{\Delta}}_{ij} = \boldsymbol{\Delta}_{ij}$. Our Finsler PHATE dissimilarities are generally asymmetric $\vec{\boldsymbol{\Delta}}_{ij}^\top \neq \vec{\boldsymbol{\Delta}}_{ij}$, with asymmetry inhereted by that of the asymmetry in the pipeline at that stage, i.e.\ of $\boldsymbol{P}^t$, which is the second and main source of asymmetry in the original PHATE pipeline.

We can also account in Finsler PHATE for the first occurrence of asymmetry appearing in the PHATE pipeline, which is the asymmetric kernel $\widetilde{\boldsymbol{K}}$. Rather than symmetrising it, we can also use as affinities $\boldsymbol{K}$ for the diffusion operator $\boldsymbol{P}$ the the asymmetric ones, i.e.\ take $\boldsymbol{K} = \widetilde{\boldsymbol{K}}$. Combining both the non-symmetrisation of the kernel and the diffusion-weighted dissimilarities provides the dissimilarities for \emph{Finsler PHATE with asymmetric kernel variant}. Symmetrising the kernel and taking diffusion-weighted dissimilarities yields the dissimilarities for \emph{Finsler PHATE with symmetric kernel variant}.

For both variants, Finsler PHATE finds an asymmetric embedding in the canonical Randers space using Finsler MDS (P2-P3), rather than the original MDS, on the asymmetric data dissimilarities $\vec{\boldsymbol{\Delta}}$.

\section{Implementation details}
\label{sec: implementation details}

\subsection{Data}
\label{sec: implementation details data}

\subsubsection{Disk data.}
\label{sec: data disk}

The points are grid-sampled using uniform polar coordinates $(\rho_i,\theta_i)$, with $20$ angles $\theta_i$ and $\tfrac{n}{20} = 15$ radii $\rho_i$.

\subsubsection{Swiss roll.}
\label{sec: data swiss roll}

The Swiss roll manifold is parametrised in the unit square by $(u,v)\in[0,1]^2$. The Swiss roll transformation used in this paper is given by $x = \big(\tilde{u}\cos(\tilde{u}), \tilde{v}, \tilde{u}\sin(\tilde{u})\big)^\top\in\mathbb{R}^3$, where the stretched parametrisation is $\tilde{u} = 3\pi \big(u + \tfrac{1}{2}\big)$ and $\tilde{v} = 20 v$.
The points were sampled by uniformly grid sampling $(u,v)$ in the unit square $[0,1]^2$ using $n = \left\lfloor \tfrac{5}{2}\sqrt{\tilde{n}} \right\rfloor \times \left\lfloor \tfrac{2}{5}\sqrt{\tilde{n}}\right\rfloor$ samples. For $\tilde{N} = 2000$ this leads to $N = 1887$ effectively sampled points on the Swiss roll.

\subsubsection{Persistence data.}
\label{sec: data persistence}

This corresponds to the clustered manifold experiment in Sec.~\ref{sec: clustered manifold toy exp}.
For $N$ samples, we first generate the labels of the datapoints. To do so, we generate non-integer and unbounded soft labels $\tilde{c}_i$ using an exponential distribution $\tilde{c}_i\sim \mathrm{Exp(\lambda_{\mathrm{exp}})}$ with rate parameter $\lambda_{\mathrm{exp}}$. We convert the unbounded soft labels to the bounded integer labels $c_i\in\{0,\cdots, C-1\}$ for $C$ number of Gaussians using $c_i = \mathrm{clip}\left(\Big\lfloor \tfrac{\tilde{c_i}}{C q_{\mathrm{exp}}(p)}\Big\rfloor, 0, C-1\right)$, where $\mathrm{clip}(\cdot, v_{\min}, v_{\max})$ is the clip function to min and maximum values $v_{\min}$ and $v_{\max}$\footnote{The purpose of the clip function is the bound the obtained integer values to at most $C-1$ by merging the largest unbounded integer labels to the last Gaussian with smallest number of datapoints.} and $q_{\mathrm{exp}}(p)$ is the $p_{\mathrm{exp}}$-th quantile of the exponential distribution and is given by $q_{\mathrm{exp}}(p) = \tfrac{-\ln(1-p_{\mathrm{exp}})}{\lambda_{\mathrm{exp}}}$. In our experiments, $N=500$, $\lambda_{\mathrm{exp}} = 1$, $p_{\mathrm{exp}}=0.99$, and $C=5$, which lead with our random seed to the following number of datapoints per Gaussian: $294, 121, 48, 28, 9$.

We then generate the $C$ random Gaussian distribution $\mathcal{N}(\mu_j, \Sigma_j)$ in the input data feature space $\mathbb{R}^n$ used for sampling the random datapoints. We randomly sample the means $\mu_j\sim\mathcal{U}([0,1]^n)$ uniformly in the unit cube of $\mathbb{R}^n$. We also randomly sample the symmetric positive definite covariance matrices $\Sigma_j$ but we squash the eigenvalues using a rescaled sigmoid function in order to bound their scale. In practice, we uniformly generate $A_j\sim \mathcal{U}([0,1]^{n\times n})$ a random matrix, from which we construct a random symmetric positive matrix $\tilde{\Sigma}_j = A A^\top$, albeit with unbounded eigenvalues and not necessarily definite. Denote $B_j$ and $\tilde\Lambda_j$ as the unitary basis of eigenvectors and the eigenvalues of $\tilde\Sigma_j = B_j \tilde\Lambda_j B_j^\top$. We bound the scale of the eigenvalues by applying the scaled sigmoid function $\sigma_s(x) = \tfrac{s}{1+e^{-x}}$ to the eigenvalues of $\tilde\Sigma_j$, giving the new diagonal matrix $\sigma_s(\tilde\Lambda_j)$. To ensure that the resulting matrix is numerically definite (no close to zero eigenvalue), we perturb the eigenvalues by $\varepsilon>0$. In summary, the random covariance matrices are given by $\Sigma_j = B_j \sigma_s(\Lambda_j)B_j^\top + \varepsilon I$, and we took $\varepsilon=10^{-5}$ and $s=0.1$ as the maximum eigenvalue scale in our experiments.

We then generate each input datapoint $x_i$ by sampling the $c_i$-th random Gaussian $\mathcal{N}(\mu_{c_i}, \Sigma_{c_i})$ according to its randomly sampled label $c_i$.

\subsubsection{Reference classification datasets.}
\label{sec: data reference classification datasets}

The sixteen reference benchmarks we use are reference datasets in modern data learning. They form an extensive and particularly wide variety of classification datasets. 
Iris (150 tabular samples, 3 classes) \cite{fisher1936use} is the only non-image task. Stepping up from the smallest images, we first consider MNIST \cite{lecun2002gradient}, Fashion-MNIST \cite{xiao2017fashion}, Kuzushiji-MNIST \cite{clanuwat2018deep}, EMNIST (train split, 60000 images, 10 classes each) and EMNIST-Balanced (train split, 112 800 images, 47 classes) \cite{cohen2017emnist} of $28\times 28$-pixel grayscale characters. At the next resolution tier, CIFAR10 and CIFAR100 (train split) \cite{krizhevsky2009learning} each provide 50000 colour images of $32\times 32$ pixels with 10 and 100 classes respectively; each picture is transformed to a 768-dimensional feature vector using a frozen CLIP \cite{radford2021learning} ViT-B32 pretrained on Imagenet as raw images are poorly embedded with traditional manifold learning techniques. Finally, we move to variable-resolution natural images represented by 2048-dimensional global-average-pooled features from a frozen Imagenet-pretrained ResNet50 \cite{he2016deep}: DTD (train split, 1880 images, 47 classes) \cite{cimpoi2014describing}, Caltech101 (8677 images, 101 classes) \cite{fei2004learning}, Caltech256 (30607 images, 257 classes) \cite{griffin2007caltech}, OxfordFlowers102 (train split, 1020 images, 102 classes) \cite{nilsback2008automated}, OxfordIIITPet (trainval split, 3680 images, 37 classes) \cite{parkhi2012cats}, GTSRB (train split, 26640 images, 43 classes) \cite{stallkamp2012man}, Imagenette (train split, 9469 images, 10 classes) \cite{howard2019imagenette} and the large-scale ImageNet (train split, 1281167 images, 1000 classes) \cite{deng2009imagenet}. No further processing is performed beyond the fixed CLIP or ResNet50 embedding step. Recall that as the original t-SNE algorithm \cite{van2008visualizing} is slow, we only ran it and its Finsler variant on the smaller datasets.

\subsection{Scores}
\label{sec: appendix scores}

\noindent \textbf{Reference classification datasets.}
Our primary quantitative scores measure the fitting quality of kMeans clustering to the class labels of the data (with same number of clusters and labels). They thus reflect how well we preserve the data manifold after embedding.
However, the clustering community often uses other scores to measure the quality of their clustering, such as the Silhouette (SIL) \cite{rousseeuw1987silhouettes}, Davies-Bouldin Index (DBI) \cite{davies2009cluster}, and Calinski-Harabasz Index (CHI) \cite{calinski1974dendrite}. However, these measures are unrelated to the data labels and solely evaluate if the clusters ``look nice''\footnote{Clusters are well-separated, condensed, spherically shaped...}. As such, they do not reflect how well we preserve the data manifold after embedding. This makes the label-unrelated scores -- SIL, DBI, and CHI -- secondary compared to our primary label-related measures --  AMI, ARI, NMI, HOM, COM V-M, and FMI. It is of no importance that a method provides clusters that ``look nicer'' if the clusters fit worse to the data labels and correspond less well to the reality of the data manifold. In the supplementary, we include to results the label-unrelated scores as an indication of the shape of our embedded clusters, but the primary and most interesting comparison quantities are the label-related ones. 
Note that methods with high clustering scores may not preserve well the manifolds, as it is well-known and as we have shown in the synthetic dataset experiments that t-SNE and UMAP preserve less the data than MDS (yet we gave a fix to improve manifold preservation). This is widely accepted by the community, depending on downstream tasks. For classification, it is desirable in manifold learning to visualise accurately separated clusters, even if it means deforming the manifold to some extent. Adding to this the fact that manifold preserving techniques like classical MDS do not scale to large data, this explains the popularity and relevance of methods like t-SNE and UMAP, and justifies the relevance of the proposed metrics. This paper does not aim to fix t-SNE or UMAP poor preservation of the manifold (although we do give an improvement), but to show how to adjust existing methods to handle and harness asymmetry. Our evaluation thus does not rely too heavily on class-label clustering metrics.

In addition to the unsupervised kMeans clustering scores, whether related or not to the data labels, we also provide supervised scores, which are by definition label-related. They consist in the accuracy of trained simple classifiers on the embeddings. In our experiments, we use both a linear classifier and a 5 nearest neighbour classifier. As test accuracy depends on the train-test splits, we report the mean results for 2-repeated stratified 5-fold cross-validation (10 runs). While the unsupervised label-related scores are the primary scores in manifold learning, supervised scores bring additional insights to the performance of the methods.

\subsection{Methods}
\label{sec: implementation details methods}

All methods share the same initial precomputed kNN graph for computing data dissimilarities. Following what is done in UMAP \cite{mcinnes2018umap}, we compute approximate nearest neighbours via the nearest neighbour descent algorithm \cite{dong2011efficient} for $k$ neighbours, with $k=15$.

We use the default hyperparameters of t-SNE and UMAP, and take the same for our Finsler t-SNE and Finsler UMAP. In particular, the min distance hyperparameter is set to $0.1$ (UMAP and Finsler UMAP) and the perplexity to $30$ (t-SNE and Finsler t-SNE).

Optionally, in cases where manifold preservation is the main objective, we propose a refinement of the traditional UMAP pipeline, to be incorporated as well in Finsler UMAP, to geodesically extend the distances of the kNN graph and truncate the extended graph to keep only the $k^+$ smallest edges per node, with $k^+ > k$.

For both t-SNE and UMAP, either the traditional Euclidean or our Finsler pipeline, dissimilarities are computed directly on the sparse kNN graph. For methods requiring geodesic extensions of the distances in the kNN graph, such as traditional MDS, like Isomap, and Finsler MDS, we need to connect the distance graph.
To do so, we follow the methodology in Scikit-learn \cite{pedregosa2011scikit} for its Isomap implementation to greedily connect the closest pair of points in the Euclidean sense between any two connected components.

For initialisation, we followed the traditional recipes in the Euclidean case and generalised them to the Finsler case for our Finsler pipelines. Vanilla and Finsler t-SNE (resp.\ UMAP) are initialised with the default deterministic PCA (resp.\ spectral) embedding. Vanilla and Finsler SMACOF, along with Finsler MDS using gradient descent, are initialised using Isomap. Such initialisation strategies are considered standard for the standard Euclidean methods (SMACOF, t-SNE and UMAP). 
Following the methodology of \cite{dages2025finsler}, as the canonical Finsler space is the Euclidean hyperplane augmented with an axis of asymmetry, the initialisations in the Finsler methods are performed with one less dimension, to which is then added an extra dimension with 0 entry, and we then rotate the initialisation so that the new axis is along $\omega$\footnote{When taking $\omega$ along the last axis, no rotation is needed.}.

For Finsler MDS with our novel gradient descent implementation instead of the unstable and unscalable Finsler Smacof algorithm provided in  \cite{dages2025finsler},
we use an Adam optimiser \cite{kingma2014adam} with default hyperparameters for $100$ epochs, weight decay of $10^{-5}$, and learning rate $\eta$. We also use cosine annealing as a learning rate scheduler with maximum temperature parameter of $100$.

Note that we do not aim for state-of-the-art manifold learning or low-dimensional embeddings. The point of this work is twofold: show that asymmetry should not be discarded so quickly when constructing dissimilarities, as it encodes information and can be useful, and provide a framework to generalise existing works to asymmetric data. We showed the generalisation of the most popular baselines widely used in the field: MDS, t-SNE, and UMAP (with MDS being generalised in stages (P1) as \cite{dages2025finsler} generalised it for stages (P2-P3)). Spectral methods have also been tackled theoretically (Sec.~\ref{sec: spectral methods appendix}), and we have explained that generalising them would result in highly non-efficient algorithms and thus are not so much worth implementing. Our choice of baselines thus suffices to provide a roadmap to generalise other embedding techniques and our numerous experiments provide extensive evidence of the contribution of our method with respect to other possible baselines.

\subsubsection{Disk data.}

In this case, we run all pipelines with our default hyperparameters. The learning rate of Finsler MDS with gradient descent is set to $\eta = 0.001$.

\subsubsection{Swiss roll.}

To better preserve the manifold, we truncate the geodesic extension of the kNN graph with $k^+ = 50$. The learning rate of Finsler MDS with gradient descent is set to $\eta = 0.1$.

\subsubsection{Persistence data.}

To better preserve the manifold, we truncate the geodesic extension of the kNN graph with $k^+ = 300$. The learning rate of Finsler MDS with gradient descent is set to $\eta = 0.1$.

\subsubsection{Classification datasets.}

We took our default hyperparameters.
On Imagenette, we additionally evaluated Poincar\'e maps \cite{klimovskaia2020poincare} with default hyperparameters. It is a symmetric hyperbolic method, providing an alternative method for comparison that is non-Euclidean yet symmetric. We are not aware of non- or pseudo-Riemannian manifold learning methods and thus did not evaluate any.

\subsection{Resources}

All methods, whether the traditional Euclidean or our Finsler ones, do not require intensive resources: non-high end commercial CPUs are enough and no GPU is needed. To achieve this performance, advanced methods, like t-SNE and UMAP and their Finsler variants, must be implemented efficiently. That is why we derive explicitly the calculation of gradients and use them in fast code,  e.g.\ Numba's \cite{lam2015numba} \verb|njit| decorator in UMAP.
In terms of time, Euclidean and Finsler UMAP are of same order of magnitude,
whereas Finsler t-SNE is slower than its Euclidean counterpart.
On Imagenette, using a Intel(R) Xeon(R) Gold 6348 CPU @ 2.60GHz, a single embedding (\textit{Optimization stage}), for a default 200 iterations, takes 10s for UMAP with random initialisation and 16s with the slower yet default spectral initialisation, whereas Finsler UMAP takes 12s. For vanilla t-SNE, it takes around 2500s compared to 9700s for Finsler t-SNE.

\section{Further experiments}
\label{sec: further experiments}

\begin{figure*}[htbp]
    \captionsetup[sub]{labelformat=empty,justification=centering,singlelinecheck=false}
    \centering
    \savebox{\largestimage}{\includegraphics[width=0.12\textwidth]{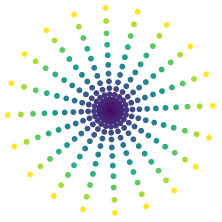}}%
    \begin{tabular}{@{}c@{}c@{}c@{}c@{}c@{}c@{}c@{}c@{}}
        \multirow[t]{2}{*}{
            \begin{tabular}{@{}cc@{}}
                \raisebox{+2em}[0pt][0pt]{%
                    \rotatebox{90}{\scriptsize $n=300$}
                }
                &
                \begin{subfigure}[b]{0.12\textwidth}
                    \centering
                    \subcaption{\scriptsize Data}
                    \includegraphics[width=\textwidth]{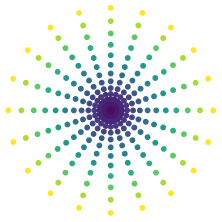}
                \end{subfigure}
            \end{tabular}
        }
        &
        
        \begin{subfigure}[b]{0.12\textwidth}
            \centering
            \subcaption{\scriptsize Isomap}
            \includegraphics[width=\textwidth]{Pictures/toy_disk_single_peak/n_300/single_peak_n_300_isomap.png}
        \end{subfigure}
        &
        \begin{subfigure}[b]{0.07\textwidth}
            \centering
            \subcaption{\scriptsize t-SNE}
            \includegraphics[width=\textwidth]{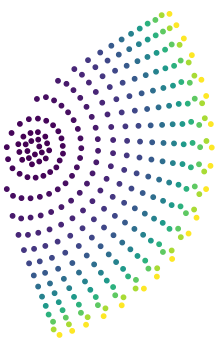}
        \end{subfigure}
        &
        \begin{subfigure}[b]{0.12\textwidth}
            \centering
            \subcaption{\scriptsize UMAP}
            \raisebox{\dimexpr.5\ht\largestimage-.5\height}{%
                \includegraphics[width=\textwidth]{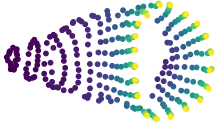}
            }
        \end{subfigure}
        &
        \begin{subfigure}[b]{0.12\textwidth}
            \centering
            \subcaption{\scriptsize Finsler MDS \\ (SMACOF)}
            \includegraphics[width=\textwidth]{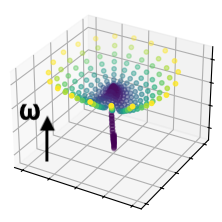}
        \end{subfigure}
        &
        \begin{subfigure}[b]{0.12\textwidth}
            \centering
            \subcaption{\scriptsize Finsler MDS \\ (GD)}
            \includegraphics[width=\textwidth]{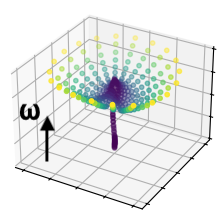}
        \end{subfigure}
        &
        \begin{subfigure}[b]{0.12\textwidth}
            \centering
            \subcaption{\scriptsize Finsler t-SNE}
            \includegraphics[width=\textwidth]{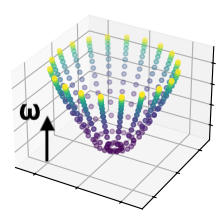}
        \end{subfigure}
        &
        \begin{subfigure}[b]{0.12\textwidth}
            \centering
            \subcaption{\scriptsize Finsler UMAP}
            \includegraphics[width=\textwidth]{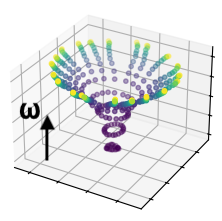}
        \end{subfigure}
        \\[0.6em]
        &
        \begin{subfigure}[b]{0.12\textwidth}
            \centering
            \includegraphics[width=\textwidth]{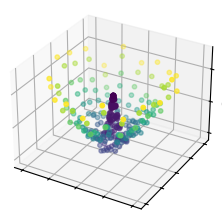}
        \end{subfigure}
        &
        \begin{subfigure}[b]{0.12\textwidth}
            \centering
            \includegraphics[width=\textwidth]{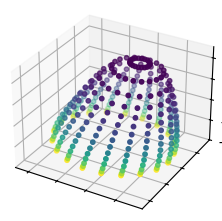}
        \end{subfigure}
        &
        \begin{subfigure}[b]{0.12\textwidth}
            \centering
            \vspace{-1em}
            \includegraphics[width=\textwidth]{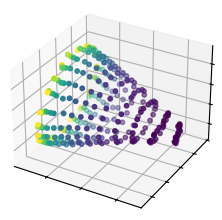}
        \end{subfigure}
        &
        \begin{subfigure}[b]{0.12\textwidth}
            \centering
            \includegraphics[width=\textwidth]{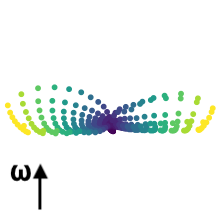}
        \end{subfigure}
        &
        \begin{subfigure}[b]{0.12\textwidth}
            \centering
            \includegraphics[width=\textwidth]{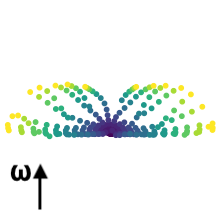}
        \end{subfigure}
        &
        \begin{subfigure}[b]{0.10\textwidth}
            \centering
            \includegraphics[width=\textwidth]{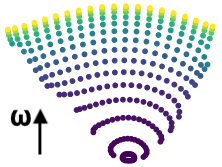}
        \end{subfigure}
        &
        \begin{subfigure}[b]{0.10\textwidth}
            \centering
            \includegraphics[width=\textwidth]{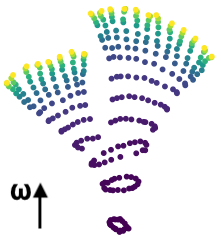}
        \end{subfigure}
        \\
        \multirow[t]{2}{*}{
            \begin{tabular}{@{}cc@{}}
                \raisebox{+2em}[0pt][0pt]{%
                    \rotatebox{90}{\scriptsize $n=1000$}
                }
                &
                \begin{subfigure}[b]{0.12\textwidth}
                    \centering
                    \subcaption{\scriptsize Data}
                    \includegraphics[width=\textwidth]{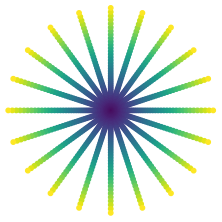}
                \end{subfigure}
            \end{tabular}
        }
        &
        \begin{subfigure}[b]{0.12\textwidth}
            \centering
            \subcaption{\scriptsize Isomap}
            \includegraphics[width=\textwidth]{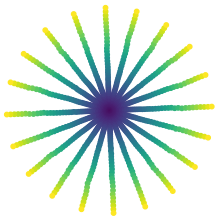}
        \end{subfigure}
        &
        \begin{subfigure}[b]{0.12\textwidth}
            \centering
            \subcaption{\scriptsize t-SNE}
            \includegraphics[width=\textwidth]{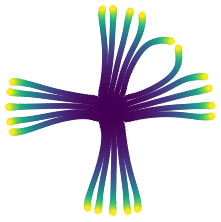}
        \end{subfigure}
        &
        \begin{subfigure}[b]{0.07\textwidth}
            \centering
            \subcaption{\scriptsize UMAP}
            \includegraphics[width=\textwidth]{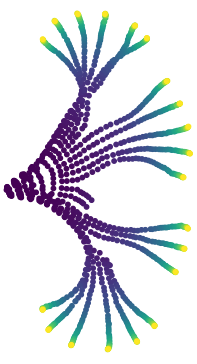}
        \end{subfigure}
        &
        \begin{subfigure}[b]{0.12\textwidth}
            \centering
            \subcaption{\scriptsize Finsler MDS \\ (SMACOF)}
            \includegraphics[width=\textwidth]{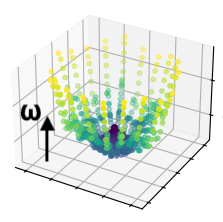}
        \end{subfigure}
        &
        \begin{subfigure}[b]{0.12\textwidth}
            \centering
            \subcaption{\scriptsize Finsler MDS \\ (GD)}
            \includegraphics[width=\textwidth]{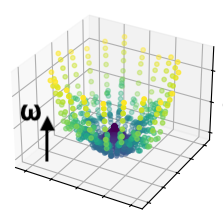}
        \end{subfigure}
        &
        \begin{subfigure}[b]{0.12\textwidth}
            \centering
            \subcaption{\scriptsize Finsler t-SNE}
            \includegraphics[width=\textwidth]{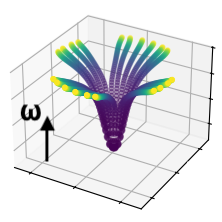}
        \end{subfigure}
        &
        \begin{subfigure}[b]{0.12\textwidth}
            \centering
            \subcaption{\scriptsize Finsler UMAP}
            \includegraphics[width=\textwidth]{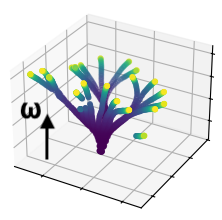}
        \end{subfigure}
        \\[0.6em]
        &
        \begin{subfigure}[b]{0.12\textwidth}
            \centering
            \includegraphics[width=\textwidth]{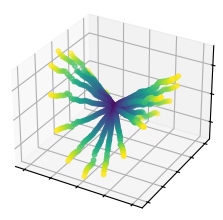}
        \end{subfigure}
        &
        \begin{subfigure}[b]{0.12\textwidth}
            \centering
            \includegraphics[width=\textwidth]{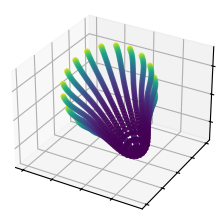}
        \end{subfigure}
        &
        \begin{subfigure}[b]{0.12\textwidth}
            \centering
            \vspace{-1em}
            \includegraphics[width=\textwidth]{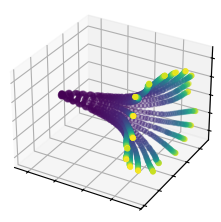}
        \end{subfigure}
        &
        \begin{subfigure}[b]{0.12\textwidth}
            \centering
            \includegraphics[width=\textwidth]{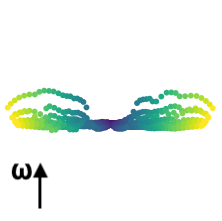}
        \end{subfigure}
        &
        \begin{subfigure}[b]{0.12\textwidth}
            \centering
            \includegraphics[width=\textwidth]{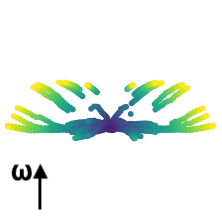}
        \end{subfigure}
        &
        \begin{subfigure}[b]{0.10\textwidth}
            \centering
            \includegraphics[width=\textwidth]{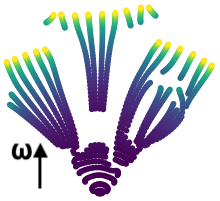}
        \end{subfigure}
        &
        \begin{subfigure}[b]{0.10\textwidth}
            \centering
            \includegraphics[width=\textwidth]{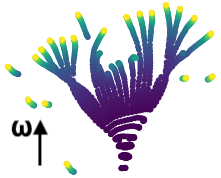}
        \end{subfigure}
        \\
        \multirow[t]{2}{*}{
            \begin{tabular}{@{}cc@{}}
                \raisebox{+2em}[0pt][0pt]{%
                    \rotatebox{90}{\scriptsize $n=3000$}
                }
                &
                \begin{subfigure}[b]{0.12\textwidth}
                    \centering
                    \subcaption{\scriptsize Data}
                    \includegraphics[width=\textwidth]{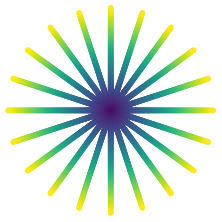}
                \end{subfigure}
            \end{tabular}
        }
        &
        \begin{subfigure}[b]{0.12\textwidth}
            \centering
            \subcaption{\scriptsize Isomap}
            \includegraphics[width=\textwidth]{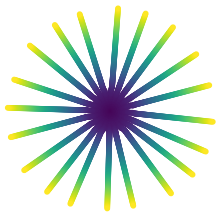}
        \end{subfigure}
        &
        \begin{subfigure}[b]{0.12\textwidth}
            \centering
            \subcaption{\scriptsize t-SNE}
            \includegraphics[width=\textwidth]{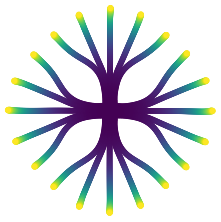}
        \end{subfigure}
        &
        \begin{subfigure}[b]{0.12\textwidth}
            \centering
            \subcaption{\scriptsize UMAP}
            \includegraphics[width=\textwidth]{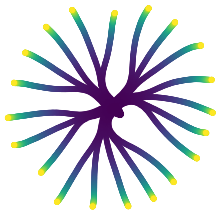}
        \end{subfigure}
        &
        \begin{subfigure}[b]{0.12\textwidth}
            \centering
            \subcaption{\scriptsize Finsler MDS \\ (SMACOF)}
            \includegraphics[width=\textwidth]{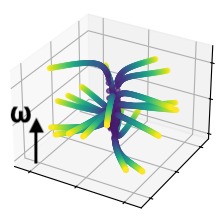}
        \end{subfigure}
        &
        \begin{subfigure}[b]{0.12\textwidth}
            \centering
            \subcaption{\scriptsize Finsler MDS \\ (GD)}
            \includegraphics[width=\textwidth]{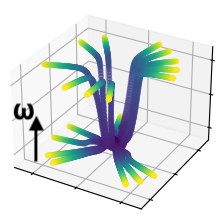}
        \end{subfigure}
        &
        \begin{subfigure}[b]{0.12\textwidth}
            \centering
            \subcaption{\scriptsize Finsler t-SNE}
            \includegraphics[width=\textwidth]{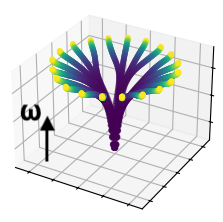}
        \end{subfigure}
        &
        \begin{subfigure}[b]{0.12\textwidth}
            \centering
            \subcaption{\scriptsize Finsler UMAP}
            \includegraphics[width=\textwidth]{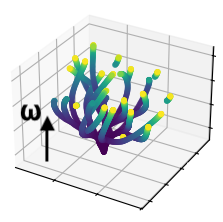}
        \end{subfigure}
        \\[0.6em]
        &
        \begin{subfigure}[b]{0.12\textwidth}
            \centering
            \includegraphics[width=\textwidth]{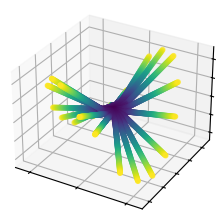}
        \end{subfigure}
        &
        \begin{subfigure}[b]{0.12\textwidth}
            \centering
            \includegraphics[width=\textwidth]{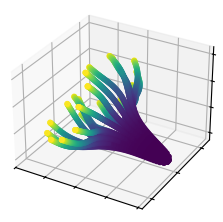}
        \end{subfigure}
        &
        \begin{subfigure}[b]{0.12\textwidth}
            \centering
            \vspace{-1em}
            \includegraphics[width=\textwidth]{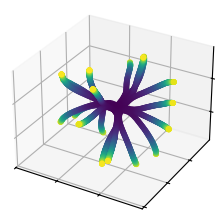}
        \end{subfigure}
        &
        \begin{subfigure}[b]{0.12\textwidth}
            \centering
            \includegraphics[width=\textwidth]{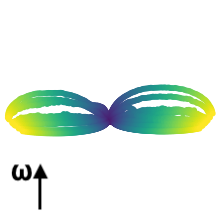}
        \end{subfigure}
        &
        \begin{subfigure}[b]{0.12\textwidth}
            \centering
            \includegraphics[width=\textwidth]{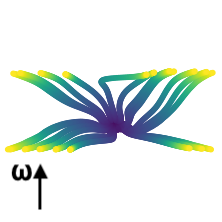}
        \end{subfigure}
        &
        \begin{subfigure}[b]{0.10\textwidth}
            \centering
            \includegraphics[width=\textwidth]{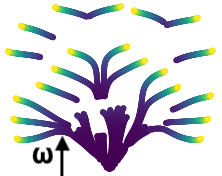}
        \end{subfigure}
        &
        \begin{subfigure}[b]{0.10\textwidth}
            \centering
            \includegraphics[width=\textwidth]{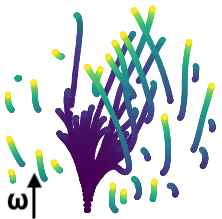}
        \end{subfigure}
    \end{tabular}

    \caption{%
    Embeddings of the toy planar data with non-uniform density. Each pair of rows corresponds to a different number of datapoints $n$, with $n=300$ (top), $n=1000$ (middle), and $n=3000$ (bottom). Within each pair of plots, the top (resp.\ bottom) row corresponds to embedding results in the recommended (resp.\ unrecommended) embedding dimension, namely $\mathbb{R}^2$ (resp.\ $\mathbb{R}^3$) for symmetric methods and $\mathbb{R}^3$ (resp.\ $\mathbb{R}^2$) for asymmetric Finsler embeddings to the canonical Randers space. The last dimension of Finsler embeddings, i.e.\ the $z$-axis (resp.\ $y$-axis) in three-dimensional (resp.\ two-dimensional) embeddings fully encodes the asymmetry, providing more information on the data than vanilla methods like Isomap. On the other hand, traditional symmetric embedding methods do not reveal the density difference between points, unlike the Finsler embedding methods, even those that use fewer dimensions than recommended for a two-dimensional data manifold.
    }
    \label{fig: toy data single peak n 300 1000 3000 2D 3D all}
    \vspace{-1em}
\end{figure*}

\subsection{Toy planar data}

We here perform two additional sets of experiments, regarding the number of datapoints and the embedding dimensionality.
For both experiments, we present the results in Fig.~\ref{fig: toy data single peak n 300 1000 3000 2D 3D all}.

\subsubsection{Number of datapoints.}
\label{sec: toy planar data number of datapoints}

In this experiment, we increase the number of points sampled on the unit disk from $N=300$ to $N\in\{1000,3000\}$, using all other parameters the same such as the number of neighbours in the kNN graph. We did not push beyond this limit 
as Finsler MDS greatly struggles when the number of data points scales. The first issue is the computation of a dense dissimilarity matrix $\boldsymbol{D}$ due to geodesic distance calculations. As it is dense, this makes minimisation updates particularly slow, for instance in the gradient descent method (GD), compared to more efficient methods like t-SNE or UMAP that heavily rely on sparse dissimilarities $p_{ij}$ to speed-up their updates. The Finsler SMACOF algorithm also heavily suffers from having many points. Unlike the original SMACOF algorithm, pseudo-inversion is required on a huge matrix of size $N^2\times N^2$ instead of $N\times N$, which leads to many issues noticed by the original authors \cite{dages2025finsler} such as instability and non-convergence. Instead, they had to rely on approximate pseudo-inverse solvers such as gmres \cite{saad1986gmres} to bring down the computation time. Nevertheless, their method still cannot scale beyond a number of points of this order of magnitude, and we were not able to get any results beyond $3000$ points with Finsler MDS.

By increasing the number of points, and keeping the same number of neighbours, we are artificially focusing on a more local graph. This leads to better preservation of individual branches\footnote{One branch corresponds to a fixed angle in polar coordinates.} in traditional methods. However, t-SNE and UMAP still poorly preserve the data manifold. This is because t-SNE and UMAP are methods favouring clustering, with $4$ main branch clusters in t-SNE and $3$ main ones in UMAP for $N=3000$, rather than manifold preservation like Isomap. Isomap accurately scales with the number of points, although some deformation is visually apparent for $N=3000$ as branches are no longer uniformly embedded angularly. Nevertheless, classical MDS is not able to provide any additional information, such as density-related information. Additionally, embedding using Isomap in a higher dimension leads to non flat embeddings, even though the original data can be perfectly embedded in a flat hyperplane of $\mathbb{R}^3$, showing another weakness of this method.

At $N=3000$, Finsler MDS becomes unstable, whether it be via the Finsler SMACOF \cite{dages2025finsler} or our gradient descent implementation. This shows that Finsler MDS, while of high interest, should only be used on datasets with a small number of points, roughly around $1000$ at most. For higher number of points, we must then switch to our Finsler t-SNE or Finsler UMAP methods, which are based on the methods t-SNE and UMAP that were designed to be used when greatly scaling up the number of datapoints. Though this will inevitably lead to a degradation in the preservation of the manifold, as these methods tend to artificially cluster, as previously discussed. That being said, both Finsler t-SNE and Finsler UMAP both scale well with the number of points. As both methods are based on matching embedding dissimilarities to a sparse dissimilarity matrix, the embedding should be analysed locally rather than globally. Note though that t-SNE provides significantly better global behaviour than UMAP thanks to its global normalisation of $q_{ij}$ which comes at the cost of speed\footnote{For a few orders of magnitude greater of $N$, t-SNE will be too slow for practical use and only UMAP will be able to run in reasonable time among the presented method. While it is fast, we should keep in mind that it very poorly preserves the data manifold for high number of points.}. Keeping in mind the local analysis of the embeddings, lower density regions are locally embedded higher than higher density regions along the last axis ($z$ in $\mathbb{R}^3$ and $y$ in $\mathbb{R}^2$), providing the valuable extra information with respect to density differences that the traditional symmetric algorithms are not able to provide.

\subsubsection{Embedding dimensionality.} 
\label{sec: toy planar data embedding dim extra minus}

In the seminal work \cite{dages2025finsler} that introduced the canonical Randers space and how to embed asymmetric data into it, it was argued that the appropriate embedding space for Finsler embeddings is $\mathbb{R}^{m+1}$ if comparing with symmetric methods embedding into the Euclidean space $\mathbb{R}^m$. This is because the canonical Randers space should be viewed as hyperplanes formed by the Euclidean space $\mathbb{R}^m$ that are orthogonal to an extra dimension $\omega$ that fully encodes asymmetry. Thus, asymmetric datapoints lying along $m$-dimensional manifolds should be encoded into the canonical Randers space $\mathbb{R}^{m+1}$. 
We did not come up with this methodology, it is the established standard. This allows to compare things that are comparable. Note that for higher embedding dimensionalities, the effect of having an extra dimension is lessened and disappears when the embedding dimensionality is larger than the data manifold dimensionality. This means that it is not sensible to embed surfaces (2-dimensional manifolds), like our toy planar disk or the Swiss roll into a canonical Randers space $\mathbb{R}^2$ (if $\lVert \omega\rVert_2\neq 0$) as it would implicitly consider that the surfaces are (asymmetric) one-dimensional manifolds, meaning a surface is incorrectly handled as a curve! We show that such poor and unrecommended choice heavily distorts the embedded manifolds (see Fig~\ref{fig: toy data single peak n 300 1000 3000 2D 3D all}), but interestingly the asymmetry between points is still revealed.
As such, in the main paper we followed the reference methodology and chose $\mathbb{R}^3$ as the embedding space for Finsler methods compared to $\mathbb{R}^2$. 

Nevertheless, we here also present, for various number of datapoints $N\in\{300, 1000, 3000\}$, the results obtained by embedding the traditional symmetric methods into the Euclidean space $\mathbb{R}^3$ and by embedding the Finsler methods into $\mathbb{R}^2$, which is now viewed as a one-dimensional space $\mathbb{R}$ formed by the $x$-axis and the $y$-axis encoding asymmetry. As the data manifold is two-dimensional, high deformations will necessarily occur in the canonical Randers space $\mathbb{R}^2$ as the one-dimensional hyperplane (a line) does not have a high enough dimension to represent accurately the two-dimensional data manifold (a disk).

Increasing the embedding dimension of traditional methods -- Isomap, t-SNE, UMAP -- leads to incorrectly non-flat embeddings, as the original data viewed from a Riemannian perspective is perfectly flat, even for Isomap which seemed to almost perfectly embed the data in $\mathbb{R}^2$ for this toy problem. These new embeddings might somewhat resemble their Finsler counterparts embedded into $\mathbb{R}^3$, which is to be expected as a similar type of objective is minimised. However, the Finsler versions are canonically aligned along the $z$-axis providing explicit information on the asymmetry in the data. On the other hand, the symmetric embeddings to $\mathbb{R}^3$ do not provide this information and instead the embedding is randomly rotated in $\mathbb{R}^3$. Likewise, when decreasing the embedding dimensionality of Finsler methods, the original disk manifold is heavily deformed due to a one dimensional Euclidean hyperplane ($x$-axis) to encode a two-dimensiona manifold (the disk), yet the asymmetry information is still explicitly presented through the positioning on the $y$-axis. In contrast, recommended embeddings to $\mathbb{R}^2$ of traditional methods often fails to accurately present the data, as in t-SNE or UMAP, or manage to present the data accurately but without any further information.

These experiments show that the strengths of our Finsler embeddings do not result from the additional embedding dimension. They not only manage to accurately preserve the original data manifold similarly to their traditional symmmetric counterparts, but they also provide additional quantitative and qualitative information regarding density differences (along geodesics) in the original data.

\subsection{Noise robustness}

We study robustness to transverse noise for our pipeline. The local scale induced by the sampling density on the data manifold is reliable only when neighbourhoods are stable. Sufficiently large transverse noise can change the kNN graph and make the asymmetry coordinate reflect more noise rather than density or other meaningful structured biases in the sampling of the data manifold, which could be a limitation of our method.

To evaluate this, we add transverse noise with increasing standard deviation $\sigma_N\in\{0, 0.01, 0.025, 0.05, 0.1, 0.2, 0.4, 0.8\}$ to our toy planar disk data 
and $\sigma_N\in\{0, 0.25, 0.5, 1, 1.5, 2, 3, 4\}$ for the Swiss roll.
We use the fixed embedding space asymmetry level $\lVert\omega\rVert_2 = 0.5$ for our Finsler methods.
We provide embeddings visualisations in
Figs.~\ref{fig: noise robustness planar disk scores} and \ref{fig: noise robustness swiss roll scores}
along with quantitative manifold preservation scores in
Figs.~\ref{fig: noise robustness planar disk embeddings} and \ref{fig: noise robustness swiss roll embeddings}.
We find that perhaps contrary to intuition, our Finsler embeddings share similar degradation performance, and thus robustness, to their symmetric counterparts. Whether this holds in general for all sorts of high-dimensional datasets is left to future work.

\begin{figure}
    \centering
    \includegraphics[width=\textwidth]{{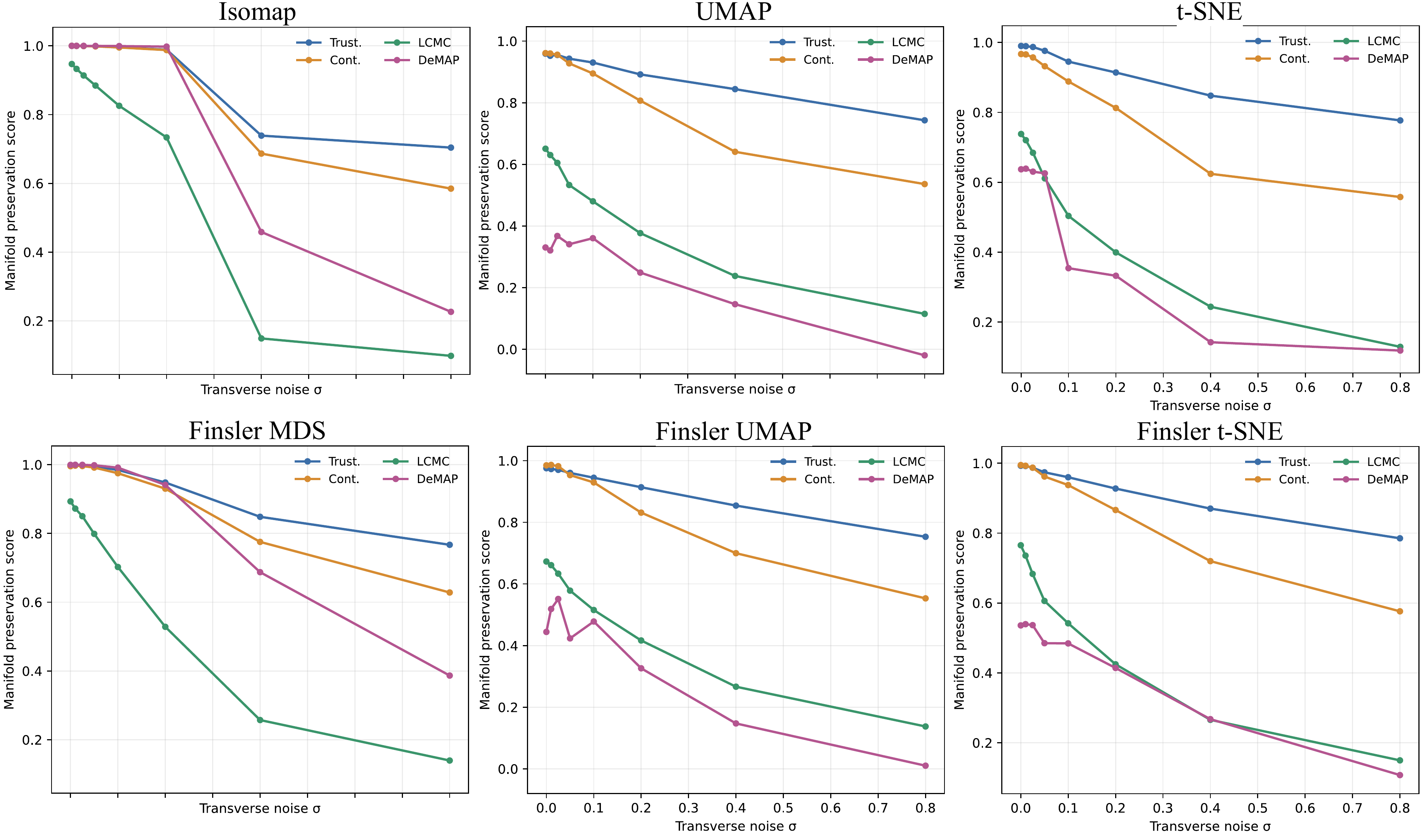}}
    \caption{Manifold preservation metrics of embeddings in Fig.~\ref{fig: noise robustness planar disk embeddings} of noisified planar disk data. Our Finsler methods' robustness are comparable to that of their symmetric counterparts.}
    \label{fig: noise robustness planar disk scores}
\end{figure}

\begin{figure}
    \centering
    \includegraphics[width=\textwidth]{{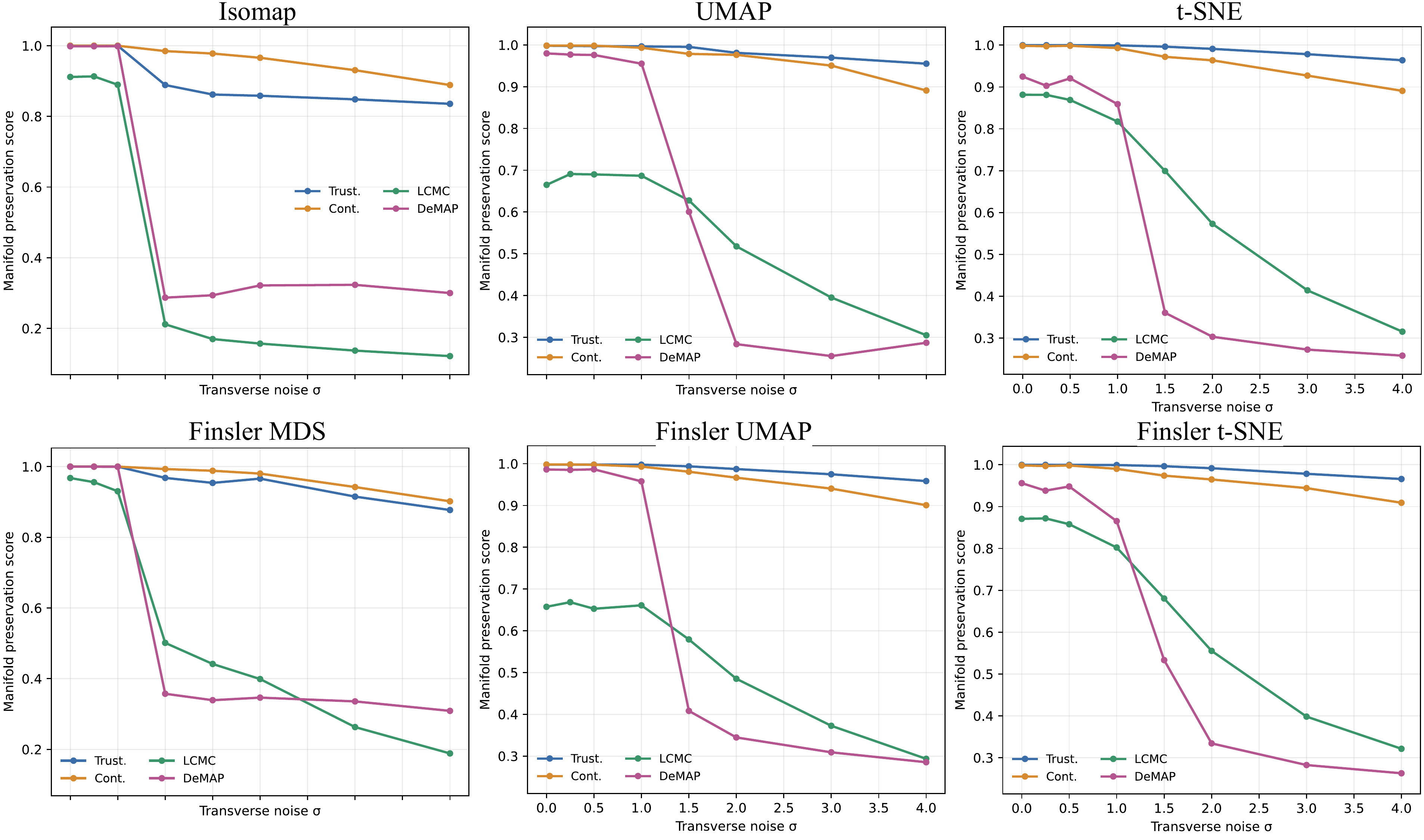}}
    \caption{Manifold preservation metrics of embeddings in Fig.~\ref{fig: noise robustness swiss roll embeddings} of noisified Swiss rolls. Our Finsler methods' robustness are comparable to that of their symmetric counterparts.}
    \label{fig: noise robustness swiss roll scores}
\end{figure}

\begin{figure}
    \centering
    \includegraphics[width=\textwidth]{{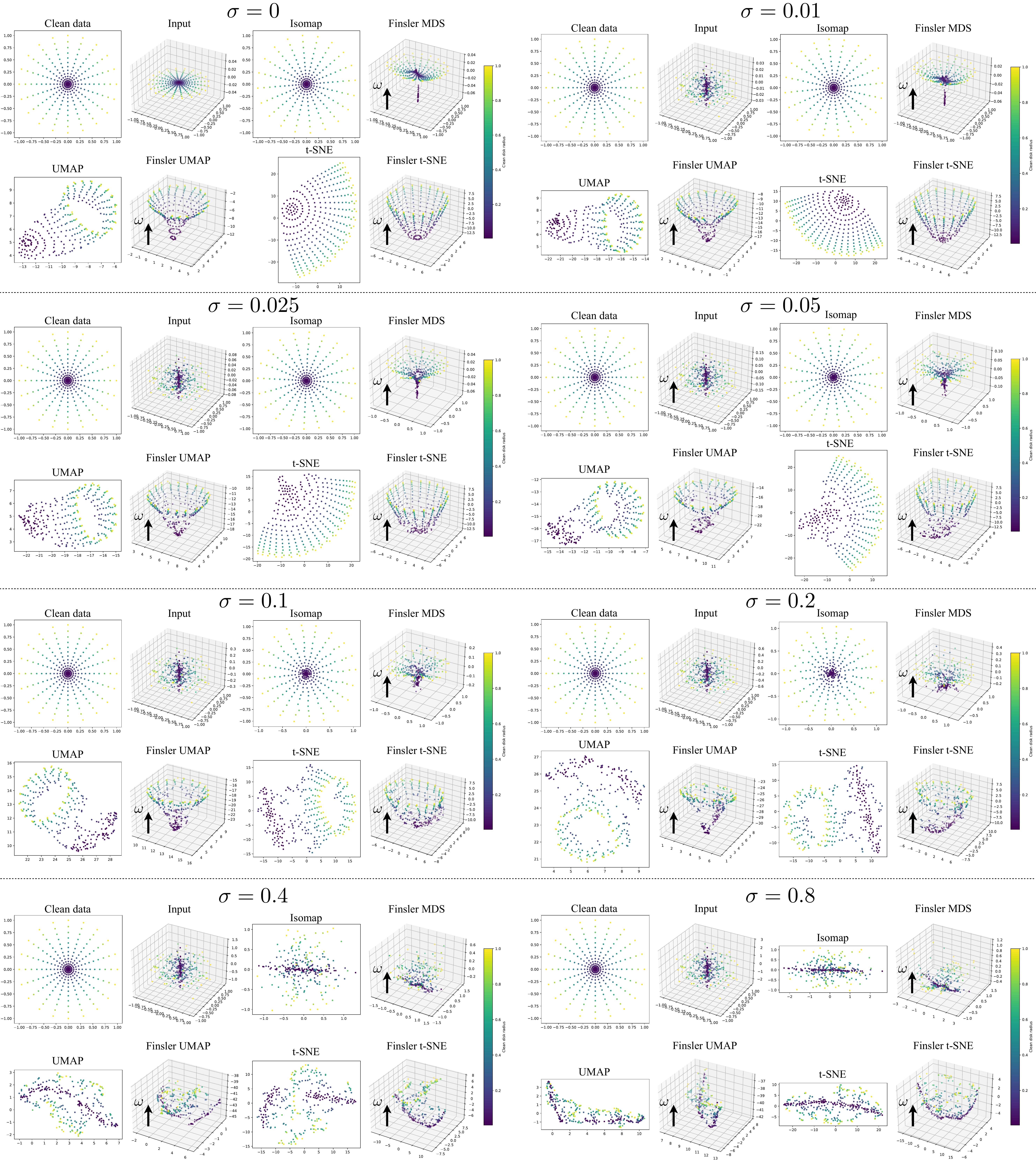}}
    \caption{Embeddings of noisified planar disk data with increasing noise level.}
    \label{fig: noise robustness planar disk embeddings}
\end{figure}

\begin{figure}
    \centering
    \includegraphics[width=\textwidth]{{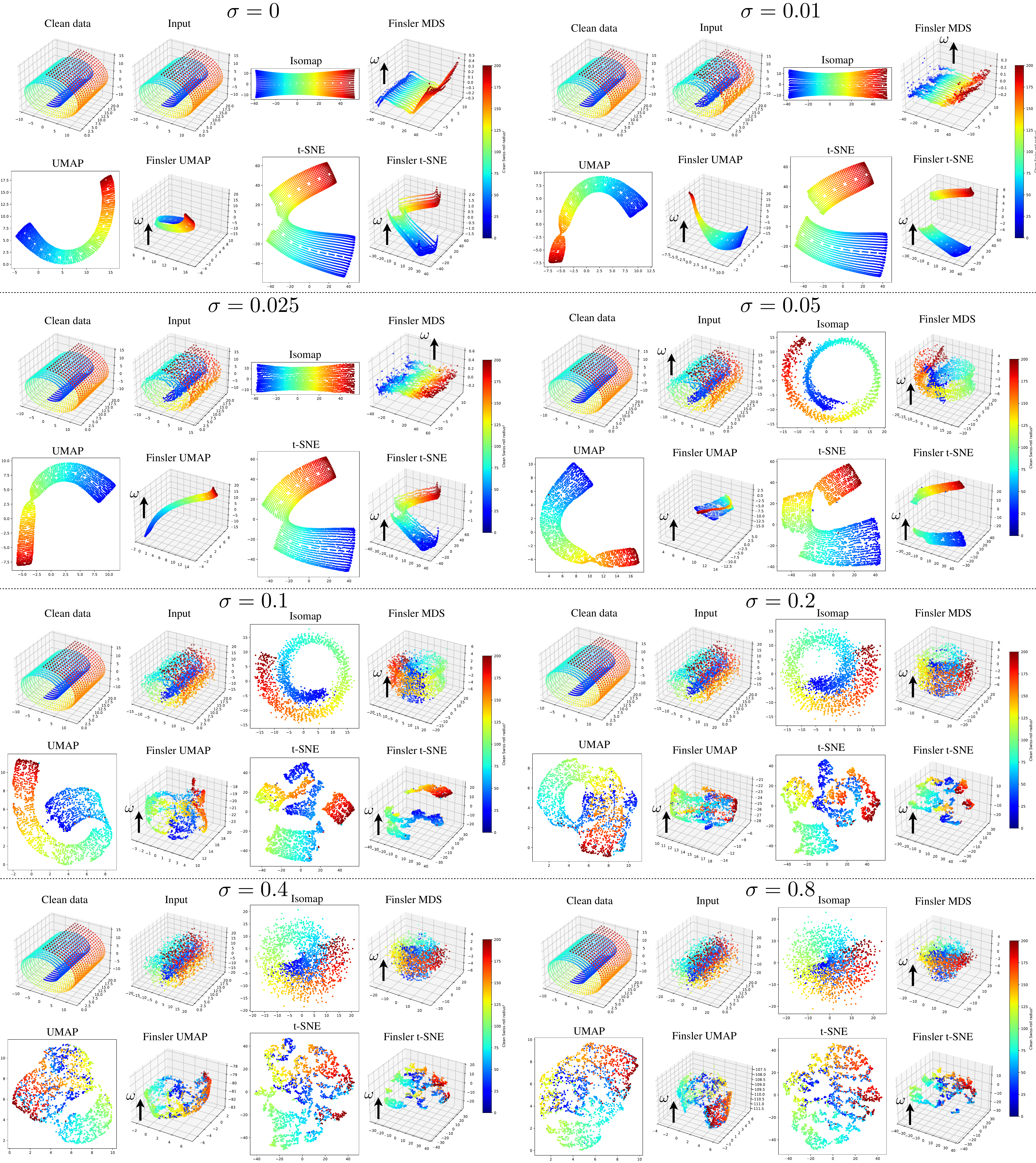}}
    \caption{Embeddings of noisified Swiss rolls with increasing noise level.}
    \label{fig: noise robustness swiss roll embeddings}
\end{figure}

\begin{figure}[t]
    \centering
    \includegraphics[width=\textwidth]{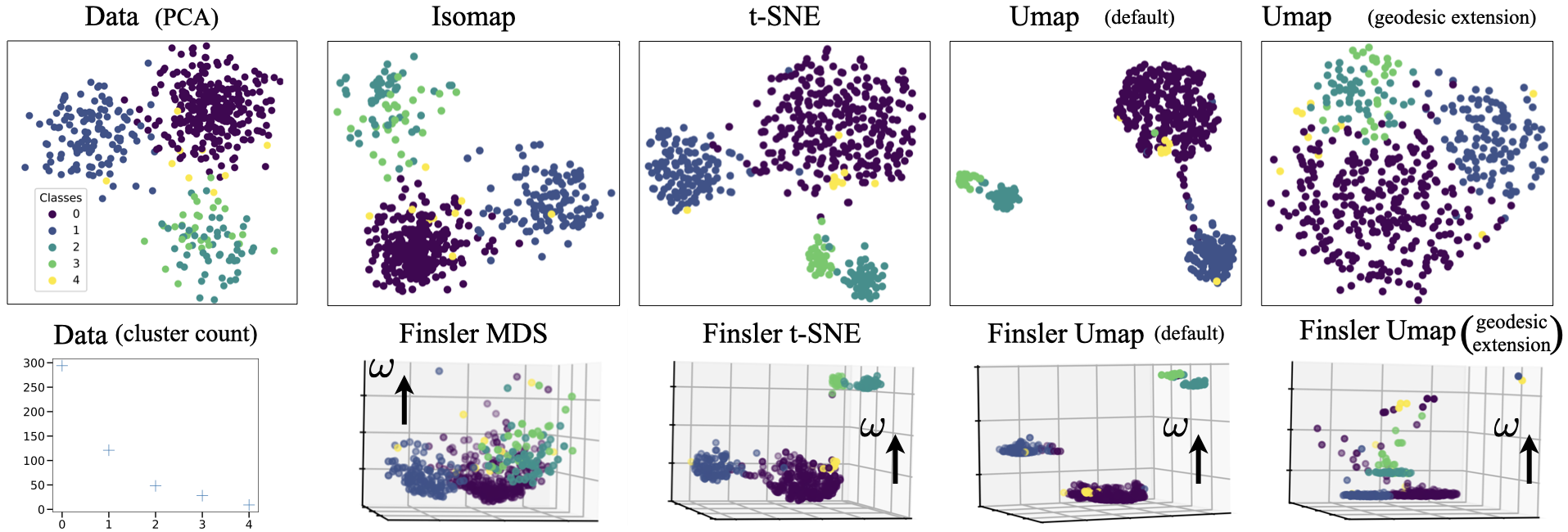}%
    \caption{
    Embeddings of simulated mutation data with exponentially decreasing samples per groundtruth Gaussian. While traditional methods cluster, with more or less success, our Finsler approach both clusters correctly and encodes a cluster hierarchy via height, as sparser low-count clusters are embedded higher. 
    This is shown either locally, when handling sparse dissimilarities (Finsler t-SNE, default Finsler UMAP), or globally when using dense dissimilarities due to geodesic extensions of graph distances (Finsler MDS, extended Finsler UMAP).
    This semantic hierarchy is absent in Euclidean methods.
    This figure is an enlarged version of Fig.~\ref{fig: persistence summary}.
    }
    \label{fig: persistence summary supmat}
    \vspace{-0.5em}
\end{figure}

\subsection{Toy clustered manifold}
\label{sec: clustered manifold toy exp}

\subsubsection{Data.}
We draw $N\!=\!500$ high-dimensional points in $\mathbb{R}^{10}$ from $5$ random Gaussians with exponentially decreasing sample counts. This simulates persistence data in which mixture components serve as sequential 
mutation states. Each agent mutates to the next state with small probability, independently of history, and mutations alter features. See Sec.~\ref{sec: data persistence} for
details.

\noindent\textbf{Methods.}
We view the input data with PCA as it is high-dimensional. Euclidean (resp.\ Finsler) methods -- Isomap, t-SNE, and UMAP (resp.\ Finsler t-SNE and UMAP) -- embed in $\mathbb{R}^2$ (resp.\ $\mathbb{R}^3$). The data has hidden semantic labels, so good embeddings must form clusters to reveal them. We run UMAP with and without geodesically extending the kNN graph (15 neighbours, 300 smallest edges retained) to promote global interactions and cluster hierarchy.

\noindent\textbf{Results.}
Results appear in Figs.~\ref{fig: persistence summary} and \ref{fig: persistence summary supmat}. Although the original manifold has no intrinsic asymmetry a priori, our asymmetric Finsler pipeline exposes a semantic hierarchy, either locally or globally: rare datapoints with more mutations (higher labels) are embedded higher than common ones with fewer mutations (lower labels). Euclidean baselines cluster but convey no hierarchy. The hierarchy is revealed thanks to asymmetric density sensitivity: as samples per Gaussian decrease, neighbourhood radii grow and cross-cluster neighbours become more likely. Dense Gaussians yield many close, same-cluster neighbours, whereas sparse Gaussians yield fewer and more distant ones, often from other clusters. Thus, our methods supplement clustering with the desired semantic ordering.
These results demonstrate the potential of our approach to provide meaningful semantic information that can be useful in practice.

\subsection{Real planar data -- US Cities}

In Fig.~\ref{fig: US cities}, we provided the side views for our asymmetric Finsler methods, which reveal via the height embedding differences the hidden information correlated with the hidden topography data. In Fig.~\ref{fig: US cities top down},
we also provide the top down views of 3D Finsler embeddings to show that they are very similar to their 2D symmetric counterparts. As such, our Finsler methods do not harm or destroy the reference symmetric embeddings, they simply add more information to them.

\begin{figure}[t]
    \centering
    \includegraphics[width=\textwidth]{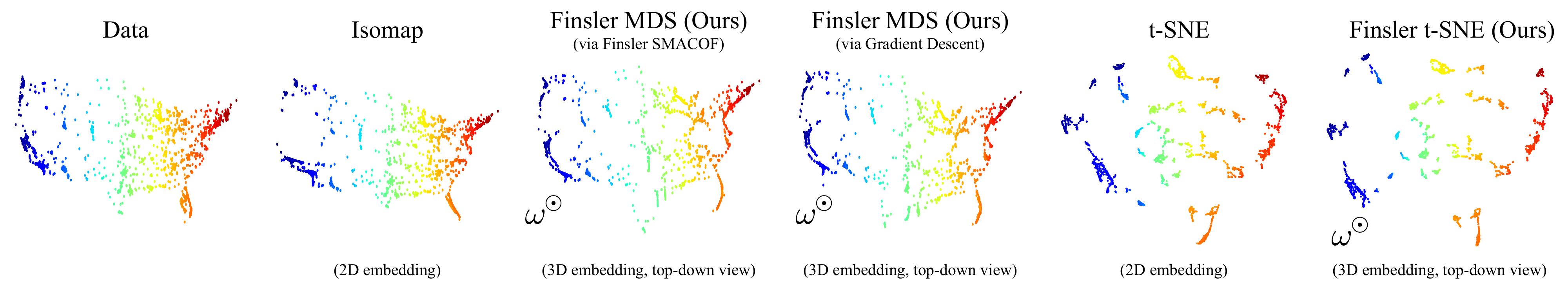}%
    \caption{
    We plot top-down view-points for our Finsler embeddings shown in Fig.~\ref{fig: US cities}, which displayed side embeddings to reveal the height asymmetry and how it correlated with topography visually. Our Finsler embeddings inherit the properties of their symmetric counterparts, including the quality of their embeddings, while providing additional asymmetric information thanks to the extra embedding dimension.
    }
    \label{fig: US cities top down}
\end{figure}%

\subsection{Reference classification datasets}
\label{sec: reference classification datasets}

\subsubsection{Raw values, clustering label-unrelated scores, supervised, and non-clustering scores}
\label{sec: raw values and label-unrelated scores}

In Fig.~\ref{fig: classif kmeans on umap finsler umap + tsne finsler tsne label-related only fixed omega}, we summarised due to page-length constraints the results by presenting the percentage difference in mean performance\footnote{The performance of our Finsler method minus that of the standard Euclidean one divided by the standard Euclidean one.} on a few of evaluated datasets.
Those results reported in the main manuscript used a fixed asymmetry level $\lVert \omega\rVert_2$, but we could alternatively select optimal asymmetry level, e.g.\ based on best AMI performance (see 
Sec.~\ref{sec: levels of emphasis on asymmetry}
for showing how this methodology does not overfit). We provide in Fig.~\ref{fig: classif kmeans on umap finsler umap + tsne finsler tsne label-related only} an alternative version of 
Fig.~\ref{fig: classif kmeans on umap finsler umap + tsne finsler tsne label-related only fixed omega} using this improved methodology that better fits to the data.
See Fig.~\ref{fig: classif kmeans on umap finsler umap + tsne finsler tsne label-related only zoom in} for the same plot on all evaluated datasets. Also, in Tabs.~\ref{tab: classif kmeans on umap finsler umap} and \ref{tab: classif kmeans on tsne finsler tsne}, we present the raw obtained (mean and standard deviation) performance on all datasets.

\begin{figure}[t]
    \centering
    \begin{subfigure}{0.35\textwidth}
        \adjincludegraphics[
            width=\textwidth, 
            trim={{0.\width} {0.\height} {0.5\width} {0.\height}},
            clip,
        ]{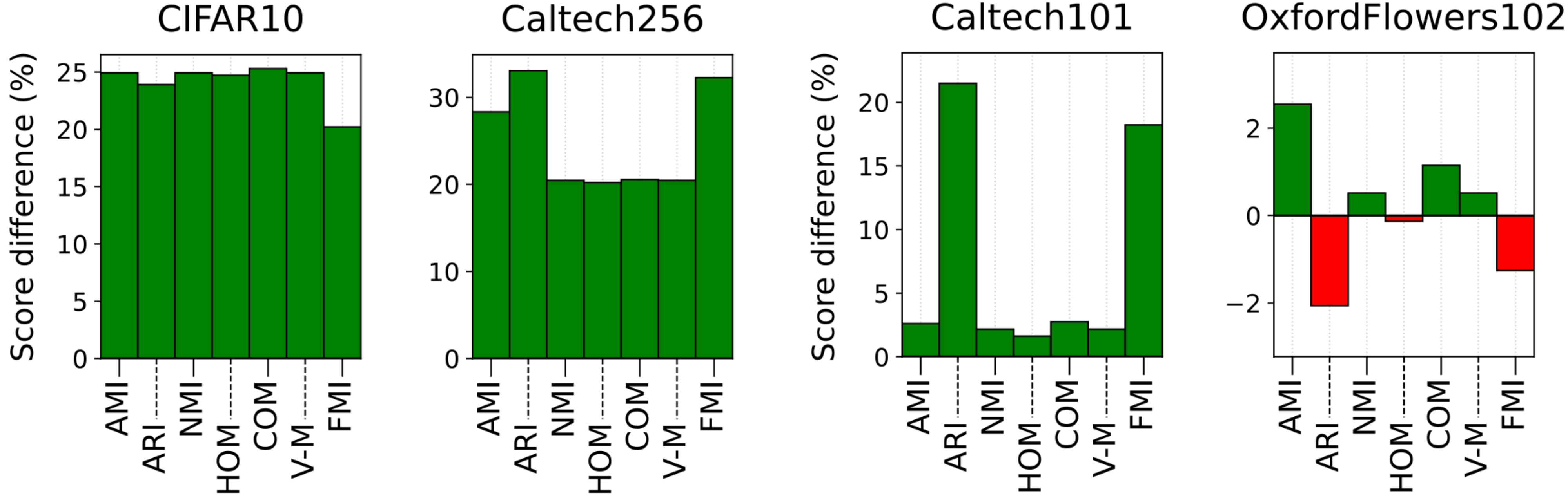}%
        \caption*{Finsler UMAP vs.\ UMAP}
    \end{subfigure}%
    \hspace{2em}
    \begin{subfigure}{0.35\textwidth}
        \adjincludegraphics[
            width=\textwidth, 
            trim={{0.5\width} {0.\height} {0.\width} {0.\height}},
            clip,
        ]   {Pictures/classif_datasets/quantitative_main_pct_diff__subset.pdf}%
        \caption*{Finsler t-SNE vs.\ t-SNE}
    \end{subfigure}
    \caption{Percentage difference in mean performance between our Finsler methods and 
    their 
    Euclidean baselines. Positive (resp.\ negative) differences, in green (resp.\ red), means we get better (resp.\ worse) scores 
    for
    kMeans 
    vs.\
    class labels. 
    See Sec.~\ref{sec: raw values and label-unrelated scores} for 
    extensive
    results on all 16 tested datasets, 
    with
    visual summaries (Fig.~\ref{fig: classif kmeans on umap finsler umap + tsne finsler tsne label-related only zoom in}), raw performance values, and standard deviations (Tabs.~\ref{tab: classif kmeans on umap finsler umap} and \ref{tab: classif kmeans on tsne finsler tsne}).
    }
    \label{fig: classif kmeans on umap finsler umap + tsne finsler tsne label-related only}
    \vspace{-1em}
\end{figure}

Additionally, we also provide the label-unrelated scores, either in visual summary form in Figs.~\ref{fig: classif kmeans on umap finsler umap + tsne finsler tsne all labels} and \ref{fig: classif kmeans on umap finsler umap + tsne finsler tsne all labels only zoom in} and in raw numbers in Tabs.~\ref{tab: classif kmeans on umap finsler umap} and \ref{tab: classif kmeans on tsne finsler tsne}. As mentioned in the main part, our Finsler pipelines and embedding methods are systematically superior to their Euclidean counterparts on scores related to the groundtruth labels. However, on the secondary label-unrelated scores, either can be better. Recall though that it matters not to have inferior label-unrelated scores if our label-related ones are superior. When the Euclidean methods provide better label-unrelated scores, they systematically underperform compared to our Finsler ones on label-related ones. 
They thus
tend to 
create ``nicer looking'' clusters that do not correspond to the original data manifold: points were artificially clustered together. Our Finsler approach thus provides significantly better quality embeddings that do not artificially distort the structure of the data manifold, unlike the Euclidean methods.
Unsupervised results are reinforced with those obtained using supervised scores (recall Tab.~\ref{tab: supervised evaluation}), as both Finsler t-SNE and Finsler UMAP consistently outperform their Euclidean baselines by large magnitudes\footnote{Or they are on par with the baselines.}.

We provide further unsupervised results by also evaluating, in a label-free $\lVert \omega\rVert_2 = 0.001$ setting, Finsler UMAP with the non-clustering scores Trustworthiness  \cite{venna2001neighborhood}, Continuity \cite{venna2001neighborhood}, Local Continuity Meta-Criteria (LCMC) \cite{chen2009local}, and Denoised Embedding Manifold Preservation (DeMAP) \cite{moon2019visualizing}. We compute the scores on subsampled embeddings with $N\leq10000$ due to memory issues. We report the results in Tab.~\ref{tab: classif data non clustering manif preservation scores}.
Finsler UMAP improves trustworthiness and LCMC on every dataset and continuity on most, DEMaP is mixed.

\begin{table}[t]
    \centering
    \caption{
    Mean unsupervised non-clustering scores. The mean and standard deviation over 10 runs is presented. (+) signifies that higher is better.
    Best/tied mean scores are bold/underlined.
    }
    \label{tab: classif data non clustering manif preservation scores}
    \vspace{0.5em}
    \setlength{\tabcolsep}{4pt}
    \renewcommand{\arraystretch}{1.1}
    \resizebox{0.9\textwidth}{!}{%
    \begin{tabular}{llcccc}
        \toprule
        {Dataset} & {Method}
        & {Trustworthiness (+)} & {Continuity (+)}
        & {LCMC (+)} & {DeMAP (+)} \\
        \midrule

        \multirow{2}{*}{{MNIST}}
        & UMAP
        & $0.955 \pm 2\times 10^{-4}$
        & $\underline{0.962} \pm 1\times 10^{-4}$
        & $0.233 \pm 7\times 10^{-4}$
        & $\mathbf{0.438} \pm 3\times 10^{-3}$ \\
        & Finsler UMAP
        & $\mathbf{0.977} \pm 8\times 10^{-4}$
        & $\underline{0.962} \pm 2\times 10^{-3}$
        & $\mathbf{0.315} \pm 2\times 10^{-3}$
        & $0.405 \pm 6\times 10^{-2}$ \\
        \midrule

        \multirow{2}{*}{{CIFAR10}}
        & UMAP
        & $0.958 \pm 7\times 10^{-4}$
        & $0.970 \pm 4\times 10^{-4}$
        & $0.153 \pm 1\times 10^{-3}$
        & $\mathbf{0.784} \pm 1\times 10^{-2}$ \\
        & Finsler UMAP
        & $\mathbf{0.973} \pm 3\times 10^{-4}$
        & $\mathbf{0.971} \pm 3\times 10^{-3}$
        & $\mathbf{0.211} \pm 9\times 10^{-4}$
        & $0.715 \pm 5\times 10^{-2}$ \\
        \midrule

        \multirow{2}{*}{{CIFAR100}}
        & UMAP
        & $0.923 \pm 4\times 10^{-4}$
        & $0.946 \pm 2\times 10^{-4}$
        & $0.161 \pm 1\times 10^{-3}$
        & $\mathbf{0.683} \pm 5\times 10^{-3}$ \\
        & Finsler UMAP
        & $\mathbf{0.948} \pm 2\times 10^{-4}$
        & $\mathbf{0.950} \pm 2\times 10^{-3}$
        & $\mathbf{0.236} \pm 1\times 10^{-3}$
        & $0.622 \pm 3\times 10^{-2}$ \\
        \midrule

        \multirow{2}{*}{{DTD}}
        & UMAP
        & $0.893 \pm 2\times 10^{-3}$
        & $0.916 \pm 1\times 10^{-3}$
        & $0.369 \pm 3\times 10^{-3}$
        & $0.491 \pm 8\times 10^{-3}$ \\
        & Finsler UMAP
        & $\mathbf{0.924} \pm 1\times 10^{-3}$
        & $\mathbf{0.933} \pm 1\times 10^{-3}$
        & $\mathbf{0.428} \pm 1\times 10^{-3}$
        & $\mathbf{0.557} \pm 7\times 10^{-3}$ \\
        \midrule

        \multirow{2}{*}{{Caltech101}}
        & UMAP
        & $0.981 \pm 3\times 10^{-4}$
        & $0.982 \pm 6\times 10^{-4}$
        & $0.347 \pm 2\times 10^{-3}$
        & $\mathbf{0.438} \pm 2\times 10^{-2}$ \\
        & Finsler UMAP
        & $\mathbf{0.984} \pm 3\times 10^{-4}$
        & $\mathbf{0.983} \pm 5\times 10^{-4}$
        & $\mathbf{0.387} \pm 2\times 10^{-3}$
        & $0.375 \pm 1\times 10^{-2}$ \\
        \midrule

        \multirow{2}{*}{{Caltech256}}
        & UMAP
        & $0.948 \pm 1\times 10^{-3}$
        & $0.945 \pm 1\times 10^{-3}$
        & $0.342 \pm 2\times 10^{-3}$
        & $0.489 \pm 1\times 10^{-2}$ \\
        & Finsler UMAP
        & $\mathbf{0.965} \pm 5\times 10^{-4}$
        & $\mathbf{0.950} \pm 2\times 10^{-3}$
        & $\mathbf{0.381} \pm 1\times 10^{-3}$
        & $\mathbf{0.491} \pm 2\times 10^{-2}$ \\
        \midrule

        \multirow{2}{*}{{OxfordFlowers102}}
        & UMAP
        & $0.881 \pm 3\times 10^{-3}$
        & $0.899 \pm 6\times 10^{-4}$
        & $0.364 \pm 3\times 10^{-3}$
        & $0.612 \pm 7\times 10^{-3}$ \\
        & Finsler UMAP
        & $\mathbf{0.911} \pm 2\times 10^{-3}$
        & $\mathbf{0.920} \pm 8\times 10^{-4}$
        & $\mathbf{0.405} \pm 4\times 10^{-3}$
        & $\mathbf{0.689} \pm 1\times 10^{-2}$ \\
        \midrule

        \multirow{2}{*}{{OxfordIIITPet}}
        & UMAP
        & $0.982 \pm 2\times 10^{-4}$
        & $\underline{0.986} \pm 6\times 10^{-4}$
        & $0.371 \pm 3\times 10^{-3}$
        & $\mathbf{0.585} \pm 4\times 10^{-2}$ \\
        & Finsler UMAP
        & $\mathbf{0.986} \pm 2\times 10^{-4}$
        & $\underline{0.986} \pm 8\times 10^{-4}$
        & $\mathbf{0.417} \pm 3\times 10^{-3}$
        & $0.497 \pm 5\times 10^{-2}$ \\
        \midrule

        \multirow{2}{*}{{Imagenette}}
        & UMAP
        & $0.982 \pm 1\times 10^{-4}$
        & $0.983 \pm 2\times 10^{-4}$
        & $0.296 \pm 1\times 10^{-3}$
        & $0.544 \pm 1\times 10^{-2}$ \\
        & Finsler UMAP
        & $\mathbf{0.988} \pm 2\times 10^{-4}$
        & $\mathbf{0.990} \pm 3\times 10^{-3}$
        & $\mathbf{0.332} \pm 3\times 10^{-3}$
        & $\mathbf{0.549} \pm 6\times 10^{-3}$ \\

        \bottomrule
    \end{tabular}
    }%
\end{table}

We also quantitatively evaluate Finsler PHATE, with (Sym.\ Ker.) or without (Asym.\ Ker.) kernel symmetrisation\footnote{Recall that even when the kernel $\boldsymbol{K}$ is a symmetrised version of $\tilde{\boldsymbol{K}}$, Finsler PHATE is still an asymmetric method as the distance between potentials is weighted asymmetrically.} on the classification datasets in the label-free setting $\lVert \omega\rVert_2 = 0.001$. We subsampled to $N=1000$ points due to the inefficient implementation of Finsler MDS of the original paper \cite{dages2025finsler}. Finsler PHATE generally improves clustering, Trustworthiness, and LCMC over its original symmetric counterpart. The asymmetric-kernel variant is mixed on MNIST and CIFAR but beneficial on Caltech data. We provide visualisations for these Finsler PHATE embeddings in Fig.~\ref{fig: qualitative classif datasets summary phate}.

\begin{table}[t]
    \centering
    \caption{Mean performance of our Finsler PHATE methods. The mean and standard deviation over 10 runs is presented. (+) signifies that higher is better.}
    \label{tab:complete_phate_results}
    \vspace{0.5em}
    \setlength{\tabcolsep}{2.5pt}
    \scriptsize
    \resizebox{\columnwidth}{!}{%
    \begin{tabular}{ll*{7}{l}@{\hspace{0.9em}}c@{\hspace{0.9em}}*{4}{l}}
    \toprule
    & & \multicolumn{7}{c}{Clustering scores} & & \multicolumn{4}{c}{Non-Clustering scores} \\
    \cmidrule(lr){3-9} \cmidrule(lr){11-14}
    Dataset & Method & AMI (+) & ARI (+) & NMI (+) & HOM (+) & COM (+) & V-M (+) & FMI (+) & & Trustworthiness (+) & Continuity (+) & LCMC (+) & DeMAP (+) \\
    \midrule
    \multirow[c]{3}{*}{MNIST} & PHATE & 0.506 ($\pm$ 0.009) & 0.358 ($\pm$ 0.024) & 0.515 ($\pm$ 0.009) & 0.509 ($\pm$ 0.009) & 0.521 ($\pm$ 0.009) & 0.515 ($\pm$ 0.009) & 0.425 ($\pm$ 0.021) & & \textbf{0.868} ($\pm$ 1e-4) & \textbf{0.940} ($\pm$ 3e-4) & 0.223 ($\pm$ 5e-4) & \textbf{0.646} ($\pm$ 3e-4) \\
     & Finsler PHATE (Sym. Ker.) & \textbf{0.536} ($\pm$ 0.005) & \textbf{0.364} ($\pm$ 0.013) & \textbf{0.544} ($\pm$ 0.005) & \textbf{0.533} ($\pm$ 0.007) & \textbf{0.556} ($\pm$ 0.008) & \textbf{0.544} ($\pm$ 0.005) & \textbf{0.435} ($\pm$ 0.009) & & 0.861 ($\pm$ 9e-4) & 0.933 ($\pm$ 1e-4) & \textbf{0.331} ($\pm$ 8e-4) & 0.569 ($\pm$ 4e-4) \\
     & Finsler PHATE (Asym. Ker.) & 0.456 ($\pm$ 0.012) & 0.314 ($\pm$ 0.024) & 0.466 ($\pm$ 0.012) & 0.456 ($\pm$ 0.009) & 0.477 ($\pm$ 0.015) & 0.466 ($\pm$ 0.012) & 0.391 ($\pm$ 0.024) & & 0.823 ($\pm$ 5e-4) & 0.909 ($\pm$ 1e-4) & 0.273 ($\pm$ 7e-4) & 0.522 ($\pm$ 8e-4) \\
    \midrule
    \multirow[c]{3}{*}{CIFAR10} & PHATE & 0.087 ($\pm$ 0.004) & 0.039 ($\pm$ 0.001) & 0.104 ($\pm$ 0.003) & \textbf{0.102} ($\pm$ 0.003) & 0.106 ($\pm$ 0.004) & 0.104 ($\pm$ 0.003) & 0.143 ($\pm$ 0.003) & & 0.867 ($\pm$ 4e-4) & \textbf{0.940} ($\pm$ 0e-4) & 0.235 ($\pm$ 3e-4) & 0.659 ($\pm$ 2e-4) \\
     & Finsler PHATE (Sym. Ker.) & \textbf{0.088} ($\pm$ 0.004) & \textbf{0.040} ($\pm$ 0.002) & \textbf{0.105} ($\pm$ 0.004) & \textbf{0.102} ($\pm$ 0.004) & \textbf{0.107} ($\pm$ 0.004) & \textbf{0.105} ($\pm$ 0.004) & \textbf{0.148} ($\pm$ 0.003) & & 0.871 ($\pm$ 3e-4) & 0.936 ($\pm$ 0e-4) & \textbf{0.236} ($\pm$ 4e-4) & 0.612 ($\pm$ 2e-4) \\
     & Finsler PHATE (Asym. Ker.) & 0.068 ($\pm$ 0.002) & 0.030 ($\pm$ 0.001) & 0.086 ($\pm$ 0.002) & 0.084 ($\pm$ 0.002) & 0.088 ($\pm$ 0.003) & 0.086 ($\pm$ 0.002) & 0.137 ($\pm$ 0.002) & & \textbf{0.877} ($\pm$ 3e-4) & 0.924 ($\pm$ 1e-4) & 0.215 ($\pm$ 5e-4) & \textbf{0.661} ($\pm$ 8e-4) \\
    \midrule
    \multirow[c]{3}{*}{CIFAR100} & PHATE & 0.031 ($\pm$ 0.002) & 0.011 ($\pm$ 0.001) & \textbf{0.508} ($\pm$ 0.002) & \textbf{0.505} ($\pm$ 0.003) & \textbf{0.511} ($\pm$ 0.001) & \textbf{0.508} ($\pm$ 0.002) & \textbf{0.022} ($\pm$ 0.001) & & 0.868 ($\pm$ 2e-4) & \textbf{0.940} ($\pm$ 0e-4) & 0.223 ($\pm$ 5e-4) & \textbf{0.646} ($\pm$ 3e-4) \\
     & Finsler PHATE (Sym. Ker.) & 0.030 ($\pm$ 0.003) & 0.011 ($\pm$ 0.001) & 0.507 ($\pm$ 0.002) & 0.503 ($\pm$ 0.003) & 0.510 ($\pm$ 0.002) & 0.507 ($\pm$ 0.002) & \textbf{0.022} ($\pm$ 0.001) & & 0.879 ($\pm$ 2e-4) & 0.936 ($\pm$ 0e-4) & 0.226 ($\pm$ 5e-4) & 0.601 ($\pm$ 3e-4) \\
     & Finsler PHATE (Asym. Ker.) & \textbf{0.033} ($\pm$ 0.004) & \textbf{0.012} ($\pm$ 0.002) & 0.504 ($\pm$ 0.003) & 0.498 ($\pm$ 0.004) & 0.509 ($\pm$ 0.003) & 0.504 ($\pm$ 0.003) & \textbf{0.022} ($\pm$ 0.002) & & \textbf{0.884} ($\pm$ 15e-4) & 0.922 ($\pm$ 18e-4) & \textbf{0.232} ($\pm$ 37e-4) & 0.569 ($\pm$ 56e-4) \\
    \midrule
    \multirow[c]{3}{*}{DTD} & PHATE & 0.323 ($\pm$ 0.008) & 0.142 ($\pm$ 0.005) & 0.495 ($\pm$ 0.005) & 0.484 ($\pm$ 0.005) & 0.506 ($\pm$ 0.006) & 0.495 ($\pm$ 0.005) & 0.163 ($\pm$ 0.005) & & 0.762 ($\pm$ 12e-4) & 0.888 ($\pm$ 4e-4) & 0.226 ($\pm$ 14e-4) & \textbf{0.486} ($\pm$ 10e-4) \\
     & Finsler PHATE (Sym. Ker.) & \textbf{0.345} ($\pm$ 0.006) & \textbf{0.153} ($\pm$ 0.007) & \textbf{0.506} ($\pm$ 0.007) & \textbf{0.490} ($\pm$ 0.009) & 0.522 ($\pm$ 0.005) & \textbf{0.506} ($\pm$ 0.007) & 0.176 ($\pm$ 0.006) & & 0.798 ($\pm$ 32e-4) & \textbf{0.890} ($\pm$ 11e-4) & 0.265 ($\pm$ 26e-4) & 0.481 ($\pm$ 20e-4) \\
     & Finsler PHATE (Asym. Ker.) & 0.343 ($\pm$ 0.005) & 0.152 ($\pm$ 0.004) & 0.499 ($\pm$ 0.005) & 0.476 ($\pm$ 0.007) & \textbf{0.525} ($\pm$ 0.004) & 0.499 ($\pm$ 0.005) & \textbf{0.184} ($\pm$ 0.004) & & \textbf{0.804} ($\pm$ 10e-4) & 0.878 ($\pm$ 8e-4) & \textbf{0.283} ($\pm$ 18e-4) & 0.454 ($\pm$ 17e-4) \\
    \midrule
    \multirow[c]{3}{*}{Caltech101} & PHATE & 0.516 ($\pm$ 0.010) & 0.283 ($\pm$ 0.016) & 0.731 ($\pm$ 0.006) & 0.747 ($\pm$ 0.007) & 0.716 ($\pm$ 0.005) & 0.731 ($\pm$ 0.006) & 0.310 ($\pm$ 0.016) & & 0.861 ($\pm$ 43e-4) & \textbf{0.933} ($\pm$ 17e-4) & 0.353 ($\pm$ 59e-4) & \textbf{0.441} ($\pm$ 26e-4) \\
     & Finsler PHATE (Sym. Ker.) & 0.563 ($\pm$ 0.006) & 0.325 ($\pm$ 0.011) & 0.752 ($\pm$ 0.004) & 0.760 ($\pm$ 0.005) & 0.744 ($\pm$ 0.003) & 0.752 ($\pm$ 0.004) & 0.347 ($\pm$ 0.012) & & \textbf{0.884} ($\pm$ 3e-4) & 0.931 ($\pm$ 2e-4) & \textbf{0.394} ($\pm$ 16e-4) & 0.438 ($\pm$ 17e-4) \\
     & Finsler PHATE (Asym. Ker.) & \textbf{0.626} ($\pm$ 0.006) & \textbf{0.631} ($\pm$ 0.020) & \textbf{0.779} ($\pm$ 0.003) & \textbf{0.765} ($\pm$ 0.002) & \textbf{0.794} ($\pm$ 0.004) & \textbf{0.779} ($\pm$ 0.003) & \textbf{0.642} ($\pm$ 0.020) & & 0.871 ($\pm$ 10e-4) & 0.918 ($\pm$ 3e-4) & 0.391 ($\pm$ 13e-4) & 0.426 ($\pm$ 19e-4) \\
    \midrule
    \multirow[c]{3}{*}{Caltech256} & PHATE & 0.197 ($\pm$ 0.006) & 0.132 ($\pm$ 0.006) & 0.760 ($\pm$ 0.002) & 0.764 ($\pm$ 0.002) & 0.757 ($\pm$ 0.002) & 0.760 ($\pm$ 0.002) & 0.137 ($\pm$ 0.006) & & 0.775 ($\pm$ 7e-4) & \textbf{0.885} ($\pm$ 3e-4) & 0.234 ($\pm$ 11e-4) & \textbf{0.443} ($\pm$ 18e-4) \\
     & Finsler PHATE (Sym. Ker.) & 0.216 ($\pm$ 0.008) & 0.139 ($\pm$ 0.007) & 0.762 ($\pm$ 0.002) & 0.762 ($\pm$ 0.002) & 0.762 ($\pm$ 0.002) & 0.762 ($\pm$ 0.002) & 0.144 ($\pm$ 0.007) & & 0.790 ($\pm$ 9e-4) & 0.874 ($\pm$ 3e-4) & 0.253 ($\pm$ 9e-4) & 0.399 ($\pm$ 23e-4) \\
     & Finsler PHATE (Asym. Ker.) & \textbf{0.268} ($\pm$ 0.005) & \textbf{0.238} ($\pm$ 0.019) & \textbf{0.773} ($\pm$ 0.002) & \textbf{0.769} ($\pm$ 0.002) & \textbf{0.777} ($\pm$ 0.002) & \textbf{0.773} ($\pm$ 0.002) & \textbf{0.242} ($\pm$ 0.019) & & \textbf{0.791} ($\pm$ 12e-4) & 0.860 ($\pm$ 4e-4) & \textbf{0.267} ($\pm$ 21e-4) & 0.342 ($\pm$ 21e-4) \\
    \midrule
    \multirow[c]{3}{*}{OxfordFlowers102} & PHATE & 0.255 ($\pm$ 0.005) & \textbf{0.114} ($\pm$ 0.005) & \textbf{0.605} ($\pm$ 0.003) & \textbf{0.583} ($\pm$ 0.004) & 0.629 ($\pm$ 0.003) & \textbf{0.605} ($\pm$ 0.003) & 0.130 ($\pm$ 0.005) & & 0.751 ($\pm$ 13e-4) & \textbf{0.863} ($\pm$ 4e-4) & 0.180 ($\pm$ 13e-4) & \textbf{0.508} ($\pm$ 4e-4) \\
     & Finsler PHATE (Sym. Ker.) & \textbf{0.280} ($\pm$ 0.007) & 0.111 ($\pm$ 0.004) & 0.601 ($\pm$ 0.004) & 0.568 ($\pm$ 0.005) & \textbf{0.639} ($\pm$ 0.004) & 0.601 ($\pm$ 0.004) & \textbf{0.135} ($\pm$ 0.004) & & \textbf{0.783} ($\pm$ 19e-4) & 0.862 ($\pm$ 5e-4) & \textbf{0.202} ($\pm$ 20e-4) & 0.498 ($\pm$ 5e-4) \\
     & Finsler PHATE (Asym. Ker.) & 0.262 ($\pm$ 0.004) & 0.110 ($\pm$ 0.003) & 0.597 ($\pm$ 0.003) & 0.567 ($\pm$ 0.004) & 0.630 ($\pm$ 0.003) & 0.597 ($\pm$ 0.003) & 0.130 ($\pm$ 0.002) & & 0.768 ($\pm$ 23e-4) & 0.854 ($\pm$ 12e-4) & 0.191 ($\pm$ 21e-4) & 0.500 ($\pm$ 26e-4) \\
    \midrule
    \multirow[c]{3}{*}{OxfordIIITPet} & PHATE & 0.633 ($\pm$ 0.006) & 0.404 ($\pm$ 0.012) & 0.705 ($\pm$ 0.005) & 0.702 ($\pm$ 0.006) & 0.708 ($\pm$ 0.005) & 0.705 ($\pm$ 0.005) & 0.421 ($\pm$ 0.011) & & 0.892 ($\pm$ 20e-4) & 0.969 ($\pm$ 2e-4) & 0.344 ($\pm$ 22e-4) & 0.783 ($\pm$ 9e-4) \\
     & Finsler PHATE (Sym. Ker.) & \textbf{0.851} ($\pm$ 0.007) & \textbf{0.719} ($\pm$ 0.015) & \textbf{0.880} ($\pm$ 0.005) & \textbf{0.871} ($\pm$ 0.005) & \textbf{0.889} ($\pm$ 0.006) & \textbf{0.880} ($\pm$ 0.005) & \textbf{0.728} ($\pm$ 0.015) & & \textbf{0.974} ($\pm$ 15e-4) & \textbf{0.985} ($\pm$ 3e-4) & \textbf{0.561} ($\pm$ 23e-4) & \textbf{0.785} ($\pm$ 34e-4) \\
     & Finsler PHATE (Asym. Ker.) & 0.785 ($\pm$ 0.004) & 0.447 ($\pm$ 0.017) & 0.822 ($\pm$ 0.004) & 0.765 ($\pm$ 0.006) & 0.888 ($\pm$ 0.005) & 0.822 ($\pm$ 0.004) & 0.518 ($\pm$ 0.011) & & 0.966 ($\pm$ 8e-4) & 0.977 ($\pm$ 3e-4) & 0.515 ($\pm$ 20e-4) & 0.618 ($\pm$ 40e-4) \\
    \midrule
    \multirow[c]{3}{*}{Imagenette} & PHATE & 0.722 ($\pm$ 0.019) & 0.600 ($\pm$ 0.026) & 0.727 ($\pm$ 0.019) & 0.722 ($\pm$ 0.019) & 0.732 ($\pm$ 0.018) & 0.727 ($\pm$ 0.019) & 0.641 ($\pm$ 0.023) & & 0.903 ($\pm$ 115e-4) & 0.952 ($\pm$ 22e-4) & 0.281 ($\pm$ 67e-4) & 0.463 ($\pm$ 79e-4) \\
     & Finsler PHATE (Sym. Ker.) & \textbf{0.909} ($\pm$ 0.023) & \textbf{0.874} ($\pm$ 0.057) & \textbf{0.910} ($\pm$ 0.023) & \textbf{0.907} ($\pm$ 0.027) & \textbf{0.914} ($\pm$ 0.019) & \textbf{0.910} ($\pm$ 0.023) & \textbf{0.888} ($\pm$ 0.051) & & \textbf{0.964} ($\pm$ 3e-4) & \textbf{0.979} ($\pm$ 1e-4) & \textbf{0.359} ($\pm$ 15e-4) & \textbf{0.690} ($\pm$ 18e-4) \\
     & Finsler PHATE (Asym. Ker.) & 0.896 ($\pm$ 0.010) & 0.829 ($\pm$ 0.030) & 0.898 ($\pm$ 0.010) & 0.886 ($\pm$ 0.012) & 0.910 ($\pm$ 0.009) & 0.898 ($\pm$ 0.010) & 0.848 ($\pm$ 0.026) & & 0.960 ($\pm$ 3e-4) & 0.975 ($\pm$ 1e-4) & 0.355 ($\pm$ 17e-4) & 0.573 ($\pm$ 30e-4) \\
    \bottomrule
    \end{tabular}%
    }
\end{table}

We can thus accurately claim that our asymmetric Finsler pipeline is beneficial compared to the traditional symmetric one.

\begin{figure}[t]
    \centering
    \includegraphics[width=0.7\columnwidth]{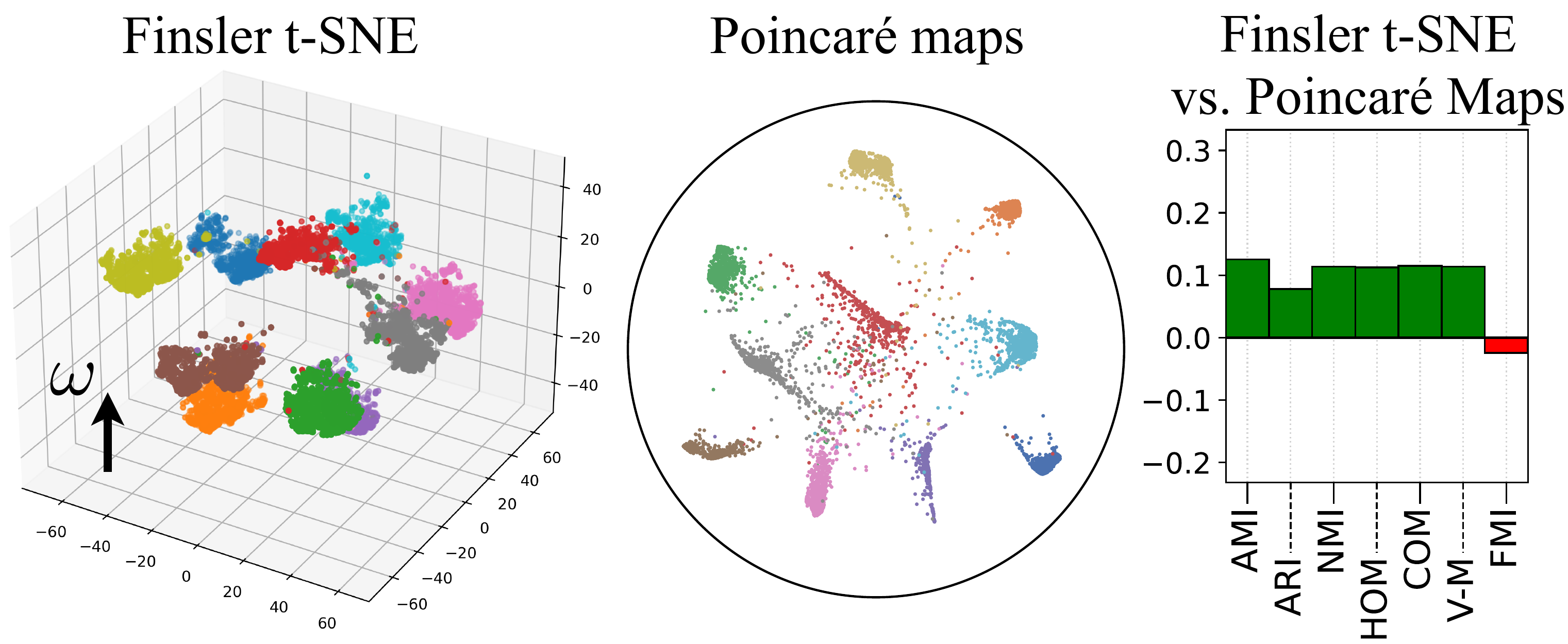}
    \caption{
    Imagenette: Finsler t-SNE vs.\ Poincaré maps,
    a 
    non-Euclidean
    hyperbolic baseline.
    Right: percentage difference in kMeans scores (positive/green means Finsler t-SNE is better). The gains reveal that modelling asymmetry with Finsler geometry may capture structure better than 
    other
    non-Euclidean embeddings. 
    }
    \label{fig: poincare maps imagenette}
\end{figure}

\begin{table*}[t]
\caption{Mean clustering performance using kMeans on UMAP and Finsler UMAP data embeddings. The mean and standard deviation over 10 runs is presented. (+)/(-) signifies that higher/lower is better.}
\label{tab: classif kmeans on umap finsler umap}
\centering
\resizebox{\textwidth}{!}{%
\begin{tabular}{llccccccccccc}
\toprule
 &  & \multicolumn{7}{c}{Label-related scores} &  & \multicolumn{3}{c}{Label-unrelated scores} \\
\cmidrule(lr){3-9}\cmidrule(lr){11-13}
Dataset & Method        & AMI (+)            & ARI (+)            & NMI (+)            & HOM (+)            & COM (+)            & V-M (+)           & FMI (+)            &  & SIL (+)            & DBI (–)            & CHI (+)             \\
\midrule
\multirow{2}{*}{Iris}           & UMAP         & 0,724 ($\pm$0,248)   & 0,688 ($\pm$0,241)   & 0,728 ($\pm$0,245)   & 0,719 ($\pm$0,243)   & 0,737 ($\pm$0,247)   & 0,728 ($\pm$0,245)   & 0,794 ($\pm$0,159)   &  & \textbf{0,721} ($\pm$0,019)   & \textbf{0,389} ($\pm$0,025)   & \textbf{2526} ($\pm$1202)    \\
                                & Finsler UMAP (Ours) & \textbf{0,794} ($\pm$0,024)   & \textbf{0,746} ($\pm$0,032)   & \textbf{0,796} ($\pm$0,023)   & \textbf{0,785} ($\pm$0,026)   & \textbf{0,808} ($\pm$0,020)   & \textbf{0,796} ($\pm$0,023)   & \textbf{0,832} ($\pm$0,020)   &  & 0,705 ($\pm$0,036)   & 0,412 ($\pm$0,047)   & 1471 ($\pm$331)     \\
\midrule
\multirow{2}{*}{MNIST}          & UMAP         & 0,847 ($\pm$0,017)   & 0,782 ($\pm$0,040)   & 0,847 ($\pm$0,017)   & 0,839 ($\pm$0,019)   & 0,855 ($\pm$0,015)   & 0,847 ($\pm$0,017)   & 0,805 ($\pm$0,035)   &  & 0,580 ($\pm$0,018)   & 0,606 ($\pm$0,033)   & \textbf{194060} ($\pm$19078) \\
                                & Finsler UMAP (Ours) & \textbf{0,858} ($\pm$0,012)   & \textbf{0,793} ($\pm$0,032)   & \textbf{0,858} ($\pm$0,012)   & \textbf{0,848} ($\pm$0,014)   & \textbf{0,868} ($\pm$0,011)   & \textbf{0,858} ($\pm$0,012)   & \textbf{0,816} ($\pm$0,028)   &  & \textbf{0,611} ($\pm$0,029)   & \textbf{0,585} ($\pm$0,060)   & 119299 ($\pm$10201) \\
\midrule
\multirow{2}{*}{Fashion-MNIST}         & UMAP         & 0,616 ($\pm$0,009)   & 0,445 ($\pm$0,028)   & 0,617 ($\pm$0,009)   & 0,610 ($\pm$0,011)   & 0,623 ($\pm$0,008)   & 0,617 ($\pm$0,009)   & 0,503 ($\pm$0,024)   &  & 0,491 ($\pm$0,017)   & 0,738 ($\pm$0,039)   & \textbf{173803} ($\pm$3963)  \\
                                & Finsler UMAP (Ours) & \textbf{0,640} ($\pm$0,012)   & \textbf{0,474} ($\pm$0,030)   & \textbf{0,640} ($\pm$0,012)   & \textbf{0,627} ($\pm$0,014)   & \textbf{0,653} ($\pm$0,014)   & \textbf{0,640} ($\pm$0,012)   & \textbf{0,533} ($\pm$0,024)   &  & \textbf{0,504} ($\pm$0,014)   & \textbf{0,736} ($\pm$0,030)   & 97213 ($\pm$13424)  \\
\midrule
\multirow{2}{*}{Kuzushiji-MNIST}         & UMAP         & 0,644 ($\pm$0,019)   & 0,497 ($\pm$0,025)   & 0,644 ($\pm$0,019)   & 0,636 ($\pm$0,019)   & 0,653 ($\pm$0,018)   & 0,644 ($\pm$0,019)   & 0,551 ($\pm$0,022)   &  & 0,556 ($\pm$0,029)   & \textbf{0,627} ($\pm$0,068)   & \textbf{117390} ($\pm$10225) \\
                                & Finsler UMAP (Ours) & \textbf{0,653} ($\pm$0,019)   & \textbf{0,504} ($\pm$0,020)   & \textbf{0,653} ($\pm$0,019)   & \textbf{0,643} ($\pm$0,018)   & \textbf{0,663} ($\pm$0,021)   & \textbf{0,653} ($\pm$0,019)   & \textbf{0,558} ($\pm$0,018)   &  & \textbf{0,571} ($\pm$0,012)   & 0,646 ($\pm$0,033)   & 60490 ($\pm$6508)   \\
\midrule
\multirow{2}{*}{EMNIST}         & UMAP         & 0,848 ($\pm$0,003)   & 0,790 ($\pm$0,019)   & 0,848 ($\pm$0,003)   & 0,842 ($\pm$0,007)   & 0,854 ($\pm$0,003)   & 0,848 ($\pm$0,003)   & 0,812 ($\pm$0,016)   &  & 0,608 ($\pm$0,005)   & 0,569 ($\pm$0,011)   & \textbf{175759} ($\pm$5755)  \\
                                & Finsler UMAP (Ours) & \textbf{0,887} ($\pm$0,024)   & \textbf{0,865} ($\pm$0,056)   & \textbf{0,887} ($\pm$0,024)   & \textbf{0,884} ($\pm$0,028)   & \textbf{0,891} ($\pm$0,020)   & \textbf{0,887} ($\pm$0,024)   & \textbf{0,879} ($\pm$0,050)   &  & \textbf{0,633} ($\pm$0,017)   & \textbf{0,563} ($\pm$0,037)   & 114745 ($\pm$8031)  \\
\midrule
\multirow{2}{*}{EMNIST-Balanced}       & UMAP         & 0,641 ($\pm$0,004)   & 0,396 ($\pm$0,008)   & 0,642 ($\pm$0,004)   & 0,639 ($\pm$0,004)   & 0,644 ($\pm$0,004)   & 0,642 ($\pm$0,004)   & 0,409 ($\pm$0,008)   &  & 0,419 ($\pm$0,004)   & \textbf{0,764} ($\pm$0,014)   & \textbf{167572} ($\pm$3708)  \\
                                & Finsler UMAP (Ours) & \textbf{0,674} ($\pm$0,007)   & \textbf{0,447} ($\pm$0,012)   & \textbf{0,674} ($\pm$0,007)   & \textbf{0,671} ($\pm$0,007)   & \textbf{0,678} ($\pm$0,007)   & \textbf{0,674} ($\pm$0,007)   & \textbf{0,459} ($\pm$0,012)   &  & \textbf{0,433} ($\pm$0,006)   & 0,824 ($\pm$0,012)   & 92918 ($\pm$5118)   \\
\midrule
\multirow{2}{*}{CIFAR10}   & UMAP         & 0,654 ($\pm$0,327)   & 0,619 ($\pm$0,310)   & 0,654 ($\pm$0,327)   & 0,651 ($\pm$0,325)   & 0,656 ($\pm$0,328)   & 0,654 ($\pm$0,327)   & 0,658 ($\pm$0,279)   &  & 0,585 ($\pm$0,115)   & 0,570 ($\pm$0,128)   & \textbf{239743} ($\pm$98339) \\
                                & Finsler UMAP (Ours) & \textbf{0,817} ($\pm$0,009)   & \textbf{0,767} ($\pm$0,029)   & \textbf{0,817} ($\pm$0,009)   & \textbf{0,812} ($\pm$0,013)   & \textbf{0,822} ($\pm$0,005)   & \textbf{0,817} ($\pm$0,009)   & \textbf{0,791} ($\pm$0,024)   &  & \textbf{0,647} ($\pm$0,016)   & \textbf{0,516} ($\pm$0,033)   & 187999 ($\pm$11072) \\
\midrule
\multirow{2}{*}{CIFAR100}  & UMAP         & 0,549 ($\pm$0,183)   & 0,315 ($\pm$0,105)   & 0,560 ($\pm$0,179)   & 0,556 ($\pm$0,178)   & 0,563 ($\pm$0,180)   & 0,560 ($\pm$0,179)   & 0,322 ($\pm$0,104)   &  & \textbf{0,444} ($\pm$0,036)   & \textbf{0,715} ($\pm$0,030)   & \textbf{123871} ($\pm$23089) \\
                                & Finsler UMAP (Ours) & \textbf{0,616} ($\pm$0,002)   & \textbf{0,362} ($\pm$0,005)   & \textbf{0,624} ($\pm$0,002)   & \textbf{0,620} ($\pm$0,003)   & \textbf{0,629} ($\pm$0,002)   & \textbf{0,624} ($\pm$0,002)   & \textbf{0,369} ($\pm$0,005)   &  & 0,426 ($\pm$0,002)   & 0,822 ($\pm$0,013)   & 61756 ($\pm$2073)   \\
\midrule
\multirow{2}{*}{DTD}   & UMAP         & 0,402 ($\pm$0,202)   & 0,237 ($\pm$0,119)   & 0,501 ($\pm$0,169)   & 0,498 ($\pm$0,168)   & 0,504 ($\pm$0,170)   & 0,501 ($\pm$0,169)   & 0,254 ($\pm$0,116)   &  & \textbf{0,480} ($\pm$0,006)   & \textbf{0,686} ($\pm$0,011)   & \textbf{3690} ($\pm$211)     \\
                                & Finsler UMAP (Ours) & \textbf{0,463} ($\pm$0,154)   & \textbf{0,274} ($\pm$0,092)   & \textbf{0,551} ($\pm$0,129)   & \textbf{0,546} ($\pm$0,127)   & \textbf{0,555} ($\pm$0,130)   & \textbf{0,551} ($\pm$0,129)   & \textbf{0,291} ($\pm$0,090)   &  & 0,449 ($\pm$0,007)   & 0,782 ($\pm$0,020)   & 1571 ($\pm$43)      \\
\midrule
\multirow{2}{*}{Caltech101} & UMAP    & 0,737 ($\pm$0,246)   & 0,389 ($\pm$0,130)   & 0,771 ($\pm$0,214)   & 0,800 ($\pm$0,222)   & 0,744 ($\pm$0,206)   & 0,771 ($\pm$0,214)   & 0,435 ($\pm$0,139)   &  & \textbf{0,677} ($\pm$0,006)   & \textbf{0,419} ($\pm$0,008)   & \textbf{154548} ($\pm$7934)  \\
                                & Finsler UMAP (Ours) & \textbf{0,821} ($\pm$0,004)   & \textbf{0,432} ($\pm$0,016)   & \textbf{0,845} ($\pm$0,003)   & \textbf{0,876} ($\pm$0,003)   & \textbf{0,815} ($\pm$0,004)   & \textbf{0,845} ($\pm$0,003)   & \textbf{0,481} ($\pm$0,016)   &  & 0,651 ($\pm$0,007)   & 0,476 ($\pm$0,017)   & 77907 ($\pm$3926)   \\
\midrule
\multirow{2}{*}{Caltech256} & UMAP    & 0,579 ($\pm$0,290)   & 0,366 ($\pm$0,183)   & 0,655 ($\pm$0,238)   & 0,658 ($\pm$0,239)   & 0,652 ($\pm$0,237)   & 0,655 ($\pm$0,238)   & 0,372 ($\pm$0,184)   &  & \textbf{0,590} ($\pm$0,005)   & \textbf{0,553} ($\pm$0,010)   & \textbf{150961} ($\pm$5464)  \\
                                & Finsler UMAP (Ours) & \textbf{0,743} ($\pm$0,001)   & \textbf{0,487} ($\pm$0,009)   & \textbf{0,789} ($\pm$0,001)   & \textbf{0,791} ($\pm$0,001)   & \textbf{0,786} ($\pm$0,001)   & \textbf{0,789} ($\pm$0,001)   & \textbf{0,492} ($\pm$0,009)   &  & 0,604 ($\pm$0,011)   & 0,553 ($\pm$0,010)   & 55076 ($\pm$1848)   \\
\midrule
\multirow{2}{*}{OxfordFlowers102} & UMAP & 0,423 ($\pm$0,212)   & 0,315 ($\pm$0,158)   & 0,713 ($\pm$0,105)   & 0,707 ($\pm$0,104)   & 0,719 ($\pm$0,106)   & 0,713 ($\pm$0,105)   & 0,322 ($\pm$0,157)   &  & \textbf{0,483} ($\pm$0,007)   & \textbf{0,665} ($\pm$0,014)   & 2583 ($\pm$122)     \\
                                & Finsler UMAP (Ours) & \textbf{0,571} ($\pm$0,005)   & \textbf{0,429} ($\pm$0,008)   & \textbf{0,785} ($\pm$0,002)   & \textbf{0,775} ($\pm$0,003)   & \textbf{0,795} ($\pm$0,003)   & \textbf{0,785} ($\pm$0,002)   & \textbf{0,437} ($\pm$0,007)   &  & 0,483 ($\pm$0,009)   & 0,722 ($\pm$0,019)   & \textbf{3652} ($\pm$526)     \\
\midrule
\multirow{2}{*}{OxfordIIITPet} & UMAP & 0,720 ($\pm$0,362)   & 0,646 ($\pm$0,323)   & 0,735 ($\pm$0,342)   & 0,732 ($\pm$0,341)   & 0,739 ($\pm$0,344)   & 0,735 ($\pm$0,342)   & 0,656 ($\pm$0,315)   &  & \textbf{0,701} ($\pm$0,006)   & \textbf{0,426} ($\pm$0,024)   & \textbf{105243} ($\pm$5334)  \\
                                & Finsler UMAP (Ours) & \textbf{0,885} ($\pm$0,005)   & \textbf{0,760} ($\pm$0,016)   & \textbf{0,892} ($\pm$0,005)   & \textbf{0,884} ($\pm$0,006)   & \textbf{0,899} ($\pm$0,005)   & \textbf{0,892} ($\pm$0,005)   & \textbf{0,768} ($\pm$0,015)   &  & 0,670 ($\pm$0,013)   & 0,478 ($\pm$0,040)   & 86351 ($\pm$6526)   \\
\midrule
\multirow{2}{*}{GTSRB} & UMAP          & 0,444 ($\pm$0,222)   & 0,206 ($\pm$0,104)   & 0,449 ($\pm$0,220)   & 0,455 ($\pm$0,223)   & 0,443 ($\pm$0,217)   & 0,449 ($\pm$0,220)   & 0,233 ($\pm$0,100)   &  & \textbf{0,411} ($\pm$0,008)   & \textbf{0,779} ($\pm$0,018)   & \textbf{34781} ($\pm$1412)   \\
                                & Finsler UMAP (Ours) & \textbf{0,515} ($\pm$0,010)   & \textbf{0,231} ($\pm$0,010)   & \textbf{0,520} ($\pm$0,010)   & \textbf{0,523} ($\pm$0,010)   & \textbf{0,517} ($\pm$0,010)   & \textbf{0,520} ($\pm$0,010)   & \textbf{0,258} ($\pm$0,010)   &  & 0,373 ($\pm$0,007)   & 0,930 ($\pm$0,024)   & 9529 ($\pm$430)     \\
\midrule
\multirow{2}{*}{Imagenette} & UMAP    & 0,750 ($\pm$0,375)   & 0,756 ($\pm$0,378)   & 0,751 ($\pm$0,374)   & 0,751 ($\pm$0,375)   & 0,751 ($\pm$0,374)   & 0,751 ($\pm$0,374)   & 0,781 ($\pm$0,340)   &  & \textbf{0,830} ($\pm$0,005)   & \textbf{0,253} ($\pm$0,007)   & \textbf{161823} ($\pm$11371) \\
                                & Finsler UMAP (Ours) & \textbf{0,846} ($\pm$0,282)   & \textbf{0,852} ($\pm$0,284)   & \textbf{0,846} ($\pm$0,281)   & \textbf{0,846} ($\pm$0,281)   & \textbf{0,846} ($\pm$0,281)   & \textbf{0,846} ($\pm$0,281)   & \textbf{0,867} ($\pm$0,255)   &  & 0,828 ($\pm$0,006)   & 0,268 ($\pm$0,016)   & 97813 ($\pm$6251)   \\
\midrule
\multirow{2}{*}{Imagenet} & UMAP    & 0,784 ($\pm$0,000)   & 0,506 ($\pm$0,001)   & 0,797 ($\pm$0,000)   & 0,793 ($\pm$0,000)   & 0,802 ($\pm$0,000)   & 0,797 ($\pm$0,000)   & 0,508 ($\pm$0,001)   &  & 0,582 ($\pm$0,002)   & 0,570 ($\pm$0,003)   & \textbf{7315700} ($\pm$51432)\\
                                & Finsler UMAP (Ours) & \textbf{0,803} ($\pm$0,001)   & \textbf{0,540} ($\pm$0,003)   & \textbf{0,816} ($\pm$0,001)   & \textbf{0,810} ($\pm$0,001)   & \textbf{0,821} ($\pm$0,001)   & \textbf{0,816} ($\pm$0,001)   & \textbf{0,542} ($\pm$0,003)   &  & \textbf{0,615} ($\pm$0,004)   & \textbf{0,565} ($\pm$0,007)   & 1930520 ($\pm$28049)\\
\bottomrule
\end{tabular}%
}
\end{table*}

\begin{table*}[t]
\caption{Clustering performance using kMeans on t-SNE and Finsler t-SNE embeddings across the smaller datasets. (+)/(-) signifies that higher/lower is better.}
\label{tab: classif kmeans on tsne finsler tsne}
\centering
\resizebox{\textwidth}{!}{%
\begin{tabular}{llccccccccccc}
\toprule
 &  & \multicolumn{7}{c}{Label-related scores} &  & \multicolumn{3}{c}{Label-unrelated scores} \\
\cmidrule(lr){3-9} \cmidrule(lr){11-13}
Dataset & Method           & AMI (+)        & ARI (+)        & NMI (+)        & HOM (+)        & COM (+)        & V-M (+)       & FMI (+)        &   & SIL (+)        & DBI (–)        & CHI (+)         \\
\midrule
\multirow{2}{*}{Iris}           & t-SNE            & 0,829          & 0,851          & 0,831          & 0,831          & 0,831          & 0,831          & 0,900          &   & \textbf{0,665} & \textbf{0,453} & 3028            \\
                                & Finsler t-SNE (Ours)   & \textbf{0,845} & \textbf{0,868} & \textbf{0,846} & \textbf{0,846} & \textbf{0,847} & \textbf{0,846} & \textbf{0,911} &   & 0,664          & 0,465          & \textbf{4822}   \\
\midrule
\multirow{2}{*}{DTD}   & t-SNE            & 0,498          & 0,288          & 0,581          & 0,578          & 0,584          & 0,581          & 0,303          &   & 0,402          & \textbf{0,804} & \textbf{2492}   \\
                                & Finsler t-SNE (Ours)   & \textbf{0,504} & \textbf{0,294} & \textbf{0,585} & \textbf{0,580} & \textbf{0,590} & \textbf{0,585} & \textbf{0,310} &   & \textbf{0,408} & 0,859          & 1186            \\
\midrule
\multirow{2}{*}{Caltech101} & t-SNE        & 0,804          & 0,377          & 0,830          & 0,865          & 0,798          & 0,830          & 0,428          &   & 0,547          & 0,610          & \textbf{23151}  \\
                                & Finsler t-SNE (Ours)   & \textbf{0,825} & \textbf{0,458} & \textbf{0,848} & \textbf{0,879} & \textbf{0,820} & \textbf{0,848} & \textbf{0,506} &   & \textbf{0,584} & \textbf{0,597} & 13419           \\
\midrule
\multirow{2}{*}{OxfordFlowers102} & t-SNE   & 0,549          & \textbf{0,388} & 0,775          & \textbf{0,767} & 0,783          & 0,775          & \textbf{0,397} &   & 0,419          & 0,740          & \textbf{1788}   \\
                                & Finsler t-SNE (Ours)   & \textbf{0,563} & 0,380          & \textbf{0,779} & 0,766          & \textbf{0,792} & \textbf{0,779} & 0,392          &   & \textbf{0,537} & \textbf{0,653} & 1690            \\
\midrule
\multirow{2}{*}{OxfordIIITPet} & t-SNE     & 0,899          & 0,814          & 0,904          & 0,900          & 0,908          & 0,904          & 0,819          &   & 0,607          & \textbf{0,554} & \textbf{15824}  \\
                                & Finsler t-SNE (Ours)   & \textbf{0,903} & \textbf{0,821} & \textbf{0,908} & \textbf{0,904} & \textbf{0,913} & \textbf{0,908} & \textbf{0,826} &   & \textbf{0,618} & 0,556          & 13912           \\
\midrule
\multirow{2}{*}{Imagenette} & t-SNE       & 0,939          & 0,946          & 0,939          & 0,939          & 0,939          & 0,939          & 0,951          &   & 0,610          & 0,536          & \textbf{22051}  \\
                                & Finsler t-SNE (Ours)   & \textbf{0,943} & \textbf{0,951} & \textbf{0,943} & \textbf{0,943} & \textbf{0,943} & \textbf{0,943} & \textbf{0,955} &   & \textbf{0,668} & \textbf{0,465} & 19651           \\
\bottomrule
\end{tabular}%
}
\end{table*}

\begin{figure*}
    \centering
    \begin{subfigure}{\textwidth}
        \adjincludegraphics[
            width=\textwidth, 
            clip,
        ]        {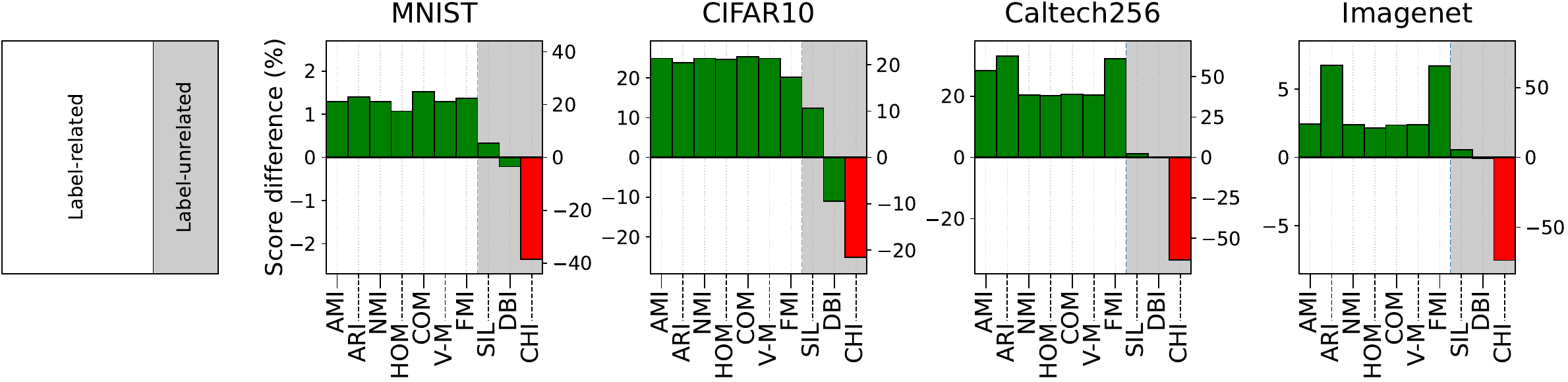}%
        \caption*{Finsler UMAP vs.\ UMAP}
    \end{subfigure}%
    \\[2em]
    \begin{subfigure}{\textwidth}
        \adjincludegraphics[
            width=\textwidth, 
            clip,
        ]   {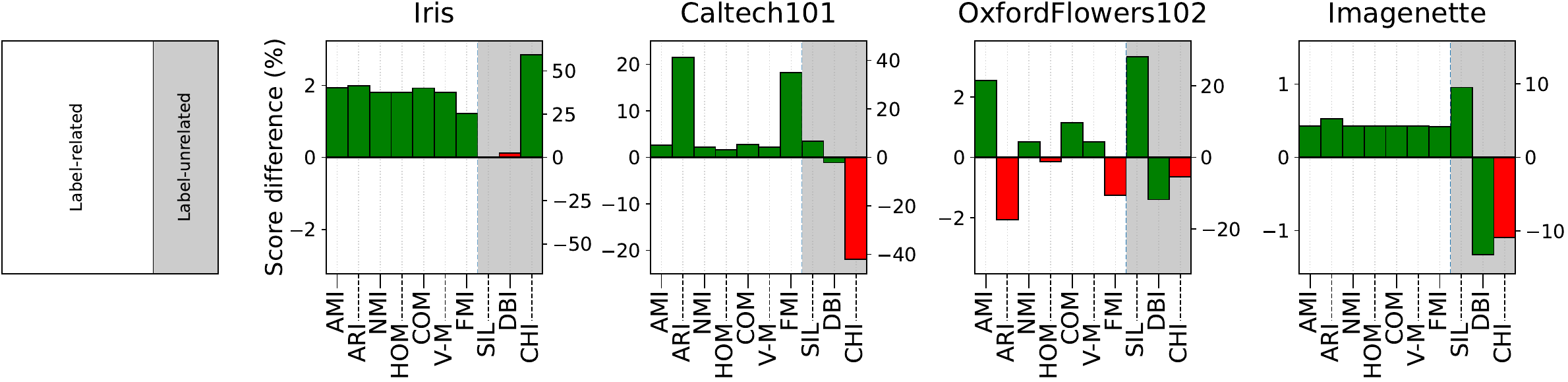}%
        \caption*{Finsler t-SNE vs.\ t-SNE}
    \end{subfigure}
    \caption{
    Percentage difference in mean performance between our Finsler methods and their traditional Euclidean baselines. Positive (resp.\ negative) differences, in green (resp.\ red) signifies we get better (resp.\ worse) scores evaluating the quality of kMeans clusters, either by comparing with groundtruth class labels or not.
    See Fig.~\ref{fig: classif kmeans on umap finsler umap + tsne finsler tsne all labels only zoom in}. Raw performance values and standard deviations are pushed to 
    Tabs~\ref{tab: classif kmeans on umap finsler umap} and \ref{tab: classif kmeans on tsne finsler tsne}.
    This figure is the same as Fig.~\ref{fig: classif kmeans on umap finsler umap + tsne finsler tsne label-related only}, except that we added to the label-related scores, in white background, the label-unrelated scores, in grey background. Having better label-unrelated score with worse label-related scores is particular bad, as it means that the method is artificially creating ``nicely looking'' clusters that do not correspond to the reality of the data manifold. This never happens for our Finsler methods, as we systematically have better label-related measures.
    }
    \label{fig: classif kmeans on umap finsler umap + tsne finsler tsne all labels}
\end{figure*}

\begin{figure*}[t]
    \centering
    \includegraphics[width=\textwidth]{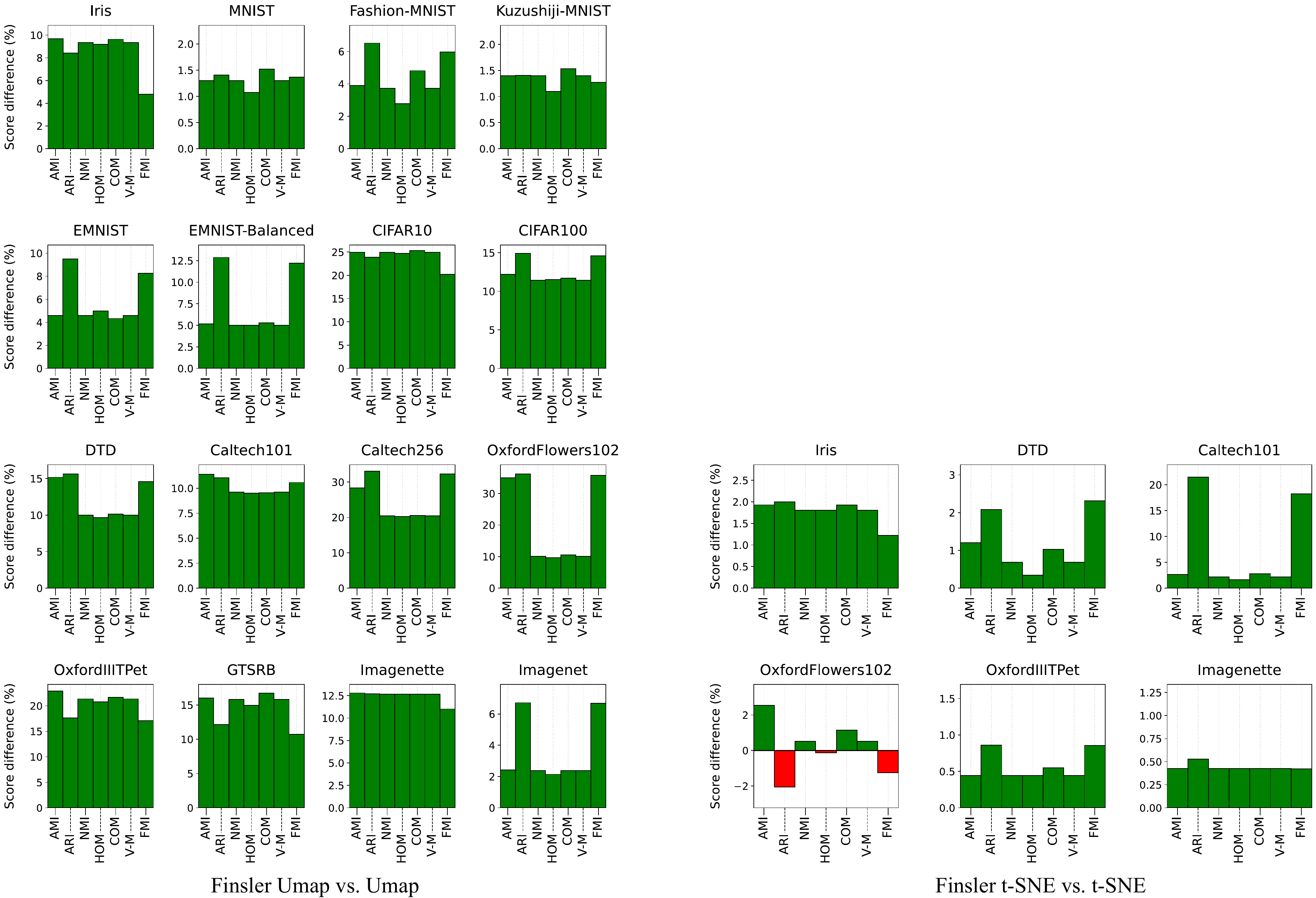}
    \caption{
    Complete version of
    Fig.~\ref{fig: classif kmeans on umap finsler umap + tsne finsler tsne label-related only} for all tested datasets, presenting the percentage difference in mean performance between our Finsler methods and their traditional Euclidean baselines. Positive (resp.\ negative) differences, in green (resp.\ red) signifies we get better (resp.\ worse) scores comparing kMeans clustering with groundtruth class labels. Raw performance values are in Tabs.~\ref{tab: classif kmeans on umap finsler umap} and \ref{tab: classif kmeans on tsne finsler tsne}.
    }
    \label{fig: classif kmeans on umap finsler umap + tsne finsler tsne label-related only zoom in}
\end{figure*}

\begin{figure*}[t]
    \centering
    \includegraphics[width=\textwidth]{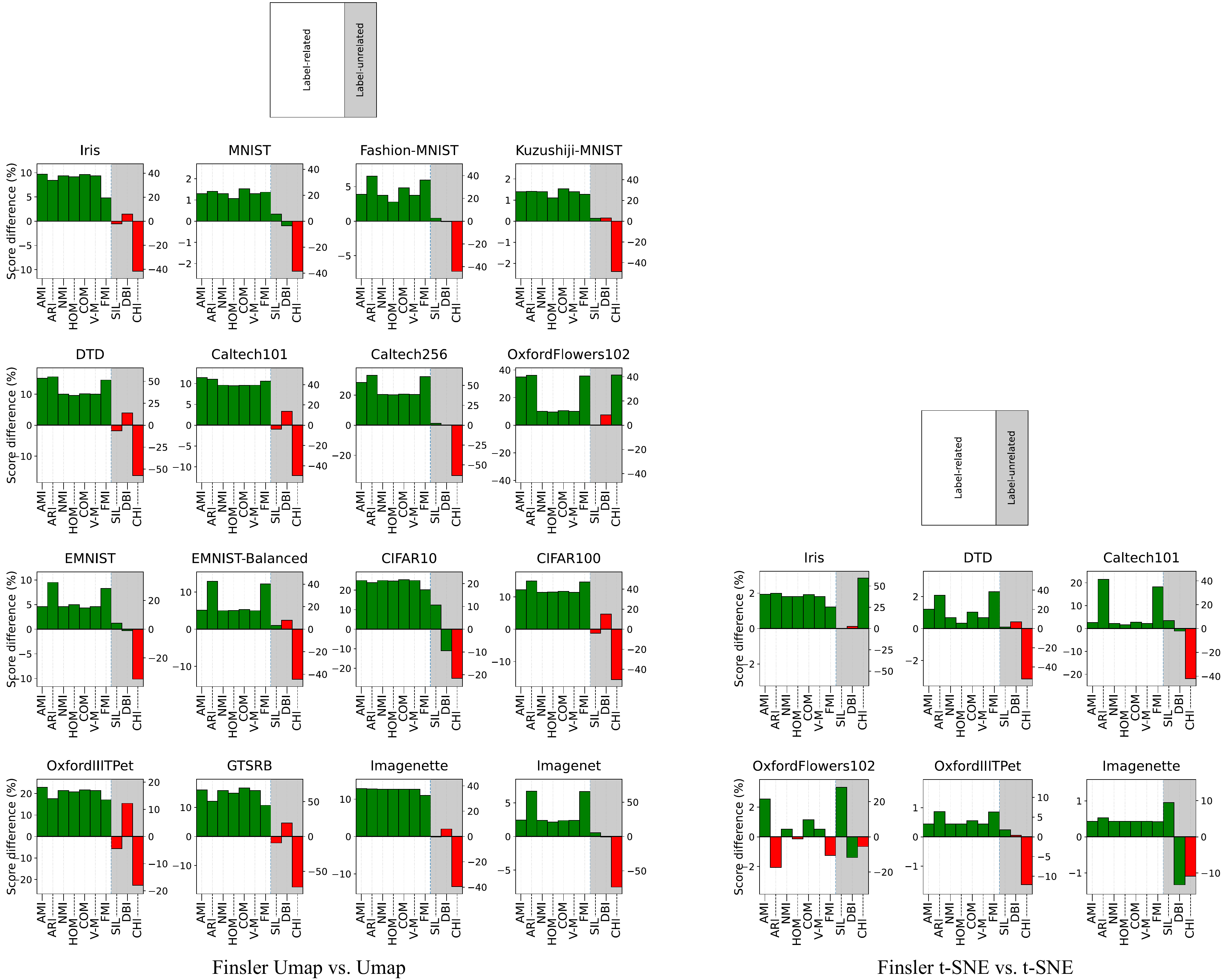}
    \caption{
    Complete version of Fig.~\ref{fig: classif kmeans on umap finsler umap + tsne finsler tsne all labels} for all datasets, presenting the percentage difference in mean performance between our Finsler methods and their traditional Euclidean baselines. Positive (resp.\ negative) differences, in green (resp.\ red) signifies we get better (resp.\ worse) evaluating the quality of kMeans clusters, either by comparing with groundtruth class labels or not. Raw performance values are in Tabs.~\ref{tab: classif kmeans on umap finsler umap} and \ref{tab: classif kmeans on tsne finsler tsne}.
    This figure is the same as Fig.~\ref{fig: classif kmeans on umap finsler umap + tsne finsler tsne label-related only zoom in}, except that we added to the label-related scores, in white background, the label-unrelated scores, in grey background. Having better label-unrelated score with worse label-related scores is particular bad, as it means that the method is artificially creating ``nicely looking'' clusters that do not correspond to the reality of the data manifold. This never happens for our Finsler methods, as we systematically have better label-related measures.
    }
    \label{fig: classif kmeans on umap finsler umap + tsne finsler tsne all labels only zoom in}
\end{figure*}

\subsubsection{Non-Euclidean symmetric method.}
\label{sec: non euclidean symmetric method}

We provide the embedding results of the hyperbolic method Poincar\'e maps in Fig.~\ref{fig: poincare maps imagenette}. Finsler t-SNE outperforms Poincar\'e maps on all scores\footnote{Except the FMI where it is on par.}. Thus, our Finsler approach yields better quality embeddings than the reference hyperbolic method in this experiment, and as demonstrated in the US cities experiment (Fig.~\ref{fig: US cities}) it also reveals additional information that is discarded in symmetric methods, including Poincar\'e maps.
Thus, not only is our Finsler pipeline superior to the traditional one, outperforming the modern reference methods (t-SNE and UMAP), it can surpass competitive methods using non-trivial changes of geometry, such as hyperbolic geometry.

\subsubsection{Full visualisations.}
\label{sec: full visualisations classification datasets}

We provide embedding visualisations on all datasets in Fig.~\ref{fig: qualitative classif datasets summary}
for traditional Euclidean UMAP and t-SNE and our Finsler UMAP and Finsler t-SNE counterparts. A random embedding was chosen for both traditional Euclidean and our Finsler UMAP.

Additionally, we also provide embedding visualisations in Fig.~\ref{fig: qualitative classif datasets summary extra min dim} using unrecommended embedding dimensions, namely $\mathbb{R}^3$ for Euclidean-based methods, instead of $\mathbb{R}^2$, and $\mathbb{\mathbb{R}}^2$ for Finsler-based methods, instead of $\mathbb{R}^3$. These experiments demonstrate that the success of Finsler-based methods is not due to the extra embedding dimension but to the asymmetric perspective on the data combined with the asymmetric structure of the embedding space.

\begin{figure*}[htbp]
    \centering
    \includegraphics[width=0.85\textwidth]{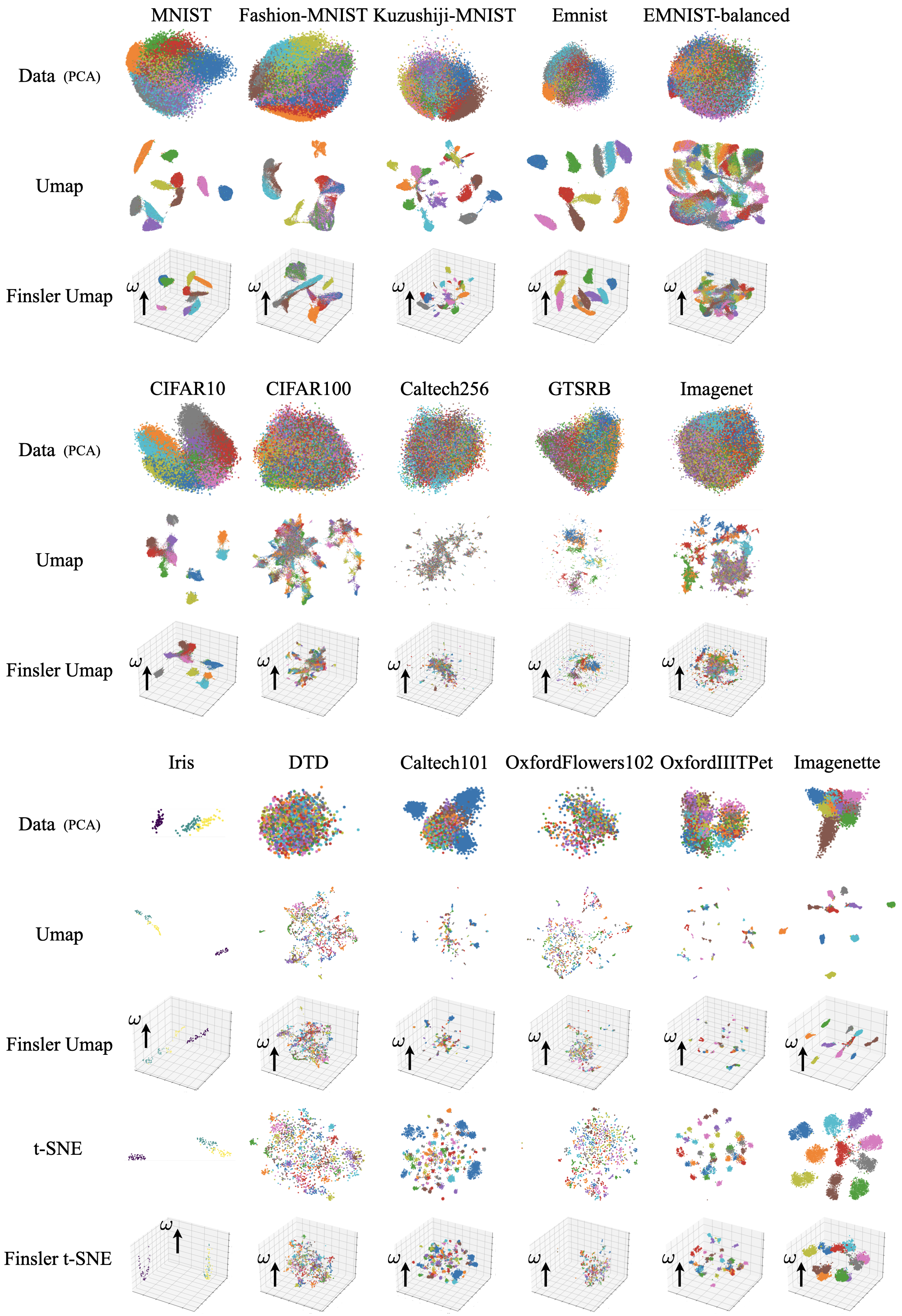}%
    \caption{
    Embeddings results, using recommended embedding dimensions, on all our tested reference classification datasets using either traditional Euclidean methods or our Finsler pipelines.
    }
    \label{fig: qualitative classif datasets summary}
\end{figure*}

\begin{figure*}[htbp]
    \centering
    \includegraphics[width=0.85\textwidth]{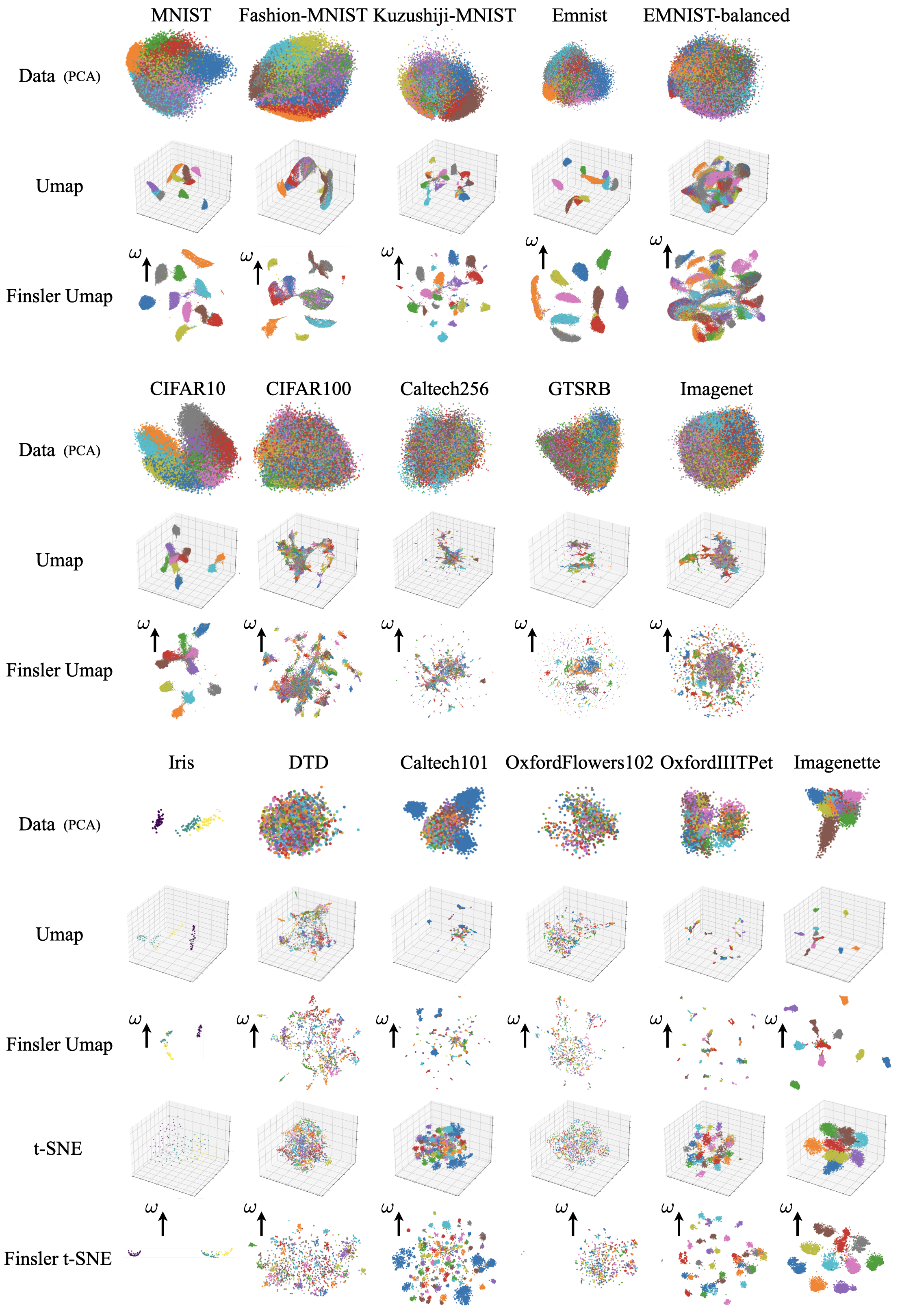}%
    \caption{
    Embeddings results, using unrecommended embedding dimensions, on all our tested reference classification datasets using either traditional Euclidean methods or our Finsler pipelines.
    }
    \label{fig: qualitative classif datasets summary extra min dim}
\end{figure*}

\begin{figure*}[htbp]
    \centering
    \includegraphics[width=0.85\textwidth]{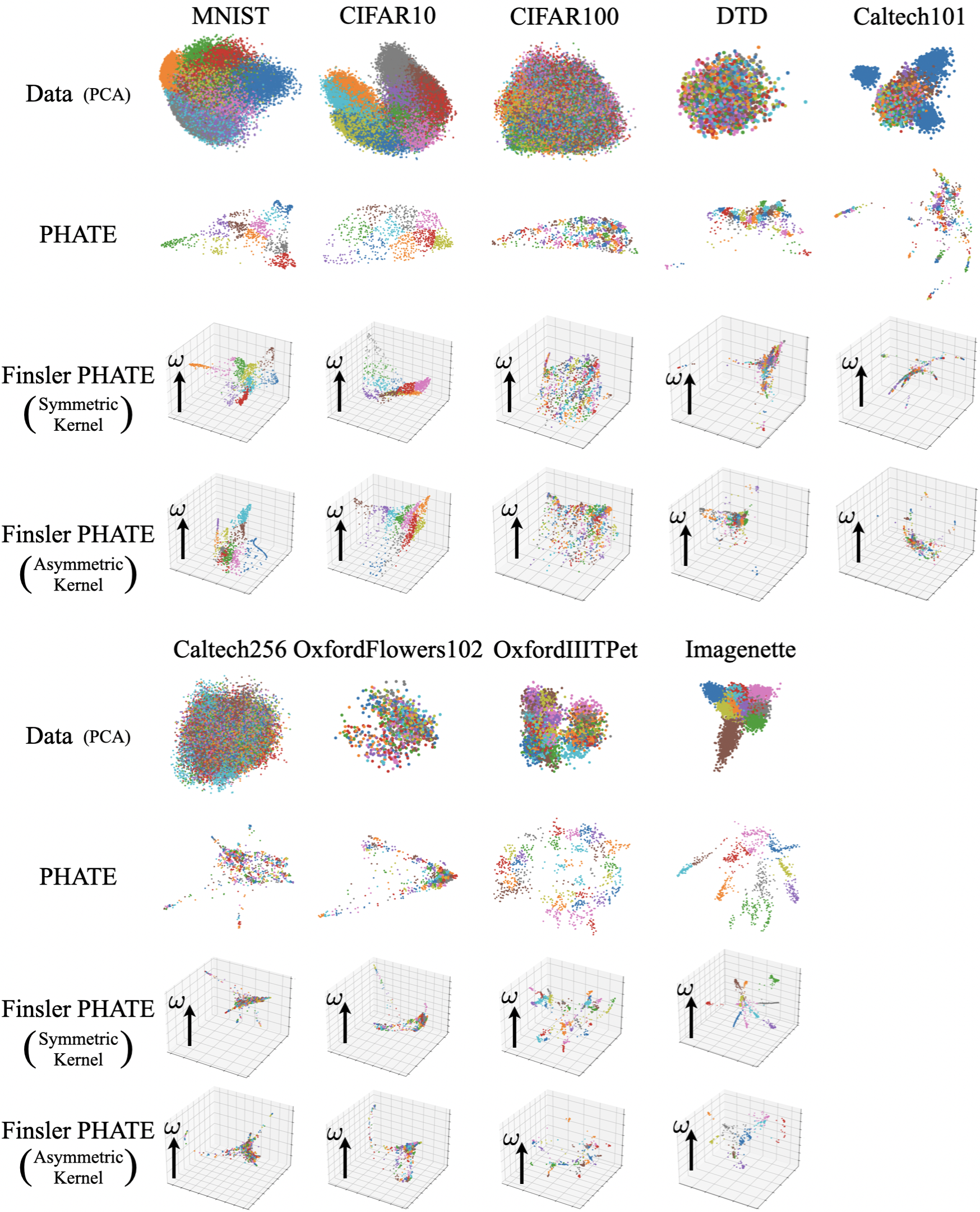}%
    \caption{
    Embeddings results, using recommended embedding dimensions, on classification datasets using either traditional Euclidean PHATE or our Finsler PHATE pipelines.
    }
    \label{fig: qualitative classif datasets summary phate}
\end{figure*}

\subsubsection{Embedding dimensionality.}
\label{sec: embedding dim classification datasets}

We evaluated both UMAP and our Finsler UMAP pipelines when increasing the embedding dimensionality\footnote{Recall that for dimensionality $d$, UMAP embeds in $\mathbb{R}^d$ while Finsler UMAP embeds in $\mathbb{R}^{d+1}$, following the methodology of \cite{dages2025finsler}.} $d\in\{2,5,10,50\}$ of the CIFAR10, OxfordFlowers102, and Imagenette datasets. We provide in Tab.~\ref{tab: classif kmeans on umap finsler umap varying dimensionality} quantitative kMeans clustering performance on these embeddings. When increasing the dimensions, our asymmetric approach still consistently yields better performance according to scores related to the labels, and thus to the data manifold. Nevertheless, the gap between the traditional Euclidean and our Finsler pipelines virtually vanishes in very high dimensions. This suggests that our unusual asymmetric perspective for analysing traditionally considered symmetric data is mostly 
beneficial when embedding data in low rather than high dimensions. Remember though that in manifold or representation learning, low dimensional, and not high dimensional, embeddings are the holy grail to search for. This means that our asymmetric Finsler perspective is not only beneficial, it is noticeably superior in the settings that matter in practice.

\begin{table*}[t]
    \caption{Mean clustering performance using kMeans on UMAP and Finsler UMAP data embeddings with varying dimensionality $d$. The mean and standard deviation over 10 runs is presented. (+)/(-) signifies that higher/lower is better.}
    \label{tab: classif kmeans on umap finsler umap varying dimensionality}
    \centering
    \resizebox{\textwidth}{!}{%
    \begin{tabular}{llcccccccccccc}
        \toprule
         &  &  & \multicolumn{7}{c}{Label-related scores} &  & \multicolumn{3}{c}{Label-unrelated scores} \\
        \cmidrule(lr){4-10}\cmidrule(lr){12-14}
        Dataset & Method & Dimensionality $d$ & AMI (+) & ARI (+) & NMI (+) & HOM (+) & COM (+) & V-M (+) & FMI (+) &  & SIL (+) & DBI (–) & CHI (+) \\
        \midrule
        \multirow{8}{*}{CIFAR10}
        & UMAP & 2  & 0,654 ($\pm$0,327) & 0,619 ($\pm$0,310) & 0,654 ($\pm$0,327) & 0,651 ($\pm$0,325) & 0,656 ($\pm$0,328) & 0,654 ($\pm$0,327) & 0,658 ($\pm$0,279) &  & 0,585 ($\pm$0,115) & 0,570 ($\pm$0,128) & \textbf{239743} ($\pm$98339) \\
        & Finsler UMAP (Ours) & 2  & \textbf{0,817} ($\pm$0,009) & \textbf{0,767} ($\pm$0,029) & \textbf{0,817} ($\pm$0,009) & \textbf{0,812} ($\pm$0,013) & \textbf{0,822} ($\pm$0,005) & \textbf{0,817} ($\pm$0,009) & \textbf{0,791} ($\pm$0,024) &  & \textbf{0,647} ($\pm$0,016) & \textbf{0,516} ($\pm$0,033) & 187999 ($\pm$11072) \\
        \cmidrule(lr){2-14}
        & UMAP & 5  & 0,814 ($\pm$0,011) & 0,763 ($\pm$0,023) & 0,814 ($\pm$0,011) & 0,810 ($\pm$0,012) & 0,819 ($\pm$0,009) & 0,814 ($\pm$0,011) & 0,788 ($\pm$0,020) &  & 0,634 ($\pm$0,013) & 0,546 ($\pm$0,038) & \textbf{228691} ($\pm$12209) \\
        & Finsler UMAP (Ours) & 5  & \textbf{0,815} ($\pm$0,011) & \textbf{0,767} ($\pm$0,025) & \textbf{0,815} ($\pm$0,011) & \textbf{0,811} ($\pm$0,013) & \textbf{0,820} ($\pm$0,009) & \textbf{0,815} ($\pm$0,011) & \textbf{0,791} ($\pm$0,021) &  & \textbf{0,638} ($\pm$0,015) & \textbf{0,541} ($\pm$0,045) & 173549 ($\pm$11579) \\
        \cmidrule(lr){2-14}
        & UMAP & 10 & 0,816 ($\pm$0,010) & \textbf{0,770} ($\pm$0,014) & 0,816 ($\pm$0,010) & \textbf{0,812} ($\pm$0,010) & 0,820 ($\pm$0,009) & 0,816 ($\pm$0,010) & \textbf{0,794} ($\pm$0,012) &  & \textbf{0,637} ($\pm$0,012) & \textbf{0,538} ($\pm$0,034) & \textbf{233395} ($\pm$9628) \\
        & Finsler UMAP (Ours) & 10 & \textbf{0,816} ($\pm$0,011) & 0,759 ($\pm$0,038) & \textbf{0,816} ($\pm$0,011) & 0,809 ($\pm$0,017) & \textbf{0,823} ($\pm$0,006) & \textbf{0,816} ($\pm$0,011) & 0,785 ($\pm$0,032) &  & 0,632 ($\pm$0,030) & 0,555 ($\pm$0,072) & 160470 ($\pm$20810) \\
        \cmidrule(lr){2-14}
        & UMAP & 50 & 0,807 ($\pm$0,017) & 0,752 ($\pm$0,034) & 0,807 ($\pm$0,017) & 0,801 ($\pm$0,021) & 0,813 ($\pm$0,014) & 0,807 ($\pm$0,017) & 0,778 ($\pm$0,029) &  & 0,630 ($\pm$0,013) & 0,561 ($\pm$0,040) & \textbf{217258} ($\pm$30880) \\
        & Finsler UMAP (Ours) & 50 & \textbf{0,816} ($\pm$0,010) & \textbf{0,763} ($\pm$0,030) & \textbf{0,816} ($\pm$0,010) & \textbf{0,810} ($\pm$0,014) & \textbf{0,822} ($\pm$0,006) & \textbf{0,816} ($\pm$0,010) & \textbf{0,788} ($\pm$0,026) &  & \textbf{0,641} ($\pm$0,014) & \textbf{0,540} ($\pm$0,045) & 143102 ($\pm$8470) \\
        \midrule
        \multirow{8}{*}{OxfordFlowers102}
        & UMAP & 2  & 0,423 ($\pm$0,212) & 0,315 ($\pm$0,158) & 0,713 ($\pm$0,105) & 0,707 ($\pm$0,104) & 0,719 ($\pm$0,106) & 0,713 ($\pm$0,105) & 0,322 ($\pm$0,157) &  & 0,483 ($\pm$0,007) & \textbf{0,665} ($\pm$0,014) & 2583 ($\pm$122) \\
        & Finsler UMAP (Ours) & 2  & \textbf{0,571} ($\pm$0,005) & \textbf{0,429} ($\pm$0,008) & \textbf{0,785} ($\pm$0,002) & \textbf{0,775} ($\pm$0,003) & \textbf{0,795} ($\pm$0,003) & \textbf{0,785} ($\pm$0,002) & \textbf{0,437} ($\pm$0,007) &  & \textbf{0,483} ($\pm$0,009) & 0,722 ($\pm$0,019) & \textbf{3652} ($\pm$526) \\
        \cmidrule(lr){2-14}
        & UMAP & 5  & 0,568 ($\pm$0,005) & 0,428 ($\pm$0,007) & 0,785 ($\pm$0,003) & 0,776 ($\pm$0,003) & 0,793 ($\pm$0,002) & 0,785 ($\pm$0,003) & 0,436 ($\pm$0,006) &  & 0,429 ($\pm$0,005) & 0,849 ($\pm$0,013) & \textbf{683} ($\pm$27) \\
        & Finsler UMAP (Ours) & 5  & \textbf{0,587} ($\pm$0,006) & \textbf{0,441} ($\pm$0,009) & \textbf{0,793} ($\pm$0,003) & \textbf{0,783} ($\pm$0,003) & \textbf{0,803} ($\pm$0,003) & \textbf{0,793} ($\pm$0,003) & \textbf{0,449} ($\pm$0,009) &  & \textbf{0,443} ($\pm$0,006) & \textbf{0,830} ($\pm$0,007) & 633 ($\pm$29) \\
        \cmidrule(lr){2-14}
        & UMAP & 10 & 0,569 ($\pm$0,004) & 0,428 ($\pm$0,007) & 0,785 ($\pm$0,002) & 0,776 ($\pm$0,003) & 0,793 ($\pm$0,002) & 0,785 ($\pm$0,002) & 0,436 ($\pm$0,007) &  & 0,419 ($\pm$0,008) & 0,880 ($\pm$0,026) & \textbf{618} ($\pm$17) \\
        & Finsler UMAP (Ours) & 10 & \textbf{0,589} ($\pm$0,006) & \textbf{0,440} ($\pm$0,007) & \textbf{0,793} ($\pm$0,003) & \textbf{0,783} ($\pm$0,004) & \textbf{0,804} ($\pm$0,003) & \textbf{0,793} ($\pm$0,003) & \textbf{0,449} ($\pm$0,006) &  & \textbf{0,432} ($\pm$0,010) & \textbf{0,846} ($\pm$0,023) & 600 ($\pm$12) \\
        \cmidrule(lr){2-14}
        & UMAP & 50 & 0,569 ($\pm$0,006) & 0,426 ($\pm$0,009) & 0,784 ($\pm$0,003) & 0,775 ($\pm$0,004) & 0,793 ($\pm$0,003) & 0,784 ($\pm$0,003) & 0,434 ($\pm$0,009) &  & 0,411 ($\pm$0,009) & 0,888 ($\pm$0,019) & \textbf{617} ($\pm$17) \\
        & Finsler UMAP (Ours) & 50 & \textbf{0,585} ($\pm$0,006) & \textbf{0,437} ($\pm$0,010) & \textbf{0,792} ($\pm$0,003) & \textbf{0,782} ($\pm$0,004) & \textbf{0,802} ($\pm$0,003) & \textbf{0,792} ($\pm$0,003) & \textbf{0,445} ($\pm$0,009) &  & \textbf{0,425} ($\pm$0,007) & \textbf{0,867} ($\pm$0,018) & 612 ($\pm$24) \\
        \midrule
        \multirow{8}{*}{Imagenette}
        & UMAP & 2  & 0,750 ($\pm$0,375) & 0,756 ($\pm$0,378) & 0,751 ($\pm$0,374) & 0,751 ($\pm$0,375) & 0,751 ($\pm$0,374) & 0,751 ($\pm$0,374) & 0,781 ($\pm$0,340) &  & \textbf{0,830} ($\pm$0,005) & \textbf{0,253} ($\pm$0,007) & \textbf{161823} ($\pm$11371) \\
        & Finsler UMAP (Ours) & 2  & \textbf{0,846} ($\pm$0,282) & \textbf{0,852} ($\pm$0,284) & \textbf{0,846} ($\pm$0,281) & \textbf{0,846} ($\pm$0,281) & \textbf{0,846} ($\pm$0,281) & \textbf{0,846} ($\pm$0,281) & \textbf{0,867} ($\pm$0,255) &  & 0,828 ($\pm$0,006) & 0,268 ($\pm$0,016) & 97813 ($\pm$6251) \\
        \cmidrule(lr){2-14}
        & UMAP & 5  & 0,941 ($\pm$0,001) & 0,947 ($\pm$0,001) & 0,941 ($\pm$0,001) & 0,941 ($\pm$0,001) & 0,941 ($\pm$0,001) & 0,942 ($\pm$0,001) & 0,953 ($\pm$0,000) &  & \textbf{0,848} ($\pm$0,002) & \textbf{0,241} ($\pm$0,004) & \textbf{132728} ($\pm$3950) \\
        & Finsler UMAP (Ours) & 5  & \textbf{0,941} ($\pm$0,001) & \textbf{0,948} ($\pm$0,001) & \textbf{0,942} ($\pm$0,001) & \textbf{0,942} ($\pm$0,001) & \textbf{0,941} ($\pm$0,001) & \textbf{0,942} ($\pm$0,001) & \textbf{0,953} ($\pm$0,000) &  & 0,840 ($\pm$0,002) & 0,258 ($\pm$0,003) & 93325 ($\pm$2307) \\
        \cmidrule(lr){2-14}
        & UMAP & 10 & 0,941 ($\pm$0,001) & 0,947 ($\pm$0,001) & 0,941 ($\pm$0,001) & 0,941 ($\pm$0,001) & 0,941 ($\pm$0,001) & 0,941 ($\pm$0,001) & 0,953 ($\pm$0,001) &  & \textbf{0,849} ($\pm$0,001) & 0,242 ($\pm$0,002) & \textbf{130262} ($\pm$3277) \\
        & Finsler UMAP (Ours) & 10 & \textbf{0,941} ($\pm$0,001) & \textbf{0,948} ($\pm$0,001) & \textbf{0,941} ($\pm$0,001) & \textbf{0,941} ($\pm$0,001) & \textbf{0,941} ($\pm$0,001) & \textbf{0,941} ($\pm$0,001) & \textbf{0,953} ($\pm$0,000) &  & 0,834 ($\pm$0,003) & \textbf{0,241} ($\pm$0,006) & 93464 ($\pm$3257) \\
        \cmidrule(lr){2-14}
        & UMAP & 50 & 0,941 ($\pm$0,001) & 0,947 ($\pm$0,000) & 0,941 ($\pm$0,001) & 0,941 ($\pm$0,001) & 0,941 ($\pm$0,001) & 0,941 ($\pm$0,001) & 0,953 ($\pm$0,000) &  & \textbf{0,848} ($\pm$0,002) & \textbf{0,244} ($\pm$0,003) & \textbf{127384} ($\pm$2830) \\
        & Finsler UMAP (Ours) & 50 & \textbf{0,941} ($\pm$0,001) & \textbf{0,948} ($\pm$0,000) & \textbf{0,941} ($\pm$0,001) & \textbf{0,941} ($\pm$0,001) & \textbf{0,941} ($\pm$0,001) & \textbf{0,941} ($\pm$0,001) & \textbf{0,953} ($\pm$0,000) &  & 0,840 ($\pm$0,001) & 0,262 ($\pm$0,003) & 88716 ($\pm$1271) \\
        \bottomrule
    \end{tabular}%
    }
\end{table*}


\subsubsection{Levels of emphasis on asymmetry.}
\label{sec: levels of emphasis on asymmetry}

The canonical Randers space used as embedding space of the Finsler-based methods depends on the choice of linear drift component $\omega$. While its orientation matters little, with the vertical upwards orientation being equivalent to any others by simply rotating the resulting embeddings, its magnitude $\lVert\omega\rVert_2\in[0,1)$ 
may be
important. Increasing  $\lVert\omega\rVert_2$ puts more emphasis on preserving asymmetry and fits well to cases with high levels of asymmetry. When $\omega\xrightarrow[]{} 0$, the embedding space becomes approximately Euclidean, which suits cases of data with little asymmetry. As such, we recommend when possible tuning $\lVert\omega\rVert_2$ to each task rather than choosing a default value.
For instance, $\lVert \omega\rVert_2$ can be heuristically chosen based on the average asymmetry in the data, and we leave for future work learning it for downstream tasks.

Nevertheless, taking a default value can be good enough.
This is what was done with $\lVert\omega\rVert_2 = 0.5$
in the set of visual experiments in existing work \cite{dages2025finsler} 
and in our quantitative supervised evaluation in Tab.~\ref{tab: supervised evaluation}.
In most results reported in this work
(Tabs.~\ref{tab: classif kmeans on umap finsler umap} and \ref{tab: classif kmeans on tsne finsler tsne}), we reported the results obtained with $\lVert\omega\rVert_2\in\{0.001, 0.01, 0.1, 0.5\}$ providing the highest mean AMI score (we did not look at other scores).
Choosing from best AMI might sound unfair and label-informed. However, results are similar with label-uninformed fixed default value as well. In a separate experiment, we show that results given by choosing the best AMI are informative and do not overfit. To show this, we perform a held-out label-informed generalisation check in which we split each dataset in half, independently embed both halves, select $\lVert\omega\rVert_2$ by AMI only on the first half, and report test scores on the second half, see Tab.~\ref{tab: clustering performance on held out data for omega selection}. These held-out results match or improve the baseline on all datasets and are consistent with the results reported on the full datasets.
To avoid advocating for class-aware tuning of $\lVert \omega\rVert_2$, we provide for completeness in Tabs.~\ref{tab: classif kmeans on umap finsler umap varying omega levels} and \ref{tab: classif kmeans on tSNE finsler tSNE varying omega levels} the full results for all tested $\lVert\omega\rVert_2$. For almost all datasets, the mean scores, and in particular the AMI, for most if not all tested levels of $\lVert\omega\rVert_2$ is superior to that of traditional UMAP. These results demonstrate that our Finsler pipeline is robust to choices of $\lVert\omega\rVert_2$. They also reveal the benefit of appropriately tuning $\lVert\omega\rVert_2$ for each task in order to provide the embedding that best captures the data and its natural asymmetry. 
For the results reported in the main manuscript in Fig.~\ref{fig: classif kmeans on umap finsler umap + tsne finsler tsne label-related only fixed omega}, we used a default label-uninformed setting of $\lVert\omega\rVert_2 = 0.001$.

\subsection{Directed graphs}
\label{sec: directed graphs}

For coherence and novelty, we focused in this work on unordered data and combining our novel construction of asymmetric dissimilarities (P1) with our novel asymmetric generalisations of existing embedders (P2-P3). 
In prior work \cite{dages2025finsler}, asymmetric dissimilarities were given, e.g.\ from directed graphs, and the focus was solely on the embedding phases of the pipeline (P2-P3). 
Likewise, we can also directly apply our novel Finsler embedders to directed graphs. We embed the toy directed graph from \cite{dages2025finsler} using 
Finsler MDS, 
Finsler UMAP, and Finsler t-SNE, along with their symmetric counterparts. For each asymmetric embedder, we use the raw dissimilarities provided by the directed graph, as in \cite{dages2025finsler}.
We can also recompute different asymmetric dissimilarities by applying our data dissimilarity construction (P1), however this operation is ill-justified as it is designed for unordered data and not directed graphs.
We provide the embeddings in Fig.~\ref{fig: finsler mds directed graph embeddings}. Our Finsler embeddings applied to the raw dissimilarities all accurately reveal the directional hierarchical structure of the data, while preserving it accurately. On the other hand, recomputing dissimilarities harms the given oracle asymmetric structure provided by the directed graph.

\begin{figure}[t]
    \centering
    \includegraphics[width=0.7\columnwidth]{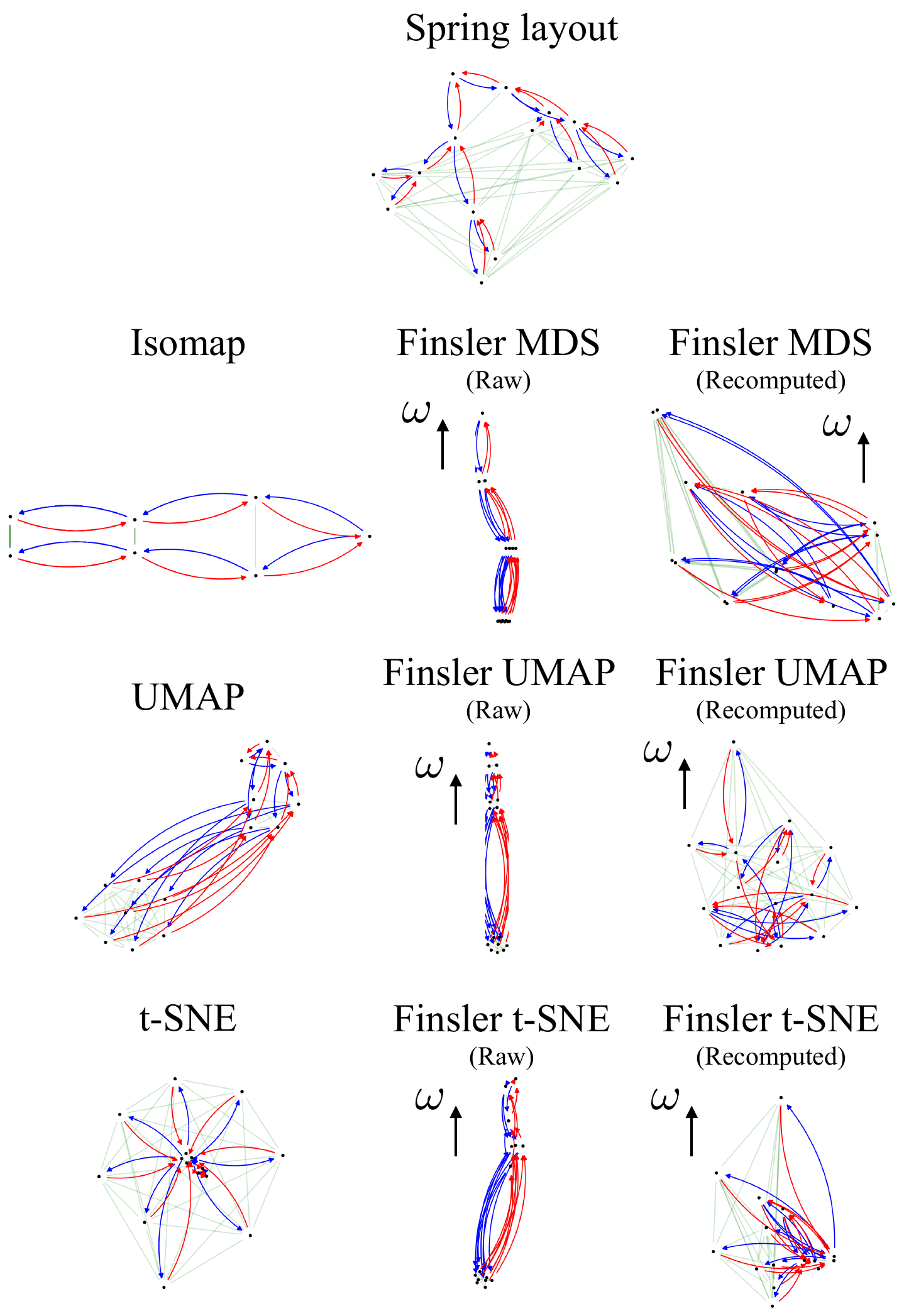}
    \caption{
    Directed graph embeddings. Finsler asymmetric embedders are either applied directly to raw data dissimilarities of the directed graph (Raw), or to recomputed dissimilarities using our data dissimilarity construction pipeline (Recomputed) designed for unordered collections and thus not for directed graphs. All presented Finsler methods are ours, including Finsler MDS-Raw as it uses our alternative gradient descent optimiser. For directed graphs, it may be wiser to use the asymmetric dissimilarities given by the graphs themselves.
    }
    \label{fig: finsler mds directed graph embeddings}
\end{figure}

\begin{table*}[t]
    \caption{
    Mean clustering performance on held-out data (second half of data), that was not used to select the best $\lVert\omega\rVert_2$ via AMI performance on non-held out data (first half of data). Performance is measured using kMeans on UMAP and Finsler UMAP data embeddings using the optimal value $\lVert \omega\rVert_2$ found on the first half of data. The mean and standard deviation over 10 runs is presented. (+) signifies higher is better.}
    \label{tab: clustering performance on held out data for omega selection}
    \centering
    \resizebox{\textwidth}{!}{%
        \begin{tabular}{@{}llcccccccc@{}}
            \toprule
            Dataset & Method & $\|\omega\|_2$ & AMI (+) & ARI (+) & NMI (+) & HOM (+) & COM (+) & V-M (+) & FMI (+) \\
            \midrule

            \multirow[c]{2}{*}{MNIST}
            & UMAP & ---
                & 0,811 ($\pm$ 0,017) & 0,720 ($\pm$ 0,039)
                & 0,811 ($\pm$ 0,017) & 0,800 ($\pm$ 0,016)
                & 0,822 ($\pm$ 0,014) & 0,811 ($\pm$ 0,017)
                & 0,752 ($\pm$ 0,034) \\
            & Finsler UMAP (Ours) & 0,001
                & \textbf{0,820} ($\pm$ 0,015) & \textbf{0,750} ($\pm$ 0,040)
                & \textbf{0,820} ($\pm$ 0,015) & \textbf{0,813} ($\pm$ 0,017)
                & \textbf{0,826} ($\pm$ 0,013) & \textbf{0,820} ($\pm$ 0,015)
                & \textbf{0,777} ($\pm$ 0,034) \\
            \midrule

            \multirow[c]{2}{*}{CIFAR10}
            & UMAP & ---
                & 0,652 ($\pm$ 0,315) & 0,618 ($\pm$ 0,311)
                & 0,651 ($\pm$ 0,322) & 0,650 ($\pm$ 0,324)
                & 0,654 ($\pm$ 0,320) & 0,653 ($\pm$ 0,326)
                & 0,656 ($\pm$ 0,278) \\
            & Finsler UMAP (Ours) & 0,001
                & \textbf{0,806} ($\pm$ 0,009) & \textbf{0,762} ($\pm$ 0,025)
                & \textbf{0,811} ($\pm$ 0,008) & \textbf{0,810} ($\pm$ 0,014)
                & \textbf{0,819} ($\pm$ 0,006) & \textbf{0,814} ($\pm$ 0,009)
                & \textbf{0,788} ($\pm$ 0,026) \\
            \midrule

            \multirow[c]{2}{*}{CIFAR100}
            & UMAP & ---
                & 0,542 ($\pm$ 0,182) & 0,311 ($\pm$ 0,106)
                & 0,557 ($\pm$ 0,180) & 0,551 ($\pm$ 0,175)
                & 0,559 ($\pm$ 0,178) & 0,555 ($\pm$ 0,177)
                & 0,318 ($\pm$ 0,105) \\
            & Finsler UMAP (Ours) & 0,001
                & \textbf{0,612} ($\pm$ 0,002) & \textbf{0,360} ($\pm$ 0,004)
                & \textbf{0,621} ($\pm$ 0,002) & \textbf{0,618} ($\pm$ 0,003)
                & \textbf{0,628} ($\pm$ 0,002) & \textbf{0,620} ($\pm$ 0,003)
                & \textbf{9,365} ($\pm$ 0,003) \\
            \midrule

            \multirow[c]{2}{*}{DTD}
            & UMAP & ---
                & 0,446 ($\pm$ 0,007) & 0,270 ($\pm$ 0,008)
                & 0,601 ($\pm$ 0,005) & 0,598 ($\pm$ 0,005)
                & 0,604 ($\pm$ 0,005) & 0,601 ($\pm$ 0,005)
                & 0,286 ($\pm$ 0,008) \\
            & Finsler UMAP (Ours) & 0,001
                & \textbf{0,457} ($\pm$ 0,003) & \textbf{0,281} ($\pm$ 0,003)
                & \textbf{0,607} ($\pm$ 0,002) & \textbf{0,601} ($\pm$ 0,003)
                & \textbf{0,613} ($\pm$ 0,003) & \textbf{0,607} ($\pm$ 0,002)
                & \textbf{0,298} ($\pm$ 0,003) \\
            \midrule

            \multirow[c]{2}{*}{Caltech101}
            & UMAP & ---
                & \textbf{0,793} ($\pm$ 0,002) & 0,411 ($\pm$ 0,012)
                & 0,836 ($\pm$ 0,003) & \textbf{0,869} ($\pm$ 0,003)
                & 0,806 ($\pm$ 0,003) & 0,836 ($\pm$ 0,003)
                & 0,459 ($\pm$ 0,011) \\
            & Finsler UMAP (Ours) & 0,01
                & \textbf{0,793} ($\pm$ 0,002) & \textbf{0,416} ($\pm$ 0,013)
                & \textbf{0,837} ($\pm$ 0,001) & \textbf{0,869} ($\pm$ 0,002)
                & \textbf{0,807} ($\pm$ 0,002) & \textbf{0,837} ($\pm$ 0,001)
                & \textbf{0,465} ($\pm$ 0,012) \\
            \midrule

            \multirow[c]{2}{*}{Caltech256}
            & UMAP & ---
                & 0,698 ($\pm$ 0,001) & 0,456 ($\pm$ 0,007)
                & 0,780 ($\pm$ 0,001) & 0,782 ($\pm$ 0,001)
                & 0,777 ($\pm$ 0,001) & 0,780 ($\pm$ 0,001)
                & 0,461 ($\pm$ 0,007) \\
            & Finsler UMAP (Ours) & 0,01
                & \textbf{0,716} ($\pm$ 0,002) & \textbf{0,478} ($\pm$ 0,008)
                & \textbf{0,792} ($\pm$ 0,001) & \textbf{0,794} ($\pm$ 0,002)
                & \textbf{0,791} ($\pm$ 0,001) & \textbf{0,792} ($\pm$ 0,001)
                & \textbf{0,484} ($\pm$ 0,008) \\
            \midrule

            \multirow[c]{2}{*}{OxfordFlowers102}
            & UMAP & ---
                & 0,395 ($\pm$ 0,011) & 0,292 ($\pm$ 0,011)
                & 0,772 ($\pm$ 0,005) & 0,759 ($\pm$ 0,006)
                & \textbf{0,786} ($\pm$ 0,004) & 0,772 ($\pm$ 0,005)
                & 0,303 ($\pm$ 0,011) \\
            & Finsler UMAP (Ours) & 0,5
                & \textbf{0,396} ($\pm$ 0,010) & \textbf{0,307} ($\pm$ 0,005)
                & \textbf{0,781} ($\pm$ 0,003) & \textbf{0,777} ($\pm$ 0,004)
                & 0,785 ($\pm$ 0,004) & \textbf{0,781} ($\pm$ 0,003)
                & \textbf{0,315} ($\pm$ 0,005) \\
            \midrule

            \multirow[c]{2}{*}{OxfordIIITPet}
            & UMAP & ---
                & 0,894 ($\pm$ 0,005) & 0,796 ($\pm$ 0,013)
                & 0,906 ($\pm$ 0,005) & 0,901 ($\pm$ 0,005)
                & 0,911 ($\pm$ 0,005) & 0,906 ($\pm$ 0,005)
                & 0,802 ($\pm$ 0,013) \\
            & Finsler UMAP (Ours) & 0,001
                & \textbf{0,899} ($\pm$ 0,004) & \textbf{0,812} ($\pm$ 0,012)
                & \textbf{0,910} ($\pm$ 0,004) & \textbf{0,907} ($\pm$ 0,004)
                & \textbf{0,913} ($\pm$ 0,003) & \textbf{0,910} ($\pm$ 0,004)
                & \textbf{0,817} ($\pm$ 0,012) \\
            \midrule

            \multirow[c]{2}{*}{Imagenette}
            & UMAP & ---
                & 0,932 ($\pm$ 0,002) & 0,935 ($\pm$ 0,002)
                & 0,932 ($\pm$ 0,002) & 0,932 ($\pm$ 0,002)
                & 0,933 ($\pm$ 0,002) & 0,932 ($\pm$ 0,002)
                & 0,941 ($\pm$ 0,002) \\
            & Finsler UMAP (Ours) & 0,001
                & \textbf{0,941} ($\pm$ 0,002) & \textbf{0,947} ($\pm$ 0,002)
                & \textbf{0,941} ($\pm$ 0,001) & \textbf{0,941} ($\pm$ 0,002)
                & \textbf{0,941} ($\pm$ 0,001) & \textbf{0,941} ($\pm$ 0,002)
                & \textbf{0,952} ($\pm$ 0,002) \\

            \bottomrule
        \end{tabular}%
    }
\end{table*}

\begin{table*}[ht]
    \caption{Mean clustering performance using kMeans on UMAP and Finsler UMAP data embeddings. We use various levels of emphasis for metric asymmetry, via $\lVert \omega\rVert_2\in\{0.001, 0.01, 0.1, 0.5\}$, in the embedding space of the Finsler methods. The mean and standard deviation over 10 runs is presented. (+)/(-) signifies that higher/lower is better.}
    \label{tab: classif kmeans on umap finsler umap varying omega levels}
    \centering
    \resizebox{\textwidth}{!}{%
    \begin{tabular}{lllccccccccccc}
        \toprule
         &  &  & \multicolumn{7}{c}{Label-related scores} &  & \multicolumn{3}{c}{Label-unrelated scores} \\
        \cmidrule(lr){4-10}\cmidrule(lr){12-14}
        Dataset & Method & $\lVert \omega\rVert_2$ & AMI (+) & ARI (+) & NMI (+) & HOM (+) & COM (+) & V-M (+) & FMI (+) &  & SIL (+) & DBI (–) & CHI (+) \\
        \midrule
        \multirow{5}{*}{Iris}
        & UMAP & ---    & 0,724 ($\pm$0,248) & 0,688 ($\pm$0,241) & 0,728 ($\pm$0,245) & 0,719 ($\pm$0,243) & 0,737 ($\pm$0,247) & 0,728 ($\pm$0,245) & 0,794 ($\pm$0,159) &  & 0,721 ($\pm$0,019) & 0,389 ($\pm$0,025) & 2526 ($\pm$1202) \\
        \cmidrule(lr){2-14}
        & Finsler UMAP (Ours) & 0,5  & 0,626 ($\pm$0,313) & 0,584 ($\pm$0,293) & 0,630 ($\pm$0,309) & 0,620 ($\pm$0,304) & 0,641 ($\pm$0,314) & 0,630 ($\pm$0,309) & 0,726 ($\pm$0,194) &  & 0,742 ($\pm$0,063) & 0,418 ($\pm$0,113) & 912 ($\pm$778) \\
        & Finsler UMAP (Ours) & 0,1  & 0,719 ($\pm$0,242) & 0,677 ($\pm$0,228) & 0,723 ($\pm$0,239) & 0,713 ($\pm$0,236) & 0,732 ($\pm$0,242) & 0,723 ($\pm$0,239) & 0,786 ($\pm$0,151) &  & 0,731 ($\pm$0,021) & 0,382 ($\pm$0,027) & 1760 ($\pm$493) \\
        & Finsler UMAP (Ours) & 0,01 & 0,722 ($\pm$0,243) & 0,683 ($\pm$0,230) & 0,726 ($\pm$0,240) & 0,717 ($\pm$0,237) & 0,735 ($\pm$0,243) & 0,726 ($\pm$0,240) & 0,790 ($\pm$0,152) &  & 0,717 ($\pm$0,018) & 0,402 ($\pm$0,031) & 1485 ($\pm$275) \\
        & Finsler UMAP (Ours) & 0,001& 0,794 ($\pm$0,024) & 0,746 ($\pm$0,032) & 0,796 ($\pm$0,023) & 0,785 ($\pm$0,026) & 0,808 ($\pm$0,020) & 0,796 ($\pm$0,023) & 0,832 ($\pm$0,020) &  & 0,705 ($\pm$0,036) & 0,412 ($\pm$0,047) & 1471 ($\pm$331) \\
        \midrule
        \multirow{5}{*}{MNIST}
        & UMAP & ---    & 0,847 ($\pm$0,017) & 0,782 ($\pm$0,040) & 0,847 ($\pm$0,017) & 0,839 ($\pm$0,019) & 0,855 ($\pm$0,015) & 0,847 ($\pm$0,017) & 0,805 ($\pm$0,035) &  & 0,580 ($\pm$0,018) & 0,606 ($\pm$0,033) & 194060 ($\pm$19078) \\
        \cmidrule(lr){2-14}
        & Finsler UMAP (Ours) & 0,5  & 0,832 ($\pm$0,029) & 0,757 ($\pm$0,052) & 0,832 ($\pm$0,029) & 0,823 ($\pm$0,032) & 0,842 ($\pm$0,026) & 0,832 ($\pm$0,029) & 0,783 ($\pm$0,046) &  & 0,550 ($\pm$0,028) & 0,687 ($\pm$0,071) & 96743 ($\pm$11252) \\
        & Finsler UMAP (Ours) & 0,1  & 0,856 ($\pm$0,019) & 0,794 ($\pm$0,054) & 0,856 ($\pm$0,019) & 0,847 ($\pm$0,024) & 0,864 ($\pm$0,014) & 0,856 ($\pm$0,019) & 0,817 ($\pm$0,046) &  & 0,600 ($\pm$0,024) & 0,594 ($\pm$0,055) & 115765 ($\pm$13484) \\
        & Finsler UMAP (Ours) & 0,01 & 0,858 ($\pm$0,012) & 0,793 ($\pm$0,032) & 0,858 ($\pm$0,012) & 0,848 ($\pm$0,014) & 0,868 ($\pm$0,011) & 0,858 ($\pm$0,012) & 0,816 ($\pm$0,028) &  & 0,611 ($\pm$0,029) & 0,585 ($\pm$0,060) & 119299 ($\pm$10201) \\
        & Finsler UMAP (Ours) & 0,001& 0,856 ($\pm$0,014) & 0,801 ($\pm$0,040) & 0,856 ($\pm$0,014) & 0,850 ($\pm$0,017) & 0,862 ($\pm$0,012) & 0,856 ($\pm$0,014) & 0,823 ($\pm$0,035) &  & 0,604 ($\pm$0,021) & 0,597 ($\pm$0,047) & 126533 ($\pm$10462) \\
        \midrule
        \multirow{5}{*}{Fashion-MNIST}
        & UMAP & ---    & 0,616 ($\pm$0,009) & 0,445 ($\pm$0,028) & 0,617 ($\pm$0,009) & 0,610 ($\pm$0,011) & 0,623 ($\pm$0,008) & 0,617 ($\pm$0,009) & 0,503 ($\pm$0,024) &  & 0,491 ($\pm$0,017) & 0,738 ($\pm$0,039) & 173803 ($\pm$3963) \\
        \cmidrule(lr){2-14}
        & Finsler UMAP (Ours) & 0,5  & 0,624 ($\pm$0,015) & 0,440 ($\pm$0,034) & 0,624 ($\pm$0,015) & 0,614 ($\pm$0,018) & 0,634 ($\pm$0,015) & 0,624 ($\pm$0,015) & 0,501 ($\pm$0,028) &  & 0,499 ($\pm$0,013) & 0,741 ($\pm$0,041) & 100874 ($\pm$22041) \\
        & Finsler UMAP (Ours) & 0,1  & 0,633 ($\pm$0,015) & 0,456 ($\pm$0,017) & 0,633 ($\pm$0,015) & 0,621 ($\pm$0,014) & 0,646 ($\pm$0,018) & 0,633 ($\pm$0,015) & 0,516 ($\pm$0,014) &  & 0,503 ($\pm$0,020) & 0,734 ($\pm$0,050) & 107025 ($\pm$14459) \\
        & Finsler UMAP (Ours) & 0,01 & 0,640 ($\pm$0,012) & 0,474 ($\pm$0,030) & 0,640 ($\pm$0,012) & 0,627 ($\pm$0,014) & 0,653 ($\pm$0,014) & 0,640 ($\pm$0,012) & 0,533 ($\pm$0,024) &  & 0,504 ($\pm$0,014) & 0,736 ($\pm$0,030) & 97213 ($\pm$13424) \\
        & Finsler UMAP (Ours) & 0,001& 0,635 ($\pm$0,021) & 0,474 ($\pm$0,022) & 0,635 ($\pm$0,021) & 0,626 ($\pm$0,017) & 0,645 ($\pm$0,026) & 0,635 ($\pm$0,021) & 0,531 ($\pm$0,022) &  & 0,488 ($\pm$0,028) & 0,795 ($\pm$0,086) & 98307 ($\pm$10707) \\
        \midrule
        \multirow{5}{*}{Kuzushiji-MNIST}
        & UMAP & ---    & 0,644 ($\pm$0,019) & 0,497 ($\pm$0,025) & 0,644 ($\pm$0,019) & 0,636 ($\pm$0,019) & 0,653 ($\pm$0,018) & 0,644 ($\pm$0,019) & 0,551 ($\pm$0,022) &  & 0,556 ($\pm$0,029) & 0,627 ($\pm$0,068) & 117390 ($\pm$10225) \\
        \cmidrule(lr){2-14}
        & Finsler UMAP (Ours) & 0,5  & 0,644 ($\pm$0,021) & 0,501 ($\pm$0,031) & 0,644 ($\pm$0,021) & 0,634 ($\pm$0,023) & 0,654 ($\pm$0,020) & 0,644 ($\pm$0,021) & 0,555 ($\pm$0,027) &  & 0,542 ($\pm$0,025) & 0,692 ($\pm$0,051) & 51443 ($\pm$6543) \\
        & Finsler UMAP (Ours) & 0,1  & 0,653 ($\pm$0,019) & 0,504 ($\pm$0,020) & 0,653 ($\pm$0,019) & 0,643 ($\pm$0,018) & 0,663 ($\pm$0,021) & 0,653 ($\pm$0,019) & 0,558 ($\pm$0,018) &  & 0,571 ($\pm$0,012) & 0,646 ($\pm$0,033) & 60490 ($\pm$6508) \\
        & Finsler UMAP (Ours) & 0,01 & 0,635 ($\pm$0,011) & 0,486 ($\pm$0,021) & 0,635 ($\pm$0,011) & 0,625 ($\pm$0,014) & 0,645 ($\pm$0,011) & 0,635 ($\pm$0,011) & 0,542 ($\pm$0,017) &  & 0,560 ($\pm$0,021) & 0,664 ($\pm$0,061) & 57863 ($\pm$7456) \\
        & Finsler UMAP (Ours) & 0,001& 0,652 ($\pm$0,020) & 0,509 ($\pm$0,037) & 0,652 ($\pm$0,020) & 0,643 ($\pm$0,023) & 0,661 ($\pm$0,017) & 0,652 ($\pm$0,020) & 0,562 ($\pm$0,031) &  & 0,559 ($\pm$0,028) & 0,680 ($\pm$0,058) & 56887 ($\pm$7713) \\
        \midrule
        \multirow{5}{*}{EMNIST}
        & UMAP & ---    & 0,848 ($\pm$0,003) & 0,790 ($\pm$0,019) & 0,848 ($\pm$0,003) & 0,842 ($\pm$0,007) & 0,854 ($\pm$0,003) & 0,848 ($\pm$0,003) & 0,812 ($\pm$0,016) &  & 0,608 ($\pm$0,005) & 0,569 ($\pm$0,011) & 175759 ($\pm$5755) \\
        \cmidrule(lr){2-14}
        & Finsler UMAP (Ours) & 0,5  & 0,855 ($\pm$0,029) & 0,803 ($\pm$0,063) & 0,855 ($\pm$0,029) & 0,849 ($\pm$0,033) & 0,861 ($\pm$0,026) & 0,855 ($\pm$0,029) & 0,824 ($\pm$0,056) &  & 0,575 ($\pm$0,033) & 0,652 ($\pm$0,064) & 97712 ($\pm$9353) \\
        & Finsler UMAP (Ours) & 0,1  & 0,881 ($\pm$0,021) & 0,855 ($\pm$0,049) & 0,881 ($\pm$0,021) & 0,877 ($\pm$0,024) & 0,885 ($\pm$0,019) & 0,881 ($\pm$0,021) & 0,870 ($\pm$0,044) &  & 0,623 ($\pm$0,019) & 0,587 ($\pm$0,034) & 109551 ($\pm$9189) \\
        & Finsler UMAP (Ours) & 0,01 & 0,869 ($\pm$0,027) & 0,817 ($\pm$0,056) & 0,869 ($\pm$0,027) & 0,861 ($\pm$0,029) & 0,877 ($\pm$0,026) & 0,869 ($\pm$0,027) & 0,837 ($\pm$0,050) &  & 0,624 ($\pm$0,032) & 0,570 ($\pm$0,075) & 109198 ($\pm$14587) \\
        & Finsler UMAP (Ours) & 0,001& 0,887 ($\pm$0,024) & 0,865 ($\pm$0,056) & 0,887 ($\pm$0,024) & 0,884 ($\pm$0,028) & 0,891 ($\pm$0,020) & 0,887 ($\pm$0,024) & 0,879 ($\pm$0,050) &  & 0,633 ($\pm$0,017) & 0,563 ($\pm$0,037) & 114745 ($\pm$8031) \\
        \midrule
        \multirow{5}{*}{EMNIST-Balanced}
        & UMAP & ---    & 0,641 ($\pm$0,004) & 0,396 ($\pm$0,008) & 0,642 ($\pm$0,004) & 0,639 ($\pm$0,004) & 0,644 ($\pm$0,004) & 0,642 ($\pm$0,004) & 0,409 ($\pm$0,008) &  & 0,419 ($\pm$0,004) & 0,764 ($\pm$0,014) & 167572 ($\pm$3708) \\
        \cmidrule(lr){2-14}
        & Finsler UMAP (Ours) & 0,5  & 0,674 ($\pm$0,007) & 0,447 ($\pm$0,012) & 0,674 ($\pm$0,007) & 0,671 ($\pm$0,007) & 0,678 ($\pm$0,007) & 0,674 ($\pm$0,007) & 0,459 ($\pm$0,012) &  & 0,433 ($\pm$0,006) & 0,824 ($\pm$0,012) & 92918 ($\pm$5118) \\
        & Finsler UMAP (Ours) & 0,1  & 0,668 ($\pm$0,003) & 0,437 ($\pm$0,008) & 0,669 ($\pm$0,003) & 0,665 ($\pm$0,003) & 0,672 ($\pm$0,003) & 0,669 ($\pm$0,003) & 0,450 ($\pm$0,007) &  & 0,416 ($\pm$0,010) & 0,839 ($\pm$0,024) & 94983 ($\pm$4762) \\
        & Finsler UMAP (Ours) & 0,01 & 0,665 ($\pm$0,004) & 0,430 ($\pm$0,008) & 0,666 ($\pm$0,004) & 0,662 ($\pm$0,004) & 0,669 ($\pm$0,005) & 0,666 ($\pm$0,004) & 0,443 ($\pm$0,008) &  & 0,416 ($\pm$0,007) & 0,849 ($\pm$0,017) & 91657 ($\pm$3083) \\
        & Finsler UMAP (Ours) & 0,001& 0,664 ($\pm$0,004) & 0,433 ($\pm$0,006) & 0,665 ($\pm$0,004) & 0,661 ($\pm$0,004) & 0,670 ($\pm$0,004) & 0,665 ($\pm$0,004) & 0,446 ($\pm$0,006) &  & 0,414 ($\pm$0,004) & 0,850 ($\pm$0,014) & 90969 ($\pm$4174) \\
        \midrule
        \multirow{5}{*}{CIFAR10}
        & UMAP & ---    & 0,654 ($\pm$0,327) & 0,619 ($\pm$0,310) & 0,654 ($\pm$0,327) & 0,651 ($\pm$0,325) & 0,656 ($\pm$0,328) & 0,654 ($\pm$0,327) & 0,658 ($\pm$0,279) &  & 0,585 ($\pm$0,115) & 0,570 ($\pm$0,128) & 239743 ($\pm$98339) \\
        \cmidrule(lr){2-14}
        & Finsler UMAP (Ours) & 0,5  & 0,707 ($\pm$0,236) & 0,643 ($\pm$0,217) & 0,707 ($\pm$0,236) & 0,698 ($\pm$0,233) & 0,717 ($\pm$0,239) & 0,707 ($\pm$0,236) & 0,681 ($\pm$0,196) &  & 0,541 ($\pm$0,082) & 0,686 ($\pm$0,116) & 102787 ($\pm$25852) \\
        & Finsler UMAP (Ours) & 0,1  & 0,732 ($\pm$0,244) & 0,689 ($\pm$0,231) & 0,732 ($\pm$0,244) & 0,728 ($\pm$0,243) & 0,737 ($\pm$0,246) & 0,732 ($\pm$0,244) & 0,721 ($\pm$0,208) &  & 0,607 ($\pm$0,104) & 0,574 ($\pm$0,147) & 159475 ($\pm$46970) \\
        & Finsler UMAP (Ours) & 0,01 & 0,809 ($\pm$0,011) & 0,752 ($\pm$0,029) & 0,809 ($\pm$0,011) & 0,802 ($\pm$0,014) & 0,816 ($\pm$0,009) & 0,809 ($\pm$0,011) & 0,778 ($\pm$0,025) &  & 0,639 ($\pm$0,018) & 0,534 ($\pm$0,047) & 174967 ($\pm$13199) \\
        & Finsler UMAP (Ours) & 0,001& 0,817 ($\pm$0,009) & 0,767 ($\pm$0,029) & 0,817 ($\pm$0,009) & 0,812 ($\pm$0,013) & 0,822 ($\pm$0,005) & 0,817 ($\pm$0,009) & 0,791 ($\pm$0,024) &  & 0,647 ($\pm$0,016) & 0,516 ($\pm$0,033) & 187999 ($\pm$11072) \\
        \midrule
        \multirow{5}{*}{CIFAR100}
        & UMAP & ---    & 0,549 ($\pm$0,183) & 0,315 ($\pm$0,105) & 0,560 ($\pm$0,179) & 0,556 ($\pm$0,178) & 0,563 ($\pm$0,180) & 0,560 ($\pm$0,179) & 0,322 ($\pm$0,104) &  & 0,444 ($\pm$0,036) & 0,715 ($\pm$0,030) & 123871 ($\pm$23089) \\
        \cmidrule(lr){2-14}
        & Finsler UMAP (Ours) & 0,5  & 0,549 ($\pm$0,183) & 0,321 ($\pm$0,107) & 0,560 ($\pm$0,179) & 0,554 ($\pm$0,177) & 0,565 ($\pm$0,181) & 0,560 ($\pm$0,179) & 0,329 ($\pm$0,107) &  & 0,411 ($\pm$0,047) & 0,845 ($\pm$0,050) & 74769 ($\pm$17273) \\
        & Finsler UMAP (Ours) & 0,1  & 0,554 ($\pm$0,185) & 0,324 ($\pm$0,108) & 0,564 ($\pm$0,181) & 0,560 ($\pm$0,179) & 0,569 ($\pm$0,182) & 0,564 ($\pm$0,181) & 0,332 ($\pm$0,107) &  & 0,411 ($\pm$0,048) & 0,833 ($\pm$0,056) & 57757 ($\pm$13248) \\
        & Finsler UMAP (Ours) & 0,01 & 0,555 ($\pm$0,185) & 0,323 ($\pm$0,108) & 0,565 ($\pm$0,181) & 0,560 ($\pm$0,179) & 0,569 ($\pm$0,182) & 0,565 ($\pm$0,181) & 0,331 ($\pm$0,107) &  & 0,408 ($\pm$0,047) & 0,835 ($\pm$0,058) & 56631 ($\pm$13153) \\
        & Finsler UMAP (Ours) & 0,001& 0,616 ($\pm$0,002) & 0,362 ($\pm$0,005) & 0,624 ($\pm$0,002) & 0,620 ($\pm$0,003) & 0,629 ($\pm$0,002) & 0,624 ($\pm$0,002) & 0,369 ($\pm$0,005) &  & 0,426 ($\pm$0,002) & 0,822 ($\pm$0,013) & 61756 ($\pm$2073) \\
        \midrule
        \multirow{5}{*}{DTD}
        & UMAP & ---    & 0,402 ($\pm$0,202) & 0,237 ($\pm$0,119) & 0,501 ($\pm$0,169) & 0,498 ($\pm$0,168) & 0,504 ($\pm$0,170) & 0,501 ($\pm$0,169) & 0,254 ($\pm$0,116) &  & 0,480 ($\pm$0,006) & 0,686 ($\pm$0,011) & 3690 ($\pm$211) \\
        \cmidrule(lr){2-14}
        & Finsler UMAP (Ours) & 0,5  & 0,448 ($\pm$0,148) & 0,261 ($\pm$0,087) & 0,539 ($\pm$0,124) & 0,534 ($\pm$0,123) & 0,543 ($\pm$0,125) & 0,539 ($\pm$0,124) & 0,278 ($\pm$0,085) &  & 0,451 ($\pm$0,013) & 0,812 ($\pm$0,025) & 4364 ($\pm$816) \\
        & Finsler UMAP (Ours) & 0,1  & 0,410 ($\pm$0,206) & 0,242 ($\pm$0,121) & 0,507 ($\pm$0,172) & 0,503 ($\pm$0,171) & 0,511 ($\pm$0,173) & 0,507 ($\pm$0,172) & 0,260 ($\pm$0,119) &  & 0,447 ($\pm$0,005) & 0,798 ($\pm$0,022) & 1639 ($\pm$100) \\
        & Finsler UMAP (Ours) & 0,01 & 0,463 ($\pm$0,154) & 0,274 ($\pm$0,092) & 0,551 ($\pm$0,129) & 0,546 ($\pm$0,127) & 0,555 ($\pm$0,130) & 0,551 ($\pm$0,129) & 0,291 ($\pm$0,090) &  & 0,449 ($\pm$0,007) & 0,782 ($\pm$0,020) & 1571 ($\pm$43) \\
        & Finsler UMAP (Ours) & 0,001& 0,462 ($\pm$0,154) & 0,272 ($\pm$0,091) & 0,551 ($\pm$0,128) & 0,547 ($\pm$0,127) & 0,555 ($\pm$0,129) & 0,551 ($\pm$0,128) & 0,289 ($\pm$0,089) &  & 0,448 ($\pm$0,008) & 0,778 ($\pm$0,019) & 1602 ($\pm$100) \\
        \midrule
        \multirow{5}{*}{Caltech101}
        & UMAP & ---    & 0,737 ($\pm$0,246) & 0,389 ($\pm$0,130) & 0,771 ($\pm$0,214) & 0,800 ($\pm$0,222) & 0,744 ($\pm$0,206) & 0,771 ($\pm$0,214) & 0,435 ($\pm$0,139) &  & 0,677 ($\pm$0,006) & 0,419 ($\pm$0,008) & 154548 ($\pm$7934) \\
        \cmidrule(lr){2-14}
        & Finsler UMAP (Ours) & 0,5  & 0,736 ($\pm$0,245) & 0,383 ($\pm$0,128) & 0,771 ($\pm$0,213) & 0,798 ($\pm$0,221) & 0,745 ($\pm$0,206) & 0,771 ($\pm$0,213) & 0,427 ($\pm$0,136) &  & 0,643 ($\pm$0,010) & 0,490 ($\pm$0,017) & 103875 ($\pm$4468) \\
        & Finsler UMAP (Ours) & 0,1  & 0,738 ($\pm$0,246) & 0,383 ($\pm$0,128) & 0,772 ($\pm$0,213) & 0,801 ($\pm$0,221) & 0,745 ($\pm$0,206) & 0,772 ($\pm$0,213) & 0,429 ($\pm$0,137) &  & 0,652 ($\pm$0,010) & 0,479 ($\pm$0,015) & 77854 ($\pm$5631) \\
        & Finsler UMAP (Ours) & 0,01 & 0,821 ($\pm$0,004) & 0,432 ($\pm$0,016) & 0,845 ($\pm$0,003) & 0,876 ($\pm$0,003) & 0,815 ($\pm$0,004) & 0,845 ($\pm$0,003) & 0,481 ($\pm$0,016) &  & 0,651 ($\pm$0,007) & 0,476 ($\pm$0,017) & 77907 ($\pm$3926) \\
        & Finsler UMAP (Ours) & 0,001& 0,739 ($\pm$0,246) & 0,393 ($\pm$0,131) & 0,774 ($\pm$0,214) & 0,802 ($\pm$0,222) & 0,747 ($\pm$0,207) & 0,774 ($\pm$0,214) & 0,438 ($\pm$0,140) &  & 0,649 ($\pm$0,011) & 0,473 ($\pm$0,022) & 74462 ($\pm$3758) \\
        \midrule
        \multirow{5}{*}{Caltech256}
        & UMAP & ---    & 0,579 ($\pm$0,290) & 0,366 ($\pm$0,183) & 0,655 ($\pm$0,238) & 0,658 ($\pm$0,239) & 0,652 ($\pm$0,237) & 0,655 ($\pm$0,238) & 0,372 ($\pm$0,184) &  & 0,590 ($\pm$0,005) & 0,553 ($\pm$0,010) & 150961 ($\pm$5464) \\
        \cmidrule(lr){2-14}
        & Finsler UMAP (Ours) & 0,5  & 0,664 ($\pm$0,221) & 0,435 ($\pm$0,145) & 0,723 ($\pm$0,182) & 0,725 ($\pm$0,183) & 0,722 ($\pm$0,182) & 0,723 ($\pm$0,182) & 0,440 ($\pm$0,145) &  & 0,568 ($\pm$0,007) & 0,644 ($\pm$0,014) & 51845 ($\pm$2257) \\
        & Finsler UMAP (Ours) & 0,1  & 0,743 ($\pm$0,001) & 0,487 ($\pm$0,009) & 0,789 ($\pm$0,001) & 0,791 ($\pm$0,001) & 0,786 ($\pm$0,001) & 0,789 ($\pm$0,001) & 0,492 ($\pm$0,009) &  & 0,590 ($\pm$0,005) & 0,604 ($\pm$0,011) & 55076 ($\pm$1848) \\
        & Finsler UMAP (Ours) & 0,01 & 0,667 ($\pm$0,222) & 0,434 ($\pm$0,145) & 0,726 ($\pm$0,183) & 0,729 ($\pm$0,183) & 0,724 ($\pm$0,182) & 0,726 ($\pm$0,183) & 0,440 ($\pm$0,145) &  & 0,586 ($\pm$0,007) & 0,608 ($\pm$0,012) & 54395 ($\pm$1912) \\
        & Finsler UMAP (Ours) & 0,001& 0,742 ($\pm$0,001) & 0,488 ($\pm$0,008) & 0,788 ($\pm$0,001) & 0,790 ($\pm$0,001) & 0,786 ($\pm$0,001) & 0,788 ($\pm$0,001) & 0,494 ($\pm$0,008) &  & 0,588 ($\pm$0,005) & 0,609 ($\pm$0,008) & 54591 ($\pm$1419) \\
        \midrule
        \multirow{5}{*}{OxfordFlowers102}
        & UMAP & ---    & 0,423 ($\pm$0,212) & 0,315 ($\pm$0,158) & 0,713 ($\pm$0,105) & 0,707 ($\pm$0,104) & 0,719 ($\pm$0,106) & 0,713 ($\pm$0,105) & 0,322 ($\pm$0,157) &  & 0,483 ($\pm$0,007) & 0,665 ($\pm$0,014) & 2583 ($\pm$122) \\
        \cmidrule(lr){2-14}
        & Finsler UMAP (Ours) & 0,5  & 0,571 ($\pm$0,005) & 0,429 ($\pm$0,008) & 0,785 ($\pm$0,002) & 0,775 ($\pm$0,003) & 0,795 ($\pm$0,003) & 0,785 ($\pm$0,002) & 0,437 ($\pm$0,007) &  & 0,483 ($\pm$0,009) & 0,722 ($\pm$0,019) & 3652 ($\pm$526) \\
        & Finsler UMAP (Ours) & 0,1  & 0,514 ($\pm$0,171) & 0,386 ($\pm$0,129) & 0,757 ($\pm$0,086) & 0,749 ($\pm$0,085) & 0,765 ($\pm$0,086) & 0,757 ($\pm$0,086) & 0,394 ($\pm$0,128) &  & 0,459 ($\pm$0,007) & 0,751 ($\pm$0,021) & 1134 ($\pm$41) \\
        & Finsler UMAP (Ours) & 0,01 & 0,517 ($\pm$0,171) & 0,388 ($\pm$0,129) & 0,758 ($\pm$0,086) & 0,750 ($\pm$0,085) & 0,767 ($\pm$0,086) & 0,758 ($\pm$0,086) & 0,396 ($\pm$0,128) &  & 0,457 ($\pm$0,009) & 0,756 ($\pm$0,021) & 1014 ($\pm$42) \\
        & Finsler UMAP (Ours) & 0,001& 0,516 ($\pm$0,170) & 0,389 ($\pm$0,129) & 0,758 ($\pm$0,086) & 0,751 ($\pm$0,085) & 0,766 ($\pm$0,086) & 0,758 ($\pm$0,086) & 0,397 ($\pm$0,129) &  & 0,461 ($\pm$0,006) & 0,746 ($\pm$0,016) & 1043 ($\pm$40) \\
        \midrule
        \multirow{5}{*}{OxfordIIITPet}
        & UMAP & ---    & 0,720 ($\pm$0,362) & 0,646 ($\pm$0,323) & 0,735 ($\pm$0,342) & 0,732 ($\pm$0,341) & 0,739 ($\pm$0,344) & 0,735 ($\pm$0,342) & 0,656 ($\pm$0,315) &  & 0,701 ($\pm$0,006) & 0,426 ($\pm$0,024) & 105243 ($\pm$5334) \\
        \cmidrule(lr){2-14}
        & Finsler UMAP (Ours) & 0,5  & 0,885 ($\pm$0,005) & 0,760 ($\pm$0,016) & 0,892 ($\pm$0,005) & 0,884 ($\pm$0,006) & 0,899 ($\pm$0,005) & 0,892 ($\pm$0,005) & 0,768 ($\pm$0,015) &  & 0,670 ($\pm$0,013) & 0,478 ($\pm$0,040) & 86351 ($\pm$6526) \\
        & Finsler UMAP (Ours) & 0,1  & 0,805 ($\pm$0,269) & 0,702 ($\pm$0,235) & 0,816 ($\pm$0,254) & 0,810 ($\pm$0,252) & 0,822 ($\pm$0,256) & 0,816 ($\pm$0,254) & 0,711 ($\pm$0,228) &  & 0,693 ($\pm$0,010) & 0,430 ($\pm$0,019) & 51588 ($\pm$5479) \\
        & Finsler UMAP (Ours) & 0,01 & 0,806 ($\pm$0,269) & 0,704 ($\pm$0,235) & 0,816 ($\pm$0,255) & 0,810 ($\pm$0,253) & 0,823 ($\pm$0,257) & 0,816 ($\pm$0,255) & 0,714 ($\pm$0,229) &  & 0,693 ($\pm$0,013) & 0,434 ($\pm$0,017) & 50056 ($\pm$2610) \\
        & Finsler UMAP (Ours) & 0,001& 0,807 ($\pm$0,270) & 0,710 ($\pm$0,237) & 0,818 ($\pm$0,255) & 0,812 ($\pm$0,253) & 0,824 ($\pm$0,257) & 0,818 ($\pm$0,255) & 0,719 ($\pm$0,231) &  & 0,697 ($\pm$0,014) & 0,433 ($\pm$0,026) & 50436 ($\pm$3656) \\
        \midrule
        \multirow{5}{*}{GTSRB}
        & UMAP & ---    & 0,444 ($\pm$0,222) & 0,206 ($\pm$0,104) & 0,449 ($\pm$0,220) & 0,455 ($\pm$0,223) & 0,443 ($\pm$0,217) & 0,449 ($\pm$0,220) & 0,233 ($\pm$0,100) &  & 0,411 ($\pm$0,008) & 0,779 ($\pm$0,018) & 34781 ($\pm$1412) \\
        \cmidrule(lr){2-14}
        & Finsler UMAP (Ours) & 0,5  & 0,440 ($\pm$0,147) & 0,190 ($\pm$0,064) & 0,445 ($\pm$0,145) & 0,450 ($\pm$0,147) & 0,441 ($\pm$0,144) & 0,445 ($\pm$0,145) & 0,218 ($\pm$0,062) &  & 0,361 ($\pm$0,012) & 0,957 ($\pm$0,029) & 9369 ($\pm$571) \\
        & Finsler UMAP (Ours) & 0,1  & 0,460 ($\pm$0,154) & 0,201 ($\pm$0,069) & 0,465 ($\pm$0,152) & 0,468 ($\pm$0,153) & 0,463 ($\pm$0,151) & 0,465 ($\pm$0,152) & 0,229 ($\pm$0,067) &  & 0,376 ($\pm$0,007) & 0,938 ($\pm$0,034) & 9379 ($\pm$609) \\
        & Finsler UMAP (Ours) & 0,01 & 0,461 ($\pm$0,154) & 0,205 ($\pm$0,069) & 0,467 ($\pm$0,153) & 0,470 ($\pm$0,154) & 0,463 ($\pm$0,152) & 0,467 ($\pm$0,153) & 0,233 ($\pm$0,067) &  & 0,374 ($\pm$0,010) & 0,931 ($\pm$0,032) & 9523 ($\pm$389) \\
        & Finsler UMAP (Ours) & 0,001& 0,515 ($\pm$0,010) & 0,231 ($\pm$0,010) & 0,520 ($\pm$0,010) & 0,523 ($\pm$0,010) & 0,517 ($\pm$0,010) & 0,520 ($\pm$0,010) & 0,258 ($\pm$0,010) &  & 0,373 ($\pm$0,007) & 0,930 ($\pm$0,024) & 9529 ($\pm$430) \\
        \midrule
        \multirow{5}{*}{Imagenette}
        & UMAP & ---    & 0,750 ($\pm$0,375) & 0,756 ($\pm$0,378) & 0,751 ($\pm$0,374) & 0,751 ($\pm$0,375) & 0,751 ($\pm$0,374) & 0,751 ($\pm$0,374) & 0,781 ($\pm$0,340) &  & 0,830 ($\pm$0,005) & 0,253 ($\pm$0,007) & 161823 ($\pm$11371) \\
        \cmidrule(lr){2-14}
        & Finsler UMAP (Ours) & 0,5  & 0,811 ($\pm$0,270) & 0,762 ($\pm$0,256) & 0,811 ($\pm$0,270) & 0,799 ($\pm$0,266) & 0,824 ($\pm$0,274) & 0,811 ($\pm$0,270) & 0,789 ($\pm$0,228) &  & 0,779 ($\pm$0,020) & 0,322 ($\pm$0,038) & 62465 ($\pm$9513) \\
        & Finsler UMAP (Ours) & 0,1  & 0,840 ($\pm$0,280) & 0,843 ($\pm$0,281) & 0,840 ($\pm$0,280) & 0,840 ($\pm$0,280) & 0,840 ($\pm$0,280) & 0,840 ($\pm$0,280) & 0,858 ($\pm$0,253) &  & 0,817 ($\pm$0,006) & 0,279 ($\pm$0,016) & 92330 ($\pm$12149) \\
        & Finsler UMAP (Ours) & 0,01 & 0,844 ($\pm$0,281) & 0,849 ($\pm$0,283) & 0,844 ($\pm$0,281) & 0,844 ($\pm$0,281) & 0,844 ($\pm$0,281) & 0,844 ($\pm$0,281) & 0,864 ($\pm$0,255) &  & 0,829 ($\pm$0,012) & 0,268 ($\pm$0,034) & 98684 ($\pm$10545) \\
        & Finsler UMAP (Ours) & 0,001& 0,846 ($\pm$0,282) & 0,852 ($\pm$0,284) & 0,846 ($\pm$0,281) & 0,846 ($\pm$0,281) & 0,846 ($\pm$0,281) & 0,846 ($\pm$0,281) & 0,867 ($\pm$0,255) &  & 0,828 ($\pm$0,006) & 0,268 ($\pm$0,016) & 97813 ($\pm$6251) \\
        \midrule
        \multirow{5}{*}{Imagenet}
        & UMAP & ---    & 0,784 ($\pm$0,000) & 0,506 ($\pm$0,001) & 0,797 ($\pm$0,000) & 0,793 ($\pm$0,000) & 0,802 ($\pm$0,000) & 0,797 ($\pm$0,000) & 0,508 ($\pm$0,001) &  & 0,582 ($\pm$0,002) & 0,570 ($\pm$0,003) & 7315700 ($\pm$51432) \\
        \cmidrule(lr){2-14}
        & Finsler UMAP (Ours) & 0,5  & 0,799 ($\pm$0,001) & 0,530 ($\pm$0,002) & 0,811 ($\pm$0,000) & 0,805 ($\pm$0,001) & 0,818 ($\pm$0,000) & 0,811 ($\pm$0,000) & 0,533 ($\pm$0,002) &  & 0,600 ($\pm$0,003) & 0,588 ($\pm$0,007) & 1505910 ($\pm$36768) \\
        & Finsler UMAP (Ours) & 0,1  & 0,803 ($\pm$0,001) & 0,539 ($\pm$0,002) & 0,816 ($\pm$0,001) & 0,810 ($\pm$0,001) & 0,821 ($\pm$0,001) & 0,816 ($\pm$0,001) & 0,542 ($\pm$0,002) &  & 0,614 ($\pm$0,004) & 0,565 ($\pm$0,008) & 1885240 ($\pm$27296) \\
        & Finsler UMAP (Ours) & 0,01 & 0,803 ($\pm$0,000) & 0,538 ($\pm$0,001) & 0,815 ($\pm$0,000) & 0,810 ($\pm$0,000) & 0,821 ($\pm$0,001) & 0,815 ($\pm$0,000) & 0,540 ($\pm$0,001) &  & 0,614 ($\pm$0,004) & 0,564 ($\pm$0,006) & 1921640 ($\pm$25941) \\
        & Finsler UMAP (Ours) & 0,001& 0,803 ($\pm$0,001) & 0,540 ($\pm$0,003) & 0,816 ($\pm$0,001) & 0,810 ($\pm$0,001) & 0,821 ($\pm$0,001) & 0,816 ($\pm$0,001) & 0,542 ($\pm$0,003) &  & 0,615 ($\pm$0,004) & 0,565 ($\pm$0,007) & 1930520 ($\pm$28049) \\
        \bottomrule
    \end{tabular}%
    }
\end{table*}

\begin{table*}[ht]
    \caption{Mean clustering performance using kMeans on t-SNE and Finsler t-SNE data embeddings across the smaller datasets. We use various levels of emphasis for metric asymmetry, via $\lVert \omega\rVert_2\in\{0.001, 0.01, 0.1, 0.5\}$, in the embedding space of the Finsler methods. (+)/(-) signifies that higher/lower is better.}
    \label{tab: classif kmeans on tSNE finsler tSNE varying omega levels}
    \centering
    \resizebox{\textwidth}{!}{%
    \begin{tabular}{lllccccccccccc}
        \toprule
         &  &  & \multicolumn{7}{c}{Label-related scores} &  & \multicolumn{3}{c}{Label-unrelated scores} \\
        \cmidrule(lr){4-10}\cmidrule(lr){12-14}
        Dataset & Method & $\lVert \omega\rVert_2$ & AMI (+) & ARI (+) & NMI (+) & HOM (+) & COM (+) & V-M (+) & FMI (+) &  & SIL (+) & DBI (–) & CHI (+) \\
        \midrule
        \multirow{5}{*}{Iris}
        & t-SNE                 & —     & 0,829 & 0,851 & 0,831 & 0,831 & 0,831 & 0,831 & 0,900 &  & 0,665 & 0,453 & 3028 \\
        \cmidrule(lr){2-14}
        & Finsler t-SNE (Ours) & 0,5   & 0,845 & 0,868 & 0,846 & 0,846 & 0,847 & 0,846 & 0,911 &  & 0,664 & 0,465 & 4822 \\
        & Finsler t-SNE (Ours) & 0,1   & 0,835 & 0,851 & 0,837 & 0,836 & 0,837 & 0,837 & 0,900 &  & 0,659 & 0,465 & 5092 \\
        & Finsler t-SNE (Ours) & 0,01  & 0,835 & 0,851 & 0,837 & 0,836 & 0,837 & 0,837 & 0,900 &  & 0,670 & 0,442 & 5063 \\
        & Finsler t-SNE (Ours) & 0,001 & 0,835 & 0,851 & 0,837 & 0,836 & 0,837 & 0,837 & 0,900 &  & 0,663 & 0,450 & 5263 \\
        \midrule
        \multirow{5}{*}{DTD}
        & t-SNE                 & —     & 0,498 & 0,288 & 0,581 & 0,578 & 0,584 & 0,581 & 0,303 &  & 0,402 & 0,804 & 2492 \\
        \cmidrule(lr){2-14}
        & Finsler t-SNE (Ours) & 0,5   & 0,497 & 0,290 & 0,579 & 0,573 & 0,584 & 0,579 & 0,306 &  & 0,440 & 0,819 & 1268 \\
        & Finsler t-SNE (Ours) & 0,1   & 0,502 & 0,298 & 0,584 & 0,582 & 0,587 & 0,584 & 0,313 &  & 0,430 & 0,826 & 1306 \\
        & Finsler t-SNE (Ours) & 0,01  & 0,496 & 0,288 & 0,578 & 0,574 & 0,583 & 0,578 & 0,305 &  & 0,424 & 0,831 & 1208 \\
        & Finsler t-SNE (Ours) & 0,001 & 0,504 & 0,294 & 0,585 & 0,580 & 0,590 & 0,585 & 0,310 &  & 0,408 & 0,859 & 1186 \\
        \midrule
        \multirow{5}{*}{Caltech101}
        & t-SNE                 & —     & 0,804 & 0,377 & 0,830 & 0,865 & 0,798 & 0,830 & 0,428 &  & 0,547 & 0,610 & 23151 \\
        \cmidrule(lr){2-14}
        & Finsler t-SNE (Ours) & 0,5   & 0,824 & 0,441 & 0,847 & 0,878 & 0,818 & 0,847 & 0,490 &  & 0,576 & 0,627 & 13206 \\
        & Finsler t-SNE (Ours) & 0,1   & 0,823 & 0,434 & 0,846 & 0,879 & 0,816 & 0,846 & 0,484 &  & 0,581 & 0,596 & 13901 \\
        & Finsler t-SNE (Ours) & 0,01  & 0,825 & 0,458 & 0,848 & 0,879 & 0,820 & 0,848 & 0,506 &  & 0,584 & 0,597 & 13419 \\
        & Finsler t-SNE (Ours) & 0,001 & 0,818 & 0,412 & 0,842 & 0,876 & 0,811 & 0,842 & 0,464 &  & 0,554 & 0,638 & 13707 \\
        \midrule
        \multirow{5}{*}{OxfordFlowers102}
        & t-SNE                 & —     & 0,549 & 0,388 & 0,775 & 0,767 & 0,783 & 0,775 & 0,397 &  & 0,419 & 0,740 & 1788 \\
        \cmidrule(lr){2-14}
        & Finsler t-SNE (Ours) & 0,5   & 0,556 & 0,381 & 0,776 & 0,764 & 0,788 & 0,776 & 0,392 &  & 0,510 & 0,712 & 1452 \\
        & Finsler t-SNE (Ours) & 0,1   & 0,563 & 0,380 & 0,779 & 0,766 & 0,792 & 0,779 & 0,392 &  & 0,537 & 0,653 & 1690 \\
        & Finsler t-SNE (Ours) & 0,01  & 0,557 & 0,388 & 0,776 & 0,763 & 0,789 & 0,776 & 0,400 &  & 0,528 & 0,650 & 1677 \\
        & Finsler t-SNE (Ours) & 0,001 & 0,558 & 0,388 & 0,776 & 0,764 & 0,789 & 0,776 & 0,400 &  & 0,527 & 0,656 & 172846 \\
        \midrule
        \multirow{5}{*}{OxfordIIITPet}
        & t-SNE                 & —     & 0,899 & 0,814 & 0,904 & 0,900 & 0,908 & 0,904 & 0,819 &  & 0,607 & 0,554 & 15824 \\
        \cmidrule(lr){2-14}
        & Finsler t-SNE (Ours) & 0,5   & 0,900 & 0,822 & 0,906 & 0,903 & 0,908 & 0,906 & 0,827 &  & 0,600 & 0,606 & 13946 \\
        & Finsler t-SNE (Ours) & 0,1   & 0,899 & 0,790 & 0,904 & 0,894 & 0,915 & 0,904 & 0,798 &  & 0,624 & 0,544 & 12501 \\
        & Finsler t-SNE (Ours) & 0,01  & 0,903 & 0,821 & 0,908 & 0,904 & 0,913 & 0,908 & 0,826 &  & 0,618 & 0,556 & 13912 \\
        & Finsler t-SNE (Ours) & 0,001 & 0,898 & 0,796 & 0,904 & 0,896 & 0,912 & 0,904 & 0,803 &  & 0,610 & 0,585 & 13367 \\
        \midrule
        \multirow{5}{*}{Imagenette}
        & t-SNE                 & —     & 0,939 & 0,946 & 0,939 & 0,939 & 0,939 & 0,939 & 0,951 &  & 0,610 & 0,536 & 22051 \\
        \cmidrule(lr){2-14}
        & Finsler t-SNE (Ours) & 0,5   & 0,943 & 0,951 & 0,943 & 0,943 & 0,943 & 0,943 & 0,955 &  & 0,668 & 0,465 & 19651 \\
        & Finsler t-SNE (Ours) & 0,1   & 0,937 & 0,944 & 0,937 & 0,937 & 0,937 & 0,937 & 0,950 &  & 0,696 & 0,451 & 19732 \\
        & Finsler t-SNE (Ours) & 0,01  & 0,936 & 0,944 & 0,936 & 0,936 & 0,936 & 0,936 & 0,950 &  & 0,703 & 0,446 & 19678 \\
        & Finsler t-SNE (Ours) & 0,001 & 0,936 & 0,944 & 0,936 & 0,936 & 0,936 & 0,936 & 0,949 &  & 0,706 & 0,443 & 19994 \\
        \bottomrule
    \end{tabular}%
    }
\end{table*}

\end{document}